\documentclass[11pt]{report}

\usepackage{suthesis-2e}
\usepackage{times}

\usepackage{latexsym}
\usepackage{amsmath}
\usepackage{amssymb}
\usepackage{amsthm}
\usepackage{bm}

\usepackage{tikz}
\usepackage{tikz-qtree}
\usepackage{tikz-dependency}
\usepackage{bussproofs}
\usepackage{mathtools}

\makeatletter
\newcommand{\@BIBLABEL}{\@emptybiblabel}
\newcommand{\@emptybiblabel}[1]{}
\makeatother
\usepackage{hyperref}

\usepackage{lingmacros}

\usepackage{algorithm}
\usepackage{algpseudocode}

\usepackage[labelformat=simple]{subcaption}

\usepackage[T1]{CJKutf8}

\newcommand{\lefttriangle}[3]{
\draw[anchor=mid] (0, 0) node[above] {#1} -- ++(-0.5, -0.5) -- ++(0.5, 0) -- cycle;
\draw[anchor=mid] (-0.35, -0.75) node[anchor=mid east] {#2} +(0.2, 0) node[anchor=mid west] {#3};
}
\newcommand{\righttriangle}[3]{
\draw[anchor=mid] (0, 0) node[above] {#1} -- ++(0.5, -0.5) -- ++(-0.5, 0) -- cycle;
\draw[anchor=mid] (0.15, -0.75) node[anchor=mid east] {#2} +(0.2, 0) node[anchor=mid west] {#3};
}

\newcommand{\lefttriangleinlinebottom}[1]{
\begin{scope}[scale=0.9,thick]
 \draw[anchor=mid] (0, 0) -- ++(-0.5, -0.5) -- ++(0.5, 0) -- cycle;
 \node[anchor=mid west] at (-0.15, -0.75) {#1};
\end{scope}
}
\newcommand{\righttriangleinline}[1]{
\begin{scope}[scale=0.9,thick]
\draw[anchor=mid] (0, 0) node[above] {#1} -- ++(0.5, -0.5) -- ++(-0.5, 0) -- cycle; 
\end{scope}
}
\newcommand{\righttriangleinlinebottom}[1]{
\begin{scope}[scale=0.9,thick]
 \draw[anchor=mid] (0, 0) -- ++(0.5, -0.5) -- ++(-0.5, 0) -- cycle;
 \node[anchor=mid east] at (0.15, -0.75) {#1};
\end{scope}
}
\newcommand{\predrectangle}[4]{
\draw (0, 0) rectangle (0.75, -0.5);
\draw[->, >=stealth'] (1.5, -0.5) node [above right] {#4} arc (0:180:0.5cm);
\draw[anchor=mid] (0, -0.75) node {#1} ++(0.75, 0) node {#2} +(0.75, 0) node {#3};
}
\newcommand{\predrectangleinline}{
\begin{scope}[scale=0.8,thick]
\draw (0, 0) rectangle (0.75, -0.5);
\draw[->, >=stealth'] (1.5, -0.5) arc (0:180:0.5cm);
\end{scope} 
}
\newcommand{\predright}[5]{
\draw (0, 0) node [above] {#1} -- ++(0.75, -0.5)-- ++(-0.75, 0) -- cycle;
\draw[->, >=stealth'] (0.25, -0.5) arc (180:0:0.6cm) node [above right] {#5};
\draw[anchor=mid] (0, -0.75) node {#2} ++(0.75, 0) node {#3} +(0.75, 0) node {#4};
}
\newcommand{\predrightinline}{
\begin{scope}[scale=0.8,thick]
\draw (0, 0) -- ++(0.75, -0.5)-- ++(-0.75, 0) -- cycle;
\draw[->, >=stealth'] (0.25, -0.5) arc (180:0:0.6cm);
\end{scope}
}
\newcommand{\headtriangle}[4]{
\draw (0, 0) node [above] {#1}-- ++(-0.5, -0.5) -- ++(1.0, 0) -- cycle;
\draw[anchor=mid] (-0.5, -0.75) node {#2} ++(0.5, 0) node {#3} +(0.5, 0) node {#4};
}
\newcommand{\headtriangleinline}[1]{
\begin{scope}[scale=0.8,thick]
 \draw (0, 0) node [above] {#1}-- ++(-0.5, -0.5) -- ++(1.0, 0) -- cycle;
\end{scope}
}
\newcommand{\headtriangleinlinebottom}[1]{
\begin{scope}[scale=0.8,thick]
 \draw (0, 0) -- ++(-0.5, -0.5) -- ++(1.0, 0) -- cycle;
 \node at (0, -0.8) {#1};
\end{scope}
}
\newcommand{\headtriangleinlinebottomfull}[3]{
\begin{scope}[scale=0.8,thick]
 \draw (0, 0) -- ++(-0.5, -0.5) -- ++(1.0, 0) -- cycle;
 \draw[anchor=mid] (-0.6, -0.8) node {#1} ++(0.6, 0) node {#2} +(0.6, 0) node {#3};
\end{scope}
}
\newcommand{\halfrectangle}[4]{
\draw (0.75, 0) node [above] {$b$} -- ++(-0.75, 0) -- ++(0, -0.5) -- ++(0.75, 0);
\draw[densely dashed] (0.75, 0) -- ++(0, -0.5);
\draw[->, >=stealth'] (1.5, -0.5) node [above right] {#4} .. controls (1.45, 0) and (0.95, 0) .. +(-0.75, 0);
\draw[anchor=mid] (0, -0.75) node {#1} ++(0.75, 0) node {#2} +(0.75, 0) node {#3};
}
\newcommand{\halfrectangleinline}{
\begin{scope}[scale=0.8,thick]
\draw (0.75, 0) -- ++(-0.75, 0) -- ++(0, -0.5) -- ++(0.75, 0);
\draw[densely dashed] (0.75, 0) -- ++(0, -0.5);
\draw[->, >=stealth'] (1.5, -0.5) .. controls (1.45, 0) and (0.95, 0) .. +(-0.75, 0); 
\end{scope}
}
\newcommand{\halfright}[5]{
\draw[->, >=stealth'] (0.25, -0.5) arc (180:0:0.6cm) node [above right] {#5};
\draw[->, >=stealth'] (1.45, -0.5) .. controls (1.45, 0) and (0.95, 0) .. +(-0.7, 0);
\draw (0.75, -0.2) node [above=7pt] {$b$} -- ++(-0.75, 0.2) node [above] {#1} -- ++(0, -0.5) -- ++(0.75, 0);
\draw[densely dashed] (0.75, -0.2) -- ++(0, -0.3);
\draw[anchor=mid] (0, -0.75) node {#2} ++(0.75, 0) node {#3} +(0.75, 0) node {#4};
}
\newcommand{\lefttrape}[3]{
\draw[anchor=mid] (0, 0) node[above] {#1} -- ++(-0.5, -0.25) -- ++(0, -0.25) -- ++(0.5, 0) -- cycle;
\draw[anchor=mid] (-0.35, -0.75) node[anchor=mid east] {#2} +(0.2, 0) node[anchor=mid west] {#3};
}
\newcommand{\righttrape}[3]{
\draw[anchor=mid] (0, 0) node[above] {#1} -- ++(0.5, -0.25) -- ++(0, -0.25) -- ++(-0.5, 0) -- cycle;
\draw[anchor=mid] (0.15, -0.75) node[anchor=mid east] {#2} +(0.2, 0) node[anchor=mid west] {#3};
}
\newcommand{\lefttrapeinline}[2]{
\begin{scope}[scale=0.9,thick]
\draw[anchor=mid] (0, 0) -- ++(-0.5, -0.25) -- ++(0, -0.25) -- ++(0.5, 0) -- cycle;
\draw[anchor=mid] (-0.35, -0.75) node[anchor=mid east] {#1} +(0.2, 0) node[anchor=mid west] {#2};
\end{scope}
}

\newtheorem{newdef}{Definition}[chapter]
\newtheorem{newlemma}{Lemma}[chapter]
\newtheorem{mytheorem}{Theorem}[chapter]

\newcommand*{\Ja}[1]{%
\begin{CJK}{UTF8}{min}#1\end{CJK}}

\newcommand{\REVISE}[1]{#1}

\dept{Informatics}
\submitdate{June 2016}

\begin{document}
\title{Left-corner Methods for Syntactic Modeling\\
with Universal Structural Constraints}
\author{Hiroshi Noji}

\beforepreface

\prefacesection{Abstract}

Explaining the syntactic variation and universals including the constraints on that variation across languages in the world is essential both from a theoretical and practical point of view.
It is in fact one of the main goals in linguistics.
In computational linguistics, these kinds of syntactic regularities and constraints could be utilized as prior knowledge about grammars, which would be valuable for improving the performance of various syntax-oriented systems such as parsers or grammar induction systems.
This thesis is about such syntactic universals.

The primary goal in this thesis is to identify better syntactic constraint or bias, that is language independent but also efficiently exploitable during sentence processing.
We focus on a particular syntactic construction called center-embedding, which is well studied in psycholinguistics and noted to cause particular difficulty for comprehension.
Since people use language as a tool for communication, one expects such complex constructions to be avoided for communication efficiency.
From a computational perspective, center-embedding is closely relevant to a {\it left-corner} parsing algorithm, which can capture the degree of center-embedding of a parse tree being constructed.
This connection suggests left-corner methods can be a tool to exploit the universal syntactic constraint that people avoid generating center-embedded structures.
We explore such utilities of center-embedding as well as left-corner methods extensively through several theoretical and empirical examinations.

We base our analysis on dependency syntax.
This is because our focus in this thesis is the language universality.
Now the number of available dependency treebanks are growing rapidly compared to the treebanks of phrase-structure grammars thanks to the recent standardization efforts of dependency treebanks across languages, such as the Universal Dependencies project.
We use these resources, consisting of more than 20 treebanks, which enable us to examine the universality of particular language phenomena, as we pursue in this thesis.

First, we quantitatively examine the universality of center-embedding avoidance using a collection of dependency treebanks.
Previous studies on center-embedding in psycholinguistics have been limited to behavioral studies focusing on particular languages or sentences.
Our study contrasts with these previous studies, and provides the first quantitative results on center-embedding avoidance.
Along with these experiments, we provide a parser that can capture the degree of center-embedding of a {\it dependency} tree being built, by extending a left-corner parsing algorithm for dependency grammars.
The main empirical finding in this study is that center-embedding is in fact a rare phenomenon across languages.
This result also suggests a left-corner parser could be utilized as a tool exploiting the universal syntactic constraints in languages.

We then explore such utility of a left-corner parser in the application of unsupervised grammar induction.
In this task, the input to the algorithm is a collection of sentences, from which the model tries to extract the salient patterns on them as a grammar.
This is a particularly hard problem although we expect the universal constraint may help in improving the performance since it can effectively restrict the possible search space for the model.
We build the model by extending the left-corner parsing algorithm for efficiently tabulating the search space except those involving center-embedding up to a specific degree.
Again, we examine the effectiveness of our approach on many treebanks, and demonstrate that often our constraint leads to better parsing performance.
We thus conclude that left-corner methods are particularly useful for syntax-oriented systems, as it can exploit efficiently the inherent universal constraints in languages.

\prefacesection{Acknowledgments}

I could not have finished writing this thesis without the support of many people.

I am greatly indebted to my advisor, Professor Yusuke Miyao for his continuous help and support throughout my PhD.
He was always giving me many constructive and encouraging suggestions when I had trouble in writing a paper, preparing presentation slides, interpreting experimental results, and so on, which cannot be listed in this space.
He is arguably the best advisor I have met so far, and also is an excellent research role model for me.

I would like to thank my thesis committee, Professor Hiroshi Nakagawa, Professor Edson Miyamoto, Professor Makoto Kanazawa, and Professor Daichi Mochihashi.
They all have different backgrounds, and gave me very valuable and insightful comments from their own perspectives.
Professor Daichi Mochihashi guided me in my masters into the research field of unsupervised learning, which was my starting point of this work focusing on unsupervised grammar induction.

My first advisor in my masters in University of Tokyo, Professor Kumiko Tanaka-Ishii, introduced me to the research filed of computational linguistics.
Though I could not interact with her in the last three years in my PhD, her advices to me had great impacts on my thinking on research.
The fundamental idea behind this thesis, exploring the universal property of language from a computational perspective, was arguably the one inspired by her philosophy on computational linguistics.

I am grateful to Mark Johnson for hosting me in total three times (!) as a visitor in Macquarie University.
Regular meetings with him were always exciting, and have sharpen my thinking on the tasks of unsupervised learning.
It was really my fortunate to have collaboration with a great researcher as you during my PhD.

As a member in Miyao lab, I had opportunities to interact with many great researchers in NII.
Special thanks to Takuya Matsuzaki, who was always welcoming informal discussion in his room, and I remember that the core research idea on this thesis, left-corner parsing, has been came up with during a conversation with him.
Pascual Mart\'{i}nez-G\'{o}mez, Yuka Tateishi, and Sumire Uematsu always gave me many constructive comments especially in my practices on presentations.
I am grateful to Bevan Johns, who proofread a part of my first draft of this thesis, and to Pontus Stenetorp, who carefully read and gave comments on my draft of the COLING paper.

I am also grateful to my lab mates, especially Han Dan and Sho Hoshino, as well as intern students to NII:
Ying Xu, who always cheered me up, encouraged me to come to the lab early in every morning, and also helped me a lot in Beijing during ACL conference; and 
Le Quang Thang, whom I really enjoyed mentoring in our supervised parsing works.


During my research life in NII, the secretaries in Miyao lab, Keiko Kigoshi and Yuki Amano, not only greatly help in many office works, but also are a few connection to the outside of the research.
I was able to continue research in NII thanks to their kind considerations.

Many thanks to my colleagues and friends at Macquarie University, especially to Dat Quoc Nguyen, Kinzang Chhogyal, and Anish Kumar.
I have enjoyed daily lunch and pub nights with you, with which I was able to stay sane without a lot of stress in my first long-term stay abroad.
I am also indebted to John Pate, who helps me a lot at the beginning of my study on unsupervised grammar induction.

Finally, I would like to thank my parents, my sister, and my brother.
Thank you for your infinite supports through my life, and thank you for making me who I am.

\afterpreface

\chapter{Introduction}
\label{chap:1}

Explaining the syntactic variation and universals including the constraints on that variation across languages in the world is essential both from a theoretical and practical point of view.
It is in fact one of the main goals in linguistics \cite{Greenberg-1963,Dryer-1992,Evans2009}.
In computational linguistics, these kinds of syntactic regularities and constraints could be utilized as prior knowledge about grammars, which would be valuable for improving the performance of various syntax-oriented systems such as parsers or grammar induction systems, e..g, by being encoded as {\it features} in a system \cite{collins-thesis1999,McDonald:2005:NDP:1220575.1220641}.
This thesis is about such syntactic universals.

Our goal in this thesis is to identify a good syntactic constraint that fits well to the natural language sentences and thus could be exploited to improve the performance of syntax-oriented systems such as parsers.
For this end, we pick up a well known linguistic phenomenon that might be universal across languages, empirically examine its language universality across diverse languages using cross-linguistic datasets, and present computational experiments to demonstrate its utility in a real application.
Along with this, we also define several computational algorithms that efficiently exploit the constraint during sentence processing.
For an application, we show that our constraint will help in the task of unsupervised syntactic parsing, or grammar induction where the goal is to find salient syntactic patterns without explicit supervision about grammars.

In linguistics, one pervasive hypothesis about the origin of such syntactic constraints is that they come from the limitations on the human cognitive mechanism and pressures associated with language acquisition and use \cite{jaeger2011language-gsc,Fedzechkina30102012}.
In other words, since the language is a tool for communication, it is natural to assume that its shape has been formed to increase the daily communication efficiency or the learnability for language learners.
The underlying commonalities in diverse languages are then understood as the outcome of such pressures that every language user might naturally suffer from.
Our focused constraint in this thesis also has its origin in the restriction of the human ability of comprehension observed in several psycholinguistic experiments, which we introduce next.

\paragraph{Center-embedding}
It is well known in the psycholinguistic literature that a nested, or center-embedded structure is particularly difficult for compherension:
\enumsentence{\# The reporter [who the senator [who Mary met attacked] ignored] the president.} \label{sent:intro:embedding}
\noindent
This sentence is called center-embedding by its syntactic construction indicated with brackets.
This observation will be the starting point of the current study.
The difficulty of center-embedded structures has been testified across a number of languages including English \cite{Gibson2000The-dependency-,Chen2005144} and Japanese \cite{COGS:COGS1067}.
Compared to these behavioral studies, the current study aims to characterize the phenomenon of center-embedding from {\it computational} and quantitative perspectives.
For instance, one significance of the current study is to show that center-embedding is in fact a rarely observed syntactic phenomenon across a variety of languages.
We verify this fact using syntactically annotated corpora, i.e., treebanks of more than 20 languages.

\paragraph{Left-corner}
Another important concept in this thesis is {\it left-corner} parsing.
A left-corner parser parses an input sentence on a {\it stack};
the distinguished property of it is that its stack depth increases only when generating, or accepting center-embedded structures.
These formal properties of left-corner parsers were studied more than 20 years ago \cite{abney91memory,conf/coling/Resnik92} although until now there exists little study concerning its empirical behaviors as well as its potential for a device to exploit syntactic regularities of languages as we investigate in this thesis.
One previous attempt for utilizing a left-corner parser in a practical application is Johnson's linear-time tabular parsing by approximating the state space of a parser by a finite state machine.
However, this trial was not successful \cite{conf/acl/Johnson98}.\footnote{The idea is that since a left-corner parser can recognize most of (English) sentences within a limited stack depth bound, e.g., 3, the number of possible stack configurations will be constant and we may construct a finite state machine for a given context-free grammar. However in practice, the grammar constant for this algorithm gets much larger, leading to $O(n^3)$ asymptotic runtime, the same as the ordinary parsing method, e.g., CKY.}

\paragraph{Dependency}
Our empirical examinations listed above will be based on the syntactic representation called {\it dependency} structures.
In computational linguistics, constituent structures have long played a central role as a representation of natural language syntax \cite{Stolcke:1995:EPC:211190.211197,P97-1003,J98-4004,klein-manning:2003:ACL} although this situation has been changed and the recent trend in the parsing community has favored dependency-based representations, which are conceptually simpler and thus often lead to more efficient algorithms \cite{Nivre2003,Yamada03,McDonald:2005:NDP:1220575.1220641,GoldbergTDP13}.
Another important reason for us to focus on this representation is that its {\it unsupervised} induction is more tractable than the constituent representation, such as phrase-structure grammars.
In fact, significant research on unsupervised parsing has been done in this decade though much of it assumes dependency trees as the underlying structure \cite{klein-manning:2004:ACL,smith-eisner-2006-acl-sa,bergkirkpatrick-EtAl:2010:NAACLHLT,marevcek-vzabokrtsky:2012:EMNLP-CoNLL,spitkovsky-alshawi-jurafsky:2013:EMNLP}.
We discuss this computational issue more in the next chapter (Section \ref{sec:2:unsupervised}).

The last, and perhaps the most essential advantage of a dependency representation is its cross-linguistic suitability.
For studying the empirical behavior of some system across a variety of languages, the resources for those languages are essential.
Compared to constituent structures, dependency annotations are available in many corpora covering more than 20 languages across diverse language families.
Each treebank typically contains thousands of sentences with manually parsed syntactic trees.
We use such large datasets to examine our hypotheses about universal properties of languages.
Though the concepts introduced above, center-embedding and left-corner parsing, are both originally defined on constituent structures, we describe in this thesis a method by which they can be extended to dependency structures via a close connection between two representations.

\section{Tasks and Motivations}
\label{sec:1:problem}

More specifically, the tasks we tackle in this thesis can be divided into the following two categories, each of which is based on specific motivations.

\subsection{Examining language universality of center-embedding avoidance}
We first examine the hypothesis that center-embedding is a language phenomenon that every language user tries to avoid {\it regardless} of language.
The quantitative study for this question across diverse languages has not yet been performed.
Two motivations exist for this analysis:
One is rather scientific:
we examine the explanatory power of center-embedding avoidance as a universal grammatical constraint.
This is ambitious though we put more weight on the second, rather practical motivation:
the possibility that avoidance of center-embedding is a good syntactic bias to restrict the space of possible tree structures of natural language sentences.
These analyses are the main topic of Chapter \ref{chap:transition}.

\subsection{Unsupervised grammar induction}
\label{sec:intro:unsupervised}
We then consider applying the constraint with center-embedding into the application of {\it unsupervised grammar induction}.
In this task, the input to the algorithm is a collection of sentences, from which the model tries to extract the salient patterns as a grammar.
This setting contrasts with the more familiar {\it supervised} parsing task in which typically some machine learning algorithm learns the mapping from a sentence to the syntactic tree based on the training examples, i.e., sentences paired with their corresponding parse trees.
In the unsupervised setting, our goal is to obtain a model that can parse a sentence without access to the correct trees for training sentences.
This is a particularly hard problem though we expect the universal syntactic constraint may help in improving the performance since it can effectively restrict the possible search space for the model.

\paragraph{Motivations}
A typical reason to tackle this task is a purely engineering one:
Although the number of languages that we can access to the resource (i.e., treebank) increases, there are still so many languages in the world for which little to no resources are available since the creation of a new treebank from scratch is still very hard and time consuming.
Unsupervised learning of grammars would be helpful for this situation, as it provides a cheap solution without requiring the manual efforts of linguistic experts.
A more realistic setting might be to use the output of an unsupervised system as the initial annotation, which could then be corrected by experts later.
In short, a better unsupervised system can reduce the effort of experts in preparing new treebanks.
This motivation can be held in any efforts of unsupervised grammar induction.

However, as we do in this thesis, the grammar induction with particular syntactic biases or constraintes would also be appealing for the following reasons as well:
\begin{itemize}
 \item We can regard this task as a typical example of more broad problems of learning syntax without explicit supervision.
       An example of such problem is a grounding task, in which the learner induces the model of (intermediate) tree structures that bridge an input sentence and its semantics, which may be represented in one of several different forms, depending on the task and corpus, e.g., the logical expression \cite{Zettlemoyer05learningto,kwiatkowksi-EtAl:2010:EMNLP} and the answer to the given question \cite{liang-jordan-klein:2011:ACL-HLT2011,berant-EtAl:2013:EMNLP,kwiatkowski-EtAl:2013:EMNLP,kushman-EtAl:2014:P14-1}.
       In these tasks, though some implicit supervision is provided, the search space is typically still very large.
       Obtaining a positive result for the current unsupervised learning, we argue, would present an important starting point for extending the current idea into such related grammar induction tasks.
       What type of supervision we should give for those tasks is also still an open problem;
       one possibility is that a good {\it prior} for general natural language syntax, as we investigate here, would reduce the amount of supervision necessary for successful learning.
       Finally, we claim that although the current study focuses on inducing dependency structures, the presented idea, avoiding center-embedding during learning, is general enough and not necessarily restricted to the dependency induction tasks.
       The main reason why we focus on dependency structures is rather computational (see Section \ref{sec:2:unsupervised}), but it may not hold in the grounded learning tasks in the previous works cited above.
       Moreover, recently more sophisticated grammars such as combinatory categorical grammars (CCGs) are shown to be learnable when appropriate light supervision is given as seed knowledge \cite{DBLP:journals/tacl/BiskH13,bisk-hockenmaier:2015:ACL-IJCNLP,AAAI159835}.
       We thus believe that the lesson from the current study will also shed light on those related learning tasks that do not assume dependency trees as the underlying structures.
 \item The final motivation is in the relevance to understanding of child language acquisition.
       Computational modeling of the language acquisition process, in particular using probabilistic models, has gained much attention in recent years \cite{Goldwater200921,kwiatkowski-EtAl:2012:EACL2012,johnson-demuth-frank:2012:ACL2012,doyle-levy:2013:NAACL-HLT}.
       Although many of those works cited above focus on modeling of relatively early acquisition problems, e.g., word segmentation from phonological inputs, some initial studies regarding acquisition mechanism of grammar also exist \cite{TACL504}.

       We argue here that our central motivation is {\it not} to get insights into the language acquisition mechanism although the structural constraint that we consider in this thesis (i.e., center-embedding) originally comes from observation of human sentence processing.
       This is because our problem setting is far from the real language acquisition scenario that a child may undergo.
       There exist many discrepancies between them;
       the most problematic one is found in the input to the learning algorithm.
       For resource reasons, the input sentences to our learning algorithm are largely written texts for adults, e.g, newswires, novels, and blogs.
       This contrasts with the relevant studies cited above on word segmentation in which the input for training is phonological forms of child directed speech, which is, however, available in only a few languages such as English.
       This poses a problem since our interest in this thesis is the language universality of the constraint, which needs many language treebanks to be evaluated.
       Another limitation of the current approach is that every model in this thesis assumes the part-of-speech (POS) of words in a sentence as its input rather than the surface form.
       This simplification makes the problem much simpler and is employed in most previous studies \cite{klein-manning:2004:ACL,smith-eisner-2006-acl-sa,bergkirkpatrick-EtAl:2010:NAACLHLT,DBLP:journals/tacl/BiskH13,grave-elhadad:2015:ACL-IJCNLP}, but it is of course an unrealistic assumption about inputs that children receive.

       Our main claim in this direction is that the success of the current approach would lead to the further study about the connection between the child language acquisition and computational modeling of the acquisition process.
       We leave the remaining discussion about this topic in the conclusion of this thesis.
\end{itemize}

\section{What this thesis is not about}
\label{sec:intro:notabout}
This thesis is not about psycholinguistic study, i.e., we do not attempt to reveal the mechanism of human sentence processing.
Our main purpose in referring to the literature in psycholinguistics is to get insights for the syntactic patterns that every language might share to some extent and thus could be exploited from a system of computational linguistics.
We would be pleased if our findings about the universal constraint affect the thinking of psycholinguists but this is not the main goal of the current thesis.

\section{Contributions}
Our contributions can be divided into the following four parts.
The first two are our conceptual, or algorithmic contributions while the latter two are the contributions of our empirical study.

\begin{description}
 \item[Left-corner dependency parsing algorithm] We show how the idea of left-corner parsing, which was originally developed for recognizing constituent structures, can be extended to dependency structures.
            We formalize this algorithm in the framework of transition-based parsing, a similar device to the pushdown automata often used to describe the parsing algorithm for dependency structures.
            The resulting algorithm has the property that its stack depth captures the degree of center-embedding of the recognizing structure.
 \item[Efficient dynamic programming] We extend this algorithm into the tabular method, i.e., chart parsing, which is necessary to combine the ideas of left-corner parsing and unsupervised grammar induction.
            In particular, we describe how the idea of head splitting \cite{Eisner:1999:EPB:1034678.1034748,eisner-2000-iwptbook}, a technique to reduce the time complexity of chart-based dependency parsing, can be applied in the current setting.
 \item[Evidence on the universality of center-embedding avoidance] We show that center-embedding is a rare construction across languages using treebanks of more than 20 languages. Such large-scale investigation has not been performed before in the literature. Our experiment is composed of two types of complementary analyses: a static, counting-based analysis of treebanks and a supervised parsing experiment to see the effect of the constraint when some amount of parse errors occurs.
 \item[Unsupervised parsing experiments with structural constraints] We finally show that our constraint does improve the performance of unsupervised induction of dependency grammars in many languages.
 \end{description}

\section{Organization of the thesis}
The following chapters are organized as follows.

\begin{itemize}
 \item In Chapter \ref{chap:bg}, we summarize the backgrounds necessary to understand the following chapters of the thesis including several syntactic representations, the EM algorithm for acquiring grammars, and left-corner parsing.
 \item In Chapter \ref{chap:corpora}, we summarize the multilingual corpora we use in our experiments in the following chapters.
 \item Chapter \ref{chap:transition} covers the topics of the first and third contributions in the previous section.
       We first define a tool, i.e., a left-corner parsing algorithm for dependency structures, for our corpus analysis in the remainder of the chapter.
 \item Chapter \ref{chap:induction} covers the remaining, second and fourth contributions in the previous section.
       Our experiments on unsupervised parsing require the formulation of the EM algorithm, which relies on chart parsing for calculating sufficient statistics.
       We thus first develop a new dynamic programming algorithm and then apply it to the unsupervised learning task.
 \item Finally, in Chapter \ref{chap:conclusion}, we summarize the results obtained in this research and give directions for future studies.
\end{itemize}

\chapter{Background}
\label{chap:bg}

The topics covered in this chapter can be largely divided into four parts.
Section \ref{sec:bg:representation} defines several important concepts for representing syntax, such as constituency and dependency, which become the basis of all topics discussed in this thesis.
We then discuss left-corner parsing and related issues in Section \ref{sec:bg:left-corner}, such as the formal definition of center-embedding, which are in particular important to understand the contents in Chapter \ref{chap:transition}.
The following two sections are more relevant to our application of unsupervised grammar induction discussed in Chapter \ref{chap:induction}.
In Section \ref{sec:2:learning}, we describe the basis of learning probabilistic grammars, such as the EM algorithm.
Finally in Section \ref{sec:2:unsupervised}, we provide the thorough survey of the unsupervised grammar induction, and make clear our motivation and standpoint for this task.

\section{Syntax Representation}
\label{sec:bg:representation}

This section introduces several representations to describe the natural language syntax appearing in this thesis, namely context-free grammars, constituency, and dependency grammars, and discuss the connection between them.
Though our focused representation in this thesis is dependency, the concepts of context-free grammars and constituency are also essential for us.
For example, context-free grammars provide the basis for probabilistic modeling of tree structures as well as parameter estimation for it;
We discuss how our dependency-based model can be represented as an instance of context-free grammars in Section \ref{sec:2:learning}.
The connection between constituency and dependency appears many times in this thesis.
For instance, the concept of {\it center-embedding} (Section \ref{sec:bg:left-corner}) is more naturally understood with constituency rather than with dependency.

This section is about syntax representation or grammars and we do not discuss {\it parsing} but to see how the analysis with a grammar looks like, we mention a {\it parse} or a parse tree, which is the result of parsing for an input string (sentence).

\subsection{Context-free grammars}
\label{sec:2:cfg}
A context-free grammar (CFG) is a useful model to describe the hierarchical syntactic structure of an input string (sentence).
Formally a CFG is a quadruple $G=(N,\Sigma,P,S)$ where $N$ and $\Sigma$ are disjoint finite set of symbols called nonterminal and terminal symbols respectively.
Terminal symbols are symbols that appear at leaf positions of a tree while nonterminal symbols appear at internal positions.
$S \in N$ is the start symbol.
$P$ is the set of rules of the form $A\rightarrow\beta$ where $A \in N$ and $\beta \in (N \cup \Sigma)^*$.

Figure \ref{fig:2:cfg} shows an example of a CFG while Figure \ref{fig:2:cfg-deriv} shows an example of a parse with that CFG.
On a parse tree {\it terminal nodes} refer to the nodes with terminal symbols (at leaf positions) while {\it nonterminal nodes} refer to other internal nodes with nonterminal symbols.
{\it Preterminal nodes} are nodes that appear just above terminal nodes (e.g., VBD in Figure \ref{fig:2:cfg-deriv}).

This model is useful because there is a well-known polynomial (cubic) time algorithm for parsing an input string with it, which also provides the basis for parameters estimation when we develop probabilistic models on CFGs (see Section \ref{sec:2:em}).

\begin{figure}[t]
 \centering
 \begin{tabular}[t]{|rll|} \hline
  S &$\rightarrow$ &NP~~VP \\
  VP&$\rightarrow$ &VBD~~NP \\
  NP&$\rightarrow$ &DT~~NN \\
  NP&$\rightarrow$ &Mary \\
  VBD&$\rightarrow$&met \\
  DT&$\rightarrow$&the \\
  NN&$\rightarrow$&senator \\\hline
 \end{tabular}
 \caption{A set of rules in a CFG in which $N = \{$S, NP, VP, VBD, DT, NN$\}$, $\Sigma = \{$Mary, met, the, senator $\}$, and $S = $ S (the start symbol).}
 \label{fig:2:cfg}
\end{figure}

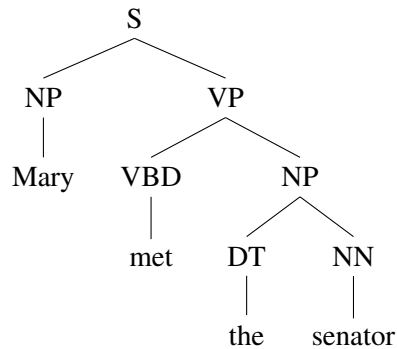
\begin{figure}[t]
 \centering
 \begin{tikzpicture}[sibling distance=10pt]
  \Tree
  [.S
  [.NP Mary ]
  [.VP [.VBD met ] [.NP [.DT the ] [.NN senator ] ] ] ]
  ]
 \end{tikzpicture}
 \caption{A parse tree with the CFG in Figure \ref{fig:2:cfg}.}
 \label{fig:2:cfg-deriv}
\end{figure}

\paragraph{Chomsky normal form}
A CFG is said to be in Chomsky normal form (CNF) if every rule in $P$ has the form $A \rightarrow B~C$ or $A \rightarrow a$ where $A,B,C \in N$ and $a \in \Sigma$;
that is, every rule is a binary rule or a unary rule and a unary rule is only allowed on a preterminal node.
The CFG in Figure \ref{fig:2:cfg} is in CNF.
We often restrict our attention to CNF as it is closely related to projective dependency grammars, our focused representation described in Section \ref{sec:2:dependency}.

\subsection{Constituency}
\label{sec:bg:constituency}
The parse in Figure \ref{fig:2:cfg-deriv} also describes the {\it constituent} structure of the sentence.
Each constituent is a group of consecutive words that function as a single cohesive unit.
In the case of tree in Figure \ref{fig:2:cfg-deriv}, each constituent is a phrase spanned by some nonterminal symbol (e.g., ``the senator'' or ``met the senator'').

We can see that the rules in Figure \ref{fig:2:cfg} define how a smaller constituents combine to form a larger constituent.
This grammar is an example of {\it phrase-structure grammars}, in which each nonterminal symbol such as NP and VP describes the syntactic role of the constituent spanned by that nonterminal.
For example, NP means the constituent is a noun phrase while VP means the one is a verb phrase.
The phrase-structure grammar is often contrasted with dependency grammars, but we note that the concept of constituency is not restricted to phrase-structure grammars and plays an important role in dependency grammars as well, as we describe next.

\subsection{Dependency grammars}

\label{sec:2:dependency}

\begin{figure}[t]
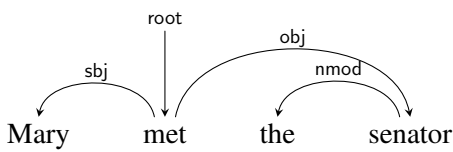

 \centering
 \begin{dependency}[theme=simple,label style={font=\sf}]
  \begin{deptext}[column sep=0.8cm]
   Mary \& met \& the \& senator \\
  \end{deptext}
  \depedge{2}{1}{${\sf sbj}$}
  \depedge{2}{4}{${\sf obj}$}
  \depedge{4}{3}{${\sf nmod}$}
  \deproot[edge unit distance=2ex]{2}{${\sf root}$}
 \end{dependency}
 \caption{Example of labelled projective dependency tree.}
 \label{fig:2:projtree}
\end{figure}

\begin{figure}[t]
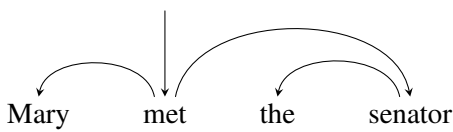

 \centering
 \begin{dependency}[theme=simple]
  \begin{deptext}[column sep=0.8cm]
   Mary \& met \& the \& senator \\
  \end{deptext}
  \depedge{2}{1}{}
  \depedge{2}{4}{}
  \depedge{4}{3}{}
  \deproot[edge unit distance=2ex]{2}{}
 \end{dependency}
 \caption{Example of unlabelled projective dependency tree.}
 \label{fig:2:projtree-unlabelled}
\end{figure}

Dependency grammars analyze the syntactic structure of a sentence as a directed tree of word-to-word dependencies.
Each dependency is represented as a directed arc from a {\it head} to a {\it dependent}, which is argument or adjunct and is modifying the head syntactically or semantically.
Figure \ref{fig:2:projtree} shows an example of an analysis with a dependency grammar.
We call these directed trees {\it dependency trees}.

The question of what is the head is a matter of debate in linguistics.
In many cases this decision is generally agreed but the analysis of certain cases is not settled, in particular those around function words \cite{zwicky199313}.
For example although ``senator'' is the head of the dependency between ``the'' and ``senator'' in Figure \ref{fig:2:projtree} some linguists argue ``the'' should be the head \cite{abney1987english}.
We discuss this problem more in Chapter \ref{chap:corpora} where we describe the assumed linguistic theory in each treebank used in our experiments.
See also Section \ref{sec:ind:eval} where we discuss that such discrepancies in head definitions cause a problem in evaluation for unsupervised systems (and our solution for that).

\paragraph{Labelled and unlabelled tree}
If each dependency arc in a dependency tree is annotated with a label describing the syntactic role between two words as in Figure \ref{fig:2:projtree}, that tree is called a {\it labeled} dependency tree.
For example the ${\sf sbj}$ label between ``Mary'' and ``met'' describes the subject-predicate relationship.
A tree is called {\it unlabeled} if those labels are omitted, as in Figure \ref{fig:2:projtree-unlabelled}.

In the remainder of this thesis, we only focus on {\it unlabeled} dependency trees although now most existing dependency-based treebanks provide labeled annotations of dependency trees.
For our purpose, dependency labels do not play the essential role.
For example, our analyses in Chapter \ref{chap:transition} are based only on the tree shape of dependency trees, which can be discussed with unlabeled trees.
In the task of unsupervised grammar induction, our goal is to induce the unlabeled dependency structures as we discuss in detail in Section \ref{sec:2:unsupervised}.

\paragraph{Constituents in dependency trees}
The idea of constituency (Section \ref{sec:bg:constituency}) is not limited to phrase-structure grammars and we can identify the constituents in dependency trees as well.
In dependency trees, a constituent is a phrase that comprises of a head and its descendants.
For example, ``met the senator'' in Figure \ref{fig:2:projtree-unlabelled} is a constituent as it comprises of a head ``met'' and its descendants ``the senator''.
Constituents in dependency trees may be more directly understood by considering a CFG for dependency grammars and the parses with it, which we describe in the following.

\subsection{CFGs for dependency grammars and spurious ambiguity}
\label{sec:bilexical}

\begin{figure}[t]
 \centering
 \begin{tabular}[t]{rlll} \hline
  \multicolumn{3}{l}{Rewrite rule}& Semantics \\ \hline
  S       &$\rightarrow$ &X[{\it a}] & Select {\it a} as the root word. \\
  X[{\it a}] &$\rightarrow$ &X[{\it a}]~~X[{\it b}] & Select {\it b} as a right modifier of {\it a}. \\
  X[{\it a}] &$\rightarrow$ &X[{\it b}]~~X[{\it a}] & Select {\it b} as a left modifier of {\it a}. \\
  X[{\it a}] &$\rightarrow$ &{\it a}          & Generate a terminal symbol. \\\hline
 \end{tabular}
 \caption{A set of template rules for converting dependency grammars into CFGs.
 {\it a} and {\it b} are lexical tokens (words) in the input sentence.
 X[{\it a}] is a nonterminal symbol indicating the head of the corresponding span is {\it a}.}
 \label{fig:2:bcfg}
\end{figure}

\begin{figure}[t]
 \centering
 \begin{tikzpicture}[sibling distance=10pt]
  \Tree
  [.S
  [.X[met]
  [.X[Mary] Mary ]
  [.X[met] [.X[met] met ] [.X[senator] [.X[the] the ] [.X[senator] senator ] ] ]
  ]
  ]
 \end{tikzpicture}
 \caption{A CFG parse that corresponds to the dependency tree in Figure \ref{fig:2:projtree-unlabelled}.}
 \label{fig:2:bcfg-deriv}
\end{figure}
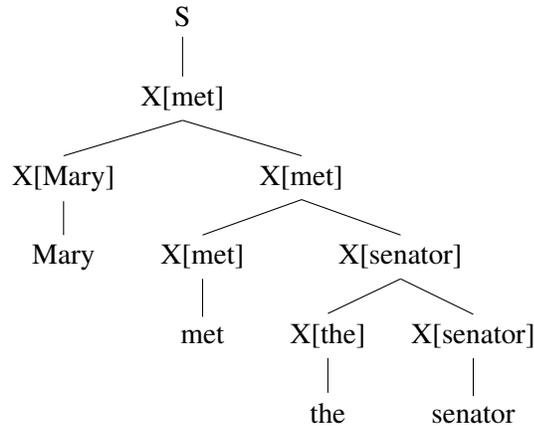

Figure \ref{fig:2:bcfg-deriv} shows an example of a CFG parse, which corresponds to the dependency tree in Figure \ref{fig:2:projtree-unlabelled} but looks very much like the constituent structure in Figure \ref{fig:2:cfg-deriv}.
With this representation, it is very clear that {\it the senator} or {\it met the senator} is a constituent in the tree.
We often rewrite an original dependency tree in this CFG form to represent the underlying constituents explicitly, in particular when discussing the extension of the concept of center-embedding and left-corner algorithm, which have originally assumed (phrase-structure-like) constituent structure, to dependency.

In this parse, every rewrite rule has one of the forms in Figure \ref{fig:2:bcfg}.
Each rule specifies one dependency arc between a head and a dependent.
For example, the rule X[senator] $\rightarrow$ X[the]~~X[senator] means that ``senator'' takes ``the'' as its left dependent.

\paragraph{Spurious ambiguity}

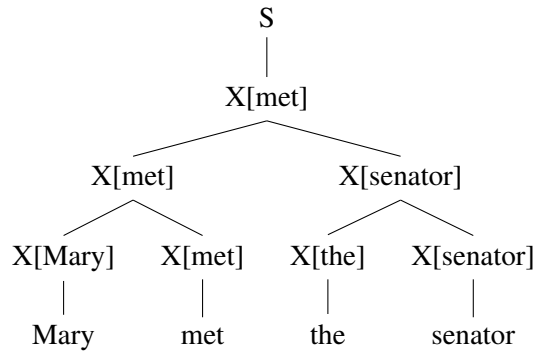
\begin{figure}[t]
 \centering
 \begin{tikzpicture}[sibling distance=10pt]
  \Tree
  [.S
  [.X[met]
  [.X[met] [.X[Mary] Mary ] [.X[met] met ] ]
  [.X[senator] [.X[the] the ] [.X[senator] senator ] ]
  ]
  ]
 \end{tikzpicture}
 \caption{Another CFG parse that corresponds to the dependency tree in Figure \ref{fig:2:projtree-unlabelled}.}
 \label{fig:2:bcfg-deriv2}
\end{figure}

On the tree in Figure \ref{fig:2:projtree-unlabelled}, we can identify ``'Mary met'' is also a constituent, which is although not a constituent in the parses in Figure \ref{fig:2:bcfg-deriv} and Figure \ref{fig:2:bcfg}.
This divergence is related to the problem of {\it spurious ambiguity}, which indicates each dependency tree may correspond to more than one CFG parse.
In fact, we can also build a CFG parse corresponding to Figure \ref{fig:2:projtree-unlabelled}, in which contrary to Figure \ref{fig:2:cfg-deriv} the constituent of ``Mary met'' is explicitly represented with the nonterminal X[met] dominating ``Mary met''.

This ambiguity becomes the problem when we analyze the type of structure for a given dependency tree, e.g., whether a tree contains any center-embedded constructions.
We will see the details and our solution for this problem later in Sections \ref{sec:oracle} and \ref{sec:memorycost}.
Another related issue with this ambiguity is that it prevents us to use the EM algorithm for learning of the models built on this CFG, which we discuss in detail in Section \ref{sec:2:sbg}.

\subsection{Projectivity}
\label{sec:bg:projective}
\begin{figure*}[t]
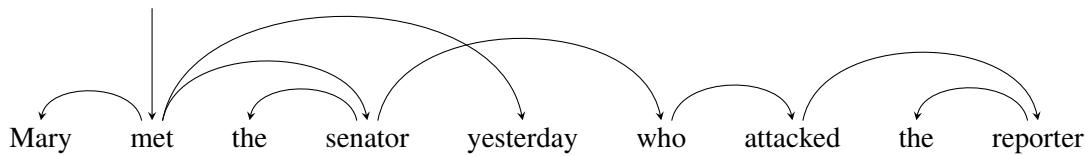

 \centering
 \begin{dependency}[theme=simple]
  \begin{deptext}[column sep=0.6cm]
   Mary \& met \& the \& senator \& yesterday \& who \& attacked \& the \& reporter \\
  \end{deptext}
  \depedge{2}{1}{}
  \depedge{2}{4}{}
  \depedge{4}{3}{}
  \depedge{2}{5}{}
  \depedge{4}{6}{}
  \depedge{6}{7}{}
  \depedge{7}{9}{}
  \depedge{9}{8}{}
  \deproot[edge unit distance=2.5ex]{2}{}
 \end{dependency}
 \caption{Example of non-projective dependency tree.}
 \label{fig:2:nonprojtree-unlabelled}
\end{figure*}

A dependency tree is called {\it projective} if the tree does not contain any {\it crossing dependencies}.
Every dependency tree appeared so far is projective.
An example of non-projective tree is shown in Figure \ref{fig:2:nonprojtree-unlabelled}.
Though we have not mentioned explicitly, the conversion method above can only handle projective dependency trees.
If we allow non-projective structures in our analysis, then the model or the algorithm typically gets much more complex \cite{mcdonald-satta:2007:IWPT2007,journals/coling/Gomez-RodriguezCW11,DBLP:journals/coling/Kuhlmann13,TACL23}.\footnote{
The maximum spanning tree (MST) algorithm \cite{McDonald:2005:NDP:1220575.1220641} enables non-projective parsing in time complexity $O(n^2)$, which is more efficient than the ordinary CKY-based algorithm \cite{Eisner:1999:EPB:1034678.1034748} though the model form (i.e., features or conditioning contexts) is restricted to be quite simple.}


Non-projective constructions are known to be relatively rare cross-linguistically \cite{nivre-EtAl:2007:EMNLP-CoNLL2007,DBLP:journals/coling/Kuhlmann13}.
Thus, along with the mathematical difficulty for handling them, often the dependency parsing algorithm is restricted to deal with only projective structures.
For example, as we describe in Section \ref{sec:2:unsupervised}, most existing systems of unsupervised dependency induction restrict their attention only on projective structures.
Note that existing treebanks contain non-projective structures in varying degree so the convention is to restrict the model to generate only projective trees and to evaluate its quality against the (possibly) non-projective gold trees.
We follow this convention in our experiments in Chapter \ref{chap:induction} and generally focus only on projective dependency trees in other chapters as well, if not mentioned explicitly.

\begin{figure*}[t]
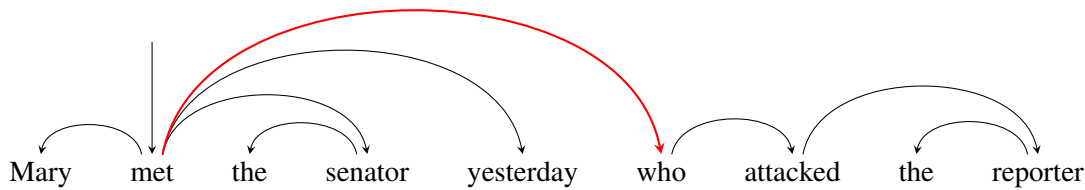

 \centering
 \begin{dependency}[theme=simple]
  \begin{deptext}[column sep=0.6cm]
   Mary \& met \& the \& senator \& yesterday \& who \& attacked \& the \& reporter \\
  \end{deptext}
  \depedge{2}{1}{}
  \depedge{2}{4}{}
  \depedge{4}{3}{}
  \depedge{2}{5}{}
  \depedge[thick,red]{2}{6}{}
  \depedge{6}{7}{}
  \depedge{7}{9}{}
  \depedge{9}{8}{}
  \deproot[edge unit distance=2.5ex]{2}{}
 \end{dependency}
 \caption{The result of pseudo projectivization to the tree in Figure \ref{fig:2:nonprojtree-unlabelled}.}
 \label{fig:2:pseudo-unlabelled}
\end{figure*}

\paragraph{Pseudo-projectivity}

There is a known technique called pseudo-projectivization \cite{nivre-nilsson:2005:ACL}, which converts any non-projective dependency trees into some projective trees with minimal modifications.
The tree in Figure \ref{fig:2:pseudo-unlabelled} shows the result of this procedure into the non-projective tree in Figure \ref{fig:2:nonprojtree-unlabelled}.\footnote{
In the original formalization, pseudo-projectivization also performs label conversions.
That is, the label on a (modified) dependency arc is changed for memorizing the performed operations;
With this memorization, the converted tree does not loose the information.
\newcite{nivre-nilsson:2005:ACL} show that non-projective dependency parsing is possible with this conversion and parsing algorithms that assume projectivity, by training and decoding with the converted forms and recovering the non-projective trees from the labeled (projective) outputs.
Since our focus in basically only unlabeled trees, we ignore those labels in Figure \ref{fig:2:pseudo-unlabelled}.
}
We perform this conversion on every tree when we analyze the structural properties of dependency trees in existing corpora in Chapter \ref{chap:transition}.
See Section \ref{sec:analysis:settings} for details.

\section{Left-corner Parsing}
\label{sec:bg:left-corner}
In this section we describe left-corner parsing and summarize related issues, e.g., its relevance to psycholinguistic studies.
Previous studies on left-corner parsing have focused only on a (probabilistic) CFG;
We will extend the discussion in this section for dependency grammars in later chapters.
In Chapter \ref{chap:transition}, we extend the idea into transition-based dependency parsing while in Chapter \ref{chap:induction}, we further extend the algorithm with efficient tabulation (dynamic programming).

A somewhat confusing fact about left-corner parsing is that there exist two variants of very different algorithms, called arc-standard and arc-eager algorithms.
The arc-standard left-corner parsing has been appeared first in the programming language literature \cite{4569645,Aho:1972:TPT:578789} and later extended for natural language parsing for improving efficiency \cite{Nederhof:1993:GLP:976744.976780} or expanding contexts captured by the model \cite{manning2000,henderson:2004:ACL}.
In the following we do {\it not} discuss these arc-standard algorithms, and only focus on the arc-eager algorithm, which has its origin in psycholinguistics \cite{Cognitive:MentalModels,abney91memory}\footnote{\newcite{Cognitive:MentalModels} introduced his left-corner parser as a cognitively plausible human parser but it has been pointed out that his parser is actually not arc-eager but arc-standard \cite{conf/coling/Resnik92}, which is (at least) not relevant to a human parser.} rather than in computer science.

\REVISE{
Left-corner parsing is closely relevant to the notion of center-embedding, a kind of branching pattern, which we characterize formally in Section \ref{sec:bg:embedding}.
We then introduce the idea of left-corner parsing through a parsing strategy in Section \ref{chap:2:left-corner-strategy} for getting intuition into parser behavior.
During Sections \ref{sec:bg:left-corner-pda} -- \ref{sec:bg:anothervariant}, we discuss the push-down automata (PDAs), a way for implementing the strategy as a parsing algorithm.
While previous studies on the arc-eager left-corner PDAs pay less attention on its theoretical properties beyond its asymptotic behavior,
in Section \ref{sec:bg:prop-left-corner}, we present a detailed, thorough analysis on the properties of the presented PDA as it plays an essential role in our exploration in the following chapters.
Although we carefully design the left-corner PDA as the realization of the presented strategy, as we see later, this algorithm differs from the one previously explained as the left-corner PDAs in the literature \cite{conf/coling/Resnik92,conf/acl/Johnson98}.
This difference is important for us.
In Section \ref{sec:bg:anothervariant} we discuss why this discrepancy occurs, as well as why we do not take the standard formalization.
Finally in Section \ref{sec:bg:psycho} we summarize the psycholinguistic relevance of the presented algorithms.
}


\subsection{Center-embedding}
\label{sec:bg:embedding}
We first define some additional notations related to CFGs that we introduced in Section \ref{sec:2:cfg}.
Let us assume a CFG $G=(N,\Sigma,P,S)$.
Then each symbol used below has the following meaning:
\begin{itemize}
 \item $A,B,C,\cdots$ are nonterminal symbols;
 \item $v,w,x,\cdots$ are strings of terminal symbols, e.g., $v \in \Sigma^*$;
 \item $\alpha,\beta,\gamma,\cdots$ are strings of terminal or nonterminal symbols, e.g., $\alpha \in (N \cup \Sigma)^*$.
\end{itemize}
In the following, we define the notion of center-embedding using left-most {\it derives} relation $\Rightarrow_{lm}$ though it is also possible to define with right-most one.
$\Rightarrow_{lm}^*$ denotes derivation in zero ore more steps while
$\Rightarrow_{lm}^+$ denotes derivation in one or more steps.
$\alpha \Rightarrow_{lm}^* \beta$ means $\beta$ can be derived from $\alpha$ by applying a list of rules in left-most order (always expanding the current left-most nonterminal).
In this order, the derivation with nonterminal symbols followed by terminal symbols, i.e., $S \Rightarrow_{lm}^+ \alpha A v$ does not appear.

For simplicity, we assume the CFG is in CNF.
It is possible to define center-embedding for general CFGs but notations are more involved, and it is sufficient for discussing our extension for dependency grammars.

Center-embedding can be characterized by the specific branching pattern found in a CFG parse, which we define precisely below.
We note that the notion of center-embedding could be defined in a different way.
In fact, as we describe later, the existence of several variants of arc-eager left-corner parsers is relevant to this arbitrariness for the definition of center-embedding.
We postpone the discussion of this issue until Section \ref{sec:bg:anothervariant}.
\begin{newdef}
 \label{def:bg:embedding}
 A CFG parse involves center-embedding if the following derivation is found in it:
  \begin{equation*}
   S \Rightarrow_{lm}^* v \underline{A} \alpha \Rightarrow_{lm}^+ v w \underline{B} \alpha \Rightarrow_{lm} v w \underline{C} D \alpha \Rightarrow_{lm}^+ v w x D \alpha; ~~ |x| \geq 2,
  \end{equation*}
 where the underlined symbol indicates that that symbol is expanded by the following $\Rightarrow$.
 The condition $|x| \geq 2$ means the constituent rooted at $C$ must comprise of more than one word.
\end{newdef}

Figure \ref{fig:2:center-embedding-pattern} shows an example of the branching pattern.
The pattern always begins with right branching edges, which are indicated by $v \underline{A} \beta \Rightarrow_{lm}^+ v w \underline{B} \beta$.
Then the center-embedding is detected if some $B$ is found which has a left child that constitutes a span larger than one word (e.g., $C$).
The final condition of the span length means the embedded subtree (rooted at $C$) has a right child.
This {\it right $\rightarrow$ left $\rightarrow$ right} pattern is the distinguished branching pattern in center-embedding.

By detecting a series of these zig-zag patterns recursively, we can measure the {\it degree} of center-embedding in a given parse.
Formally,
\begin{newdef}
  \label{def:bg:embed-depth}
 If the following derivation is found in a CFG parse:
 \begin{align}
  S \Rightarrow_{lm}^* v \underline{A} \alpha &\Rightarrow_{lm}^+ v w_1 \underline{B_1} \alpha \Rightarrow_{lm}^+ v w_1 \underline{C_1} \beta_1 \alpha \nonumber\\ 
  &\Rightarrow_{lm}^+ v w_1 w_2 \underline{B_2} \beta_1 \alpha
  \Rightarrow_{lm}^+ v w_1 w_2 \underline{C_2} \beta_2 \beta_1 \alpha \nonumber\\
  &\Rightarrow_{lm}^+ \cdots \label{eq:bg:embed-depth} \\
  &\Rightarrow_{lm}^+ v w_1 \cdots w_{m'} \underline{B_{m'}} \beta_{{m'}-1} \cdots \beta_1 \alpha
  \Rightarrow_{lm}^+ v w_1 \cdots w_{m'} \underline{C_{m'}} \beta_{m'} \beta_{{m'}-1} \cdots \beta_1 \alpha \nonumber\\
  &\Rightarrow_{lm}^+ v w_1 \cdots w_{m'} x \beta_{m'} \beta_{{m'}-1} \cdots \beta_1 \alpha;~~|x| \geq 2. \nonumber,
 \end{align}
 the degree of center-embedding in it is the maximum value $m$ among all possible values of $m'$ (i.e., $m \geq m'$).
 \end{newdef}
Each line in Eq. \ref{eq:bg:embed-depth} corresponds to the detection of additional embedding, except the last line that checks the length of the most embedded constituent.
Figures \ref{fig:2:depth-1} and \ref{fig:2:depth-2} show examples of degree one and two parses, respectively.
These are the {\it minimal} parses for each degree, meaning that degree two occurs only when the sentence length $\geq$ 6.
Note that the form of the last transform in the first line (i.e., $v w_1 \underline{B_1} \alpha \Rightarrow_{lm}^+ v w_1 \underline{C_1} \beta_1 \alpha$) does not match to the one in Definition \ref{def:bg:embedding} (i.e., $v w \underline{B} \alpha \Rightarrow_{lm} v w \underline{C} D \alpha$).
This modification is necessary because the first left descendant of $B_1$ in Eq. \ref{eq:bg:embed-depth} is not always the starting point of further center-embedding.

\label{sec:bg:center-embed}
\begin{figure}[t]
 \centering
 \begin{minipage}[t]{.4\linewidth}
  \centering
  \begin{tikzpicture}[sibling distance=12pt]
   \tikzset{level distance=25pt}
   \Tree
   [.$S$ \edge[densely dashed];
   [.$A$
   $w$
   [.$B$
   [.$C$ \edge[roof]; {~~~~~~~} ]
   $D$
   ]
   ]
   ]
  \end{tikzpicture}
  \subcaption{Pattern of center-embedding}
  \label{fig:2:center-embedding-pattern}
 \end{minipage}
 \begin{minipage}[t]{.4\linewidth}
  \centering
  \begin{tikzpicture}[sibling distance=12pt]
   \tikzset{level distance=25pt}
   \Tree
   [.$S$
    [.$A'$ $a$ ]
     [.$B$
      [.$C$ [.$B'$ $b$ ] [.$C'$ $c$ ] ]
      [.$D$ $d$ ] ] ]
  \end{tikzpicture}
  \subcaption{Degree one parse}
  \label{fig:2:depth-1}
 \end{minipage}
 \begin{minipage}[t]{.4\linewidth}
  \centering
  \vspace{10pt}
  \begin{tikzpicture}[sibling distance=12pt]
   \tikzset{level distance=25pt}
   \Tree
   [.$S$
    [.$A''$ $a'$ ]
     [.$B_1$
      [.$C_1$
       [.$A'$ $a$ ]
        [.$B_2$
         [.$C_2$ [.$B'$ $b$ ] [.$C'$ $c$ ] ]
         [.$D_2$ $d$ ] ] ]
      [.$D_1$ $d'$ ] ] ]
  \end{tikzpicture}
  \subcaption{Degree two parse}
  \label{fig:2:depth-2}
 \end{minipage}
 \begin{minipage}[t]{.4\linewidth}
  \centering
  \vspace{58pt}
  \begin{tikzpicture}[sibling distance=12pt]
   \tikzset{level distance=25pt}
   \Tree
   [.$S$ [.$B$ [.$A'$ $a$ ] [.$C$ [.$B'$ $b$ ] [.$C'$ $c$ ] ] ] [.$D$ $d$ ] ]
  \end{tikzpicture}
  \subcaption{Not center-embedding}
  \label{fig:2:not-center-embedding}
 \end{minipage}
 \caption{A parse involves center-embedding if the pattern in (a) is found in it.
 (b) and (c) are the minimal patterns with degree one and two respectively.
 (d) is the symmetry of (b) but we regard this as not center-embedding.
 }
 \label{fig:2:define-center-embedding}
\end{figure}
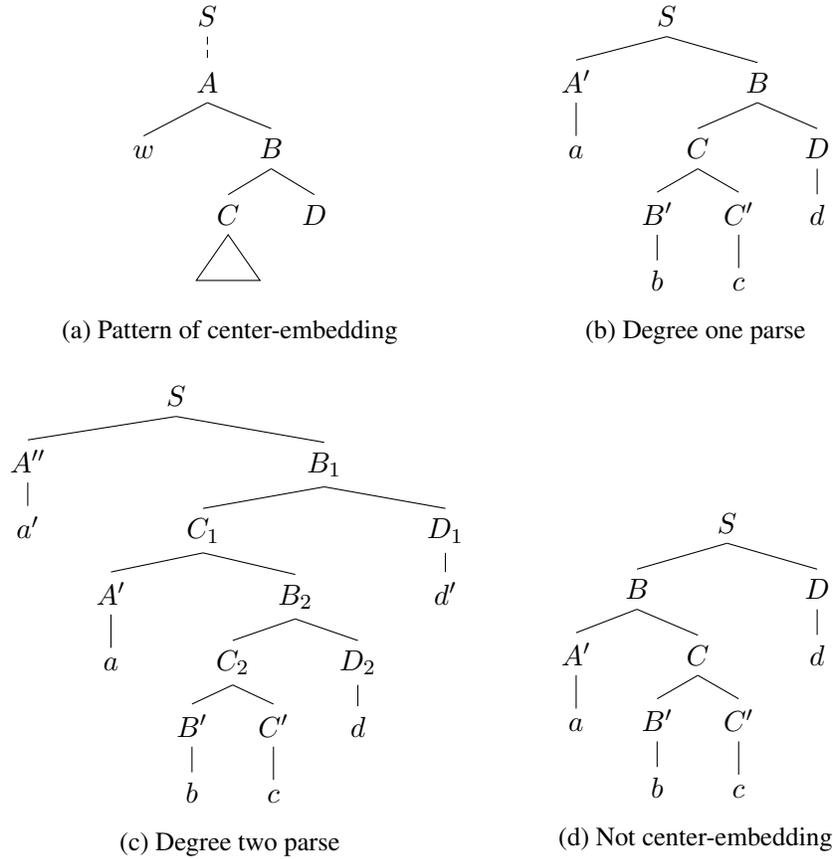

Other notes about center-embedding are summarized as follows:
\begin{itemize}
 \item It is possible to calculate the depth $m$ by just traversing every node in a parse once in a left-to-right, depth-first manner.
       The important observation for this algorithm is that the value $m'$ in Eq. \ref{eq:bg:embed-depth} is deterministic for each node, suggesting that we can fill the depth of each node top-down.
       We then pick up the maximum depth $m$ among the values found at terminal nodes.
 \item In the definitions above, the pattern of center-embedding always starts with a right directed edge.
       This means the similar, but opposite pattern found in Figure \ref{fig:2:not-center-embedding} is not center-embedding.
       Left-corner parsing that we will introduce next also distinguish only the pattern in Figure \ref{fig:2:depth-1}, not Figure \ref{fig:2:not-center-embedding}.
\end{itemize}

\subsection{Left-corner parsing strategy}
\label{chap:2:left-corner-strategy}

\begin{figure}[t]
 \centering
 \newcommand{\enum}[1]{{\small {\color{blue}{#1}}}}
 \newcommand{\mynode}[2]{\hspace{-8pt}{\enum{#1}} \hspace{-2pt}#2}
 \definecolor{g70}{gray}{0.7}
 \begin{minipage}[t]{.4\linewidth}
  \centering
  \begin{tikzpicture}[sibling distance=12pt]
   \tikzset{level distance=25pt}
   \Tree
   [.\mynode{10}{$S$}
    \edge node[left, pos=.3]{\enum{11}};
    [.\mynode{6}{$E$}
     \edge node[left, pos=.3]{\enum{7}};
     [.\mynode{2}{$F$}
      \edge node[left, pos=.3]{\enum{3}};
      [.\mynode{1}{$A$} {$a$} ]
      \edge node[right, pos=.3]{\enum{5}};
      [.\mynode{4}{$B$} {$b$} ]
     ]
     \edge node[right, pos=.3]{\enum{9}};
     [.\mynode{8}{$C$} {$c$} ]
    ]
    \edge node[right, pos=.3]{\enum{13}};
    [.\mynode{12}{$D$} {$d$} ]
   ]
  \end{tikzpicture}
  \subcaption{Left-branching}\label{subfig:left-branching}
 \end{minipage}
 \begin{minipage}[t]{.4\linewidth}
  \centering
  \begin{tikzpicture}[sibling distance=12pt]
   \tikzset{level distance=25pt}
   \Tree
   [.\mynode{2}{$S$}
    \edge node[left, pos=.3]{\enum{3}};
    [.\mynode{1}{$A$} {$a$} ]
    \edge node[right, pos=.3]{\enum{7}};
    [.\mynode{5}{$E$}
     \edge node[left, pos=.3]{\enum{6}};
     [.\mynode{4}{$B$} {$b$} ]
     \edge node[right, pos=.3]{\enum{11}};
     [.\mynode{9}{$F$}
      \edge node[left, pos=.3]{\enum{10}};
      [.\mynode{8}{$C$} {$c$} ]
      \edge node[right, pos=.3]{\enum{13}};
      [.\mynode{12}{$D$} {$d$} ]
     ]
    ]
   ]
  \end{tikzpicture}
  \subcaption{Right-branching}\label{subfig:right-branching}
 \end{minipage}
 \begin{minipage}[t]{.4\linewidth}
  \centering
  \vspace{10pt}
  \begin{tikzpicture}[sibling distance=12pt]
   \tikzset{level distance=25pt}
   \Tree
   [.\mynode{2}{$S$}
    \edge node[left, pos=.3]{\enum{3}};
    [.\mynode{1}{$A$} {$a$} ]
    \edge node[right, pos=.3]{\enum{11}};
    [.\mynode{9}{$E$}
     \edge node[left, pos=.3]{\enum{10}};
     [.\mynode{5}{$F$}
      \edge node[left, pos=.3]{\enum{6}};
      [.\mynode{4}{$B$} {$b$} ]
      \edge node[right, pos=.3]{\enum{8}};
      [.\mynode{7}{$C$} {$c$} ]
     ]
     \edge node[right, pos=.3]{\enum{13}};
     [.\mynode{12}{$D$} {$d$} ]
    ]
   ]
  \end{tikzpicture}
  \subcaption{Center-embedding}\label{subfig:embedding}
 \end{minipage}
 \begin{minipage}[t]{.4\linewidth}
  \centering
  \vspace{10pt}
  \begin{tikzpicture}[sibling distance=12pt]
   \tikzset{level distance=25pt}
   \Tree
   [.{$S$}
    \edge node[left, pos=.3]{};
    [.$A$ {$a$} ]
    \edge[g70] node[right, pos=.3]{};
    [.{\color{g70}{$E$}}
     \edge[g70] node[left, pos=.3]{};
     [.{$F$}
      \edge node[left, pos=.3]{};
      [.$B$ {$b$} ]
      \edge[g70] node[right, pos=.3]{};
      [.\color{g70}{$C$} \edge[g70]; \color{g70}{$c$} ]
     ]
     \edge[g70] node[right, pos=.3]{};
     [.\color{g70}{$D$} \edge[g70]; \color{g70}{$d$} ]
    ]
   ]
  \end{tikzpicture}
  
  \subcaption{Connected subtrees}\label{subfig:connected}
 \end{minipage}
 \caption{(a)--(c) Three kinds of branching structures with numbers on symbols and arcs showing the order of recognition with a left-corner strategy. (d) a partial parse of (c) using a left-corner strategy just after reading symbol $b$, with gray edges and symbols showing elements not yet recognized; The number of connected subtrees here is 2.}\label{fig:structures}
\end{figure}
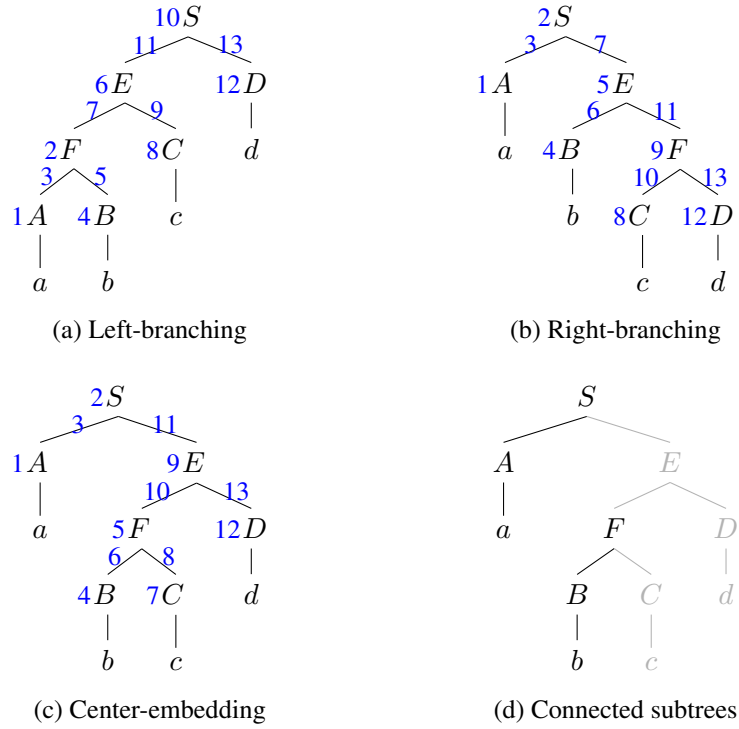

A parsing strategy is a useful abstract notion for characterizing the properties of a parser and gaining intuition into parser behavior.
Formally, it can be understood as a particular mapping from a CFG to the push-down automata that generate the same language \cite{nederhof-satta:2004:ACL1}.
Here we follow \newcite{abney91memory} and consider a parsing strategy as a specification of the order that each node and arc on a parse is recognized during parsing.
The corresponding push-down automata can then be understood as the device that provides the operational specification to realize such specific order of recognition, as we describe in Section \ref{sec:bg:left-corner-pda}.

We first characterize left-corner parsing with a parsing strategy to discuss its notable behavior for center-embedded structures.
The left-corner parsing strategy is defined by the following order of recognizing nodes and arcs on a parse tree:
\begin{enumerate}
 \item A node is recognized when the subtree of its first (left most) child has been recognized. \label{enum:bg:strategy-node}
 \item An arc is recognized when two nodes it connects have been recognized. \label{enum:bg:strategy-arc}
\end{enumerate}

We discuss the property of the left-corner strategy based on its behavior on three kinds of distinguished tree structures called left-branching, right-branching, and center-embedding, each shown in Figure \ref{fig:structures}.
The notable property of the left-corner strategy is that it generates disconnected tree fragments only on a center-embedded structure as shown in Figure \ref{fig:structures}.
Specifically, in Figure \ref{subfig:embedding}, which is center-embedding, after reading $b$ it reaches 6 but $a$ and $b$ cannot be connected at this point.
It does not generate such fragments for other structures; e.g., for the right-branching structure (Figure \ref{subfig:right-branching}), it reaches 7 after reading $b$ so $a$ and $b$ are connected by a subtree at this point.
The number of tree fragments grows as the degree of center-embedding increases.

As we describe later, the property of the left-corner strategy is appealing from a psycholinguistic viewpoint.
Before discussing this relevance, which we summarize in Section \ref{sec:bg:psycho}, in the following we will see how this strategy can be actually realized as the parsing algorithm first.

\subsection{Push-down automata}
\label{sec:bg:left-corner-pda}
We now discuss how the left-corner parsing strategy described above is implemented as a parsing algorithm, in particular as push-down automata (PDAs), the common device to define a parsing algorithm following a specific strategy.
\REVISE{
As we mentioned, this algorithm is not exactly the same as the one previously proposed as the left-corner PDA \cite{conf/coling/Resnik92,conf/acl/Johnson98}, which we summarize in Section \ref{sec:bg:anothervariant}.

PDAs assume a CFG, and specify how to build parses with that grammar given an input sentence.
Note that for simplicity we only present algorithms specific for CFGs in CNF, although both presented algorithms can be extended for general CFGs.
}

\paragraph{Notations}

We define a PDA as a tuple $(\Sigma, Q, q_{init}, q_{final}, \Delta)$ where $\Sigma$ is an alphabet of input symbols (words) in a CFG, $Q$ is a finite set of stack symbols (items), including the initial stack symbol $q_{init}$ and the final stack symbol $q_{final}$, and $\Delta$ is a finite set of transitions.
A transition has the form $\sigma_{1} \xmapsto{a} \sigma_{2}$ where $\sigma_1,\sigma_2 \in Q^*$ and $a \in \Sigma \cup \{ \varepsilon \}$; $\varepsilon$ is an empty string.
This can be applied if the stack symbols $\sigma_{1}$ are found to be the top few symbols of the stack and $a$ is the first symbol of the unread part of the input.
After such a transition, $\sigma_1$ is replaced with $\sigma_2$ and the next input symbol $a$ is treated as having been read.
If $a=\varepsilon$, the input does not proceed.
Note that our PDA does not explicitly have a set of states;
instead, we encode each state into stack symbols for simplicity as \newcite{DBLP:journals/corr/cs-CL-0404009}.

Given a PDA and an input sentence of length $n$, a {\it configuration} of a PDA is a pair ($\sigma$, $i$) where a stack $\sigma \in Q^*$ and $i$ is an input position $1 \leq i \leq n$, indicating how many symbols are read from the input.
The initial configuration is $(q_{initial}, 0)$ and the PDA {\it recognizes} a sentence if it reaches $(q_{final}, n)$ after a finite number of transitions.

\begin{figure*}[t]
 \centering
  \begin{tabular}[t]{|lll|} \hline
   Name & Transition & Condition \\ \hline
   {\sc Shift} & $\varepsilon \xmapsto{a} A$ & $A \rightarrow a \in P$ \\
   {\sc Scan}  & $A/B \xmapsto{a} A$ & $B \rightarrow a \in P$ \\
   {\sc Prediction} & $A \xmapsto{\varepsilon} B/C$ & $B \rightarrow A~C \in P$ \\
   {\sc Composition} & $A/B~C \xmapsto{\varepsilon} A/D$ & $B \rightarrow C~D \in P$ \\ \hline
 \end{tabular}
 \caption{
 \REVISE{
 The set of transitions in a push-down automaton that parses a CFG $(N,\Sigma,P,S)$ with the left-corner strategy.
 $a \in \Sigma$; $A,B,C,D\in N$.
 The initial stack symbol $q_{init}$ is the start symbol of the CFG $S$, while the final stack symbol $q_{final}$ is an empty stack symbol $\varepsilon$.
 }
 }
 \label{fig:bg:ourpda}
\end{figure*}

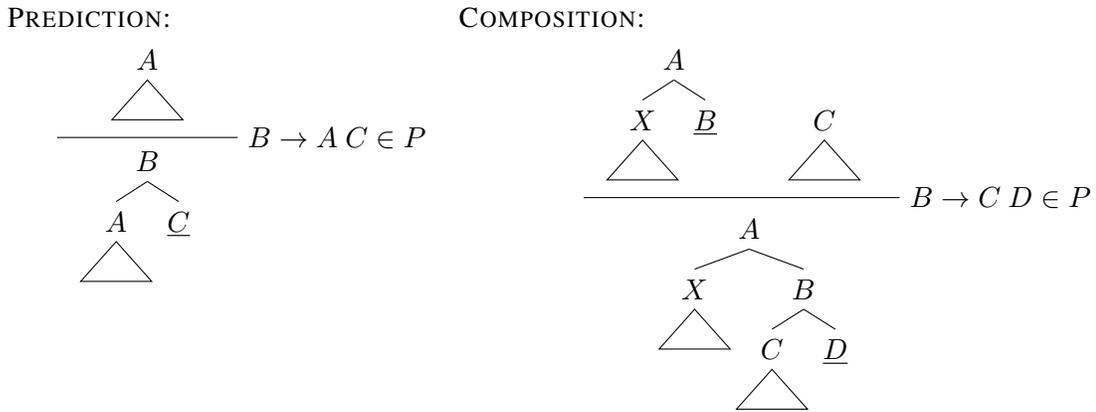
\begin{figure}[t]
 \begin{tikzpicture}[level distance=0.8cm]
  \begin{scope}[xshift=2cm,yshift=0cm]
   \node at (0, 0) [anchor=west] {\sc Prediction:};
   \begin{scope}[xshift=2cm, yshift=-0.7cm]
    \Tree
    [.$A$ \edge[roof]; {~~~~~~~~} ];
    \draw (-1.2, -0.9) -- +(2.4, 0) node[right] {$B \rightarrow A ~ C \in P$};
    \begin{scope}[xshift=0cm, yshift=-1.35cm]
     \Tree
     [.$B$ [.$A$ \edge[roof]; {~~~~~~~~} ] \underline{$C$} ];
    \end{scope}     
   \end{scope}
  \end{scope}
  \begin{scope}[xshift=8cm,yshift=0cm]
   \node at (0, 0) [anchor=west] {\sc Composition:};
   \begin{scope}[xshift=3cm, yshift=-0.7cm]
    \begin{scope}[xshift=0cm, yshift=0cm]
     \Tree
     [.$A$ [.$X$ \edge[roof]; {~~~~~~~~} ]  $\underline{B}$ ];
    \end{scope}
    \begin{scope}[xshift=2cm, yshift=-0.8cm]
     \Tree
     [.$C$ \edge[roof]; {~~~~~~~~} ];
    \end{scope}
    \draw (-1.2, -1.7) -- +(4.2, 0) node[right] {$B \rightarrow C ~ D \in P$};
    \begin{scope}[xshift=1cm, yshift=-2.25cm]
     \Tree
     [.$A$ [.$X$ \edge[roof]; {~~~~~~~~} ] [.$B$ [.$C$ \edge[roof]; {~~~~~~~~} ] \underline{$D$} ] ];
    \end{scope}     
   \end{scope}
  \end{scope}
 \end{tikzpicture}
 \caption{Graphical representations of inferences rules of {\sc Prediction} and {\sc Composition} defined in Figure \ref{fig:bg:ourpda}.
 An underlined symbol indicates that the symbol is predicted top-down.}
 \label{fig:2:pda-binary}
\end{figure}

\paragraph{PDA}


\REVISE{
We develop a left-corner PDA to achieve the recognition order of nodes and arcs by the left-corner strategy that we formulated in Section \ref{chap:2:left-corner-strategy}.
In our left-corner PDA, each stack symbol is either a nonterminal $A\in N$, or a pair of nonterminals $A/B$, where $A,B \in N$.
$A/B$ is used for representing an {\it incomplete} constituent, waiting for a subtree rooted at $B$ being substituted.
In this algorithm, $q_{init}$ is an empty stack symbol $\varepsilon$ while $q_{final}$ is the stack symbol $S$, the start symbol of a given CFG.

Figure \ref{fig:bg:ourpda} lists the set of transitions in this PDA.
{\sc Prediction} and {\sc Composition} are the key operations for achieving the left-corner strategy.
Specifically, {\sc Prediction} operation first recognizes a parent node of a subtree (rooted at $A$) bottom-up, and then predicts its sibling node top-down.
This is graphically explained in Figure \ref{fig:2:pda-binary}.
We notice that this operation just corresponds to the policy \ref{enum:bg:strategy-node} of the strategy about the order of recognizing nodes.\footnote{
\label{fn:bg:prediction}
Though the strategy postpones the recognition of the sibling node, we can interpret that the predicted sibling (i.e., $C$) by {\sc Prediction} is still not recognized.
It is recognized by {\sc Scan} or {\sc Composition}, which introduce the same node bottom-up and matches two nodes, i.e., the top-down predicted node and the bottom-up recognized node.
}
The policy \ref{enum:bg:strategy-arc} about the order of connecting nodes is also essential, and it is realized by another key operation of {\sc Composition}.
This operation involves two steps.
First, it performs the same prediction operation as {\sc Prediction} for the top stack symbol.
It is $C$ in Figure \ref{fig:bg:ourpda}, and the result is $B/D$, i.e., a subtree rooted at $B$ predicting the sibling node $D$.
It then connects this subtree and the second top subtree, i.e., $A/B$.
This is done by matching two identical nodes of different views, i.e., top-down predicted node $B$ in $A/B$ and bottom-up recognized node $B$ in $B/D$.
This matching operation is the key for achieving the policy \ref{enum:bg:strategy-arc}, which demands that two recognized nodes be connected immediately.
In {\sc Composition}, these two nodes are $A$, which is already recognized, and $B$, which is just recognized bottom-up by the first prediction step.\footnote{
As mentioned in footnote \ref{fn:bg:prediction}, we regard the predicted node $B$ in $A/B$ as not yet being recognized.
}

In the following, we distinguish two kinds of transitions in Figure \ref{fig:bg:ourpda}:
{\sc Shfit} and {\sc Scan} operations belong to {\it shift} transitions\footnote{We use small caps to refer to a specific action, e.g., {\sc Shift}, while ``shift'' refers to an action type.}, as they proceed the input position of the configuration.
This is not the case in {\sc Prediction} and {\sc Composition}, and we call them {\it reduce} transitions.

The left-corner strategy of \newcite{abney91memory} has the property that the maximum number of unconnected subtrees during enumeration equals the degree of center-embedding.
The presented left-corner PDA is an implementation of this strategy and essentially has the same property;
that is, its maximum stack depth during parsing is equal the degree of center-embedding of the resulting parse.
The example of this is shown next, while the formal discussion is provided in Section \ref{sec:bg:prop-left-corner}.
}


\begin{figure*}[t]
 \centering
  \begin{tabular}[t]{rlll} \hline
   Step & Action & Stack         & Read symbol \\ \hline
        &        & $\varepsilon$ & \\
   1    & {\sc Shift} & $A'$ & $a$ \\
   2    & {\sc Predict} & $S/B$ & \\
   3    & {\sc Shift} & $S/B~B'$ & $b$ \\
   4    & {\sc Predict} & ${\color{red}{S/B~C/C'}}$ & \\
   5    & {\sc Scan} & $S/B~C$ & $c$ \\
   6    & {\sc Composition} & $S/D$ & \\
   7    & {\sc Scan} & $S$ & $d$ \\ \hline
 \end{tabular}
 \caption{An example of parsing process by the left-corner PDA to recover the parse in Figure \ref{fig:2:depth-1} given an input sentence $a~b~c~d$.
 It is step 4 that occurs stack depth two after a reduce transition.}
 \label{fig:bg:pda-example}
\end{figure*}

\begin{figure}[t]
 \centering
  \begin{minipage}[t]{0.3\textwidth}
   \centering
   \begin{tabular}{ccc}
    $S$ &$\rightarrow$ & $A'~B$ \\
    $B$ &$\rightarrow$ & $C~D$ \\
    $C$ &$\rightarrow$ & $B'~C'$ \\
   \end{tabular}
 \end{minipage}
 \begin{minipage}[t]{0.3\textwidth}
  \centering
  \begin{tabular}{ccc}
   $A'$ &$\rightarrow$ & $a$ \\
   $B'$ &$\rightarrow$ & $b$ \\
   $C'$ &$\rightarrow$ & $c$ \\
   $D$ &$\rightarrow$ & $d$ \\
  \end{tabular}
 \end{minipage}
 \caption{A CFG that is parsed with the process in Figure \ref{fig:bg:pda-example}.}
 \label{fig:bg:cfg-embed}
\end{figure}

\paragraph{Example}

Figure \ref{fig:bg:pda-example} shows an example of parsing process given a CFG in Figure \ref{fig:bg:cfg-embed} and an input sentence $a~b~c~d$.
The parse tree contains one degree of center-embedding found in Figure \ref{fig:2:depth-1}, and this is illuminated in Figure \ref{fig:bg:pda-example} with the appearances of stack depth of two, in particular before reading symbol $c$, which exactly corresponds to the step 4 on the Figure \ref{subfig:embedding}.

\subsection{Properties of the left-corner PDA}
\label{sec:bg:prop-left-corner}

\REVISE{
In this section, we formally establish the connection between the left-corner PDA and the center-embeddedness of a parse.
The result is also essential when discussing the property of our extended algorithm for dependency grammars presented in Chapter \ref{chap:transition};
see Section \ref{sec:memorycost} for details.

The following lemmas describe the basic properties of the left-corner PDA, which will be the basis in the further analysis.

\begin{newlemma}
 \label{lemma:bg:shift-reduce}
 In a sequence of transitions to arrive the final configuration $(q_{final}, n)$ of the left-corner PDA (Figure \ref{fig:bg:ourpda}), shift (i.e., {\sc Shift} or {\sc Scan}) and reduce (i.e., {\sc Prediction} or {\sc Composition}) transitions occur alternately.
\end{newlemma}
\begin{proof}
 Reduce transitions are only performed when the top symbol of the stack is {\it complete}, i.e., of the form $A$.
 Then, since each reduce transition makes the top symbol of the stack incomplete, two consecutive reduce transitions are not applicable.
 Conversely, shift transitions make the top stack symbol complete.
 We cannot perform {\sc Scan} after some shift transition, since it requires an incomplete top stack symbol.
 If we perform {\sc Shift} after a shift transition, the top two stack symbols become complete, but we cannot combine these two symbols since the only way to combine two symbols on the stack is {\sc Composition}, while it requires the second top symbol to be incomplete.
\end{proof}

\begin{newlemma}
 \label{lemma:bg:after-reduce}
 In the left-corner PDA, after each reduce transition, every item remained on the stack is an incomplete stack symbol of the form $A/B$.
\end{newlemma}
\begin{proof}
 From Lemma \ref{lemma:bg:shift-reduce}, a shift action is always followed by a reduce action, and vice versa.
 We call a pair of some shift and reduce operations a {\it push} operation.
 In each push operation, a shift operation adds at most one stack symbol on the stack, which is always replaced with an incomplete symbol by the followed reduce transition.
 Thus after a reduce transition no complete symbol remains on the stack.
\end{proof}

We can see that transitions in Figure \ref{fig:bg:pda-example} satisfy these conditions.
Intuitively, the existence of center-embedding is indicated by the accumulated incomplete symbols on the stack, each of which corresponds to each line on the derivation in Eq. \ref{eq:bg:embed-depth}.
This is formally stated as the following theorem, which establishes the connection between the stack depth of the left-corner PDA and the degree of center-embedding.

\begin{mytheorem}
 \label{thoerem:bg:stack-depth}
 Given a CFG parse, its degree of center-embedding is equal to the maximum value of the stack depth after a reduce transition minus one for recognizing that parse on the left-corner PDA.
\end{mytheorem}
For example, for a CFG parse with one degree of center-embedding, the maximum stack depth after a reduce transition is two, which is indicated at step 4 in Figure \ref{fig:bg:pda-example}.
We leave the proof of this theorem in Appendix \ref{chap:app:analyze-pda}.

Note that Theorem \ref{thoerem:bg:stack-depth} says nothing about the stack depth after a shift transition, which generally is not equal to the degree of center-embedding.
We discuss this issue more when presenting the algorithm for dependency grammars; see Section \ref{sec:memorycost}.
}

\subsection{Another variant}
\label{sec:bg:anothervariant}
\begin{figure*}[t]
 \centering
  \begin{tabular}[t]{|lll|} \hline
   Name & Transition & Condition \\ \hline
   {\sc Shift} & $A \xmapsto{a} A{-}B$ & $B \rightarrow a \in P$ \\
   {\sc Scan}  & $A \xmapsto{a} \varepsilon$ & $A \rightarrow a \in P$ \\
   {\sc Prediction} & $A{-}B \xmapsto{\varepsilon} A{-}C~D$ & $C \rightarrow B~D \in P$ \\
   {\sc Composition} & $A{-}B \xmapsto{\varepsilon} C$ & $A \rightarrow B~C \in P$ \\ \hline
 \end{tabular}
 \caption{A set of transitions in another variant of the left-corner PDA appeared in Resnik (1992).
 $a \in \Sigma$; $A,B,C,D\in N$.
 Differently from the PDA in Figiure \ref{fig:bg:ourpda}, the initial stack symbol $q_{init}$ is $S$ while $q_{final}$ is an empty stack symbol $\varepsilon$.
 }
 \label{fig:bg:pda1}
\end{figure*}

\REVISE{
We now present another variant of the left-corner PDA appeared in the literature \cite{conf/coling/Resnik92,conf/acl/Johnson98}.
We will see that this algorithm has a different property with respect to the stack depth and the degree of center-embedding than Theorem \ref{thoerem:bg:stack-depth}.
In particular, this difference is relevant to the structures that are recognized as center-embedding for the algorithm, which has not been precisely discussed so far;
\newcite{journals/coling/SchulerAMS10} give comparison of two algorithms but from a different perspective.

Figure \ref{fig:bg:pda1} shows the list of possible transitions in this variant.
The crucial difference between two PDAs is in the form of initial and final stack symbols.
That is, in this PDA the initial stack symbol $q_{initial}$ is $S$, while $q_{final}$ is an empty symbol $\varepsilon$, which are opposite in the PDA that we discussed so far (Section \ref{sec:bg:left-corner-pda}).

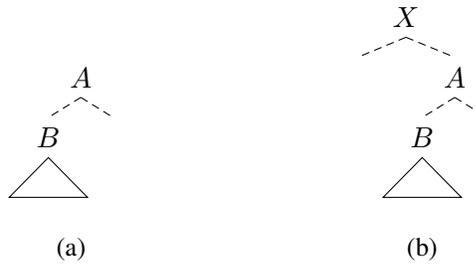
\begin{figure}[t]
 \centering
 \begin{minipage}[t]{0.3\textwidth}
  \centering
  \begin{tikzpicture}[level distance=0.8cm]
    \Tree
    [.$A$
     \edge[densely dashed];
     [.$B$ \edge[roof]; {~~~~~~~~~} ]
     \edge[densely dashed]; {\color{white}{A}}
    ]
  \end{tikzpicture}
  \subcaption{}\label{fig:bg:a-b:a}
 \end{minipage}
 \begin{minipage}[t]{0.3\textwidth}
  \centering
  \begin{tikzpicture}[level distance=0.8cm]
    \Tree
    [.$X$
    \edge[densely dashed]; {\color{white}{A}}
    \edge[densely dashed];
    [.$A$
     \edge[densely dashed];
     [.$B$ \edge[roof]; {~~~~~~~~~} ]
     \edge[densely dashed]; {\color{white}{A}}
    ]
    ]
  \end{tikzpicture}
  \subcaption{}\label{fig:bg:a-b:b}
 \end{minipage}
 \caption{Stack symbols of the left-corner PDA of Figure \ref{fig:bg:pda1}.
 Both trees correspond to symbol $A{-}B$ where $A$ is the current goal while $B$ is the recognized nonterminal.
 Note that $A$ may be a right descendant of another nonterminal (e.g., $X$), which dominates a larger subtree.}
 \label{fig:bg:a-b}
\end{figure}

Also in this variant, the form of stack symbols is different.
Instead of $A/B$, which represents a subtree waiting for $B$, it has $A{-}B$, which means that $B$ is the {\it left-corner} in a subtree rooted at $A$, and has been already recognized.
In other words, $A$ is the current goal, which the PDA tries to build, while $B$ represents a finished subtree.
This is schematically shown in Figure \ref{fig:bg:a-b:a}.

Parsing starts with $q_{initial}=S$, which immediately changes to $S{-}A$ where $A\rightarrow a$, and $a \in \Sigma$ is the initial token of the sentence.
{\sc Prediction} is similar to the one in our variant:
It expands the currently recognized structure, and also predicts the sibling symbol (i.e., $D$), which becomes a new goal symbol.
{\sc Composition} looks very different, but has the similar sense of transition.
In the symbol $A/B$, $A$ is not limited to $S$, in which case $A$ is some right descendant of another nonterminal, as depicted in Figure \ref{fig:bg:a-b:b}.
The sense of {\sc Composition} in Figure \ref{fig:bg:pda1} is that we finish recognition of the left subtree of $A$ (i.e., the tree rooted at $B$) and change the goal symbol to $C$, the sibling of $B$.
If we consider this transition in the form of Figure \ref{fig:bg:a-b:b}, it looks similar to the one in Figure \ref{fig:bg:ourpda};
that is, the corresponding transition in our variant is $X/A~B \xmapsto{\varepsilon} C$.
Instead, in the current variant, the root nonterminal of a subtree $X$ is not kept on the stack, and the goal symbol is moved from top to bottom.
This is the reason why the final stack symbol $q_{final}$ is empty.
The final goal for the PDA is always the preterminal for the last token of the sentence, which is then finally removed by {\sc Scan}.

\paragraph{Example}

\begin{figure}[t]
 \centering
  \begin{tabular}[t]{rlll} \hline
   Step & Action & Stack         & Read symbol \\ \hline
        &        & $S$ & \\
   1    & {\sc Shift} & $S{-}A'$ & $a$ \\
   2    & {\sc Composition} & $B$ & \\
   3    & {\sc Shift} & $B{-}B'$ & $b$ \\
   4    & {\sc Predict} & $B{-}C~C'$ & \\
   5    & {\sc Scan} & $B{-}C$ & $c$ \\
   6    & {\sc Composition} & $D$ & \\
   7    & {\sc Scan} & $\varepsilon$ & $d$ \\ \hline
 \end{tabular}
 \caption{Parsing process of the PDA in Figure \ref{fig:bg:pda1} to recover the parse in Figure \ref{fig:2:depth-1} given the CFG in Figure \ref{fig:bg:cfg-embed} and an input sentence $a~b~c~d$. The stack depth keeps one in every step after a shift transition.}
 \label{fig:bg:pda-example:2}
\end{figure}

This PDA has slightly different characteristics in terms of stack depth and the degree of center-embedding, which we point out here with some examples.
In particular, it regards the parse in Figure \ref{fig:2:not-center-embedding} as singly (degree one) center-embedded, while the one in Figure \ref{fig:2:depth-1} as not center-embedded.
That is, it has just the opposite properties to the PDA that we discussed in Section \ref{sec:bg:left-corner-pda}.

We first see how the situation changes for the CFG that we gave an example in Figure \ref{fig:bg:pda-example}, which analyzed the parse in Figure \ref{fig:2:depth-1}.
See Figure \ref{fig:bg:pda-example:2}.
Contrary to our variant, this PDA has the property that its stack depth after some {\it shift} transitions increases as the degree of center-embedding increases.\footnote{
This again contrasts with our variant (Theorem \ref{thoerem:bg:stack-depth}).
This is because in the PDA in \ref{fig:bg:pda1}, new stack element is introduced with a reduce transition (i.e., {\sc Prediction}), and center-embedding is detected with the followed {\sc Shift}, which does not decrease the stack depth.
In our variant, on the other hand, new stack element is introduced by {\sc Shift}.
Center-embedding is detected if this new element remains on the stack after a reduce transition (by {\sc Prediction}).
}
In this case, these are steps 3, 5, and 7, all of which has a stack with only one element.
The main reason why it does not increase the stack depth is in the first {\sc Composition} operation, which changes the stack symbol to $B$.
After that, since the outside structure of $B$ is already processed, the remaining tree looks just like left-branching, which the left-corner PDA including this variant processes without increasing the stack depth.

\begin{figure}[t]
 \centering
  \begin{tabular}[t]{rlll} \hline
   Step & Action & Stack         & Read symbol \\ \hline
        &        & $S$ & \\
   1    & {\sc Shift} & $S{-}A'$ & $a$ \\
   2    & {\sc Prediction} & $S{-}B~C$ & \\
   3    & {\sc Shift} & {\color{red}{$S{-}B~C{-}B'$}} & $b$ \\
   4    & {\sc Composition} & $S{-}B~C'$ & \\
   5    & {\sc Scan} & $S{-}B$ & $c$ \\
   6    & {\sc Composition} & $D$ & \\
   7    & {\sc Scan} & $\varepsilon$ & $d$ \\ \hline
 \end{tabular}
 \caption{Parsing process of the PDA in Figure \ref{fig:bg:pda1} to recover the parse in Figure \ref{fig:2:not-center-embedding}.
 The stack depth after a shift transition increases at step 3.}
 \label{fig:bg:pda-example:3}
\end{figure}

On the other hand, for the parse in Figure \ref{fig:2:not-center-embedding}, this PDA increases the stack depth as simulated in Figure \ref{fig:bg:pda-example:3}.
At step 2, the PDA introduces new goal symbol $C$, which remains on the stack after the followed {\sc Shift}.
This is the pattern of transitions with which this PDA increase its stack depth, and it occurs when processing the zig-zag patterns starting from left edges, not right edges as in our variant.

\paragraph{Discussion}

We have pointed out that there are two variants of (arc-eager) left-corner PDAs, which suffer from slightly different conditions under which their stack depth increases.
From an empirical point of view, the only common property is its asymptotic behavior.
That is, both linearly increase the stack depth as the degree of center-embedding increases.
The difference is rather subtle, i.e., the condition of beginning center-embedding (left edges or right edges).

Historically, the variant introduced in this section (Figure \ref{fig:bg:pda1}) has been thought as the realization of the left-corner PDA \cite{conf/coling/Resnik92,conf/acl/Johnson98}.
However, as we have seen, if we base development of the algorithm on the parsing strategy (Section \ref{chap:2:left-corner-strategy}), our variant can be seen as the correct implementation of it, as only our variant preserves the transparent relationship between the stack depth and the disconnected trees generated during enumeration by the strategy.

\newcite{conf/coling/Resnik92} did not design the algorithm based on the parsing strategy, but from an existing arc-standard left-corner PDA \cite{4569645,Cognitive:MentalModels}, which also accepts an empty stack symbol as the final configuration.
His main argument is that the arc-eager left-corner PDA can be obtained by introducing a {\sc Composition} operation, which does not exist in the arc-standard PDA.
Interestingly, there is another variant of the arc-standard PDA \cite{Nederhof:1993:GLP:976744.976780}, which instead accepts the $S$ symbol\footnote{To be precise, the stack item of \newcite{Nederhof:1993:GLP:976744.976780} is a dotted rule like [S$\rightarrow$ NP$\bullet$VP] and parsing finishes with an item of the form [S$\rightarrow\alpha\bullet$] with some $\alpha$.}. If we extend this algorithm by introducing {\sc Composition}, we get very similar algorithm to the one we presented in Section \ref{sec:bg:left-corner-pda} with the same stack depth property.

Thus, we can conclude that Resnik's argument is correct in that a left-corner PDA can be {\it arc-eager} by adding composition operations, but depending on which arc-standard PDA we employ as the basis, the resulting arc-eager PDA may have different characteristics in terms of stack depth.
In particular, the initial and final stack configurations are important.
If the based arc-standard PDA accepts the empty stack symbol as in \newcite{4569645}, the corresponding arc-eager PDA regards the pattern beginning with right edges as center-embedding.
The direction becomes opposite if we start from the PDA that accepts the non-empty stack symbol as in \newcite{Nederhof:1993:GLP:976744.976780}.

Our discussion in the following chapters is based on the variant we presented in Section \ref{sec:bg:left-corner-pda}, which is relevant to \newcite{Nederhof:1993:GLP:976744.976780}.
However, we do not make any claims such that this algorithm is superior to the variant we introduced in this section.
Both are correct arc-eager left-corner PDAs, and we argue that the choice is rather arbitrary.
This arbitrariness is further discussed next, along with the limitation of both approaches as the psycholinguistic models.

Finally, our variant of the PDA in Section \ref{sec:bg:left-corner-pda} has been previously presented in \newcite{journals/coling/SchulerAMS10} and \newcite{vanschijndel-schuler:2013:NAACL-HLT}, though they do not mention the relevance of the algorithm to the parsing strategy.
Their main concern is in the psychological plausibility of the parsing model, and they argue that this variant is more plausible due to its inherent bottom-up nature (not starting from the predicted $S$ symbol).
They do not point out the difference of two algorithms in terms of the recognized center-embedded structures as we discussed here.
}

\subsection{Psycholinguistic motivation and limitation}
\label{sec:bg:psycho}
We finally summarize left-corner parsing and relevant theories in the psycholinguistics literature.
One well known observation about human language processing is that the sentences with multiply center-embedded constructions are quite difficult to understand, while left- and right-branching constructions seem to cause no particular difficulty \cite{Miller1963-MILFMO,gibson1998dlt}.
\eenumsentence{\item[a.]\# The reporter [who the senator [who Mary met] attacked] ignored the president.\label{sent:2:embedding}
\item[b.]Mary met the senator [who attacked the reporter [who ignored the president]].
}
The sentence (\ref{sent:2:embedding}a) is an example of a center-embedded sentence while (\ref{sent:2:embedding}b) is a right-branching sentence.
This observation matches the behavior of left-corner parsers, which increase its stack depth in processing center-embedded sentences only, as we discussed above.

It has been well established that center-embedded structures are a generally difficult construction \cite{gibson1998dlt,Chen2005144}, and this connection between left-corner parsers and human behaviors motivated researchers to investigate left-corner parsers as an approximation of human parsers \cite{Roark:2001:RPP:933637,journals/coling/SchulerAMS10,vanschijndel-schuler:2013:NAACL-HLT}.
The most relevant theory in psycholinguistics that accounts for the difficulty of center-embedding is the one based on the {\it storage cost} \cite{Chen2005144,COGS:COGS1067,nakatani2008}, i.e., the cost associated with keeping incomplete materials in memory.\footnote{
Another explanation for this difficulty is retrieval-based accounts such as the {\it integration cost} \cite{gibson1998dlt,Gibson2000The-dependency-} in the dependency locality theory.
We do not discuss this theory since the connection between the integration cost and the stack depth of left-corner parsers is less obvious, and it has been shown that the integration cost itself is not sufficient to account for the difficulty of center-embedding \cite{Chen2005144,COGS:COGS1067}.
}
For example, \newcite{Chen2005144} and \newcite{COGS:COGS1067} find that people read more center-embedded sentences more slower than less center-embedded sentences, in particular when entering new embedded clauses, through their reading time experiments of English and Japanese, respectively.
This observation suggests that there exists some sort of {\it storage} component in human parsers, which is consumed when processing more nested structures, as in the stack of left-corner parsers.

However, as we claimed in Section \ref{sec:intro:notabout}, our main goal in this thesis is not to deepen understanding of the mechanism of human sentence processing.
One reason of this is that there are some discrepancies between the results in the articles cited above and the behavior of our left-corner parser, which we summarize below.
Another, and perhaps more important limitation of left-corner parsers as an approximation of human parsers is that it cannot account for the sentence difficulties not relevant to center-embedding, such as the garden path phenomena:
\enumsentence{\# The horse raced past the barn fell,} \label{sent:2:gardenpath}
in which people feel difficulty at the last verb {\it fell}.
Also there exist some cases in which nested structures {\it do} facilitate comprehension, known as {\it anti-locality} effects \cite{konieczny2000,10.2307/4490268}.
These can be accounted for by another, non-memory-based theory called expectation-based account \cite{Hale2001,Levy20081126}, which is orthogonal in many aspects to the memory-based account \cite{jaeger2011language-gsc}.
We do not delve into those problems further and in the following we focus on the issues of the former mentioned above, which is relevant to our definition of center-embedding as well as the choice of the variant of left-corner PDAs (Section \ref{sec:bg:anothervariant}).

\begin{figure}[t]
 \centering
 {\small
 \begin{tikzpicture}[sibling distance=12pt]
  \tikzset{level distance=25pt}
   \Tree
   [.S
    [.NP 
     [.NP \edge[roof]; {The reporter} ]
     [.$\bar{\textrm{S}}$
      [.WP who ]
      [.S
       [.NP 
        [.NP
         [.DT the ]
         [.NP senator ]
        ]
        [.$\bar{\textrm{S}}$
         [.WP who ]
         [.S [.NP Mery ] [.VP met ] ]
        ]
       ]
       [.VP attacked ]
      ]
     ]
    ]
    [.VP \edge[roof]; {ignored the president} ]
   ]
 \end{tikzpicture}
 }
 \caption{The parse of the sentence (\ref{sent:2:embedding}a).}
 \label{fig:bg:reporter-parse}
\end{figure}

\paragraph{Discrepancies in definitions of center-embedding}
We argue here that sometimes the stack depth of our left-corner parser {\it underestimates} the storage cost for some center-embedded sentences in which linguists predict greater difficulty for comprehension.
More specifically, though \newcite{Chen2005144} claims the sentence (\ref{sent:2:embedding}a) is {\it doubly} center-embedded, our left-corner parser recognizes this is {\it singly} center-embedded, as its parse does not contain the zig-zag pattern in Figure \ref{fig:2:depth-2} (but in Figure \ref{fig:2:depth-1}).
Figure \ref{fig:bg:reporter-parse} shows the parse.
This discrepancy occurs due to our choice for the definition of center-embedding discussed in Section \ref{sec:bg:embedding}.
In our definition (Definition \ref{def:bg:embedding}), center-embedding always starts with a right edge.
In the case like Figure \ref{fig:bg:reporter-parse}, two main constituents ``The reporter ... attacked'' and ``ignored the president'' are connected with a left edge, and this is the reason why our definition of center-embedding as well as our left-corner parser predicts that this parse is singly nested.


Here we note that although our left-corner parser underestimates the center-embeddedness in some cases, it correctly estimates the relative difficulty of sentence (\ref{sent:2:embedding}a) compared to less nested sentences below.
\eenumsentence{
\item[a.]The senator [who Mary met] ignored the president.\label{sent:bg:sigle}
\item[b.]The reporter ignored the president.}
The problem is that both sentences above are recognized as not center-embedded although some literature in psycholinguistics (e.g., \newcite{Chen2005144}) assumes it is singly center-embedded.

{\Ja{書記が}}

\REVISE{
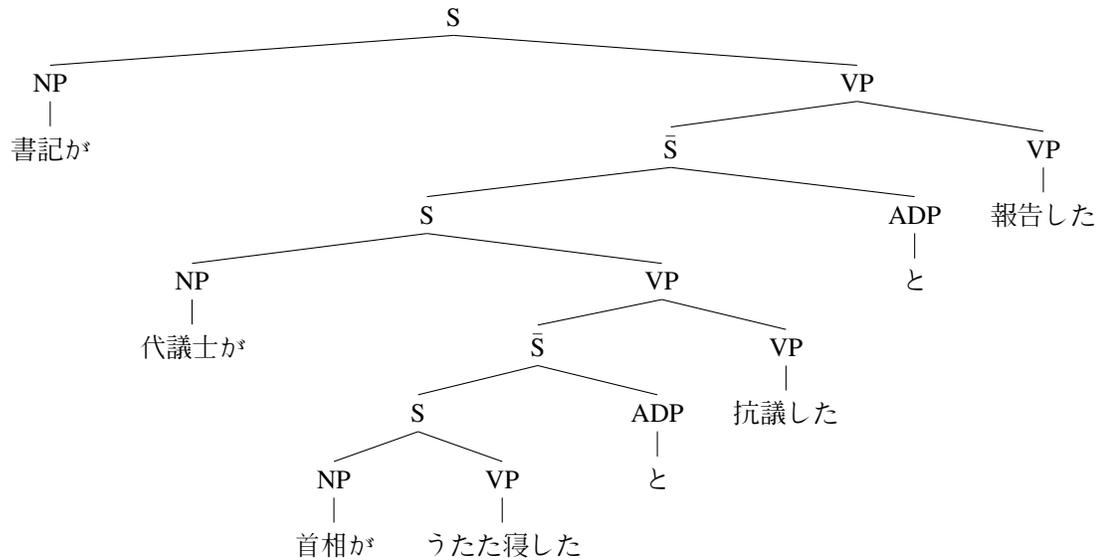
\begin{figure}[t]
 \centering
 {\small
 \begin{tikzpicture}[sibling distance=12pt]
  \tikzset{level distance=25pt}
   \Tree
     [.S
      [.NP {\Ja{書記が}} ]
      [.VP
       [.$\bar{\textrm{S}}$
        [.S
         [.NP {\Ja{代議士が}} ]
         [.VP
          [.$\bar{\textrm{S}}$
           [.S
            [.NP {\Ja{首相が}} ]
            [.VP {\Ja{うたた寝した}} ]
           ]
           [.ADP {\Ja{と}} ]
          ]
          [.VP {\Ja{抗議した}} ]
         ]
        ]
        [.ADP {\Ja{と}} ]
       ]
       [.VP {\Ja{報告した}} ]
      ]
     ]
 \end{tikzpicture}
 }
 \caption{The parse of the sentence (\ref{sent:bg:japanese-nested}).}
 \label{fig:bg:shoki-parse}
\end{figure}
}

However, this mismatch does not mean that our left-corner parser always underestimates the predicted center-embeddedness by linguists.
\REVISE{
We give further examples below to make explicit the points.
\begin{itemize}
 \item As the example below \cite{nakatani2008} indicates, often in the parse of a Japanese sentence the degree of center-embedding matches the prediction by linguists.
 \eenumsentence{\item[\#]\Ja{書記が [代議士が [首相が うたた寝した と] 抗議した と] 報告した}\\
                         secretary-nom [congressman-nom [prime minister-nom dozed comp] protested comp] reported\\
                         The secretary reported that the congressman protested that the prime minister had dozed.\label{sent:bg:japanese-nested}
       }
       The parse is shown in Figure \ref{fig:bg:shoki-parse}, which contains the pattern in Figure \ref{fig:2:depth-2}.
       This is because two constituents ``{\Ja{書記が}}'' and ``{\Ja{代議士が ... 報告した}}'' are connected with a right edge in this case.
 \item This observation may suggest that our left-corner parser always underestimates the degree of center-embedding for specific languages, e.g., English.
       However, this is not generally true since we can make an English example in which two predictions are consistent, as in Japanese sentence, e.g., by making the sentence (\ref{sent:2:embedding}) as a large complement as follows:
       \enumsentence{\# He said [the reporter [who the senator [who Mary met] attacked] ignored the president].\label{sent:2:embedding:match}}
       In the example, ``He said'' does not cause additional embedding, as the constituent ``the reporter ... president'' is not embedded internally, and thus linguists predict that this is still doubly center-embedded.
       On the other hand, the parse now involves the pattern in Figure \ref{fig:2:depth-2}, suggesting that the predictions are consistent in this case.
\end{itemize}

The point is that since our left-corner parser (PDA) only regards the pattern starting from right edges as center-embedding, it underestimates the prediction by linguists when the direction of outermost edge in the parse is left, as in Figure \ref{fig:bg:reporter-parse}.
Though there might be some language specific tendency (e.g., English sentences might be often underestimated) we do not make such claims here, since the degree of center-embedding in our definition is determined purely in terms of the tree structure, as indicated by sentence (\ref{sent:2:embedding:match}).
We perform the relevant empirical analysis on treebanks in Chapter \ref{chap:transition}.

From the psycholinguistics viewpoint, this discrepancy might make our empirical studies in the following chapters less attractive.
However, as we noted in Section \ref{sec:intro:notabout}, our central motivation is rather to capture the universal constraint that every language may suffer from, though is computationally tractable, which we argue does not necessarily reflect correctly the difficulties reported by psycholinguistic experiments.

As might be predicted, the results so far become opposite if we employ another variant of PDA that we formulated in Section \ref{sec:bg:anothervariant}, in which the stack depth increases on the pattern starting from left edges, as in Figure \ref{fig:2:not-center-embedding}.
This variant of PDA estimates that the degree of center-embedding on the parse in Figure \ref{fig:bg:reporter-parse} will be two, while that of Figure \ref{fig:bg:shoki-parse} will be one.
This highlights that the reason of the observed discrepancies is mainly due to the computational tractability:
We can develop a left-corner parser so that its stack depth increases on center-embedded structures indicated by some zig-zag patterns, which are always starting from left (the variant of \newcite{conf/coling/Resnik92}), or right (our variant).
However, from an algorithm perspective, it is hard to allow both left and right directions, and this is the assumption of psycholinguists.

Again, we do argue that our choice for the variant of the left-corner PDA is rather arbitrary.
This choice may impact the empirical results in the following chapters, where we examine the relationships between parses on the treebanks and the incurred stack depth.
In the current study, we do not empirically compare the behaviors of two PDAs, which we leave as one of future investigations.

}

\section{Learning Dependency Grammars}
\label{sec:2:learning}

In this section we will summarize several basic ideas about learning and parsing of dependency grammars.
The dependency model with valence \cite{klein-manning:2004:ACL} is the most popular model for unsupervised dependency parsing, which will be the basis of our experiments in Chapter \ref{chap:induction}.
We formalize this model in Section \ref{sec:2:dmv} as a special instance of split bilexical grammars (Section \ref{sec:2:sbg}).
Before that, this section first reviews some preliminaries on a learning mechanism, namely, probabilistic context-free grammars (Section \ref{sec:2:pcfg}), chart parsing (Section \ref{sec:2:cky}), and parameter estimation with the EM algorithm (Section \ref{sec:2:em}).

\subsection{Probabilistic context-free grammars}
\label{sec:2:pcfg}

Here we start the discussion with probabilistic context-free grammars (PCFGs) because they will allow use of generic parsing and parameter estimation methods that we describe later.
However, we note that for the grammar to be applied these algorithms, the grammar should not necessarily be a PCFG. 
We will see that in fact split bilexical grammars introduced later cannot always be formulated as a correct PCFG.
Nevertheless, we begin this section with the discussion of PCFGs mainly because:
\begin{itemize}
 \item the ideas behind chart parsing algorithms (Sections \ref{sec:2:cky} and \ref{sec:2:em}) can be best understood with a simple PCFG; and
 \item we can obtain a natural generalization of these algorithms to handle a special class of grammars (not PCFGs) including split bilexical grammars.
       We will describe the precise condition for a grammar to be applied these algorithms later.
\end{itemize}

Formally a PCFG is a tuple $G=(N,\Sigma,P,S,\theta)$ where $(N,\Sigma,P,S)$ is a CFG (Section \ref{sec:2:cfg}) and $\theta$ is a vector of non-negative real values indexed by production rules $P$ such that
\begin{equation}
 \sum_{A\rightarrow \beta \in P_A} \theta_{A \rightarrow \beta} = 1, \label{eqn:2:normalize}
\end{equation}
where $P_A \subset P$ is a collection of rules of the form $A \rightarrow \beta$.
We can interpret $\theta_{A \rightarrow \beta}$ as the conditional probability of choosing a rule $A \rightarrow \beta$ given that the nonterminal being expanded is $A$.

With this model, we can calculate the score (probability) of a parse as the product of rules appeared on that.
Let a parse be $z$ that contains rules $r_1,r_2,\cdots$;
then the probability of $z$ under the given PCFG is
\begin{align}
 P(z|\theta) &= \prod_{r_i \in z} \theta_{r_i} \\
             &= \prod_{r \in P} \theta_r^{f(r,z)}, \label{eqn:2:pz}
\end{align}
where $f(r,z)$ is the number of occurrences of a rule $r$ in $z$.

We can also interpret a PCFG as a directed graphical model that defines a distribution over CFG parses.
The generative process is described as follows:
Starting at $S$ (start symbol), it chooses to apply a rule $S\rightarrow \beta$ with probability $\theta_{S\rightarrow \beta}$;
$\beta$ defines the symbols of the children, which are then expanded recursively to generate their subtrees.
This process stops when all the leaves of the tree are terminals.
Note that this process also generates a sentence $x$, which is obtained by concatenating every terminal symbol in $z$, meaning that:
\begin{equation}
 P(z|\theta) = P(x,z|\theta). \label{eqn:2:pz-equality}
\end{equation}

Some questions arise when applying this model to real parsing applications like grammar induction:
\begin{description}
 \item[Parsing] Given a PCFG $G$, how to find the best (highest probablity) parse among all possible parses?
 \item[Learning] Where do the probabilities, or rule weights $\theta$ come from?
\end{description}
The nice property of PCFGs is that there is a very general solution for these questions.
We first discuss the first question in Section \ref{sec:2:cky}, and then deal with the second question in Section \ref{sec:2:em}.

\subsection{CKY Algorithm}
\label{sec:2:cky}

Let us define some notations first.
We assume the input sentence is a length $n$ sentence, $x = x_1 x_2 \cdots x_n$ where $x_i \in \Sigma$.
For $i \leq j$, $x_{i,j} = x_i x_{i+1} \cdots x_{j}$ denotes an input substring.
We assume the grammar is in CNF (Section \ref{sec:2:cfg}), which makes the discussion much simpler.

Given an input sentence $x$ and a PCFG with parameters $\theta$, the goal of parsing is to solve the following argmax problem:
\begin{equation}
 z' = \arg\max_{z\in \mathcal{Z}(x)} P(z|\theta), \label{eq:2:cky-argmax}
\end{equation}
where $\mathcal{Z}(x)$ is a set of all possible parses on $x$.
Now we describe a general algorithm to solve this problem in polynomial time called the CKY algorithm, which also plays an essential role in parameter esmation of $\theta$ discussed in Section \ref{sec:2:em}.

For the moment we simplify the problem as calculating the {\it probability} of the best parse instead of the best parse itself (Eq. \ref{eq:2:cky-argmax}).
We later describe that the argmax problem can be solved with a small modification to this algorithm.

The CKY algorithm is a kind of chart parsing.
For an input string, there are too many, exponential number of parses, which prohibit to enumerate one by one.
To enable search in this large space, a chart parser divides the problem into subproblems, each of which analyze a small span $x_{i,j}$ and then is combined into an analysis of a larger span.

Let $C$ be a chart, which gives mapping from a signature of a subspan, or an item $(i,j,N)$ to a real value.
That is, each cell of $C$ keeps the probability of the best (highest score) analysis for a span $x_{i,j}$ with symbol $N$ as its root.
Algorithm \ref{alg:cky} describes the CKY parsing, in which each chart cell is filled recursively.
The procedure {\sc Fill}($i,j,N$) return a value with memoization.
This procedure is slightly different from the ones found in some textbooks \cite{manning99foundations}, which instead fill chart cells in specific order.
We do no take this approach since our memoization-based technique need not care about correct order of filling chart cells, which is somewhat involved for more complex grammars such as the one we present in Chapter \ref{chap:induction}.

\begin{algorithm}[t]
 \caption{CKY Parsing}\label{alg:cky}
 \hspace{18pt}
 \begin{minipage}{0.95\textwidth}
 \begin{algorithmic}[1]
 \Procedure{Parse}{$x$}
 \State $C[1,n,S] \leftarrow {\textrm{\sc Fill}}(1,n,S)$ \Comment{Recursively fills chart cells.}
 \State \Return $C[1,n,S]$
 \EndProcedure
 \Procedure{Fill}{$i,j,N$}
 \If{$(i,j,N) \not\in C$}
 \If{$i = j$} \label{alg:cky:fill-start}
 \State $C[i,i,N] \leftarrow \theta_{N \rightarrow x_{i}}$ \Comment{Terminal expansion.} \label{alg:cky:terminal}
 \Else
 \State $C[i,j,N] \leftarrow \max_{N \rightarrow A~B \in R; k \in [i,j]} \theta_{N \rightarrow A~B} \times {\textrm{\sc Fill}}(i,k,A) \times {\textrm{\sc Fill}}(k,j,B)$ \label{alg:cky:recursive}
 \EndIf \label{alg:cky:fill-end}
 \EndIf
 \State \Return $C[i,j,N]$
\EndProcedure
 \end{algorithmic}  
 \end{minipage}
\end{algorithm}

The crucial point in this algorithm is the recursive equation in line \ref{alg:cky:recursive}.
The assumption here is that since the grammar is context-free, the parses of subspans (e.g., spans with signatures $(i,k,A)$ and $(k,j,B)$) in the best parse of $x_{i,j}$ should also be the best parses at subspan levels.
The goal of the algorithm is to fill the chart cell for an item $(1,n,S)$ where $S$ is the start symbol, which corresponds the analysis of the full span (whole sentence).

To get the best parse, what we need to modify in the algorithm \ref{alg:cky} is to keep backpointers to the cells of the best children when filling each chart cell during lines \ref{alg:cky:fill-start}--\ref{alg:cky:fill-end}.
This is commonly done by preparing another chart, in which each cell keeps not a numerical value but backpointers into other cells in that chart.
The resulting algorithm is called the {\it Viterbi} algorithm, the details of which are found in the standard textbooks \cite{manning99foundations}.

\begin{figure}[t]
 \centering
  \begin{minipage}[t]{.5\linewidth}
   \begin{tikzpicture}[thick, level distance=0.8cm]
    \node {\sc Terminal:};
    \draw (1, -0.5) -- +(2, 0) node[right] {$N\rightarrow x_i \in R$};
    \begin{scope}[xshift=2cm, yshift=-0.9cm]
     \Tree
     [.$N$ \edge[roof,thick]; {~~~~~~~~} ];
     \node at (0, -0.85cm) {$i~~~~~~~~~~i$};
    \end{scope}
   \end{tikzpicture}
  \end{minipage}
 \begin{minipage}[t]{.5\linewidth}
  \begin{tikzpicture}[thick, level distance=0.8cm]
   \node {\sc Binary:};
   \begin{scope}[xshift=2cm, yshift=0cm]
    \Tree
    [.$A$ \edge[roof,thick]; {~~~~~~~~} ];
   \end{scope}
   \begin{scope}[xshift=4cm, yshift=0cm]
    \Tree
    [.$B$ \edge[roof,thick]; {~~~~~~~~} ];
   \end{scope}
   \node at (3, -0.85) {$i~~~~~~~~~~k~~~~~~k~~~~~~~~~~j$};
   \draw (1.2, -1.1) -- +(3.5, 0) node[right] {$N\rightarrow A~B \in R$};
   \begin{scope}[xshift=3cm, yshift=-1.45cm]
    \Tree
    [.$N$ \edge[roof,thick]; {~~~~~~~~} ];
    \node at (0, -0.85cm) {$i~~~~~~~~~~j$};
   \end{scope}
  \end{tikzpicture}
 \end{minipage}
 \caption{Inference rules of the CKY algorithm. {\sc Terminal} rules correspond to the terminal expansion in line \ref{alg:cky:terminal} of the Algorithm \ref{alg:cky}; {\sc Binary} rules correspond to the one in line \ref{alg:cky:recursive}. Each rule specifies how an analysis of a larger span (below ---) is derived from the analyses of smaller spans (above ---) provided that the input and grammar satisfy the side conditions in the right of ---.}
 \label{fig:2:cky}
\end{figure}
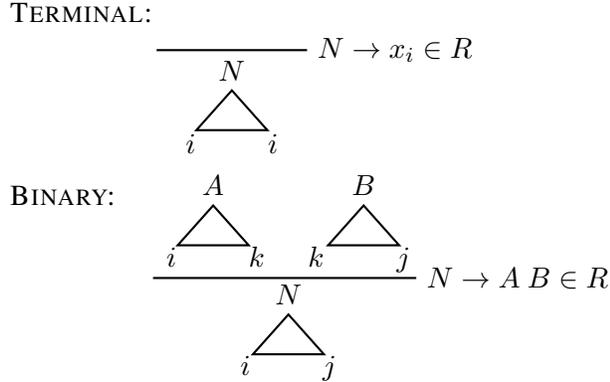

The behavior of the CKY parsing is characterized by the procedure of filling chart cells in lines \ref{alg:cky:fill-start}--\ref{alg:cky:fill-end} of Algorithm \ref{alg:cky}.
We often write this procedure visually as in Figure \ref{fig:2:cky}.
With this specification we can analyze the time complexity of the CKY algorithm, which is $O(n^3|P|)$ where $|P|$ is the number of allowed rules in the grammar because each rule in Figure \ref{fig:2:cky} is performed only once \footnote{Once the chart cell of some item is calculated, we can access to the value of that cell in $O(1)$.} and there are at most $O(n^3|P|)$ ways of instantiations for {\sc Binary} rules.\footnote{{\sc Terminal} rules are instantiated at most $O(n|N|)$ ways, which is smaller than $O(n^3|R|)$.}

\subsection{Learning parameters with EM algorithm}
\label{sec:2:em}

Next we briefly describe how rule weights $\theta$ can be estimated given a collection of input sentences.
This is the setting of {\it unsupervised} learning.
In supervised learning, we can often learn parameters more easily by counting rule occurrences in the training treebank \cite{P97-1003,J98-4004,klein-manning:2003:ACL}.

\paragraph{EM algorithm}

Assume we have a set of training examples $\mathbf x = x^{(1)}, x^{(2)},\cdots,x^{(m)}$.
Each example is a sentence $x^{(i)} = x^{(i)}_1 x^{(i)}_2 \cdots x^{(i)}_{n_i}$ where $n_i$ is the length of the $i$-th sentence.
Our goal is to estimate parameters $\theta$ given $\mathbf x$, which is good in some criteria.

The EM algorithm is closely related to maximum likelihood estimation in that it tries to estimate $\theta$, which maximizes the following log-likelihood of the observed data $\mathbf x$:
\begin{equation}
L(\theta,\mathbf x) = \sum_{1 \leq i \leq m} \log p(x^{(i)} | \theta) = \sum_{1 \leq i \leq m} \log \sum_{z \in \mathcal{Z}(x^{(i)})} p(x^{(i)}, z | \theta), \label{eqn:2:emobj}
\end{equation}
where $p(x^{(i)}, z|\theta)$ is given by Eq. \ref{eqn:2:pz} due to Eq. \ref{eqn:2:pz-equality}.
However, calculating $\theta$ that maximizes this objective is generally intractable \cite{DEMP1977}.
The idea of EM algorithm is instead of getting the optimal $\theta$, trying to increase Eq. \ref{eqn:2:emobj} up to some point to find the locally optimal values of $\theta$, starting from some initial values $\theta_0$.
It is an iterative procedure and updates parameters as $\theta^{(0)} \rightarrow \theta^{(1)} \rightarrow \cdots$ until specific number of iterations (or until $L(\theta,\mathbf x)$ does not increase).

Each iteration of the EM algorithm proceeds as follows.
\begin{description}
 \item[E-step] Given the current parameters $\theta^{(t)}$, calculate the expected counts of each rule $e(r|\theta^{(t)})$ as
            \begin{equation}
             e(r|\theta^{(t)}) = \sum_{1 \leq i \leq m} e_{x^{(i)}}(r|\theta^{(t)}),
            \end{equation}
            where $e_{x}(r|\theta^{(t)})$ is the expected counts of $r$ in a sentence $x$, given by
            \begin{equation}
             e_{x}(r|\theta^{(t)}) = \sum_{z \in \mathcal{Z}(x)} p(z|x) f(r,z). \label{eqn:2:ex}
            \end{equation}
            where $f(r,z)$ is the number of times that $r$ appears in $z$.
            As in the parsing problem in Eq. \ref{eq:2:cky-argmax}, it is impossible to directly calculate Eq. \ref{eqn:2:ex} by enumerating every parse.
            Below, we describe how this calculation becomes possible with the dynamic programming algorithm called the inside-outside algorithm, which is similar to CKY.
 \item[M-step] Update the parameters as follows:
 \begin{equation}
  \theta^{(t+1)}_{A \rightarrow \beta} = \frac{e(A \rightarrow \beta | \theta^{(t)})}{\sum_{\alpha: A\rightarrow \alpha \in R} e(A \rightarrow \alpha | \theta^{(t)}) }
 \end{equation}
            This update is similar to the standard maximum likelihood estimation in the supervised learning setting, in which we observe the explicit counts of each rule.
            In the EM algorithm we do not explicitly observe rule counts, so we use the expected counts calculated with the previously estimated parameters.
            We can show that this procedure always increases the log likelihood (Eq. \ref{eqn:2:emobj}) until convergence though the final parameters are not globally optimum.
\end{description}

\paragraph{Inside-outside algorithm}

We now explain how the expected rule counts $e_{x}(r|\theta)$ are obtained for each $r$ given sentence $x$.
Let $r$ be a binary rule $r = A \rightarrow B~C$.
First, it is useful to decompose $e_{x}(r|\theta)$ as the expected counts on each subspan as follows:
\begin{equation}
 e_{x}(r|\theta) = \sum_{1\leq i \leq k \leq j \leq n_x} e_x(z_{i,k,j,r}|\theta).
\end{equation}
$e_x(z_{i,k,j,r}|\theta)$ is the expected counts of an event that the following fragment occurs in a parse $z$.
\begin{center}
\tikz[level distance=0.8cm]{
  \Tree 
  [.$A$
  [.$B$ \edge[roof]; {~~~~~~~~} ]
  [.$C$ \edge[roof]; {~~~~~~~~} ]
  ]
 \node at (0,-1.7) {$i~~~~~~~~~k~~~~~~~~~j$};
}
\end{center}

\noindent $z_{i,k,j,r}$ is an indicator variable\footnote{We omit dependence for $x$ for simplicity.} that is 1 if the parse contains the fragment above.
Because the expected counts for an indicator variable are the same as the conditional probability of that variable \cite{Bishop:2006:PRM:1162264}, we can rewrite $e_x(z_{i,k,j,r}|\theta)$ as follows:
\begin{align}
 e_x(z_{i,k,j,r}|\theta) &= p(z_{i,k,j,r} = 1 | x, \theta) \\
                  &= \frac{p(z_{i,k,j,r} = 1, x | \theta)}{p(x|\theta)}. \label{eqn:2:ew}
\end{align}
Intuitively the numerator in Eq. \ref{eqn:2:ew} is the total probability for generating parse trees that yield $x$ and contain the fragment $z_{i,k,j,r}$.
The denominator $p(x|\theta)$ is the marginal probability of the sentence $x$.

We first consider how to calculate $p(x|\theta)$, which can be done with a kind of CKY parsing;
what we have to modify is just to replace the $\max$ operation in the line \ref{alg:cky:recursive} in Algorithm \ref{alg:cky} by the summation operation.
Then each chart cell $C[i,j,N]$ keeps the marginal probability for a subspan $x_{i,j}$ rooted at $N$.
After filling the chart, $C[1,n_x,S]$ is the sentence marginal probability $p(x|\theta)$.
The marginal probability for the signature $(i,j,N)$ is called the {\it inside} probability, and this chart algorithm is called the inside algorithm, which calculates inside probabilities by filling chart cells recursively.

\REVISE{
Calculation of $p(z_{i,k,j,r} = 1, x | \theta)$ is more elaborate so we only sketch the idea here.
Analogous to the inside probability introduced above, we can also define the outside probability $O(i,j,N)$, which is the marginal probability for the outside of the span with signature $(i,j,N)$;
that is,
\begin{equation}
 O(i,j,N) = p(x_{1,i-1},N,x_{j+1,n}|\theta),
\end{equation}
in which $N$ roots the subspan $x_{i,j}$.
Since $I(i,j,N) = p(x_{x_{i,j}}| N,\theta)$, given $r = N \rightarrow A~B$ we obtain:
\begin{equation}
 p(z_{i,k,j,r} = 1, x_{i,j}| N, \theta) = \theta_{N \rightarrow A~B} \times I(i,k,A) \times I(k,j,B), 
\end{equation}
which is the total probability of parse trees for the subspan $x_{i,j}$ that contains the fragment indicated by $z_{i,k,j,r}$.
Combining these two terms, we obtain:
\begin{align}
 p(z_{i,k,j,r} = 1, x | \theta) &= p(z_{i,k,j,r} = 1, x_{i,j}| N, \theta) \times p(x_{1,i-1},N,x_{j+1,n}|\theta) \\
  &= \theta_{N \rightarrow A~B} \times I(i,k,A) \times I(k,j,B) \times O(i,j,N). \label{eqn:2:io}
\end{align}
}

\subsection{When the algorithms work?}
\label{sec:2:when}
So far we have assumed the underlying grammar is a PCFG, for which we have introduced two algorithms, the CKY algorithm and the EM algorithm with inside-outside calculation.
However as we noted in the beginning of Section \ref{sec:2:pcfg}, the scope of these algorithms is not limited to PCFGs.
What is the precise condition under which these algorithms can be applied?

PCFGs are a special instance of weighted CFGs, in which each rule has a weight but the sum of rule weights from a parent nonterminal (Eq. \ref{eqn:2:normalize}) is not necessarily normalized.
As we see next, the split bilexical grammars in Section \ref{sec:2:sbg} can always be converted to a weighted CFG but may not be converted to a PCFG.

The CKY algorithm can be applied to {\it any} weighted CFGs.
This is easily verified because only the assumption in the Algorithm \ref{alg:cky} is that the grammar is context-free for being able to divide a larger problem into smaller subproblems.

The condition for the inside-outside algorithm is more involved.
Let us assume that we have a generative model of a parse, which is not originally parameterized with a PCFG, and also we have a weighted CFG designed so that the score that this weighted CFG gives to a (CFG) parse is the same as the probability that the original generative model assigns to the corresponding parse in the original form (not CFG) (The model in Section \ref{sec:2:sbg} is an example of such cases.).

Then, the necessary condition on this weighted CFG is that there is no spurious ambiguity between two representations (Section \ref{sec:bilexical});
that is, a CFG parse can be uniquely converted to the parse in the original form, and vice versa.
The main reason why the spurious ambiguity cause a problem is that the quantities used to calculate expected counts (Eq. \ref{eqn:2:ew}) are not correctly defined if the spurious ambiguity exists.
For example, the sentence marginal probability $p(x|\theta)$ would not be equal to the inside probability for the whole sentence $I(1,n_x,S)$ when the spurious ambiguity exists since if there are multiple CFG derivations for a single parse in the original form the inside probability calculated with the weighted CFG would overestimate the true sentence marginal probability.
The same issue happens for another quantity in Eq. \ref{eqn:2:io}.
On the other hand, if there is one-to-one correspondence between two representations, it should hold at the smaller subspan levels and it is this transparency that guarantees the correctness of the inside-outside algorithm even if the grammar is not strictly a PCFG.

\subsection{Split bilexical grammars}
\label{sec:2:sbg}

\begin{figure}[t]
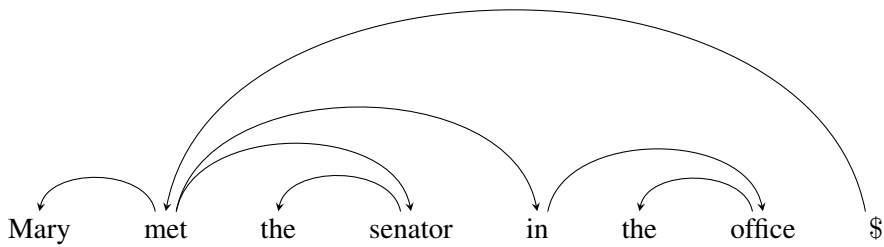

 \centering
 \begin{dependency}[theme=simple]
  \begin{deptext}[column sep=0.8cm]
   Mary \& met \& the \& senator \& in \& the \& office \& $\$$ \\
  \end{deptext}
  \depedge{2}{1}{}
  \depedge{2}{4}{}
  \depedge{4}{3}{}
  \depedge{2}{5}{}
  \depedge{5}{7}{}
  \depedge{7}{6}{}
  \depedge{8}{2}{}
 \end{dependency}
 \caption{Example of a projective dependency tree generated by a SBG. $\$$ is always placed at the end of a sentence, which has only one dependent in the left direction.}
 \label{fig:2:deptree-sbg}
\end{figure}

The split bilexical grammars, or SBGs \cite{eisner-2000-iwptbook} is a notationally simpler variant of split head-automaton grammars \cite{Eisner:1999:EPB:1034678.1034748}.
Here we describe this model as a generalization of the specific model described in Section \ref{sec:2:dmv}, the dependency model with valence (DMV) \cite{klein-manning:2004:ACL}.
We will give a somewhat in-depth explanation of this grammar below because it will be the basis of our proposal in Chapter \ref{chap:induction}.

The explanation below basically follows \newcite{eisner2010}.
A SBG defines a distribution over projective dependency trees.
This model can easily be converted to an equivalent weighted CFG, although some effort is needed to remove the {\it spurious ambiguity}.
We will show that by removing it the time complexity can also be improved from $O(n^5)$ to $O(n^3)$.
In Chapter \ref{chap:induction} we will invent the similar technique for our model that follows the left-corner parsing strategy.

\paragraph{Model}
A (probabilistic) SBG is a tuple $G_{SBG}=(\Sigma, \$, L, R)$.
$\Sigma$ is an alphabet of words that may appear in a sentence.
$\$ \not\in \Sigma$ is a distinguished root symbol, which we describe later; let $\bar \Sigma\ = \Sigma \cup \{\$\}$.
$L$ and $R$ are functions from $\bar \Sigma$ to probabilistic $\epsilon$-free finite-state automata over $\Sigma$;
that is, for each $a \in \bar \Sigma$ the SBG specifies ``left'' and ``right'' probabilistic FSAs, $L_a$ and $R_a$.
We write $q \xmapsto{~a'~} r \in R_{a}$ to denote a state transition from $q$ to $r$ by adding $a'$ to $a$'s right dependents when the current right state of $a$ is $q$.
Also each model defines $\textit{init}(L_a)$ and $\textit{init}(R_a)$ that return the set of initial states for $a$ in either direction (usually the initial state is unique given the head $a$ and the direction).
$\textit{final}(L_a)$ is a set of final states;
$q \in \textit{final}(L_a)$ means that $a$ can stop generating its left dependents.
We will show that by changing the definitions of these functions several generative models over dependency trees can be represented in a unified framework.

The model generates a sentence $x_1 x_2 \cdots x_{n} \$$ along with a parse, given the root symbol $\$$, which is always placed at the end.
An example of a parse is shown in Figure \ref{fig:2:deptree-sbg}.
SBGs define the following generative process over dependency trees:
\begin{enumerate}
 \item The root symbol $\$$ generates a left dependent $a$ from $q\in \textit{init}(L_\$)$.
       $a$ is regarded as the conventional root word in a sentence (e.g., {\it met} in Figure \ref{fig:2:deptree-sbg}).
 \item Recursively generates a parse tree.
       Given the current head $a$, the model generates its left dependents and its right dependents.
       This process is head-outward, meaning that the closest dependent is generated first.
       For example, the initial left state of {\it met}, $q \in \textit{init}(L_{\textit met})$ generates {\it Mary}.
       The right dependents are generated as follows:
       First {\it senator} is generated from $q_0 \in \textit{init}(R_{\textit met})$.
       Then the state may be changed by a transition $q_0 \xmapsto{\textit{senator}} q_1$ to $q_1 \in R_{\textit met}$, which generates {\it in}.
       The process stops when every token stops generating its left and right dependents.
\end{enumerate}

This model can generalize several generative models over dependency trees.
Given $a$, $L_a$ and $R_a$ define distributions over $a$'s left and right children, respectively.
Since the automata $L_a$ and $R_a$ have the current state, we can define several distributions by customizing the topology of state transitions.
For example if we define the automata of the form:

\begin{center}
\tikz[thick, level distance=0.8cm, >=stealth']{
  \def\state#1{
  \draw (0,0) circle (0.25);
  \draw (0,0) circle (0.5);
  \node at (45:0.75) {#1};
  }
  \state{$q_0$};
  \draw[->] (0.55, 0) -- +(1.9, 0);
  \begin{scope}[xshift=3.0cm,yshift=0cm]
   \state{$q_1$};
  \end{scope}
  \draw[rounded corners=8pt,->] (3.55, 0.25) -- ++(0.9, 0) -- ++(0, -0.5) -- ++(-0.9, 0);
}
\end{center}

\noindent it would allow the first (closet) dependent to be chosen differently from the rest ($q_0$ defines the probability of the first dependent).
If we remove $q_0$, the resulting automata are \tikz[baseline=-4pt]{
\draw (0,0) circle (2pt);
\draw (0,0) circle (5pt);
\draw[rounded corners=3pt,->,>=stealth'] (6pt,3pt) -- ++(12pt,0) -- ++(0,-6pt) -- ++(-12pt,0);
} with a single state $q_1$, so token $a$'s left (or right) dependents are conditionally independent of one another given $a$.\footnote{SBGs can also be used to encode second-order {\it adjacent} dependencies, i.e., $a$'s left or right dependent to be dependent on its sibling word just generated before, although in this case there exists more efficient factorization that leads to a better asymptotic runtime \cite{McDonald2006,johnson:2007:ACLMain}.}


\paragraph{Spurious ambiguity in naive $\boldsymbol{O(n^5)}$ parsing algorithm}

We now describe that a SBG can be converted to a weighted CFG, though the distributions associated with $L_a$ and $R_a$ cannot always be encoded in a form of a PCFG.
Also, as we see below, the grammar suffers from the spurious ambiguity, which prevent us to apply the inside-outside algorithm (see Section \ref{sec:2:when}).

The key observation for this conversion is that we can represent a subtree of SBGs as the following triangle, which can be seen as a special case of a subtree used in the ordinary CKY parsing.
\begin{center}
 \vspace{-5pt}
 \tikz[thick, level distance=0.8cm]{
 \Tree
  [.$q_1~q_2$ \edge[roof,thick]; {~~~~~~~~} ];
  \node at (0, -0.85) {$i~~~~h~~~~j$};
 }\vspace{-10pt}
\end{center}

\noindent The main difference from the ordinary subtree representation is that it is decorated with an additional index $h$, which is the position of the head word.
For example in the analysis of Figure \ref{fig:2:deptree-sbg}, a subtree on {\it met the senator} is represented by setting $i=2, j=4, h=2$.
$q_1$ and $q_2$ are the current $h$'s left and right states respectively.

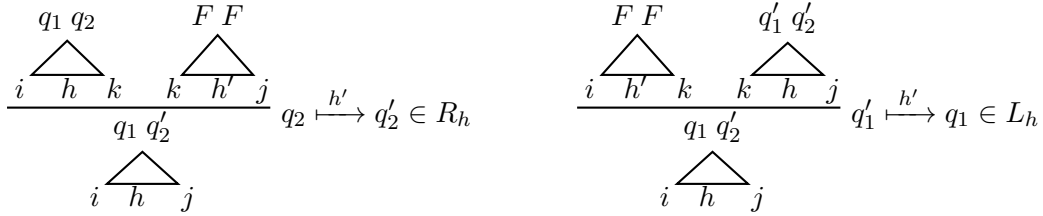
\begin{figure}[t]
 \centering
 \begin{minipage}[t]{.49\linewidth}
  \centering
  \begin{tikzpicture}[thick, level distance=0.8cm]
  \begin{scope}[xshift=2cm, yshift=0cm]
   \Tree
   [.$q_1~q_2$ \edge[roof,thick]; {~~~~~~~~} ];
  \end{scope}
  \begin{scope}[xshift=4cm, yshift=0cm]
   \Tree
   [.$F~F$ \edge[roof,thick]; {~~~~~~~~} ];
  \end{scope}
  \node at (3, -0.86) {$i~~~~~h~~~~k~~~~~~k~~~~h'~~~j$};
  \draw (1.2, -1.1) -- +(3.5, 0) node[right] {$q_2 \xmapsto{~h'~} q_2' \in R_{h}$};
  \begin{scope}[xshift=3cm, yshift=-1.45cm]
   \Tree
   [.$q_1~q_2'$ \edge[roof,thick]; {~~~~~~~~} ];
   \node at (0, -0.86cm) {$i~~~~h~~~~~j$};
  \end{scope}
  \end{tikzpicture}
 \end{minipage}
 \begin{minipage}[t]{.49\linewidth}
 \centering
 \begin{tikzpicture}[thick, level distance=0.8cm]
  \begin{scope}[xshift=2cm, yshift=0cm]
   \Tree
   [.$F~F$ \edge[roof,thick]; {~~~~~~~~} ];
  \end{scope}
  \begin{scope}[xshift=4cm, yshift=0cm]
   \Tree
   [.$q_1'~q_2'$ \edge[roof,thick]; {~~~~~~~~} ];
  \end{scope}
  \node at (3, -0.86) {$i~~~~h'~~~~k~~~~~~k~~~~h~~~~j$};
  \draw (1.2, -1.1) -- +(3.5, 0) node[right] {$q_1' \xmapsto{~h'~} q_1 \in L_{h}$};
  \begin{scope}[xshift=3cm, yshift=-1.45cm]
   \Tree
   [.$q_1~q_2'$ \edge[roof,thick]; {~~~~~~~~} ];
   \node at (0, -0.86cm) {$i~~~~h~~~~~j$};
  \end{scope}
 \end{tikzpicture}
 \end{minipage}
 \caption{Binary inference rules in the naive CFG conversion.
 $F$ means the state is a final state in that direction.
 Both left and right consequent items (below ---) have the same item but from different derivations, suggesting 1) the weighted grammar is not a PCFG;
 and 2) there is the spurious ambiguity.}
 \label{fig:2:binary-naive-sbg}
\end{figure}

We can assume a tuple ($h$, $q_1$, $q_2$) to comprise a nonterminal symbol of a CFG.
Then the grammar is a PCFG if the normalization condition (Eq. \ref{eqn:2:normalize}) is satisfied for every such tuple.
Note now each rule looks like $(h,q_1,q_2) \rightarrow \beta$.

Figure \ref{fig:2:binary-naive-sbg} explains why the grammar cannot be a PCFG.
We can naturally associate PCFG rule weights for these rules with transition probabilities given by the automata $L_h$ and $R_h$.\footnote{Here and the following, we occasionally abuse the notation and use $L_h$ or $R_h$ to mean the automaton associated with word $x_h$ at index position $h$.}
However, then the sum of rule weights of the converted CFG starting from symbol $(a,q_1,q_2')$ is not equal to 1.
The left rule of Figure \ref{fig:2:binary-naive-sbg} means the converted CFG would have rules of the form $(a,q_1,q_2') \rightarrow (a,q_1,q_2)~(a',F,F)$.
The weights associated with these rules are normalized across $a' \in \Sigma$, as state transitions are deterministic given $q_2'$ and $a'$.
The problem is that the same signature, i.e., $(a,q_1,q_2')$ on the same span can also be derived from another rule in the right side of Figure \ref{fig:2:binary-naive-sbg}.
This distribution is also normalized, meaning that the sum of rule weights is not 1.0 but 2.0, and thus the grammar is a weighted CFG, not a PCFG.
The above result also suggests that the grammar suffers from the spurious ambiguity, 
Below we describe how this ambiguity can be removed with modifications.

As a final note, the time complexity of the algorithm in Figure \ref{fig:2:binary-naive-sbg} is very inefficient, $O(n^5)$ because there are five free indexes in each rule.
This is in contrast to the complexity of original CKY parsing, which is $O(n^3)$.
The refinement described next also fix this problem, and we obtain the $O(n^3)$ algorithm for parsing general SBGs.

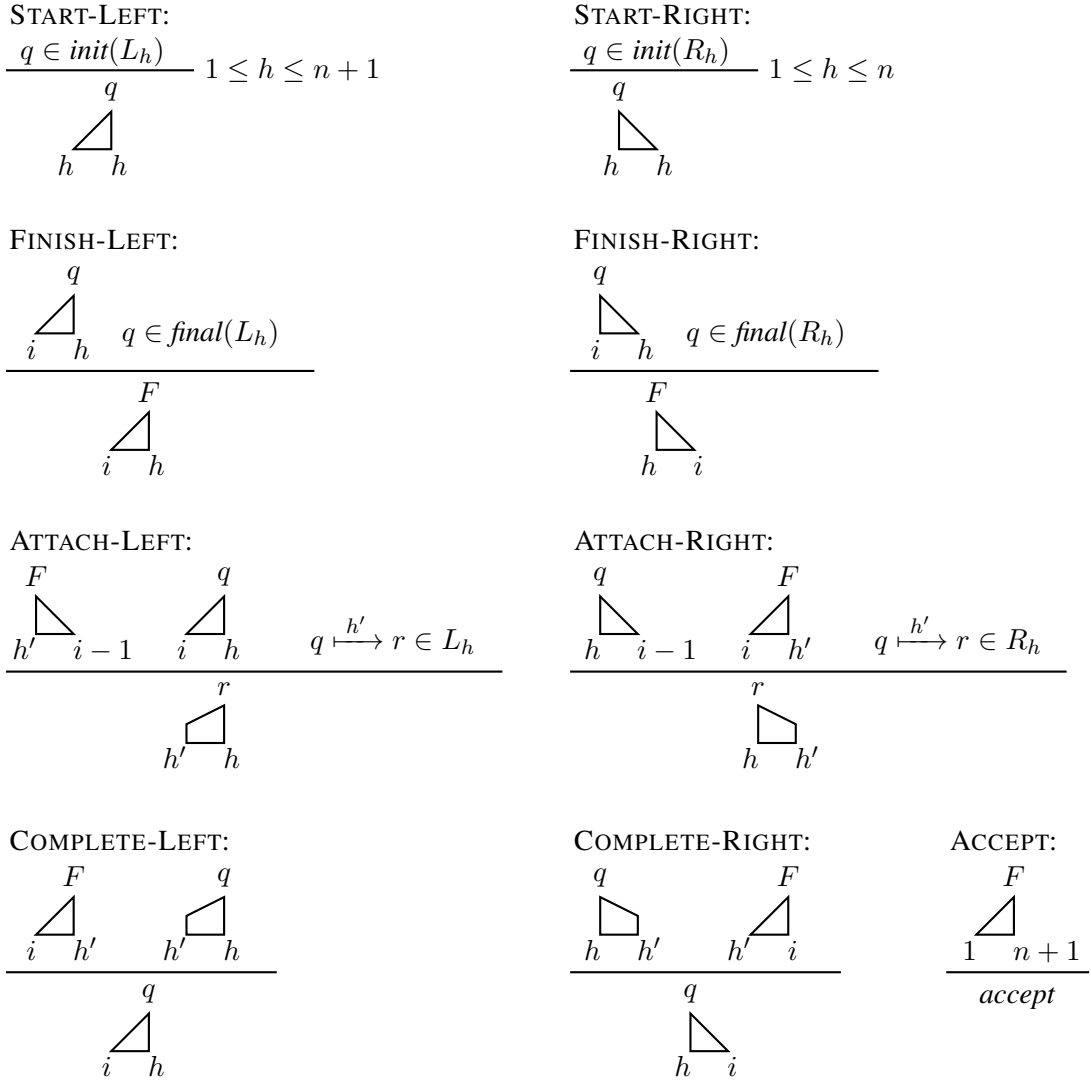
\begin{figure}[p]
 \begin{tikzpicture}[thick, level distance=0.8cm, >=stealth']
  \begin{scope}[xshift=0cm,yshift=0cm]
   \node at (0, 0) [anchor=west] {\sc Start-Left:};
   \node at (1.25, -0.5) {$q \in \textit{init}(L_h)$};
   \draw (0.1, -0.75) -- +(2.5, 0) node[right] {$1 \leq h \leq n+1$};
   \begin{scope}[xshift=1.5cm, yshift=-1.3cm]
    \lefttriangle{$q$}{$h$}{$h$};
   \end{scope}
  \end{scope}
  
  \begin{scope}[xshift=0cm,yshift=-3cm]
   \node at (0, 0) [anchor=west] {\sc Finish-Left:};
   \begin{scope}[xshift=1.0cm, yshift=-0.75cm]
    \lefttriangle{$q$}{$i$}{$h$};
   \end{scope}
   \node at (1.5, -1.25) [anchor=west] {$q \in \textit{final}(L_h)$};
   \draw (0.1, -1.75) -- +(4.1, 0);
   \begin{scope}[xshift=2.0cm, yshift=-2.3cm]
    \lefttriangle{$F$}{$i$}{$h$};
   \end{scope}
  \end{scope}
  
  \begin{scope}[xshift=7.5cm,yshift=0cm]
   \node at (0, 0) [anchor=west] {\sc Start-Right:};
   \node at (1.25, -0.5) {$q \in \textit{init}(R_h)$};
   \draw (0.1, -0.75) -- +(2.5, 0) node[right] {$1 \leq h \leq n$};
   \begin{scope}[xshift=0.75cm, yshift=-1.3cm]
    \righttriangle{$q$}{$h$}{$h$};
   \end{scope}
  \end{scope}

  \begin{scope}[xshift=7.5cm,yshift=-3cm]
   \node at (0, 0) [anchor=west] {\sc Finish-Right:};
   \begin{scope}[xshift=0.5cm, yshift=-0.75cm]
    \righttriangle{$q$}{$i$}{$h$}
   \end{scope}
   \node at (1.5, -1.25) [anchor=west] {$q \in \textit{final}(R_h)$};
   \draw (0.1, -1.75) -- +(4.1, 0);
   \begin{scope}[xshift=1.25cm, yshift=-2.3cm]
    \righttriangle{$F$}{$h$}{$i$}
   \end{scope}
  \end{scope}

  \begin{scope}[xshift=0cm,yshift=-7cm]
   \node at (0, 0) [anchor=west] {\sc Attach-Left:};
   \begin{scope}[xshift=0.5cm, yshift=-0.75cm]
    \righttriangle{$F$}{$h'$}{$i-1$}
   \end{scope}
   \begin{scope}[xshift=3.0cm, yshift=-0.75cm]
    \lefttriangle{$q$}{$i$}{$h$}
   \end{scope}
   \node at (4.0, -1.25) [anchor=west] {$q \xmapsto{~h'~} r \in L_h$};
   \draw (0.1, -1.75) -- +(6.6, 0);
   \begin{scope}[xshift=3.0cm, yshift=-2.2cm]
    \lefttrape{$r$}{$h'$}{$h$}
   \end{scope}
  \end{scope}

  \begin{scope}[xshift=7.5cm,yshift=-7cm]
   \node at (0, 0) [anchor=west] {\sc Attach-Right:};
   \begin{scope}[xshift=0.5cm, yshift=-0.75cm]
    \righttriangle{$q$}{$h$}{$i-1$}
   \end{scope}
   \begin{scope}[xshift=3.0cm, yshift=-0.75cm]
    \lefttriangle{$F$}{$i$}{$h'$}
   \end{scope}
   \node at (4.0, -1.25) [anchor=west] {$q \xmapsto{~h'~} r \in R_h$};
   \draw (0.1, -1.75) -- +(6.6, 0);
   \begin{scope}[xshift=2.6cm, yshift=-2.2cm]
    \righttrape{$r$}{$h$}{$h'$}
   \end{scope}
  \end{scope}

  \begin{scope}[xshift=0cm,yshift=-11cm]
   \node at (0, 0) [anchor=west] {\sc Complete-Left:};
   \begin{scope}[xshift=1.0cm, yshift=-0.75cm]
    \lefttriangle{$F$}{$i$}{$h'$}
   \end{scope}
   \begin{scope}[xshift=3.0cm, yshift=-0.75cm]
    \lefttrape{$q$}{$h'$}{$h$}
   \end{scope}
   \draw (0.1, -1.75) -- +(3.6, 0);
   \begin{scope}[xshift=2.0cm, yshift=-2.3cm]
    \lefttriangle{$q$}{$i$}{$h$}
   \end{scope}
  \end{scope}

  \begin{scope}[xshift=7.5cm,yshift=-11cm]
   \node at (0, 0) [anchor=west] {\sc Complete-Right:};
   \begin{scope}[xshift=0.5cm, yshift=-0.75cm]
    \righttrape{$q$}{$h$}{$h'$}
   \end{scope}
   \begin{scope}[xshift=3.0cm, yshift=-0.75cm]
    \lefttriangle{$F$}{$h'$}{$i$}
   \end{scope}
   \draw (0.1, -1.75) -- +(3.6, 0);
   \begin{scope}[xshift=1.7cm, yshift=-2.3cm]
    \righttriangle{$q$}{$h$}{$i$}
   \end{scope}
  \end{scope}

  \begin{scope}[xshift=12.5cm,yshift=-11cm]
   \node at (0, 0) [anchor=west] {\sc Accept:};
   \begin{scope}[xshift=1.0cm, yshift=-0.75cm]
    \lefttriangle{$F$}{$1$}{$n+1$};
   \end{scope}
   \draw (0.1, -1.75) -- +(1.9, 0);
   \node at (0.4, -2.1) [anchor=west] {{\it accept}};
  \end{scope}
 \end{tikzpicture}
 \caption{An algorithm for parsing SBGs in $O(n^3)$ given a length $n$ sentence.
 The $n+1$-th token is a dummy root token $\$$, which only has one left dependent (sentence root).
 $i,j,h$, and $h'$ are index of a token in the given sentence while $q,r$, and $F$ are states.
 $L_h$ and $R_h$ are left and right FSA of the $h$-th token in the sentence.
 Each item as well as a statement about a state (e.g., $r\in \textit{final}(L_p)$) has a weight and the weight of a consequent item (below ---) is obtained by the product of the weights of its antecedent items (above ---).
 }
 \label{fig:bg:ded-sbg}
\end{figure}

\paragraph{$\boldsymbol{O(n^3)}$ algorithm with head-splitting}
The main reason why the algorithm in Figure \ref{fig:2:binary-naive-sbg} causes the spurious ambiguity is because given a head there is no restriction in the order of collecting its left and right dependents.
This problem can be handled by introducing new items called {\it half constituents} denoted by
\tikz[thick]{
\draw (0, -0.2) -- ++(-0.5, -0.45) -- ++(0.5, 0) -- cycle;
} ({\it left constituents}) and 
\tikz[thick]{
\draw (0, -0.2) -- ++(0.5, -0.45) -- ++(-0.5, 0) -- cycle;
} ({\it right constituents}),
which represent left and right span separately given a head.
For example in the dependency tree in Figure \ref{fig:2:deptree-sbg}, a phrase ``Mary met'' comprises a left constituent while ``met the senator'' comprises a right constituent.
In the new algorithm these two constituents are expanded separately with each other and there is no {\it mixed} states $(q_1, q_2)$ as in the items in Figure \ref{fig:2:binary-naive-sbg}.
Eliminating these mixed states is the key to eliminate the spurious ambiguity.

Figure \ref{fig:bg:ded-sbg} shows new algorithm, which can be understood as follows:
\begin{itemize}
 \item {\sc Attach-Left} and {\sc Complete-Left} (or the {\sc Right} counterpart) are the essential components of the algorithm.
       The idea is when combining two constituents headed by $h'$ and $h$ ($h' < h$) into a large constituent headed by $h$, we decompose
       an original constituent \tikz[baseline=-20pt]{\headtriangleinlinebottom{$h'$}} into its left and right half constituents, and combine those fragments in order.
       {\sc Attach-Left} does the first part, i.e., collects the right constituent \tikz[baseline=-20pt]{\righttriangleinlinebottom{$h'$}}.
       The resulting trapezoid \tikz[baseline=-20pt]{\lefttrapeinline{$h'$}{$h$}} represents an intermediate parsing state, which means the recognition of the right half part of $h'$ has finished while the remaining left part yet unfinished.
       {\sc Complete-Left} does the second part and collects the remaining left constituent \tikz[baseline=-20pt]{\lefttriangleinlinebottom{$h'$}}.
       {\sc Attach-Right} and {\sc Complete-Right} do the opposite operations and collect the right dependents of some head.
 \item On this process, the state $F$ in both left and right constituents ensure that they can be a dependent of others.
 \item {\sc Start-Left} and {\sc Start-Right} correspond to the terminal rules of the ordinary CKY algorithm (Figure \ref{fig:2:cky})
       though we segment it into the left and right parts.
       Note that the root symbol $\$$ at the $n+1$ position only applies {\sc Start-Left} because it must not have any right dependents.
       Commonly the left automaton $L_{\$}$ is designed to have only one dependent; otherwise, the algorithm may allow the fragmental parses with more than one root tokens.
 \item Differently from the inference rules in Figure \ref{fig:2:binary-naive-sbg}, we put the state transitions, e.g., $q \xmapsto{~h'~} r \in R_h$ as antecedent items (above ---) of each rule instead of the side condition.
       These modifications are to make the weight calculation at each rule more explicit.
       Specifically, when we develop a model in this framework, each state transition, i.e., $q\in \textit{init}(L_h)$, $q \in \textit{final}(L_h)$, and $q \xmapsto{~h'~} r \in L_h$ (or for $R_h$) has an associated weight.
       Also when we run the CKY or the related algorithm, each chart cell that corresponds to some constituent (triangle or trapezoid) has a weight.
       Thus, this formulation makes explicit that the weight of the consequent item (below ---) is obtained by the product of all weights of the antecedent items (above ---).
       We describe the particular parameterization of these transitions to achieve DMV in Section \ref{sec:2:dmv}.
 \item The grammar is not in CNF since it contains unary rules at internal positions.
       The inside-outside algorithm can still be applied by assuming null element (which has weight 1) in either child position in Algorithm \ref{alg:cky}.
\end{itemize}

There is no spurious ambiguity.
However, again this grammar is not always a PCFG.
In particular, the grammar for a dependency model with valence (DMV), which we describe next, is not a PCFG.
See the {\sc Finish-Left} rule in the algorithm.
A particular model such as DMV may associate a score for this rule to explicitly model an event that a head $h$ stops generating its left dependents.
In such cases, the weights for CFG rules $(F,h) \rightarrow (q,h)$ do not define a correct (normalized) distribution given the parent symbol $(F,h)$.
This type of inconsistency happens due to discrepancy between the underlying parsing strategies in two representations:
The PCFGs assume the tree generation is a top-down process while the SBGs assume it is bottom-up.
Nevertheless we can use the inside-outside algorithm as in PCFGs because there is no spurious ambiguity and each derivation in a CFG parse correctly gives a probability that the original SBG would give to the corresponding dependency tree.

Also time complexity is improved to $O(n^3)$.
This is easily verified since there appear at most three indexes on each rule.
The reason of this reduction is we no longer use full constituents with a head index \tikz[baseline=-20pt]{\headtriangleinlinebottom{$h$}}, which itself consumes three indexes, leading to an asymptotically inefficient algorithm.

\subsection{Dependency model with valence}
\label{sec:2:dmv}

\begin{figure}[t]
 \centering
 \begin{tabular}[t]{ll}\hline
  Transition & Weight (DMV parameters) \\ \hline
  $q_0 \in \textit{init}(L_h)$ & 1.0 \\  
  $q_0 \in \textit{final}(L_h)$ & $\theta_{\textsc{s}}(\textsc{stop}|h,\leftarrow,\textsc{true})$ \\
  $q_1 \in \textit{final}(L_h)$ & $\theta_{\textsc{s}}(\textsc{stop}|h,\leftarrow,\textsc{false})$ \\
  $q_0 \xmapsto{~d~} q_1 \in L_h$ & $\theta_{\textsc{a}}(d|h, \leftarrow) \cdot \theta_{\textsc{s}}(\neg \textsc{stop}|h,\leftarrow,\textsc{true})$ \\
  $q_1 \xmapsto{~d~} q_1 \in L_h$ & $\theta_{\textsc{a}}(d|h, \leftarrow) \cdot \theta_{\textsc{s}}(\neg \textsc{stop}|h,\leftarrow,\textsc{false})$ \\\hline
 \end{tabular}
 \caption{Mappings between FSA transitions of SBGs and the weights to achieve DMV.
 $\theta_\textit{s}$ and $\theta_\textit{a}$ are parameters of DMV described in the body.
 The right cases (e.g., $q_0 \in \textit{init}(R_a)$) are omitted but defined similary.
 $h$ and $d$ are both word types, not indexes in a sentence (contrary to Figure \ref{fig:bg:ded-sbg}).
 }
 \label{fig:bg:dmv-param-as-sbg}
\end{figure}

Now it is not so hard to formulate the famous model, dependency model with valence (DMV), on which our unsupervised model is based, as a special instance of SBGs.
This can be done by defining transitions of each automaton as well as the associated weights.
In DMV, each $L_h$ or $R_h$ given head $h$ has only two states $q_0$ and $q_1$, both of which are in finished states, i.e., {\sc Finish-Left} and {\sc Finish-Right} in Figure \ref{fig:bg:ded-sbg} can always be applied.
$q_0$ is the initial state and $q_0 \xmapsto{\textit{~h~}} q_1$ while $q_1 \xmapsto{\textit{~h~}} q_1$, meaning that we only distinguish the generation process of the first dependent from others.

The associated weights for transitions in Figure \ref{fig:bg:ded-sbg} are summarized in Figure \ref{fig:bg:dmv-param-as-sbg}.
Each weight is a product of DMV parameters, which are classified into two types of multinomial distributions $\theta_\textsc{s}$ and $\theta_\textsc{a}$.
Generally we write $\theta_\textsc{type}(d|c)$ for a multinomial parameter in which {\sc type} defines a type of multinomial, $c$ is a conditioning context, and $d$ is a decision given the context.
DMV has the following two types of parameters:
\begin{itemize}
 \item $\theta_\textsc{s}(\textit{stop}|h,\textit{dir},\textit{adj})$:
       A Bernoulli random variable to decide whether or not to attach further dependents in the current direction $\textit{dir} \in \{ \leftarrow, \rightarrow \}$.
       The decision $\textit{stop} \in \{\textsc{stop},\neg \textsc{stop}\}$.
       The adjacency $\textit{adj} \in \{\textsc{true, false}\}$ is the key factor to distinguish the distributions of the first and other dependents.
       It is \textsc{true} if $h$ has no dependent yet in \textit{dir} direction.
 \item $\theta_\textsc{a}(d|h,\textit{dir})$:
       A probability that $d$ is attached as a new dependent of $h$ in \textit{dir} direction.
\end{itemize}

The key behind the success of the DMV was the introduction of the valence factor in stop probabilities \cite{klein-manning:2004:ACL}.
Intuitively, this factor can capture the difference of the expected number of dependents for each head.
For example, in English, a verb typically takes one dependent (subject) in the left direction while several dependents in the right direction.
DMV may capture this difference with a higher value of $\theta_\textsc{s}(\neg \textsc{stop}|h, \leftarrow, \textsc{true})$ and a lower value of $\theta_\textsc{s}(\neg \textsc{stop}|h, \leftarrow, \textsc{false})$.
On the other hand, in the right direction, $\theta_\textsc{s}(\neg \textsc{stop}|h, \rightarrow, \textsc{false})$ might be higher, facilitating to attach several dependents.

\paragraph{Inference}
With the EM algorithm, we try to update parameters $\theta_\textsc{s}$ and $\theta_\textsc{a}$.
This is basically done with the inside-outside algorithm though one complicated point is that some transitions in Figure \ref{fig:bg:dmv-param-as-sbg} are associated with products of parameters, not a single parameter.
This situation contrats with the original inside-outside algorithm for PCFGs where each rule is associated with only a single parameter (e.g., $A\rightarrow \beta$ and $\theta_{A\rightarrow \beta}$).

In this case the update can be done by first collecting the expected counts of each transition in a SBG, and then converting it to the expected counts of a DMV parameter.
For example, let $e_x(\textsc{Attach-Left},q,h,d|\theta)$ be the expected counts of the {\sc Attach-Left} rule between head $h$ with state $q$ and dependent $h'$ in a sentence $x$.
We can obtain $e_x(h,d,\leftarrow|\theta)$, the expected counts of an attachment parameter of DMV as follows:
\begin{equation}
 e_x(h,d,\leftarrow|\theta) = e_x(\textsc{Attach-Left},q_0,h,d|\theta) + e_x(\textsc{Attach-Left},q_1,h,d|\theta).
\end{equation}
These are then normalized to update the parameters (as in Section \ref{sec:2:em}).
Similary the counts of the non-stop decision $e_x(h,\neg \textsc{stop}, \leftarrow,\textsc{true}|\theta)$, associated with $\theta_{\textsc{s}}(\neg \textsc{stop}|h,\leftarrow,\textsc{true})$, is obtained by:
\begin{equation}
 e_x(h,\neg \textsc{stop},\leftarrow, \textsc{true}|\theta) = \sum_{h'} e_x(\textsc{Attach-Left},q_0,h,h'|\theta),
\end{equation}
where $h'$ is a possible left dependent (word type) of $h$.

\subsection{Log-linear parameterization}
\label{sec:bg:loglinear}

In Chapter \ref{chap:induction}, we build our model based on an extended model of DMV with {\it features}, which we describe in this section.
We call this model {\it featurized DMV}, which first appeared in \newcite{bergkirkpatrick-EtAl:2010:NAACLHLT}.
We use this model since it is relatively a simple extension to DMV (among others) while known to boost the performance well.

The basic idea is that we replace each parameter of the DMV as the following log-linear model:
\begin{equation}
 \theta_\textsc{a}(d|h,\textit{dir}) = \frac{ \exp( \mathbf w^\intercal \mathbf f (d,h,\textit{dir},\textsc{a}) ) }{ \sum_{d'} \exp( \mathbf w^\intercal \mathbf f (d',h,\textit{dir},\textsc{a}) ) },
\end{equation}
where $\mathbf w$ is a weight vector and $\mathbf f (d,h,\textit{dir},\textsc{a})$ is a feature vector for an event that $h$ takes $d$ as a dependent in \textit{dir} direction.
Note that contrary to the more familiar log-linear models in NLP, such as the conditional random fields \cite{Lafferty:2001:CRF:645530.655813,finkel-kleeman-manning:2008:ACLMain}, it does not try to model the whole structure with a single log-linear model.
Such approaches make it possible to exploit more richer global structural features though inference gets more complex and challenging in particular in an unsupervised setting \cite{smith-eisner:2005:ACL,NIPS2014_5344}.

In this model, the features can only be exploited from the conditioning context and decision of each original DMV parameter.
The typical information captured with this method is the back-off structures between parameters.
For example, some feature in $\mathbf f$ is the one ignoring direction, which facilitates sharing of statistical strength of attachments between $h$ and $d$.
\newcite{bergkirkpatrick-EtAl:2010:NAACLHLT} also report that adding back-off features that use the coarse POS tags is effective, e.g.,
ones replacing actual $h$ or $d$ with a coarse indicator, such as whether $h$ belongs to a (coarse) noun category or not, when the original dataset provides finer POS tags (e.g., pronoun or proper noun).

The EM-like procedure can be applied to this model with a little modification, which instead of optimizing parameters $\theta$ directly, optimizes weight vector $\mathbf w$.
The E-step is exactly the same as the original algorithm.
In the M-step, we optimize $\mathbf w$ to increase the marginal log-likelihood (Eq. \ref{eqn:2:emobj}) using the gradient-based optimization method such as L-BFGS \cite{Liu89onthe} with the expected counts obtained from the E-step.
In practice, we optimize the objective with the regularization term to prevent overfitting.

\section{Previous Approaches in Unsupervised Grammar Induction}
\label{sec:2:unsupervised}

This section summarizes what has been done in the sutdy of unsupervised grammar induction in particular in this decade from \newcite{klein-manning:2004:ACL}, which was the first study breaking the simple baseline method in English experiments.
Here we focus on the setting of {\it monolingual} unsupervised parsing, which we first define in Section \ref{sec:bg:task}.
Related approaches utilizing some kind of supervised information, such as semi-supervised learning \cite{haghighi-klein:2006:COLACL} or transfer learning in which existing high quality parsers (or treebanks) for some languages (typically English) are transferred into parsing models of other languages \cite{mcdonald-petrov-hall:2011:EMNLP,naseem-barzilay-globerson:2012:ACL2012,mcdonald-EtAl:2013:Short,tackstrom-mcdonald-nivre:2013:NAACL-HLT} exist.
These approaches typically achieve higher accuracies though we do not touch here.

\subsection{Task setting}
\label{sec:bg:task}

The typical setting of unsupervised grammar induction is summarized as follows:
\begin{itemize}
 \item During training, the model learns its parameters using (unannotated) sentences only.
       Sometimes the model uses external resources, such as Wikipedia articles \cite{marevcek-straka:2013:ACL2013} to exploit some statistics (e.g., n-gram) in large corpora but does not rely on any syntactic annotations.
 \item To remedy the data sparseness (or the learning difficulty), often the model assumes part-of-speech (POS) tags as the input instead of surface forms.\footnote{
       Some work, e.g., \newcite{seginer:2007:ACLMain} does not assume this convention as we describe in Section \ref{sec:bg:unsup-other}.
       }
       This assumption greatly simplifies the problem though it loses much crucial information for disambiguation.
       For example, the model may not be able to disambiguate prepositional phrase (PP) attachments based on semantic cues as {\it supervised} parsers would do.
       Consider two phrases {\it eat sushi with tuna} and {\it eat sushi with chopsticks}.
       The syntactic structures for these two are different, but POS-based models may not distinguish between them as they both look the same under the model, e.g., {\sc verb noun adp noun}; {\sc adp} is an adposition.
       Therefore the main challenge of unsupervised grammar induction is often to acquire more basic structures or the word order, such that an adjective tends to modify a noun.
 \item The POS-based models are further divided into two categories, whether it can or cannot access to the {\it semantics} of each POS tag.
       The example of the former is \newcite{naseem-EtAl:2010:EMNLP}, which utilizes the information, e.g., a verb tend to be the root of a sentence.
       This approach is sometimes called {\it lightly} supervised learning.
       The latter approach, which we call {\it purely} unsupervised learning, does not access to such knowledge.
       In this case, the only necessary input for the model is the clustering of words, not the {\it label} for each cluster.
       This is advantageous in practice since it can be based on the output of some unsupervised POS tagger, which cannot identify the semantics (label) of each induced cluster.
       Though the problem settings are slightly different in two approaches, we discuss both here as it is unknown what kind of prior linguistic knowledge is necessary for learning grammars.
       Note that it is also an ongoing study how to achieve lightly supervised learning from the output of unsupervised POS taggers with a small amount of manual efforts \cite{Bisk:2015:ACLShort}.
\end{itemize}

\paragraph{Evaluation}
The evaluation of unsupervised systems is generally difficult and controversial.
This is particularly true in unsupervised grammar induction.
The common procedure, which most works described below employ, is to compare the system outputs and the gold annotated trees just as in the supervised case.
That is, we evaluate the quality of the system in terms of accuracy measure, which is precision, recall, and F1-score for constituent structures and an attachment score for dependency structures.
This is inherently flawed in some sense mainly because it cannot take into account the variation in the notion of linguistically {\it correct} structures.
For example, some dependency structures, such as coordination structures, are analyzed in several ways (see also Section \ref{sec:corpora:heads});
each of which is {\it correct} under a particular syntactic theory \cite{popel-EtAl:2013:ACL2013} but the current evaluation metric penalizes unless the prediction of the model matches the gold data currently used.
We do not discuss the solution to this problem here.
However, we try to minimize the effect of such variations in our experiments in Chapter \ref{chap:induction}.
See Section \ref{sec:ind:eval} for details.

\subsection{Constituent structure induction}

As we saw in Section \ref{sec:2:em}, the EM algorithm provides an easy way for learning parameters of any PCFGs.
This motivated the researchers to use the EM algorithm for obtaining syntactic trees without human efforts (annotations).
In early such attempts, the main focus for the induced structures has been phrase-structure trees.

However, it has been well known that such EM-based approaches perform poorly to recover the syntactic trees that linguists assume to be correct \cite{books/daglib/0080794,manning99foundations,carl1999}.
The reasons are mainly two-folds:
One is that the EM algorithm is just a hill climbing method so it cannot reach the global optimum solution.
Since the search space of the grammar is highly complex, this local maxima problem is particularly a severe problem;
\newcite{Carroll92twoexperiments} observed that randomly initialized EM algorithms always converge to different grammars, which all are far from the target grammar.
Another crucial problem is the inherent difficulty in the induction of PCFGs.
In the general setting, the fixed structure for the model is just the start symbol and observed terminal symbols.
The problem is that although terminal symbols are the most informative source for learning, that information does not correctly propagate to the higher level in the tree since each nonterminal label here is just an abstract symbol (hidden categorical variable) and has less meaning.
For example, when the model has a rule $y_1 \rightarrow y_2~y_3$ and $y_2$ and $y_3$ dominate some subtrees, $y_1$ dominates a larger constituent but its relevance to the yield (i.e., dominated terminal symbols) sharply diminishes.

For these reasons, so far the only successful PCFG-based constituent structure induction methods are by giving some amount of supervision, e.g., constraints on possible bracketing \cite{pereira-schabes:1992:ACL} and possible rewrite rules \cite{Carroll92twoexperiments}.
\newcite{johnson-griffiths-goldwater:2007:main} reported that the situation does not change with the sampling-based Bayesian inference method.

Non PCFG-based constituent structure induction has been explored since early 2000s with some success.
The common idea behind these approaches is not collapsing each span into the (less meaningful) nonterminal symbols.
\newcite{W01-0713} and \newcite{klein-manning:2002:ACL} are such attempts, in which the model tries to learn whether some yields (n-gram) comprises a constituent or not.
All parameters are connected to terminal symbols and thus the problem in propagating information from the terminal symbols is alleviated.
\newcite{ponvert-baldridge-erk:2011:ACL-HLT2011} present a chunking-based heuristic method to improve the performance in this line of models.
\newcite{seginer:2007:ACLMain} is another successful constituent induction system;
We discuss his method in Section \ref{sec:bg:unsup-other} as it has some relevance to our approach.

\subsection{Dependency grammar induction}
Due to the difficulty in PCFG-based constituent structure induction, most recent works in PCFG induction has focused on dependency as its underlying structures.
The dependency model with valence (DMV) \cite{klein-manning:2004:ACL} that we introduced in Section \ref{sec:2:dmv} is the most successful approach in such dependency-based models.
As we saw, this model can be represented as an instance of weighted CFGs and thus parameter estimation is possible with the EM algorithm.
This makes an extension on both model and inference easier, and results in many extensions on DMV in a decade as we summarize below.

We divide those previous approaches in largely two categories in whether the model relies on light supervision on dependency rules or not (see also Section \ref{sec:bg:task}).
The goal of every study introduced below can be seen to identify the necessary bias or supervision for an unsupervised parser to learn accurate grammars without explicitly annotated corpora.

\paragraph{Purely unsupervised approaches}

Generally, purely unsupervised methods perform worse than the other, lightly supervised approaches \cite{DBLP:conf/aaai/BiskH12}.

We first mention that the success of most works discussed below including the original DMV in \newcite{klein-manning:2004:ACL} rely on a heuristic initialization technique often called the harmonic initializer, which we describe in details in Section \ref{sec:ind:setting}.
Since the EM algorithm is the local search method, it suffers from the local optima problem, meaning that it is sensitive to the initialization.
Intuitively the harmonic initializer initializes the parameters to favor shorter dependencies.
\newcite{gimpel-smith:2012:NAACL-HLT2} reports that DMV {\it without} this initialization performs very badly;
the accuracy on English experiments (Wall Street Journal portion of the Penn treebank) significantly drops from 44.1 to 21.3.
Most works cited below rely on this technique, but some does not;
in that case, we mention it explicitly (e.g., \newcite{marevcek-vzabokrtsky:2012:EMNLP-CoNLL}).

Bayesian modeling and inference are popular approach for enhancing the probabilistic models.
In dependency grammar induction, \newcite{cohen-smith:2009:NAACLHLT09}, \newcite{headdeniii-johnson-mcclosky:2009:NAACLHLT09}, and \newcite{blunsom-cohn:2010:EMNLP} are examples of such approaches.
\newcite{cohen-smith:2009:NAACLHLT09} extend the baseline DMV model with somewhat complex priors called shared logistic normal priors, which enable to tie parameters of related POS tags (e.g., subcategories of nouns) to behave similarly.
This is conceptually similar to the feature-based log-linear model \cite{bergkirkpatrick-EtAl:2010:NAACLHLT} that we introduced in Section \ref{sec:bg:loglinear}.
They employ variational EM for the inference technique.

\newcite{headdeniii-johnson-mcclosky:2009:NAACLHLT09} develop carefully designed Bayesian generative models, which are also estimated with the variational EM.
This model is one of the few examples of a {\it lexicalized} model, i.e., utilizing words (surface forms) in additional to POS tags.
This is a {\it partially} lexicalized model, meaning that the words that appear less than 100 times in the training data is unlexicalized.
Another technique introduced in this paper is random initialization with model selection;
they report that the performance improves by running a few iteration of EM in thousands of randomly initialized models and then picking up one with the highest likelihood.
However, this procedure is too expensive and the later works do not follow it.

\newcite{blunsom-cohn:2010:EMNLP} is one of the current state-of-the-art methods in purely unsupervised approach.
In the shared task at the Workshop on Inducing Linguistic Structure (WILS) \cite{gelling-EtAl:2012:WILS}, it performs competitively to the lightly supervised CCG-based approach \cite{DBLP:conf/aaai/BiskH12} on average across 10 languages.
The model is partially lexicalized as in \newcite{headdeniii-johnson-mcclosky:2009:NAACLHLT09}.
Though the basic model is an extended model of the DMV, they encode the model on Bayesian tree substitution grammars \cite{Cohn:2010:ITG:1756006.1953031}, which enable to model larger tree fragments than the original DMV does.

\newcite{marevcek-vzabokrtsky:2012:EMNLP-CoNLL} and \newcite{marevcek-straka:2013:ACL2013} present methods that learn the grammars using some principle of dependencies, which they call the {\it reducibility} principle.
They argue that phrases that will be the dependent of another token (head) are often {\it reducible}, meaning that the sentence without such phrases is probably still grammatical.
They calculate the reducibility of each POS n-gram using large raw text corpus from Wikipedia articles and develop a model that biases highly reducible sequences to become dependents.
\newcite{marevcek-straka:2013:ACL2013} encode the reducibility on DMV and find that their method is not sensitive to initialization.
This approach nicely exploits the property of heads and dependents and they report the state-of-the-art scores on the datasets of CoNLL shared tasks \cite{buchholz-marsi:2006:CoNLL-X,nivre-EtAl:2007:EMNLP-CoNLL2007}.

Another line of studies on several extensitions or heuristics for improving DMV has been explored by Spitkovsky and colleagues.
For example, \newcite{spitkovsky-EtAl:2010:CONLL} report that sometimes the Viterbi objective instead of the EM objective leads to the better model and \newcite{spitkovsky-alshawi-jurafsky:2010:NAACLHLT} present the heuristics that starts learning from the shorter sentences only and gradually increases the training sentences.
\newcite{spitkovsky-alshawi-jurafsky:2013:EMNLP} is their final method.
While the reported results are impressive (very competitive to \newcite{marevcek-straka:2013:ACL2013}), the method, which tries to avoid the local optima with the combination of several heuristics, e.g., changing the objective, forgetting some previous counts, and changing the training examples, is quite complex and it makes difficult to analyze which component most contributes to attained improvements.

The important study for us is \newcite{smith-eisner-2006-acl-sa}, which explores a kind of {\it structural} bias to favor shorter dependencies.
This becomes one of our baseline models in Chapter \ref{chap:induction}; See section \ref{sec:ind:structural-const} for more details.
We argue however that the motivation in their experiment is slightly different from the other studies cited above and us.
Along with a bias to favor shorter dependencies, they also investigate the technique called {\it structural annealing}, in which the strength of the imposed bias is gradually relaxed.
Note here that by introducing such new techniques, the number of adjustable parameters (i.e., hyperparameters) increases.
\newcite{smith-eisner-2006-acl-sa} choose the best setting of those parameters based on the {\it supervised} model setting, in which the annotated development data is used for choosing the best model.
This is however not the unsupervised learning setting.
In our experiments in Chapter \ref{chap:induction}, we thus do not explore the annealing technique and just compare the effects of different structural biases, that is, the shorter dependency length bias and our proposing bias of limiting center-embedding.

\paragraph{Lightly supervised approaches}

\newcite{naseem-EtAl:2010:EMNLP} is the first model utilizing the light supervision on dependency rules.
The rules are specified declaratively as dependencies between POS tags such as $\textsc{verb} \rightarrow \textsc{noun}$ or $\textsc{noun} \rightarrow \textsc{adj}$.
Then the model parameters are learned with the {\it posterior regularization} technique \cite{journals/jmlr/GanchevGGT10}, which biases the posterior distribution at each EM iteration to put more weights on those specified dependency rules.
\newcite{naseem-EtAl:2010:EMNLP} design 13 universal rules in total, and show state-of-the-art scores across a number of languages.
This is not purely unsupervised approach but it gives an important upper bound on the required knowledge to achieve reasonable accuracies on this task.
For example, \newcite{DBLP:journals/tacl/BiskH13} demonstrate that the competitive scores to them is obtainable with a smaller amount of supervision by casting the model on CCG.

\newcite{sogaard:2012:WILS} present a heuristic method that does {\it not} learn anything, but just build a parse tree deterministically;
For example it always recognizes the left-most verb to be the root word of the sentence.
Although this is extremely simple, S{\o}gaard reports it beats many systems submitted in the WILS shared task \cite{gelling-EtAl:2012:WILS}, suggesting that often such declarative dependency rules alone can capture the basic word order of the language.

\newcite{grave-elhadad:2015:ACL-IJCNLP} is the recently proposed strong system, which also relies on the declarative rules.
Instead of the generative models as in most works cited above, they formulate their model in a discriminative clustering framework \cite{Xu05maximummargin}, with which the objective becomes convex and the optimality is satisfied.
This system becomes our strong baseline in the experiments in Chapter \ref{chap:induction}.
Note that their system relies on in total 12 rules between POS tags.
We explore how the competitive model to this system can be achieved with our structural constraints as well as a smaller amount of supervision.

\subsection{Other approaches}
\label{sec:bg:unsup-other}

There also exist some approaches that do not induce dependency nor constituent structures directly.
Typically for evaluation reasons the induced structure is converted to either form.

\paragraph{Common cover link}
\newcite{seginer:2007:ACLMain} presents his own grammar formalism called the common cover link (CCL), which looks similar to the dependency structure but differs in many points.
For example, in CCL, every link between words is fully connected at every prefix position in the sentence.
His parser and learning algorithm are fully incremental;
He argues that the CCL structure as well as the incremental processing constraint effectively reduces the search space of the model.

This approach may be conceptually similar to our approach in that both try to reduce the search space of the model that comes from the constraint on human sentence processing (incremental left-to-right processing).
However, his model, the grammar formalism (CCL), and the learning method are highly coupled with each other and it makes difficult to separate some component or idea in his framework for other applications.
We instead investigate the effect of our structural constraint as a single component, which is much simpler and the idea can easily be portable to other applications.

Though CCL is similar to dependency, he evaluates the quality of the output on constituent-based bracketing scores by converting the CCL output to the equivalent constituent representation.
We thus do not compare our approach to his method directly in this thesis.

\paragraph{CCG induction}
In the lexicalized grammar formalisms such as CCGs \cite{opac-b1095877}, each nonterminal symbol in a parse tree encodes semantics about syntax and is not an arbitrary symbol unlike previous CFG-based grammar induction approaches.
This observation motivates recent attempts for inducing CCGs with a small amount of supervision.

\newcite{DBLP:conf/aaai/BiskH12} and \newcite{DBLP:journals/tacl/BiskH13} present generative models over CCG trees and demonstrate that it achieves state-of-the-art scores on a number of languages in the WILS shared task dataset.
For evaluation, after getting a CCG derivation, they extract dependencies by reading off predicate and argument (or modifier) structures encoded in CCG categories.
\newcite{bisk-hockenmaier:2015:ACL-IJCNLP} present a model extension and thorough error analysis while \newcite{Bisk:2015:ACLShort} show how the idea can be applied when no identity on POS tags (e.g., whether a word cluster is {\sc verb} or not) is given with a small manual effort.

The key to the success of their approach is in the seed knowledge about category assignments for each input token.
In CCGs or related formalisms, it is known that a parse tree is build almost deterministically if every category for input tokens are assigned correctly \cite{matuzaki:2007,lewis-steedman:2014:EMNLP2014}.
In other words, the most difficult part in those parsers is the assignments of lexical categories, which highly restrict the ambiguity in the remaining parts.
Bisk and Hockenmair efficiently exploit this property of CCG by restricting possible CCG categories on POS tags.
Their seed knowledge is that a sentence root should be a verb or a noun, and a noun should be an argument of a verb.
They encode this knowledge to the model by seeding that only a verb can be assigned category ${\sf S}$ and only a noun can be assigned category ${\sf N}$.
Then, they expand the possible candidate categories for each POS tag in a bootstrap manner.
For example, a POS tag next to a verb is allowed to be assigned category ${\sf S \backslash S}$ or ${\sf S/S}$, and so on.\footnote{The rewrite rules of CCG are defined by a small set of combinatory rules.
For example, the rule ${\sf (S \backslash N)/N~~N \rightarrow S \backslash N}$ is an example of the forward application rule, which can be generally written as ${\sf X/Y~~Y \rightarrow X}$.
The backward application does the opposite: ${\sf Y~~X \backslash Y \rightarrow X}$.
}
Figure \ref{fig:bg:ccg-seed} shows an example of this bootstrapping process, which we borrow from \newcite{DBLP:journals/tacl/BiskH13}.

\begin{figure}[t]
  \centering
 \begin{tabular}{c|c|c|c}
  {\it The} & {\it man} & {\it ate} & {\it quickly} \\
  {\sc dt} & {\sc nns} & {\sc vbd} & {\sc rb} \\ \hline
  ${\sf N/N}$ & $\bm{\mathsf{N}}{\sf ,S/S}$& $\bm{\mathsf{S}}{\sf ,N/N}$&${\sf S \backslash S}$ \\
  ${\sf (S/S)/(S/S)}$ & ${\sf (N \backslash N)/(N \backslash N)}$ & ${\sf S\backslash N}$ & ${\sf (N \backslash N) \backslash (N \backslash N)}$\\
  &${\sf (N/N) \backslash (N/N)}$ & ${\sf (S/S) \backslash(S/S)}$ & \\
  &&${\sf (S \backslash S)/(S \backslash S)}$&\\
 \end{tabular}
 \caption{An example of bootstrapping process for assigning category candidates in CCG induction borrowed from Bisk and Hockenmaier (2013).
 {\sc dt, nns, vbd,} and {\sc rb} are POS tags.
 Bold categories are the initial seed knowledge, which is expanded by allowing the neighbor token to be a modifier.
 }
 \label{fig:bg:ccg-seed}
\end{figure}

After the process, the parameters of the generative model are learned using variational EM.
During this phase, the category spanning the whole sentence is restricted to be ${\sf S}$, or ${\sf N}$ if no verb exists in the sentence.
This mechanism highly restricts the search space and allows efficient learning.

Finally, \newcite{AAAI159835} explore another direction for learning CCG with small supervision.
Unlike Bisk and Hockenmaier's models that are based on gold POS tags, they try to learn the model from surface forms but with an incomplete tag dictionary mapping some words to possible categories.
The essential difference between these two approaches is how to provide the seed knowledge to the model and it is an ongoing research topic (and probably one of the main goal in unsupervised grammar induction) to specify what kind of information should be given to the model and what can be learned from such seed knowledge.

\subsection{Summary}

This section surveyed the previous studies in unsupervised and lightly supervised grammar induction.
As we have seen, dependency is the only structure that can be learned effectively with the well-studied techniques, e.g., PCFGs and the EM algorithm, except CCG, which may have a potential to replace this although the model tends to be inevitably more complex.
For simplicity, our focus in thesis is dependency, but we argue that the success in dependency induction indicates that the idea could be extended to learning of the other grammars, e.g., CCG as well as more basic CFG-based constituent structures.

The key to the success of previous dependency-based approaches can be divided into the following categories:
\begin{description}
 \item[Initialization] The harmonic initializer is known to boost the performance and used in many previous models including \newcite{cohen-smith:2009:NAACLHLT09}, \newcite{bergkirkpatrick-EtAl:2010:NAACLHLT}, and \newcite{blunsom-cohn:2010:EMNLP}.
 \item[Principles of dependency] The reducibility of \newcite{marevcek-vzabokrtsky:2012:EMNLP-CoNLL} and \newcite{marevcek-straka:2013:ACL2013} efficiently exploits the principle property in dependency and thus learning gets more stable.
 \item[Structural bias] \newcite{smith-eisner:2005:ACL} explores the effect of shorter dependency length bias, which is similar to the harmonic initialization but is more explicit.
 \item[Rules on POS tags] \newcite{naseem-EtAl:2010:EMNLP} and \newcite{grave-elhadad:2015:ACL-IJCNLP} shows parameter-based constraints on POS tags can boost the performance.
            \newcite{sogaard:2012:WILS} is the evidence that such POS tag rules are already powerful in themselves to achieve reasonable scores.
\end{description}

The most relevant approach to ours that we present in Chapter \ref{chap:induction} is the structural bias of \newcite{smith-eisner:2005:ACL};
However, as we have mentioned, they combine the technique with annealing and the selection of initialization method, which are tuned with the supervised model selection.
Thus they do not explore the effect of a single structural bias, which is the main interest in our experiments.
As another baseline, we also compare the performance with harmonic initialized models.
The reducibility and rules on POS tags possibly have orthogonal effects to the structural bias.
We will explore a small number of rules and see the combination effects with our structural constraints to get insights on the effect of our constraint when some amount of external supervision is provided.

\chapter{Multilingual Dependency Corpora}
\label{chap:corpora}

Cross-linguality is an important concept in this thesis.
In the following chapters, we explore a syntactic regularities or universals that exist in languages in several ways including a corpus analyses, a supervised parsing study (Chapter \ref{chap:transition}), and an unsupervised parsing study (Chapter \ref{chap:induction}).
All these studies were made possible by recent efforts for the development of multilingual corpora.
This chapter summarizes the properties and statistics of the dataset we use in our experiments.

First, we survey the problem of the ambiguity in the definitions of {\it head} that we noted when introducing dependency grammars in Section \ref{sec:2:dependency}.
This problem is critical for our purpose;
for example, if our unsupervised induction system performs so badly for a particular language, we do not know whether the reason is in the (possibly distinguished) annotation style or the inherent difficulty of that language (see also Section \ref{sec:ind:eval}).
In particular, we describe the duality of head, i.e., {\it function} head and {\it content} head, which is the main source of the reason why there can be several dependency representations for a particular syntactic construction.

We then summarize the characteristics of the treebanks that we use.
The first dataset, CoNLL shared tasks dataset \cite{buchholz-marsi:2006:CoNLL-X,nivre-EtAl:2007:EMNLP-CoNLL2007} is the first large collection of multilingual dependency treebanks (19 languages in total) in the literature, although is just a collection of existing treebanks and lacking annotation consistency across languages.
This dataset thus may not fully adequate for our cross-linguistic studies.
We will introduce this dataset and use it in our experiments mainly because it was our primary dataset in the preliminary version of the current study \cite{noji-miyao:2014:Coling}, which was done when more adequate dataset such as Universal Dependencies \cite{DEMARNEFFE14.1062} were not available.
We use this dataset only for the experiments in Chapter \ref{chap:transition}.
Universal Dependencies (UD) is a recent initiative to develop cross-linguistically consistent treebank annotation for many languages \cite{nivre2015ud}.
We choose this dataset as our primary resource for cross-linguistic experiments since currently it seems the best dataset that keeps the balance between the typological diversity in terms of the number of languages or language families and the annotation consistency.
We finally introduce another recent annotation project called Google universal treebanks \cite{mcdonald-EtAl:2013:Short}.
We use this dataset only for our unsupervised parsing experiments in Chapter \ref{chap:induction} mainly for comparing the performance of our model with the current state-of-the-art systems.
This dataset is a preliminary version of UD, so its data size and consistency is inferior.
We summarize the major differences of approaches in two corpora in Section \ref{sec:corpora:google}.

\section{Heads in Dependency Grammars}
\label{sec:corpora:heads}

\begin{figure}[t]
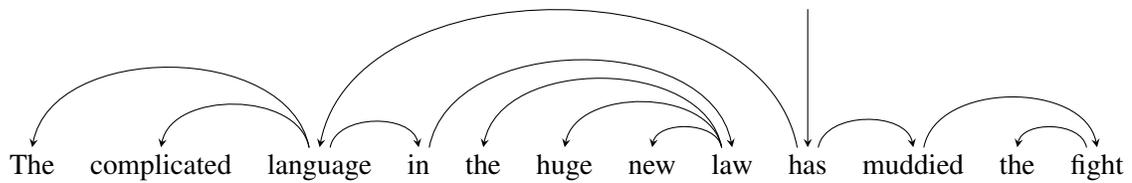
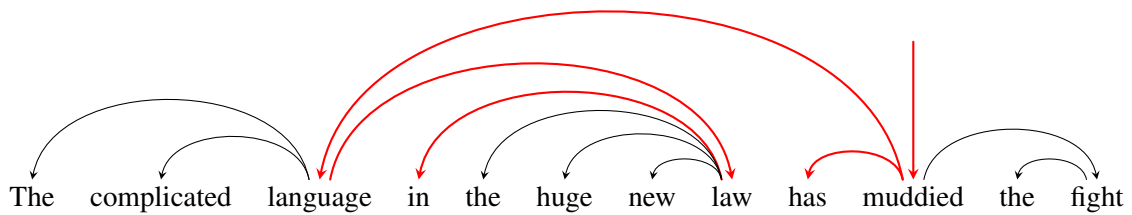
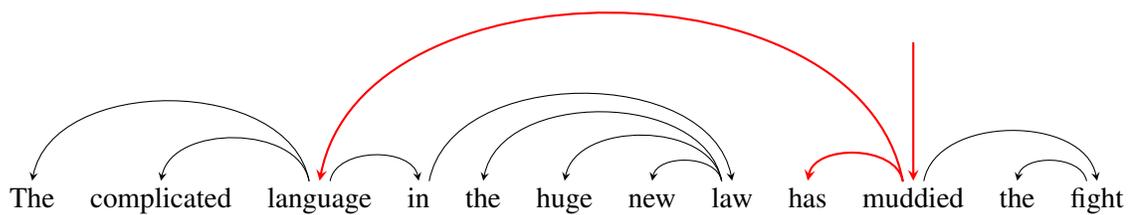

 \begin{minipage}[t]{.99\linewidth}
  \centering
  \begin{dependency} [theme=simple]
  \begin{deptext}[column sep=0.3cm]
   The \& complicated \& language \& in \& the \& huge \& new \& law \& has \& muddied \& the \& fight \\
  \end{deptext}
  \depedge{3}{1}{}
  \depedge{3}{2}{}
  \depedge{3}{4}{}
  \depedge{4}{8}{}
  \depedge{8}{7}{}
  \depedge{8}{6}{}
  \depedge{8}{5}{}
  \depedge{9}{3}{}
  \depedge{9}{10}{}
  \depedge{10}{12}{}
  \depedge{12}{11}{}
  \deproot[edge unit distance=3ex]{9}{}
  \end{dependency}
  \subcaption{Analysis on CoNLL dataset.}
  \label{subfig:corpora:conll-en}
 \end{minipage}
 \begin{minipage}[t]{.99\linewidth}
  \centering 
  \begin{dependency} [theme=simple]
  \begin{deptext}[column sep=0.3cm]
   The \& complicated \& language \& in \& the \& huge \& new \& law \& has \& muddied \& the \& fight \\
  \end{deptext}
  \depedge{3}{1}{}
  \depedge{3}{2}{}
  \depedge[red,thick]{3}{8}{}
  \depedge[red,thick]{8}{4}{}
  \depedge{8}{7}{}
  \depedge{8}{6}{}
  \depedge{8}{5}{}
  \depedge[red,thick]{10}{3}{}
  \depedge[red,thick]{10}{9}{}
  \depedge{10}{12}{}
  \depedge{12}{11}{}
  \deproot[edge unit distance=3ex,red,thick]{10}{}
  \end{dependency}
  \subcaption{Analysis on Stanford universal dependencies (UD).}
  \label{subfig:corpora:ud-en}
 \end{minipage}
 \begin{minipage}[t]{.99\linewidth}
  \centering 
  \begin{dependency} [theme=simple]
  \begin{deptext}[column sep=0.3cm]
   The \& complicated \& language \& in \& the \& huge \& new \& law \& has \& muddied \& the \& fight \\
  \end{deptext}
  \depedge{3}{1}{}
  \depedge{3}{2}{}
  \depedge{3}{4}{}
  \depedge{4}{8}{}
  \depedge{8}{7}{}
  \depedge{8}{6}{}
  \depedge{8}{5}{}
  \depedge[red,thick]{10}{3}{}
  \depedge[red,thick]{10}{9}{}
  \depedge{10}{12}{}
  \depedge{12}{11}{}
  \deproot[edge unit distance=3ex,red,thick]{10}{}
  \end{dependency}
  \subcaption{Analysis on Google universal treebanks.}
  \label{subfig:corpora:google-en}
 \end{minipage}
 \caption{Each dataset that we use employs the different kind of annotation style.
 Bold arcs are ones that do not exist in the CoNLL style tree (a).
 }
 \label{fig:corpora:compare-en}
\end{figure}

Let us first see the examples.
Figure \ref{fig:corpora:compare-en} shows how the analysis of an English sentence would be changed across the datasets we use.
Every analysis is {\it correct} under some linguistic theory.
We can see that two analyses between the CoNLL style (Figure \ref{subfig:corpora:conll-en}) and the UD style (Figure \ref{subfig:corpora:ud-en}) are largely different, in particular around function words (e.g., {\it in} and {\it has}).

\paragraph{Function and content heads}
\newcite{zwicky199313} argues that there is a duality in the notion of heads, namely, function heads and content heads.
In the view of function heads, the head of each constituent is the word that determines the syntactic role of it.
The CoNLL style tree is largely function head-based;
For example, the head in constituent ``in the huge new law'' in Figure \ref{subfig:corpora:conll-en} is ``in'', since this preposition determines the syntactic role of the phrase (i.e., prepositional phrase modifying another noun or verb phrase).
The construction of ``has muddied'' is similar;
In the syntactic view, the auxiliary ``has'' becomes the head since it is this word that determines the aspect of this sentence (present perfect).

\begin{figure}[t]
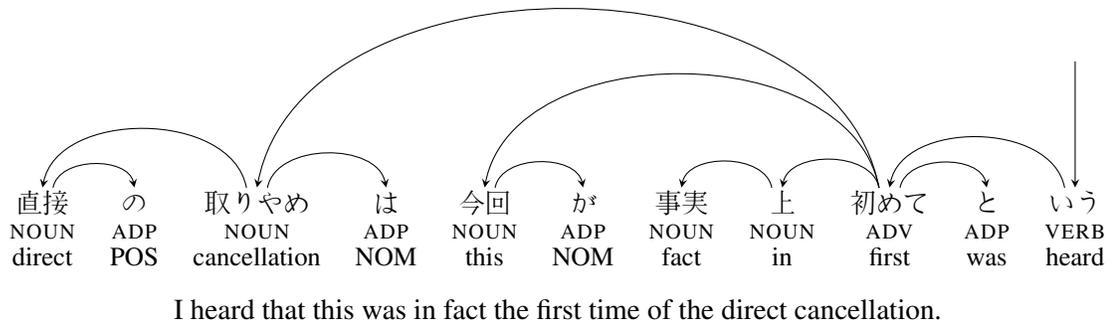

 \centering
 {\small
 \begin{dependency} [theme=simple]
  \begin{deptext}[column sep=0.3cm]
   \Ja{直接} \& \Ja{の} \& \Ja{取りやめ} \& \Ja{は} \& \Ja{今回} \& \Ja{が} \& \Ja{事実} \& \Ja{上} \& \Ja{初めて} \& \Ja{と} \& \Ja{いう} \\
   \textsc{noun} \& \textsc{adp} \& \textsc{noun} \& \textsc{adp} \& \textsc{noun} \& \textsc{adp} \& \textsc{noun} \& \textsc{noun} \& \textsc{adv} \& \textsc{adp} \& \textsc{verb} \\
   \Ja{direct} \& \Ja{POS} \& \Ja{cancellation} \& \Ja{NOM} \& \Ja{this} \& \Ja{NOM} \& \Ja{fact} \& \Ja{in} \& \Ja{first} \& \Ja{was} \& \Ja{heard} \\
  \end{deptext}
  \depedge{3}{1}{}
  \depedge{1}{2}{}
  \depedge{9}{3}{}
  \depedge{3}{4}{}
  \depedge{9}{5}{}
  \depedge{5}{6}{}
  \depedge{8}{7}{}
  \depedge{9}{8}{}
  \depedge{11}{9}{}
  \depedge{9}{10}{}
  \deproot[edge unit distance=3ex]{11}{}
 \end{dependency}
 }
 I heard that this was in fact the first time of the direct cancellation.
 \caption{A dependency tree in the Japanese UD. \textsc{noun, adv, verb,} and \textsc{adp} are assigned POS tags.}
 \label{fig:corpora:ud-jp}
\end{figure}

In another view of content heads, the head of each constituent is selected to be the word that most contributes to the semantics of it.
This is the design followed in the UD scheme \cite{nivre2015ud}.
For example, in Figure \ref{subfig:corpora:ud-en} the head of constituent ``in the huge new law'' is the noun ``law'' instead of the preposition.
Thus, in UD, every dependency arc is basically from a content word (head) to another content or function word (dependent).
Figure \ref{fig:corpora:ud-jp} shows an example of sentence in Japanese treebank of UD.
We can see that every function word (e.g., {\sc adp}) is attached to some content word, such as {\sc noun} and {\sc adv} (adverb).

\paragraph{Other variations}

\begin{figure}[t]
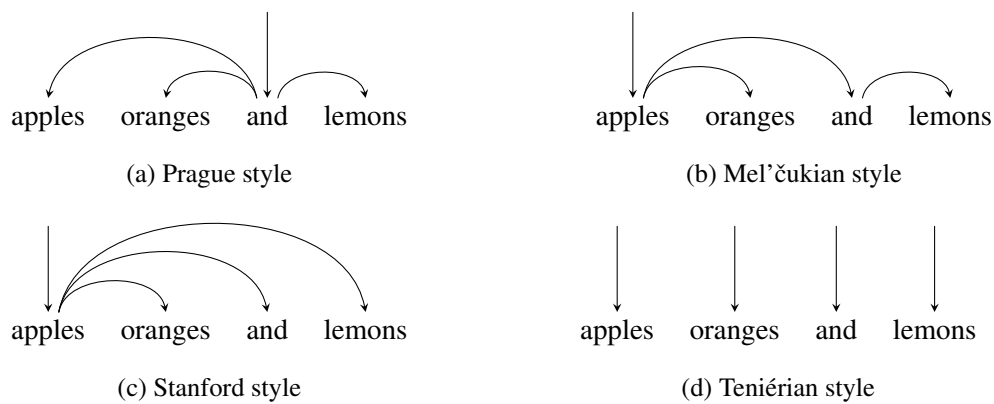

 \begin{minipage}[t]{.49\linewidth}
  \centering
  \begin{dependency} [theme=simple]
  \begin{deptext}[column sep=0.3cm]
   apples \& oranges \& and \& lemons \\
  \end{deptext}
  \depedge{3}{1}{}
  \depedge{3}{2}{}
  \depedge{3}{4}{}
  \deproot[edge unit distance=2ex]{3}{}
  \end{dependency}
  \subcaption{Prague style}
  \label{subfig:corpora:prague-coord}
 \end{minipage}
 \begin{minipage}[t]{.49\linewidth}
  \centering
  \begin{dependency} [theme=simple]
  \begin{deptext}[column sep=0.3cm]
   apples \& oranges \& and \& lemons \\
  \end{deptext}
  \depedge{1}{2}{}
  \depedge{1}{3}{}
  \depedge{3}{4}{}
  \deproot[edge unit distance=2ex]{1}{}
  \end{dependency}
  \subcaption{Mel'\v{c}ukian style}
  \label{subfig:corpora:melcuk-coord}
 \end{minipage}
 \begin{minipage}[t]{.49\linewidth}
  \centering
  \begin{dependency} [theme=simple]
  \begin{deptext}[column sep=0.3cm]
   apples \& oranges \& and \& lemons \\
  \end{deptext}
  \depedge{1}{2}{}
  \depedge{1}{3}{}
  \depedge{1}{4}{}
  \deproot[edge unit distance=2ex]{1}{}
  \end{dependency}
  \subcaption{Stanford style}
  \label{subfig:corpora:stanford-coord}
 \end{minipage}
 \begin{minipage}[t]{.49\linewidth}
  \centering
  \begin{dependency} [theme=simple]
  \begin{deptext}[column sep=0.3cm]
   apples \& oranges \& and \& lemons \\
  \end{deptext}
  \deproot[edge unit distance=2ex]{1}{}
  \deproot[edge unit distance=2ex]{2}{}
  \deproot[edge unit distance=2ex]{3}{}
  \deproot[edge unit distance=2ex]{4}{}
  \end{dependency}
  \subcaption{Teni\'{e}rian style}
  \label{subfig:corpora:tenierian-coord}
 \end{minipage}
 \caption{Four styles of annotation for coordination.}
 \label{fig:corpora:coord}
\end{figure}

Another famous construction that has several variations in analysis is coordination, which is inherently {\it multiple-head} construction and is difficult to deal with in dependency.
\newcite{popel-EtAl:2013:ACL2013} give detailed analysis of coordination structures employed in several existing treebanks.
There are roughly four families of approaches \cite{halmedt} in existing treebanks as shown in Figure \ref{fig:corpora:coord}.
Each annotation style has the following properties:
\begin{description}
 \item[Prague] All conjuncts are headed by the conjunction \cite{cs}.
 \item[Mel'\v{c}ukian] The first/last conjunct is the head, and others are organized in a chain \cite{Melcuk:1988}.
 \item[Stanford] The first conjunct is the head, others are attached directly to it \cite{mm2008stdm}.
 \item[Teni\'{e}rian] There is no common head, and all conjuncts are attached directly to the node modified by the coordination structure \cite{Tesniere1959}.
\end{description}

Note that through our experiments we do not make any claims on which annotation style is the most appropriate for dependency analysis.
In other words, we do not want to commit to a particular linguistic theory.
The main reason why we focus on UD is that it is the dataset with the highest annotation consistency across languages now available, as we describe in the following.

\section{CoNLL Shared Tasks Dataset}

This dataset consists of 19 language treebanks used in the CoNLL shared tasks 2006 \cite{buchholz-marsi:2006:CoNLL-X} and 2007 \cite{nivre-EtAl:2007:EMNLP-CoNLL2007}, in which the task was the multilingual supervised dependency parsing.
See the list of languages and statistics in Table \ref{tab:corpora:conll}.
There is generally no annotation consistency across languages in various constructions.
For example, the four types of coordination annotation styles all appear in this dataset;
Prague style is used in, e.g., Arabic, Czech, and English, while Mel'\v{c}ukian style is found in, e.g., German, Japanese, and Swedish, etc.
Function and content head choices are also mixed across languages as well as the constructions in each language.
For example, in English, the basic style is function head-based while some exceptions are found in e.g., infinitive marker in a verb phrase, such as ``... allow executives to report ...'' in which the head of ``to'' is ``report'' instead of ``allow''.
The idea of regarding a determiner as a head is the extreme of function head-based view \cite{abney1987english,Hudson:2004-01-01T00:00:00:0929-998X:7}, and most treebanks treat a noun as a head while the determiner head is also employed in some treebank, such as Danish.
\newcite{halmedt} gives more detailed survey on the differences of annotation styles in this dataset.

\begin{table}[t]
\centering
\scalebox{1.0}{
\begin{tabular}[t]{l r r r r r r r} \hline
Language & \#Sents. & \multicolumn{4}{c}{\#Tokens} & Punc. & Av. len.\\
         &          & $\leq$10& $\leq$15 & $\leq$20 & $\leq\infty$ & (\%) & \\ \hline
Arabic & 3,043 & 2,833 & 5,193 & 8,656 & 116,793 & 8.3 & 38.3\\
Basque & 3,523 & 7,865 & 19,351 & 31,384 & 55,874 & 18.5 & 15.8\\
Bulgarian & 13,221 & 34,840 & 75,530 & 114,687 & 196,151 & 14.3 & 14.8\\
Catalan & 15,125 & 9,943 & 31,020 & 66,487 & 435,860 & 11.6 & 28.8\\
Chinese & 57,647 & 269,772 & 326,275 & 337,908 & 342,336 & 0.0 & 5.9\\
Czech & 25,650 & 48,452 & 110,516 & 191,635 & 437,020 & 14.7 & 17.0\\
Danish & 5,512 & 10,089 & 24,432 & 40,221 & 100,238 & 13.9 & 18.1\\
Dutch & 13,735 & 40,816 & 75,665 & 110,118 & 200,654 & 11.2 & 14.6\\
English & 18,791 & 13,969 & 47,711 & 106,085 & 451,576 & 12.2 & 24.0\\
German & 39,573 & 66,741 & 164,738 & 292,769 & 705,304 & 13.5 & 17.8\\
Greek & 2,902 & 2,851 & 8,160 & 16,076 & 70,223 & 10.1 & 24.1\\
Hungarian & 6,424 & 8,896 & 23,676 & 42,796 & 139,143 & 15.5 & 21.6\\
Italian & 3,359 & 5,035 & 12,350 & 21,599 & 76,295 & 14.7 & 22.7\\
Japanese & 17,753 & 52,399 & 81,561 & 105,250 & 157,172 & 11.6 & 8.8\\
Portuguese & 9,359 & 13,031 & 30,060 & 54,804 & 212,545 & 14.0 & 22.7\\
Slovene & 1,936 & 4,322 & 9,647 & 15,555 & 35,140 & 18.0 & 18.1\\
Spanish & 3,512 & 3,968 & 9,716 & 18,007 & 95,028 & 12.5 & 27.0\\
Swedish & 11,431 & 20,946 & 55,670 & 96,907 & 197,123 & 10.9 & 17.2\\
Turkish & 5,935 & 21,438 & 34,449 & 44,110 & 69,695 & 16.0 & 11.7\\\hline
\end{tabular}
}
 \caption{Overview of CoNLL dataset (mix of training and test sets).
 Punc. is the ratio of punctuation tokens in a whole corpus.
 Av. len. is the average length of a sentence.}
 \label{tab:corpora:conll}
\end{table}

The dataset consists of the following treebanks.
Note that some languages (Arabic, Chinese, Czech, and Turkish) are used in both 2006 and 2007 shared tasks in different versions; in which case we use only 2007 data.
Also a number of treebanks, such as Basque, Chinese, English, etc, are annotated originally in phrase-structure trees, which are converted to dependency trees with heuristics rules extracting a head token from each constituent.
\begin{description}
 \item[Arabic:] Prague Arabic Dependency Treebank 1.0 \cite{ar}.
 \item[Basque:] 3LB Basque treebank \cite{eu}.
 \item[Bulgarian:] BulTreeBank \cite{bg}.
 \item[Catalan:] The Catalan section of the CESS-ECE Syntactically and Semantically Annotated Corpora \cite{catalan}.
 \item[Chinese:] Sinica treebank \cite{sinica}.
 \item[Czech:] Prague Dependency Treebank 2.0 \cite{cs}.
 \item[Danish:] Danish Dependency Treebank \cite{da}.
 \item[Dutch:] Alpino treebank \cite{nl}.
 \item[English:] The Wall Street Journal portion of the Penn Treebank \cite{Marcus93buildinga}.
 \item[German:] TIGER treebank \cite{de}.
 \item[Greek:] Greek Dependency Treebank \cite{el}.
 \item[Hungarian:] Szeged treebank \cite{hu}.
 \item[Italian:] A subset of the balanced section of the Italian SyntacticSemantic Treebank \cite{it}.
 \item[Japanese:] Japanese Verbmobil treebank \cite{ja}. This is mainly the collection of speech conversations and thus the average length is relatively short.
 \item[Portuguese:] The Bosque part of the Floresta sint\'{a}(c)tica \cite{pt} covering both Brazilian and European Portuguese.
 \item[Slovene:] Slovene Dependency Treebank \cite{sl}.
 \item[Spanish:] Cast3LB \cite{spanish}.
 \item[Swedish:] Talbanken05 \cite{sv}.
 \item[Turkish:] METU-Sabancı Turkish Treebank used in CoNLL 2007 \cite{tr}.
\end{description}

\section{Universal Dependencies}

UD is a collection of treebanks each of which is designed to follow the annotation guideline based on the Stanford typed dependencies \cite{mm2008stdm}, which is in most cases content head-based as we mentioned in Section \ref{sec:corpora:heads}.
We basically use the version 1.1 of this dataset, from which we exclude Finnish-FTB since UD also contains another Finnish treebank, which is larger, and add Japanese, which is included in version 1.2 dataset first.
Typically a treebank is created by first transforming trees in an existing treebank with some script into the trees to follow the annotation guideline, and then manually correcting the errors.

Another characteristic of this dataset is the set of POS tags and dependency labels are consistent across languages.
Appendix \ref{chap:app:ud-pos} summarizes the POS tagset of UD.
We do not discuss dependency labels since we omit them.

Below is the list of sources of treebanks.
We omit the languages if the source is the same as the CoNLL dataset described above.
Note the source of some languages, such as English and Japanese, are changed from the previous dataset.
See Table \ref{tab:corpora:ud} for the list of all 19 languages as well as the statistics.
\begin{description}
 \item[Croatian:] SETimes.HR \cite{hr}.
 \item[English:] English Web \cite{silveira14gold}.
 \item[Finnish:] Turku Dependency Treebank \cite{fi}.
 \item[German:] Google universal treebanks (see Section \ref{sec:corpora:google}).
 \item[Hebrew:] Hebrew Dependency Treebank \cite{he}.
 \item[Indonesian:] Google universal treebanks (see Section \ref{sec:corpora:google}).
 \item[Irish:] Irish Dependency Treebank \cite{ga}.
 \item[Japanese]: Kyoto University Text Corpus 4.0 \cite{Kawahara02constructionof,ja-ud}.
 \item[Persian]: Uppsala Persian Dependency Treebank \cite{faupdt}.
 \item[Spanish]: Google universal treebanks (see Section \ref{sec:corpora:google}).
\end{description}

\begin{table}[t]
\centering
\scalebox{1.0}{
\begin{tabular}[t]{l r r r r r r r} \hline
Language & \#Sents. & \multicolumn{4}{c}{\#Tokens} & Punc. & Av. len.\\
         &          & $\leq$10& $\leq$15 & $\leq$20 & $\leq\infty$ & & \\ \hline
Basque & 5,273 & 19,597 & 38,612 & 51,305 & 60,563 & 17.3 & 11.4\\
Bulgarian & 9,405 & 27,903 & 58,386 & 84,318 & 125,592 & 14.3 & 13.3\\
Croatian & 3,957 & 3,850 & 12,718 & 26,614 & 87,765 & 12.9 & 22.1\\
Czech & 87,913 & 160,930 & 377,994 & 654,559 & 1,506,490 & 14.6 & 17.1\\
Danish & 5,512 & 10,089 & 24,432 & 40,221 & 100,238 & 13.8 & 18.1\\
English & 16,622 & 36,189 & 74,361 & 115,511 & 254,830 & 11.7 & 15.3\\
Finnish & 13,581 & 39,797 & 85,601 & 123,036 & 181,022 & 14.6 & 13.3\\
French & 16,468 & 13,988 & 51,525 & 106,303 & 400,627 & 11.1 & 24.3\\
German & 15,918 & 24,418 & 74,400 & 135,117 & 298,614 & 13.0 & 18.7\\
Greek & 2,411 & 2,229 & 6,707 & 13,493 & 59,156 & 10.6 & 24.5\\
Hebrew & 6,216 & 5,527 & 17,575 & 35,128 & 158,855 & 11.5 & 25.5\\
Hungarian & 1,299 & 1,652 & 5,196 & 9,913 & 26,538 & 14.6 & 20.4\\
Indonesian & 5,593 & 6,890 & 23,009 & 42,749 & 121,923 & 14.9 & 21.7\\
Irish & 1,020 & 1,901 & 3,695 & 6,202 & 23,686 & 10.6 & 23.2\\
Italian & 12,330 & 24,230 & 51,033 & 79,901 & 277,209 & 11.2 & 22.4\\
Japanese & 9,995 & 6,832 & 24,657 & 54,395 & 267,631 & 10.8 & 26.7\\
Persian & 6,000 & 6,808 & 18,011 & 34,191 & 152,918 & 8.7 & 25.4\\
Spanish & 16,006 & 10,489 & 40,087 & 88,665 & 432,651 & 11.0 & 27.0\\
Swedish & 6,026 & 13,045 & 31,343 & 51,333 & 96,819 & 10.7 & 16.0\\\hline
\end{tabular}
 }
 \caption{Overview of UD dataset (mix of train/dev/test sets).
 Punc. is the ratio of punctuation tokens in a whole corpus.
 Av. len. is the average length of a sentence.}
 \label{tab:corpora:ud}
\end{table}

\section{Google Universal Treebanks}
\label{sec:corpora:google}

This dataset is a collection of 12 languages treebanks, i.e., Brazilian-Portuguese, English, Finnish, French, German, Italian, Indonesian, Japanese, Korean, Spanish and Swedish.
Most treebanks are created by hand in this project except the following two languages:
\begin{description}
 \item[English:] Automatically convert from the Wall Street Journal portion of the Penn Treebank \cite{Marcus93buildinga} (with a different conversion method than the CoNLL dataset).
 \item[Swedish:] Talbanken05 \cite{sv} as in CoNLL dataset.
\end{description}

Basically every treebank follow the annotation guideline based on the Stanford typed dependencies as in UD, but contrary to UD, the annotation of Google treebanks is not fully content head-based.
As we show in Figure \ref{subfig:corpora:google-en}, it annotates specific constructions in function head-based, in particular {\sc adp} phrases.

We do not summarize the statistics of this dataset here as we use it only in our experiments in Chapter \ref{chap:induction} where we will see the statistics of the subset of the data that we use (see Section \ref{sec:ind:dataset}).

\chapter{Left-corner Transition-based Dependency Parsing}
\label{chap:transition}

Based on several recipes introduced in Chapter \ref{chap:bg}, we now build a left-corner parsing algorithm operating on dependency grammars.
In this chapter, we formalize the algorithm as a transition system for dependency parsing \cite{Nivre:2008} that roughly corresponds to the dependency version of a push-down automaton (PDA).

We have introduced PDAs with the left-corner parsing strategy for CFGs (Section \ref{sec:bg:left-corner-pda}) as well as a conversion method of any projective dependency trees into an equivalent CFG parse (Section \ref{sec:bilexical}).
Thus one may suspect that it is straightforward to obtain a left-corner parsing algorithm for dependency grammars by, e.g., developing a CFG parser that will build a CFG parse encoding dependency information at each nonterminal symbol.

\REVISE{
In this chapter, however, we take a different approach to build an algorithm in a non-trivial way.
One reason for this is because such a CFG-based approach cannot be an {\it incremental} algorithm.
On the other hand, our algorithm in this chapter is incremental; that is, it can construct a partial parse on the stack, without seeing the future input tokens.
Incrementality is important for assessing parser performance with a comparison to other existing parsing methods, which basically assume incremental processing.
We perform such empirical comparison in Sections \ref{sec:analysis} and \ref{sec:parse}.

For example, let us assume to build a parse in Figure \ref{subfig:abigdog:cfg}, which corresponds to the CFG parse for a dependency tree on ``a big dog''.
To recognize this parse on the left-corner PDA in Section \ref{sec:bg:left-corner-pda}, after shifting token ``a'' (which becomes X[a]), the PDA may covert it to the symbol ``X[dog]/X[dog]''.
However, for creating such a symbol, we have to know that ``dog'' will appear on the remaining inputs at this point, which is impossible in incremental parsing.
This contrasts with the left-corner parser for phrase-structure grammars that we considered in Section \ref{sec:bg:left-corner-pda} in which there is only a finite inventory of nonterminal, which might be predicted.

The algorithm we formalize in this chapter does not introduce such symbols to enable incremental parsing.
We do so by introducing a new concept, a {\it dummy} node, which efficiently abstracts the predicted structure of a subtree in a compact way.
Another important point is since this algorithm directly operates on a dependency tree (not via a CFG form), we can get intuition into how the left-corner parser builds a dependency parse tree.
This becomes important when developing efficient tabulating algorithm with head-splitting (Section \ref{sec:2:sbg}) in Chapter \ref{chap:induction}.
}

\REVISE{
We formally define our algorithm as a {\it transition system}, a stack-based formalization like push-down automata and is the most popular way for obtaining algorithms for dependency grammars \cite{Nivre2003,Yamada03,Nivre:2008,RodriguezNivre2013divisible}.
As we discussed in Section \ref{sec:bg:left-corner-pda}, a left-corner parser can capture the degree of center-embedding of a construction by its stack depth.
Our algorithm preserves this property, and its stack depth increases only when processing dependency structures involving center-embedding.
}

The empirical part of this chapter comprises of two kinds of experiments:
First, we perform a corpus analysis to show that our left-corner algorithm consistently requires less stack depth to recognize annotated trees relative to other algorithms across languages.
The result also suggests the existence of a syntactic universal by which deeper center-embedding is a rare construction across languages, which has not yet been quantitatively examined cross-linguistically.
The second experiment is a supervised parsing experiment, which can be seen as an alternative way to assess the parser's ability to capture important syntactic regularities.
In particular, we will find that the parser using our left-corner algorithm is consistently less sensitive to the decoding constraints of stack depth bound across languages.
Conversely, the performance of other dependency parsers such as the arc-eager parser is largely affected by the same constraints.

\REVISE{
The motivation behind these comparisons is to examine whether the stack depth of a left-corner parser is in fact a meaningful measure to explain the syntactic universal among other alternatives, which would be valuable for other applications such as unsupervised grammar induction that we explore in Chapter \ref{chap:induction}.
}

The first experiment is a {\it static} analysis, which strictly analyzes the observed tree forms in the treebanks, while the second experiment takes {\it parsing errors} into account.
Though the result of the first experiment seems clearer to claim a universal property of language, the result of the second experiment might also be important for real applications.
Specifically we will find that the rate of performance drop with a decoding constraint is smaller than the expected value from the coverage result of the first experiment.
This suggests that a good approximation of the observed syntactic structures in treebanks is available from a highly restricted space if we allow small portion of parse errors.
Since real applications always suffer from parse errors, this result is more appealing for finding a good constraint to restrict the possible tree structures.

This chapter proceeds as follows:
Since our empirical concern is the relative performance of our left-corner algorithm compared to existing transition-based algorithms, we begin the discussion in this chapter with a survey of stack depth behavior in existing algorithms in Section \ref{sec:others}.
This discussion is an extension of a preliminary survey about the incrementality of transition systems by \newcite{nivre:2004:IncrementalParsing}, which is (to our knowledge) the only study discussing how stack elements increase for a particular dependency structures in some algorithm.
Then, in Section \ref{sec:left-corner}, we develop our new transition system that follows a left-corner parsing strategy for dependency grammars and discuss the formal properties of the system, such as the spurious ambiguity of the system and its implications, which are closely relevant to the spurious ambiguity problem we discussed in Section \ref{sec:bilexical}.
The empirical part is devoted to Sections \ref{sec:analysis} and \ref{sec:parse}, focusing on the static corpus analysis and supervised parsing experiments, respectively.
Finally, we give discussion along with the the relevant previous studies in Section \ref{sec:relatedwork} to conclude this chapter.

The preliminary version of this chapter appeared as \newcite{noji-miyao:2015:jnlp}, which was itself an extension of \newcite{noji-miyao:2014:Coling}.
Although these previous versions limited the dataset to the one in the CoNLL shared tasks \cite{buchholz-marsi:2006:CoNLL-X,nivre-EtAl:2007:EMNLP-CoNLL2007}, we add new analysis on Universal dependencies \cite{DEMARNEFFE14.1062} (see also Chapter \ref{chap:corpora}).
The total number of analyzed treebanks is 38 in total across 26 languages.

\section{Notations}
\label{sec:trans:notation}
We first introduce several important concepts and notations used in this chapter.

\paragraph{Transition system}

Every parsing algorithm presented in this chapter can be formally defined as a transition system.
The description below is rather informal; See \newcite{Nivre:2008} for more details.
A transition system is an abstract state machine that processes sentences and produces parse trees.
It has a set of {\it configurations} and a set of {\it transition actions} applied to a configuration.
Each system defines an {\it initial configuration} given an input sentence.
The parsing process proceeds by repeatedly applying an action to the current configuration.
After a finite number of transitions the system arrives at a {\it terminal configuration}, and the dependency tree is read off the terminal configuration.

Formally, each configuration is a tuple $(\sigma,\beta,A)$;
here, $\sigma$ is a stack, and we use a vertical bar to signify the append operation, e.g., $\sigma=\sigma'|\sigma_1$ denotes $\sigma_1$ is the topmost element of stack $\sigma$.
Further, $\beta$ is an input buffer consisting of token indexes that have yet to be processed;
here, $\beta=j|\beta'$ indicates that $j$ is the first element of $\beta$.
Finally, $A \subseteq V_w \times V_w$ is a set of arcs given $V_w$, a set of token indexes for sentence $w$.

\paragraph{Transition-based parser}

We distinguish two similar terms, a transition system and a transition-based parser in this chapter.
A transition system formally characterizes how a tree is constructed via transitions between configurations.
On the other hand, a parser is built on a transition system, and it selects the best {\it action sequence} (i.e., the best parse tree) for an input sentence probably with some scoring model.
Since a transition system abstracts the way of constructing a parse tree, when we mention a {\it parsing algorithm}, it often refers to a transition system, not a parser.
Most of the remaining parts of this chapter is about transition systems, except Section \ref{sec:parse}, in which we compare the performance of several parsers via supervised parsing experiments.

\paragraph{Center-embedded dependency structure}

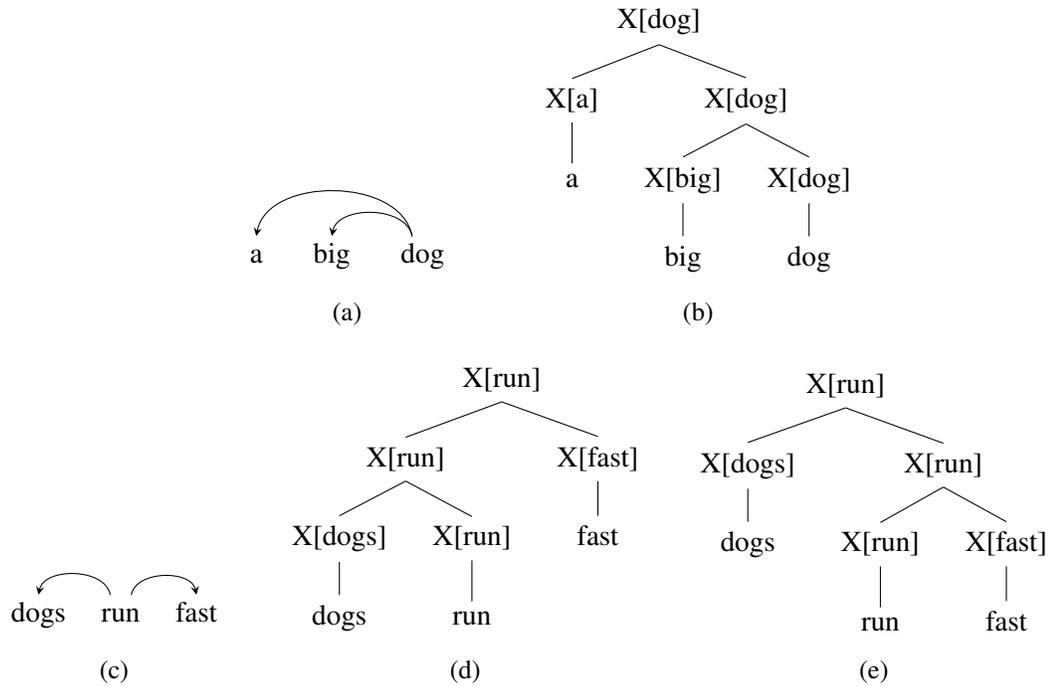
\begin{figure}[t]
 \centering
 \begin{minipage}[t]{.3\linewidth}
  \centering
  \begin{dependency}[theme=simple]
   \begin{deptext}[column sep=0.5cm]
    a \& big \& dog \\
   \end{deptext}
   \depedge{3}{1}{}
   \depedge{3}{2}{}
  \end{dependency}
  \subcaption{}\label{subfig:abigdog}
 \end{minipage}
 \begin{minipage}[t]{.3\linewidth}
  \centering
  \begin{tikzpicture}[sibling distance=10pt]
   \Tree
   [.X[dog]
     [.X[a] a ]
     [.X[dog]
       [.X[big] big ]
       [.X[dog] dog ]
     ]
   ]
  \end{tikzpicture}
  \subcaption{}\label{subfig:abigdog:cfg}
 \end{minipage}

 \vspace{10pt}
 \begin{minipage}[t]{.25\linewidth}
  \centering
  \begin{dependency}[theme=simple]
   \begin{deptext}[column sep=0.3cm]
    dogs \& run \& fast \\
   \end{deptext}
   \depedge{2}{1}{}
   \depedge{2}{3}{}
  \end{dependency}
  \subcaption{}\label{subfig:dogsrunfast}
 \end{minipage}
 \begin{minipage}[t]{.35\linewidth}
  \centering
  \begin{tikzpicture}[sibling distance=10pt]
   \Tree
   [.X[run]
     [.X[run] [.X[dogs] dogs ] [.X[run] run ]
     ]
     [.X[fast] fast ]
   ]
  \end{tikzpicture}
  \subcaption{}\label{subfig:dogsrunfast:cfg1}
 \end{minipage}
 \begin{minipage}[t]{.35\linewidth}
  \centering
  \begin{tikzpicture}[sibling distance=10pt]
   \Tree
   [.X[run]
     [.X[dogs] dogs ]
     [.X[run]
       [.X[run] run ]
       [.X[fast] fast ]
     ]
   ]
  \end{tikzpicture}
  \subcaption{}\label{subfig:dogsrunfast:cfg2}
 \end{minipage}
 \caption{Conversions from dependency trees into CFG parses; (a) can be uniquely converted to (b), while (c) can be converted to both (d) and (e).}\label{fig:reduction}
\end{figure}

The concept of center-embedding introduced in Section \ref{sec:bg:embedding} is originally defined on a constituent structure, or a CFG parse.
Remember that a dependency tree also encodes constituent structures implicitly (see Figure \ref{fig:reduction}) but the conversion from a dependency tree into a CFG parse (in CNF) is not unique, i.e., there is a spurious ambiguity (see Section \ref{sec:bilexical}).
This ambiguity implies there is a subtlety for defining the degree of center-embedding for a dependency structure.

We argue that the tree structure of a given dependency tree (i.e., whether it belongs to center-embedding) cannot be determined by a given tree itself;
We can determine the tree structure of a dependency tree {\it only if} we have some one-to-one conversion method from a dependency tree to a CFG parse.
For example some conversion method may always convert a tree of Figure \ref{subfig:dogsrunfast} into the one of Figure \ref{subfig:dogsrunfast:cfg1}.
In other words, the tree structure of a dependency tree should be discussed along with such a conversion method.
We discuss this subtlety more in Section \ref{sec:oracle}.

We avoid this ambiguity for a while by restricting our attention to the tree structures like Figure \ref{subfig:abigdog} in which we can obtain the corresponding CFG parse uniquely.
For example the dependency tree in Figure \ref{subfig:abigdog} is an example of a {\it right-branching} dependency tree.
Similarly we call a given dependency tree is center-embedding, or left- (right-)branching, depending on the implicit CFG parse when there is no conversion ambiguity.

\section{Stack Depth of Existing Transition Systems}
\label{sec:others}
This section surveys how the stack depth of existing transition systems grows given a variety of dependency structures.
These are used as baseline systems in our experiments in Sections \ref{sec:analysis} and \ref{sec:parse}.

\subsection{Arc-standard}
\label{sec:standard}

\begin{figure}[t]
 \centering
 \begin{minipage}[b]{.24\linewidth}
  \centering
  \begin{dependency}[theme=simple]
   \begin{deptext}[column sep=0.4cm]
    $a$ \& $b$ \& $c$ \\
   \end{deptext}
   \depedge{1}{2}{}
   \depedge{2}{3}{}
  \end{dependency}
  \subcaption{}\label{subfig:right-dep1}
 \end{minipage}
 \begin{minipage}[b]{.24\linewidth}
  \centering
  \begin{dependency}[theme=simple]
   \begin{deptext}[column sep=0.4cm]
    $a$ \& $b$ \& $c$ \\
   \end{deptext}
   \depedge{3}{1}{}
   \depedge{3}{2}{}
  \end{dependency}
  \subcaption{}\label{subfig:right-dep2}
 \end{minipage}
 \begin{minipage}[b]{.24\linewidth}
  \centering
  \begin{dependency}[theme=simple]
   \begin{deptext}[column sep=0.4cm]
    $a$ \& $b$ \& $c$ \\
   \end{deptext}
   \depedge{1}{3}{}
   \depedge{3}{2}{}
  \end{dependency}
  \subcaption{}\label{subfig:right-dep3}
 \end{minipage}
  \begin{minipage}[b]{.24\linewidth}
   \centering
   \begin{tikzpicture}[sibling distance=10pt]
    \tikzset{level distance=25pt}
    \Tree
    [.X [.X[$a$] $a$ ] [.X [.X[$b$] $b$ ] [.X[$c$] $c$ ] ] ]
   \end{tikzpicture}   
   \subcaption{}\label{subfig:right-dep-constituency}
  \end{minipage}
  \caption{(a)-(c) Right-branching dependency trees for three words and (d) the corresponding CFG parse.}
  \label{fig:right-deps}
\end{figure}
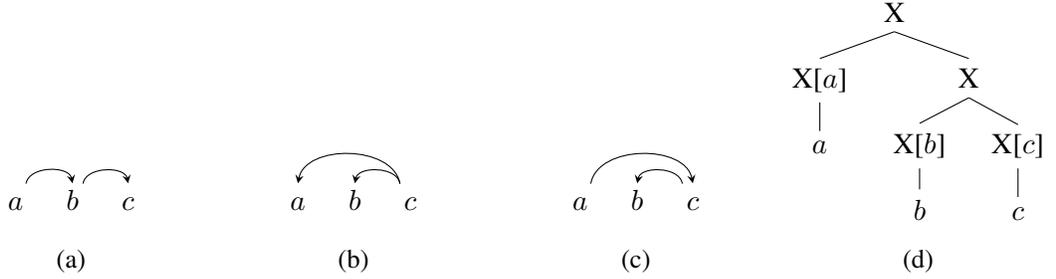

The arc-standard system \cite{nivre:2004:IncrementalParsing} consists of the following three transition actions, with $(h, d)$ representing a dependency arc from $h$ (head) to $d$ (dependent).
\begin{itemize}
 \item {\sc Shift}: $(\sigma,j|\beta,A) \mapsto (\sigma | j, \beta, A)$;
 \item {\sc LeftArc}: $(\sigma|\sigma'_2|\sigma'_1, \beta, A) \mapsto (\sigma|\sigma'_1, \beta, A \cup \{ (\sigma'_1, \sigma'_2) \})$;
 \item {\sc RightArc}: $(\sigma|\sigma'_2|\sigma'_1, \beta, A) \mapsto (\sigma|\sigma'_2, \beta, A \cup \{ (\sigma'_2, \sigma'_1) \})$.
\end{itemize}

We first observe here that the stack depth of the arc-standard system increases linearly for a right-branching structure, such as $a^\curvearrowright b^\curvearrowright c^\curvearrowright \cdots$, in which the system first shifts all words on the stack before connecting each pair of words.
\newcite{nivre:2004:IncrementalParsing} analyzed this system and observed that stack depth grows when processing a dependency tree that becomes right-branching with a CFG conversion.
Figure \ref{fig:right-deps} shows these dependency trees for three words;
the system must construct a subtree of $b$ and $c$ before connecting $a$ to either, thus increasing stack depth.
This occurs because the system builds a tree in a bottom-up manner, i.e., each token collects all dependents before being attached to its head.
The arc-standard system is essentially equivalent to the push-down automaton of a CFG in CNF with a bottom-up strategy \cite{nivre:2004:IncrementalParsing}, so it has the same property as the bottom-up parser for a CFG.
This equivalence also indicates that its stack depth increases for center-embedded structures.

\subsection{Arc-eager}
The arc-eager system \cite{Nivre2003} uses the following four transition actions:
\begin{itemize}
 \item {\sc Shift}: $(\sigma,j|\beta,A) \mapsto (\sigma | j, \beta, A)$;
 \item {\sc LeftArc}: $(\sigma|\sigma'_1, j|\beta, A) \mapsto (\sigma, j|\beta, A \cup \{ (j, \sigma'_1) \})$ ~~~ (if $\neg \exists k, (k,\sigma'_1) \in A$);
 \item {\sc RightArc}: $(\sigma|\sigma'_1, j|\beta, A) \mapsto (\sigma|\sigma'_1|j, \beta, A \cup \{ (\sigma'_1, j) \})$;
 \item {\sc Reduce}: $(\sigma|\sigma'_1,\beta,A) \mapsto (\sigma,\beta,A)$ ~~~ (if $\exists k, (k,\sigma'_1) \in A$).
\end{itemize}
Note that {\sc LeftArc} and {\sc Reduce} are not always applicable.
{\sc LeftArc} requires that $\sigma'_1$ is not a dependent of any other tokens, while {\sc Reduce} requires that $\sigma'_1$ is a dependent of some token (attached to its head).
These conditions originate from the property of the arc-eager system by which each element on the stack may not be disjoint.
In this system, two successive tokens on the stack may be combined with a left-to-right arc, i.e., $a^\curvearrowright b$, thus constituting a {\it connected component}.

\begin{table}[t]
 \centering
 \begin{tabular}[t]{lccc} \hline
                  &Left-branching &Right-branching  &Center-embedding \\ \hline
  Arc-standard    &$O(1)$         &$O(n)$       &$O(n)$ \\
  Arc-eager       &$O(1)$         &$O(1\sim n)$ &$O(1\sim n)$ \\ 
  Left-corner     &$O(1)$         &$O(1)$       &$O(n)$ \\ \hline
 \end{tabular}
 \caption{Order of required stack depth for each structure for each transition system.
 $O(1\sim n)$ means that it recognizes a subset of structures within a constant stack depth but demands linear stack depth for the other structures.
 }
 \label{tab:order}
\end{table}

For this system, we slightly abuse the notation and define stack depth as the number of connected components, not as the number of tokens on the stack, since our concern is the syntactic bias that may be captured with measures on the stack.
With the definition based on the number of tokens on the stack, the arc-eager system would have the same stack depth properties as the arc-standard system.
As we see below, the arc-eager approach has several interesting properties with this modified definition.\footnote{
The stack of the arc-eager system can be seen as the stack of stacks; i.e., each stack element itself is a stack preserving a connected subtree (a right spine).
Our definition of stack depth corresponds to the depth of this stack of stacks.
}

From this definition, unlike the arc-standard system, the arc-eager system recognizes the structure shown in Figure \ref{subfig:right-dep1} and more generally $a^\curvearrowright b^\curvearrowright c^\curvearrowright \cdots$ within constant depth (just one) since it can connect all tokens on the stack with consecutive {\sc RightArc} actions.
More generally, the stack depth of the arc-eager system never increases as long as all dependency arcs are left to right.
This result indicates that the construction of the arc-eager system is no longer purely bottom-up and makes it difficult to formally characterize the stack depth properties based on the tree structure.

We argue two points regarding the stack depth of the arc-eager system.
First, it recognizes a subset of the right-branching structures within a constant depth, as we analyzed above, while increasing stack depth linearly for other right-branching structures, including the trees shown in Figures \ref{subfig:right-dep2} and \ref{subfig:right-dep3}.
Second, it recognizes a subset of the center-embedded structures within a constant depth, such as \tikz[baseline=-.5ex]{
\node at (0.0, 0.0) {$a^\curvearrowright b ^\curvearrowright c ~~ d$,};
\draw[line width=0.35pt,->] (-0.16,0.08) .. controls(-0.09,0.3) and (0.50,0.3) .. (0.57,0.08);
}, which becomes center-embedded when converted to a constituent tree with all arcs left-to-right.
For other center-embedded structures, the stack depth grows linearly as with the arc-standard system.

We summarize the above results in Table \ref{tab:order}.
The left-corner transition system that we propose next has the properties of the third row of the table, and its stack depth grows only on center-embedded dependency structures.

\subsection{Other systems}
All systems in which stack elements cannot be connected have the same properties as the arc-standard system because of their bottom-up constructions including the hybrid system of \newcite{kuhlmann-gomezrodriguez-satta:2011:ACL-HLT2011}.
\newcite{kitagawa-tanakaishii:2010:Short} and \newcite{sartorio-satta-nivre:2013:ACL2013} present an interesting variant that attaches one node to another node that may not be the head of a subtree on the stack.
We do not explore these systems in our experiments because their stack depth essentially has the same properties as the arc-eager system, e.g., their stack depth does not always grow on center-embedded structures, although it grows on some kinds of right-branching structures.

\section{Left-corner Dependency Parsing}
\label{sec:left-corner}

In this section, we develop our dependency transition system with the left-corner strategy.
Our starting point is the push-down automaton for a CFG that we developed in Section \ref{sec:bg:left-corner-pda}.
We will describe how the idea in this automaton can be extended for dependency trees by introducing the concept of {\it dummy nodes} that abstract the prediction mechanism required to achieve the left-corner parsing strategy.

\subsection{Dummy node}

The key characteristic of our transition system is the introduction of a dummy node in a subtree, which is needed to represent a subtree containing predicted structures, such as the symbol $A/B$ in Figure \ref{fig:bg:ourpda}, which predicts an existence of $B$ top-down.
To intuitively understand the parser actions, we present a simulation of transitions for the sentence shown in Figure \ref{subfig:right-dep2} for which all existing systems demand a linear stack depth.
Our system first shifts $a$ and then conducts a {\it prediction} operation that yields subtree \hspace{-3pt}\tikz[baseline=-1.ex]{
\draw[->] (-0.12,-0.08) -- (-0.5,-0.25);
\node at (0,0) {$x$};
\node at (-0.6,-0.3) {$a$};
}, where $x$ is a dummy node.
Here, we predict that $a$ will become a left dependent of an incoming word.
Next, it shifts $b$ to the stack and then conducts a {\it composition} operation to obtain a tree \hspace{-3pt}\tikz[baseline=-1.ex]{
\draw[->] (-0.12,-0.02) -- (-0.6,-0.25);
\draw[->] (-0.03,-0.1) -- (-0.2,-0.25);
\node at (0,0) {$x$};
\node at (-0.7,-0.3) {$a$};
\node at (-0.25,-0.3) {$b$};
}.
Finally, $c$ is inserted into the position of $x$, thus recovering the tree.

\subsection{Transition system}
\label{subsec:transitionsystem}

Our system uses the same notation for a configuration as other systems presented in Section \ref{sec:others}.
Figure \ref{fig:config} shows an example of a configuration in which the $i$-th word in a sentence is written as $w_i$ on the stack.
Each element on the stack is a list representing a right spine of a subtree, which is similar to \newcite{kitagawa-tanakaishii:2010:Short} and \newcite{sartorio-satta-nivre:2013:ACL2013}.
Here, right spine $\sigma_i = \langle \sigma_{i1},\sigma_{i2},\cdots,\sigma_{ik} \rangle$ consists of all nodes in a descending path from the head of $\sigma_i$, i.e., from $\sigma_{i1}$, taking the rightmost child at each step.
We also write $\sigma_i = \sigma'_i | \sigma_{ik}$, meaning that $\sigma_{ik}$ is the rightmost node of spine $\sigma_i$.
Each element of $\sigma_i$ is an index of a token in a sentence or a subtree rooted at a dummy node, $x(\lambda)$, where $\lambda$ is the set of left dependents of $x$.
We state that right spine $\sigma_i$ is {\it complete} if it does not contain any dummy nodes, while $\sigma_i$ is {\it incomplete} if it contains a dummy node.\footnote{\REVISE{$\sigma_i$ with a dummy node corresponds to a stack symbol of the form $A/B$ in the left-corner PDA, which we called incomplete in Section \ref{sec:bg:left-corner-pda}.
Thus, the meaning of these notions (i.e., complete and incomplete) is the same in two algorithms.
The main reason for us to use spine-based notation stems from our use of a dummy node, which postpones the realization of dependency arcs connected to it.
To add arcs appropriately to $A$ when a dummy is filled with a token, it is necessary to keep the surrounding information of the dummy node (this occurs in {\sc Insert} and {\sc RightComp}), which can be naturally traced by remembering each right spine.
}}

\REVISE{
All transition actions in our system are defined in Figure \ref{fig:actions}.
{\sc Insert} is essentially the same as the {\sc Scan} operation in the original left-corner PDA for CFGs (Figure \ref{fig:bg:ourpda}).
Other changes are that we divide {\sc Prediction} and {\sc Composition} into two actions, left and right.
As in the left-corner PDA, by a shift action, we mean {\sc Shift} or {\sc Insert}, while a reduce action means one of prediction and composition actions.
}

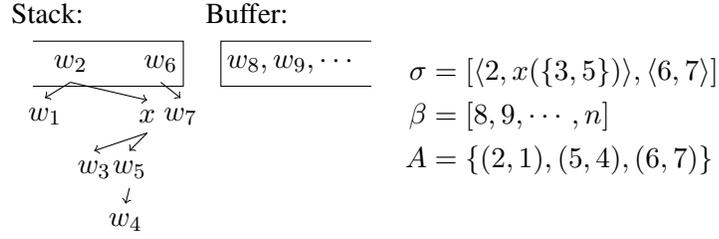
\begin{figure}[t]
 \hspace{4.0cm}Stack: \hspace{1.5cm}Buffer:

 \vspace*{0.2cm}\hspace{4.0cm}
 \begin{tikzpicture}[edge from parent/.style={draw,->},sibling distance=15pt,level distance=20pt]

  \path[draw] (0, 0) -- ++(2, 0) -- ++(0, -0.6) -- ++(-2, 0);
  \node at (0.5, -0.3) {$w_2$}
		[sibling distance=0.68cm]
    child[missing]
		child { node {$w_1$} }
    child[missing]
    child {
			node {$x$}
				[sibling distance=0.47cm]
				child { node {$w_3$} }
				child { node {$w_5$}
          [sibling distance=3pt]
          child { node {$w_4$} }
          child[missing]
        }
				child[missing]
        child[missing]
		};
  \node at (1.7, -0.3) {$w_6$}
    child[missing]
    child{ node {$w_7$} };

  \path[draw] (4.5, 0) -- ++(-2.0, 0) -- ++(0, -0.6) -- ++(2.0, 0);

  \node at (3.45, -0.3) {$w_8,w_9,\cdots$};

  \node at (7,-1.0) {$
  \begin{aligned}
   \sigma &= [\langle 2,x(\{3,5\}) \rangle, \langle 6,7 \rangle ] \\
   \beta &= [8,9,\cdots,n] \\
   A &= \{(2,1),(5,4),(6,7)\}
  \end{aligned}
  $};
 \end{tikzpicture}
 \caption{Example configuration of a left-corner transition system.}
 \label{fig:config}
\end{figure}

\paragraph{Shift Action}
\REVISE{
As in the left-corner PDA, {\sc Shift} moves a token from the top of the buffer to the stack.
{\sc Insert} corresponds to the {\sc Scan} operation of the PDA, and replaces a dummy node on the top of the stack with a token from the top of the buffer.
Note that before doing a shift action, a dummy $x$ can be replaced by any words, meaning that arcs from and to $x$ are unspecified.
This is the key to achieve incremental parsing (see the beginning of this chapter).
It is {\sc Insert} that these arcs are first specified, by filling the dummy node with an actual token.
As in the left-corner PDA, the top element of the stack must be complete after a shift action.
}

\paragraph{Reduce Action}

\REVISE{
As in the left-corner PDA, a reduce action is applied when the top element of the stack is complete, and changes it to an incomplete element.

{\sc LeftPred} and {\sc RightPred} correspond to {\sc Prediction} in the left-corner PDA.
Figure \ref{fig:correspondence} describes the transparent relationships between them.
{\sc LeftPred} makes the head of the top stack element (i.e., $\sigma_{11}$) as a left dependent of a new dummy $x$, while {\sc RightPred} predicts a dummy $x$ as a right dependent of $a$.
In these actions, if we think the original and resulting dependency forms in CFG, the correspondence to {\sc Prediction} in the PDA is apparent.
Specifically, the CFG forms of the resulting trees in both actions are the same.
The only difference is the head label of the parent symbol, which is $x$ in {\sc LeftPred} while $a$ in {\sc RightPred}.
}

\begin{figure}[t!]
 \centering
 \begin{tabular}[t]{|lll|} \hline
  {\sc Shift} &$(\sigma, j|\beta, A) \mapsto (\sigma | \langle j \rangle, \beta, A)$ & \\
  {\sc Insert} &$(\sigma | \langle \sigma'_1 | i | x(\lambda) \rangle), j|\beta, A) \mapsto (\sigma | \langle \sigma'_1 | i | j \rangle, \beta, A \cup \{ (i,j) \} \cup \{ \cup_{k\in\lambda} (j,k) \}$ & \\ \hline
  {\sc LeftPred}&$(\sigma | \langle \sigma_{11}, \cdots \rangle, \beta, A) \mapsto (\sigma | \langle x(\sigma_{11}) \rangle, \beta, A)$ & \\
  {\sc RightPred}&$(\sigma | \langle \sigma_{11}, \cdots \rangle, \beta, A) \mapsto (\sigma | \langle \sigma_{11}, x(\emptyset) \rangle, \beta, A)$ & \\
  {\sc LeftComp}&$(\sigma | \langle \sigma_2'| x(\lambda) \rangle | \langle \sigma_{11},\cdots \rangle, \beta, A) \mapsto (\sigma | \langle \sigma_2'| x(\lambda \cup \{\sigma_{11}\})\rangle, \beta, A)$ & \\
  {\sc RightComp}& $(\sigma | \langle \sigma_2'| x(\lambda) \rangle | \langle \sigma_{11},\cdots \rangle, \beta, A) \mapsto (\sigma | \langle \sigma_2'| \sigma_{11} | x(\emptyset) \rangle, \beta, A \cup \{\cup_{k\in\lambda} (\sigma_{11},k)\})$ & \\ \hline
 \end{tabular}
 \caption{Actions of the left-corner transition system including two shift operations (top) and reduce operations (bottom).}
 \label{fig:actions}
\end{figure}

\begin{figure}[t]
 \centering
 \begin{tabular}[t]{ll}
  {\sc LeftPred:}&{\sc RightPred:} \\
  \parbox{0.4\linewidth}{%
  \hspace{10pt}
  \begin{tikzpicture}[level distance=22pt, sibling distance=20pt,edge from parent/.style={draw,->}]
   \node (a) at (1.0, -0.6) {$a$};
   \path[draw] (0, -1.0) -- +(2.0, 0.0);

   \node at (1.0, -1.4) {$x$}
    child { node {$a$} }
    child[missing];

   \begin{scope}[xshift=4.0cm,yshift=-0.1cm,edge from parent/.style={draw}]
    \Tree[.X[$a$] $a$ ]
   \end{scope}
   \path[draw] (2.5, -1.0) -- +(3.0, 0.0);
   
   \begin{scope}[xshift=4.0cm,yshift=-1.33cm,edge from parent/.style={draw}]
    \Tree[.X[$x$] [.X[$a$] $a$ ] [.X[$x$] ] ]
   \end{scope}
  \end{tikzpicture}  
  }  &{%
  \parbox{0.4\linewidth}{%
  \hspace{10pt}
  \begin{tikzpicture}[level distance=22pt, sibling distance=20pt,edge from parent/.style={draw,->}]
   \node (a) at (1.0, -0.6) {$a$};
   \path[draw] (0, -1.0) -- +(2.0, 0.0);

   \node at (1.0, -1.4) {$a$}
    child[missing]
    child { node {$x$} };

   \begin{scope}[xshift=4.0cm,yshift=-0.1cm,edge from parent/.style={draw}]
    \Tree[.X[$a$] $a$ ]
   \end{scope}
   \path[draw] (2.5, -1.0) -- +(3.0, 0.0);
   
   \begin{scope}[xshift=4.0cm,yshift=-1.33cm,edge from parent/.style={draw}]
    \Tree[.X[$a$] [.X[$a$] $a$ ] [.X[$x$] ] ]
   \end{scope}
  \end{tikzpicture}
  }} \\
  {\sc LeftComp:}&{\sc RightComp:} \\
  \parbox{0.4\linewidth}{%
  \hspace{10pt}
  \begin{tikzpicture}[level distance=20pt, sibling distance=20pt,edge from parent/.style={draw,->}]
   \node (a) at (1.5, -1.0) {$a$};
   \node (b) at (0.5, -0.6) {$b$}
    child[missing]
    child { node {$x$} };
   \path[draw] (0, -1.6) -- +(2.0, 0.0);

   \node at (0.7, -2.0) {$b$}
    [sibling distance=40pt]
    child[missing]
    child {
     node {$x$}
      [sibling distance=20pt]
      child { node {$a$} }
      child[missing] };

   \begin{scope}[xshift=5.5cm,yshift=-0.8cm,edge from parent/.style={draw}]
    \Tree[.X[$a$] $a$ ]
   \end{scope}
   
   \path[draw] (2.5, -1.6) -- +(3.3, 0.0);
   
   \begin{scope}[xshift=3.7cm,yshift=-0.1cm,edge from parent/.style={draw},level distance=20pt]
    \Tree[.X[$b$] [.X[$b$] $b$ ] X[$x$] ]
   \end{scope}
   \begin{scope}[xshift=3.8cm,yshift=-1.93cm,edge from parent/.style={draw},level distance=22pt,sibling distance=3pt]
    \Tree[.X[$b$] [.X[$b$] $b$ ] [.X[$x$] [.X[$a$] $a$ ] X[$x$] ] ]
   \end{scope}
  \end{tikzpicture}  
  }  &{%
  \parbox{0.4\linewidth}{%
  \hspace{10pt}
  \begin{tikzpicture}[level distance=20pt, sibling distance=20pt,edge from parent/.style={draw,->}]
   \node (a) at (1.5, -1.0) {$a$};
   \node (b) at (0.5, -0.6) {$b$}
    child[missing]
    child { node {$x$} };
   \path[draw] (0, -1.6) -- +(2.0, 0.0);

   \node at (0.7, -2.0) {$b$}
    [sibling distance=40pt]
    child[missing]
    child {
     node {$a$}
      [sibling distance=20pt]
      child[missing]
      child { node {$x$} } };

   \begin{scope}[xshift=5.5cm,yshift=-0.8cm,edge from parent/.style={draw}]
    \Tree[.X[$a$] $a$ ]
   \end{scope}
   \path[draw] (2.5, -1.6) -- +(3.3, 0.0);
   
   \begin{scope}[xshift=3.7cm,yshift=-0.1cm,edge from parent/.style={draw},level distance=20pt]
    \Tree[.X[$b$] [.X[$b$] $b$ ] X[$x$] ]
   \end{scope}
   \begin{scope}[xshift=3.8cm,yshift=-1.93cm,edge from parent/.style={draw},level distance=22pt,sibling distance=3pt]
    \Tree[.X[$b$] [.X[$b$] $b$ ] [.X[$a$] [.X[$a$] $a$ ] X[$x$] ] ]
   \end{scope}
  \end{tikzpicture}
  }} \\
 \end{tabular}
 \caption{Correspondences of reduce actions between dependency and CFG.
 We only show minimal example subtrees for simplicity.
 However, $a$ can have an arbitrary number of children, so can $b$ or $x$, provided $x$ is on a right spine and has no right children.}
 \label{fig:correspondence}
\end{figure}
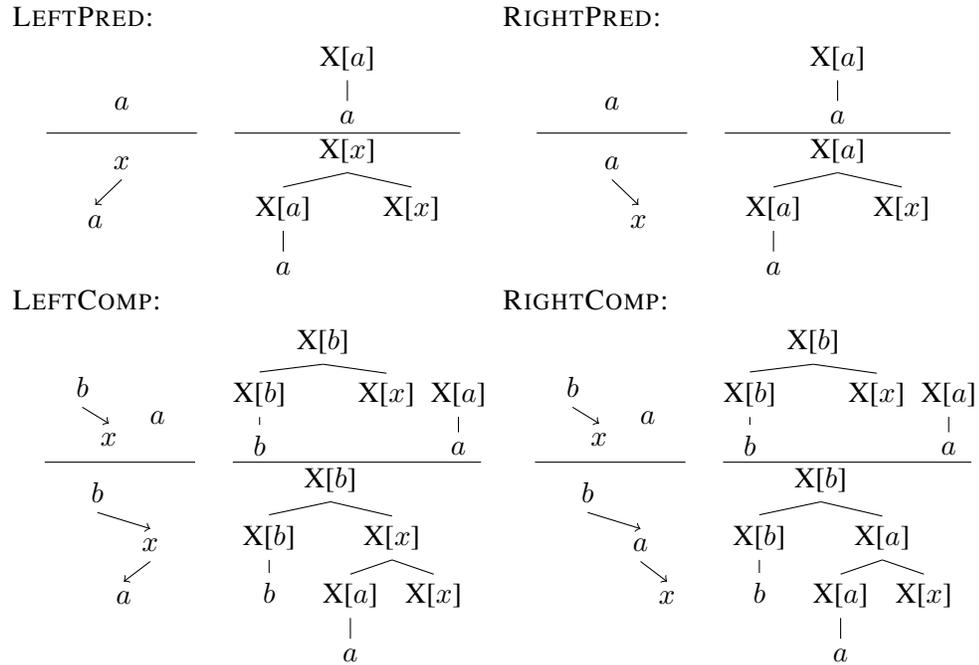

\REVISE{
A similar correspondence holds between {\sc RightComp}, {\sc LeftComp}, and {\sc Composition} in the PDA.
We can interpret {\sc LeftComp} as two step operations as in {\sc Composition} in the PDA (see Section \ref{sec:bg:left-corner-pda}):
It first applies {\sc LeftPred} to the top stack element, resulting in \hspace{-3pt}\tikz[baseline=-1.ex]{
\draw[->] (-0.12,-0.08) -- (-0.5,-0.25);
\node at (0,0) {$x$};
\node at (-0.6,-0.3) {$a$};
}, and then unifies two $x$s to comprise a subtree.
The connection to {\sc Composition} is apparent from the figure.
{\sc RightComp}, on the other hand, first applies {\sc RightPred} to $a$, and then combines the resulting tree and the second top element on the stack.
This step is a bit involved, which might be easier to understand with the CFG form.
On the CFG form, two unified nodes are X[$x$], which is predicted top-down, and X[$a$], which is recognized bottom-up (with the first {\sc RightPred} step).\footnote{If the parent of X[$x$] is also X[$x$], all of them are filled with $a$ recursively.
This situation corresponds to the case in which dummy node $x$ in the dependency tree has multiple left dependents, as in the resulting tree by {\sc LeftComp} in Figure \ref{fig:correspondence}.
}
Since $x$ can be unified with any tokens, at this point, $x$ in X[$x$] is filled with $a$.
Returning to dependency, this means that we insert the subtree rooted at $a$ (after being applied {\sc RightPred}) into the position of $x$ in the second top element.

Note that from Figure \ref{fig:correspondence}, we can see that the dummy node $x$ can only appears in a right spine of a CFG subtree.
Now, we can reinterpret {\sc Insert} action on the CFG subtree, which attaches a token to a (predicted) preterminal X[$x$], as in {\sc Scan} of the PDA, and then fills every $x$ in a right spine with a shifted token.
This can be seen as a kind of unification operation.

\paragraph{Relationship to the left-corner PDA}

As we have seen so far, though our transition system directly operates dependency trees, we can associate every step with a process to expand (partial) CFG parses as in the manner that the left-corner PDA would do.\footnote{
To be precise, we note that this CFG-based expansion process cannot be written in the form used in the left-corner PDA.
For example, if we write items in {\sc LeftComp} and {\sc RightComp} in Figure \ref{fig:correspondence} in the form of $A/B$, both results in the same transition: X[$b$]/X[$x$] X[$a$] $\mapsto$ X[$b$]/X[$x$].
This is due to our use of a dummy node $x$, which plays different roles in two actions (e.g., {\sc RigthComp} assumes the first $x$ is $a$) but the difference is lost with this notation.
We thus claim that a CFG-based expansion step corresponds to a step in the left-corner PDA in that every action in the former expands the tree in the same way as the corresponding action of the left-corner PDA, as explained by Figure \ref{fig:correspondence} (for reduce actions) and the body (for {\sc Insert}); the equivalence of {\sc Shift} is apparent.
}
In every step, the transparency of two tree forms, i.e., dependency and CFG, is always preserved.
This means that at the final configuration the CFG parse would be the one corresponding to the resulting dependency tree, and also at each step the stack depth is identical to the one that is incurred during parsing the final CFG parse with the original left-corner PDA.
We will see this transparency with an example next.
The connection between the stack depth and the degree of center-embedding, that we established in Theorem \ref{thoerem:bg:stack-depth} for the PDA, also directly holds in this transition system.
We restate this for our transition system in Section \ref{sec:memorycost}.

\paragraph{Example}

\begin{figure*}[t]
 \centering
  \begin{tabular}[t]{rllll} \hline
   Step & Action         & Stack ($\sigma$)   & Buffer ($\beta$) & Set of arcs ($A$) \\ \hline
        &                & $\varepsilon$ & $a~b~c~d$ & $\emptyset$ \\
   1    & {\sc Shift}    & $\langle a \rangle$           & $b~c~d$   & $\emptyset$ \\
   2    & {\sc RightPred}& $\langle a, x(\emptyset) \rangle$ & $b~c~d$ & $\emptyset$ \\
   3    & {\sc Shift}    & $\langle a, x(\emptyset) \rangle \langle b \rangle$ & $c~d$ & $\emptyset$ \\
   4    & {\sc RightPred}& {\color{red}{$\langle a, x(\emptyset) \rangle \langle b, x(\emptyset) \rangle$}} & $c~d$ & $\emptyset$ \\
   5    & {\sc Insert}   & $\langle a, x(\emptyset) \rangle \langle b, c \rangle$ & $d$ & ${(b,c)}$ \\
   6    & {\sc RightComp}& $\langle a, b, x(\emptyset) \rangle$ & $d$ & ${(b, c), (a, b)}$ \\
   7    & {\sc Insert}    & $\langle a, b, d \rangle$ &  & ${(b, c), (a, b), (b, d)}$ \\ \hline
 \end{tabular}
 \caption{An example parsing process of the left-corner transition system.}
 \label{fig:trans:example}
\end{figure*}

For an example, Figure \ref{fig:trans:example} shows the transition of configurations during parsing a tree \tikz[baseline=-.5ex]{
\node at (0.0, 0.0) {$a^\curvearrowright b ^\curvearrowright c ~~ d$,};
\draw[line width=0.35pt,->] (-0.16,0.08) .. controls(-0.09,0.3) and (0.50,0.3) .. (0.57,0.08);
} which corresponds to the parse in Figure \ref{fig:2:depth-1} and thus involves one degree of center-embedding.
Comparing to Figure \ref{fig:bg:pda-example}, we can see that two transition sequences for the PDA (for CFG) and the transition system (for dependency) are essentially the same:
the differences are that {\sc Prediction} and {\sc Composition} are changed to the corresponding actions (in this case, {\sc RightPred} and {\sc RigthComp}) and {\sc Scan} is replaced with {\sc Insert}.
This is essentially due to the transparent relationships between them that we discussed above.
As in Figure \ref{fig:bg:pda-example}, the stack depth two after a reduce action indicates center-embedding, which is step 4.
}

\paragraph{Other properties}

As in the left-corner PDA, this transition system also performs shift and reduce actions alternately (the proof is almost identical to the case of PDA).
Also, given a sentence of length $n$, the number of actions required to arrive at the final configuration is $2n-1$, because every token except the last word must be shifted once and reduced once, and the last word is always inserted as the final step.

Every projective dependency tree is derived from at least one transition sequence with this system, i.e., our system is {\it complete} for the class of projective dependency trees \cite{Nivre:2008}.
\REVISE{
Though we omit the proof, this can be shown by appealing to the transparency between the transition system and the left-corner PDA, which is complete for a given CFG.
}

However, our system is {\it unsound} for the class of projective dependency trees, meaning that a transition sequence on a sentence does not always generate a valid projective dependency tree.
We can easily verify this claim with an example.
Let $a~b~c$ be a sentence and consider the action sequence ``{\sc Shift LeftPred Shift LeftPred Insert}'' with which we obtain the terminal configuration of $\sigma= [ x(a), c ]; \beta=[]; A=\{ (b,c) \}$\footnote{For clarity, we use words instead of indices for stack elements.}, but this is not a valid tree.
The arc-eager system also suffers from a similar restriction \cite{NivreAEP14}, which may lead to lower parse accuracy.
Instead of fixing this problem, in our parsing experiment, which is described in Section \ref{sec:parse}, we implement simple post-processing heuristics to combine those fragmented trees that remain on the stack.

\subsection{Oracle and spurious ambiguity}
\label{sec:oracle}
This section presents and analyzes an oracle function for the transition system defined above.
An oracle for a transition system is a function that returns a correct action given the current configuration and a set of gold arcs.
The reasons why we develop and analyze the oracle are mainly two folds:
First, we use this in our empirical corpus study in Section \ref{sec:analysis};
that is, we analyze how stack depth increases during simulation of recovering dependency trees in the treebanks.
Such simulation requires the method to extract the correct sequence of actions to recover the given tree.
Second, we use it to obtain training examples for our supervised parsing experiments in Section \ref{sec:parse}.
This is more typical reason to design the oracles for transition-based parsers \cite{Nivre:2008,GoldbergTDP13}.

Below we also point out the deep connection between the design of an oracle and a conversion process of a dependency into a (binary) CFG parse, which becomes the basis of the discussion in Section \ref{sec:memorycost} on the degree of center-embedding of a given dependency tree, the problem we left in Section \ref{sec:trans:notation}.

Since our system performs shift and reduce actions interchangeably, we need two functions to define the oracle.
Let $A_g$ be a set of arcs in the gold tree and $c$ be the current configuration.
We select the next shift action if the stack is empty (i.e., the initial configuration) or the top element of the stack is incomplete as follows:
\begin{itemize}
 \item {\sc Insert}:
       Let $c=(\sigma| \langle \sigma'_1 | i | x(\lambda) \rangle, j|\beta, A)$. $i$ may not exist.
       The condition is:
       \begin{itemize}
        \item if $i$ exists, $(i,j)\in A_g$ and $j$ has no dependents in $\beta$;
        \item otherwise, $\exists k\in \lambda;(j,k)\in A_g$.
       \end{itemize}
 \item {\sc Shift}: otherwise.
\end{itemize}
If the top element on the stack is complete, we select the next reduce action as follows:
\begin{itemize}
 \item {\sc LeftComp}:
       Let $c=(\sigma| \langle \sigma'_2 | i | x(\lambda) \rangle | \langle \sigma_{11},\cdots \rangle, \beta, A)$. $i$ may not exist.
       Then
       \begin{itemize}
        \item if $i$ exists, $\sigma_{11}$ has no dependents in $\beta$ and $i$'s next dependent is the head of $\sigma_{11}$;
        \item otherwise, $\sigma_{11}$ has no dependents in $\beta$ and $k\in\lambda$ and $\sigma_{11}$ share the same head.
       \end{itemize}
 \item {\sc RightComp}:
       Let $c=(\sigma| \langle \sigma'_2 | i | x(\lambda) \rangle | \langle \sigma_{11},\cdots \rangle, \beta, A)$. $i$ may not exist.
       Then
       \REVISE{
       \begin{itemize}
        \item if $i$ exists, the rightmost dependent of $\sigma_{11}$ is in $\beta$ and $(i,\sigma_{11}) \in A_g$;
        \item otherwise, the rightmost dependent of $\sigma_{11}$ is in $\beta$ and $\exists k\in \lambda,(\sigma_{11},k)\in A_g$.
       \end{itemize}
       }
 \item {\sc RightPred}: if $c=(\sigma| \langle \sigma_{11},\cdots \rangle, \beta, A)$ and $\sigma_{11}$ has a dependent in $\beta$.
 \item {\sc LeftPred}: otherwise.
\end{itemize}
Essentially, each condition ensures that we do not miss any gold arcs by performing the transition.
This is ensured at each step so we can recover the gold tree in the terminal configuration.
We use this oracle in our experiments in Sections \ref{sec:analysis} and \ref{sec:parse}.

\paragraph{Spurious ambiguity}
Next, we observe that the developed oracle above is not a unique function to return the gold action.
Consider sentence $a ^\curvearrowleft b ^\curvearrowright c$, which is a simplification of the sentence shown in Figure \ref{subfig:dogsrunfast}.
If we apply the oracle presented above to this sentence, we obtain the following sequence:
\begin{equation}
 \textsc{Shift} \rightarrow \textsc{LeftPred} \rightarrow \textsc{Insert} \rightarrow \textsc{RightPred} \rightarrow \textsc{Insert} \label{eqn:oracleseq1}
\end{equation}
Note, however, that the following transitions also recover the parse tree:
\begin{equation}
 \textsc{Shift} \rightarrow \textsc{LeftPred} \rightarrow \textsc{Shift} \rightarrow \textsc{RightComp} \rightarrow \textsc{Insert} \label{eqn:oracleseq2}
\end{equation}
This is a kind of spurious ambiguities that we mentioned several times in this thesis (Sections \ref{sec:bilexical}, \ref{sec:2:sbg}, and \ref{sec:trans:notation}).
Although in the transition-based parsing literature some works exist to improve parser performances by utilizing this ambiguity \cite{GoldbergTDP13} or by eliminating it \cite{TACL68}, here we do not discuss such practical problems and instead {\it analyze} the differences in the transitions leading to the same tree.

Here we show that the spurious ambiguity of the transition system introduced above is essentially due to the spurious ambiguity of transforming a dependency tree into a CFG parse (Section \ref{sec:bilexical}).
We can see this by comparing the implicitly recognized CFG parses with the two action sequences above.
In sequence (\ref{eqn:oracleseq1}), {\sc RightPred} is performed at step four, meaning that the recognized CFG parse has the form $((a~b)~c)$, while that of sequence (\ref{eqn:oracleseq2}) is $(a~(b~c))$ due to its {\sc RightComp} operation.
This result indicates an oracle for our left-corner transition system implicitly binarizes a given gold dependency tree.
The particular binarization mechanism associated with the oracle presented above is discussed next.

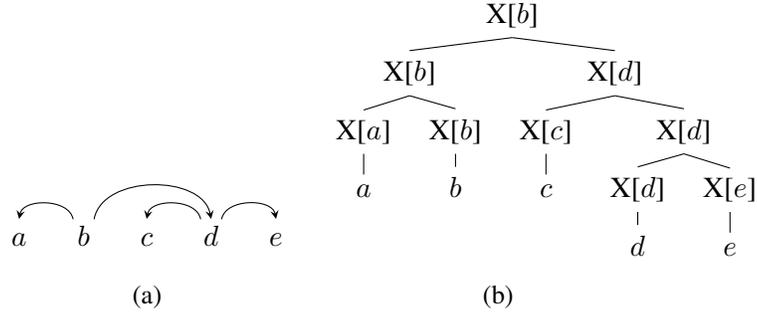
\begin{figure}[t]
 \centering
 \begin{minipage}[b]{.3\linewidth}
  \centering
  \begin{dependency}[theme=simple]
   \begin{deptext}[column sep=0.5cm]
    $a$ \& $b$ \& $c$ \& $d$ \& $e$ \\
   \end{deptext}
   \depedge{2}{1}{}
   \depedge{2}{4}{}
   \depedge{4}{3}{}
   \depedge{4}{5}{}
  \end{dependency}
  \subcaption{}\label{subfig:binarize}
 \end{minipage}
 \begin{minipage}[b]{.3\linewidth}
  \centering
  \begin{tikzpicture}[sibling distance=7pt]
   \tikzset{level distance=22pt}
   \Tree
   [.X[${b}$]
     [.X[${b}$] [.X[$a$] $a$ ] [.X[$b$] $b$ ] ]
     [.X[${d}$]
       [.X[$c$] $c$ ]
       [.X[${d}$] [.X[$d$] $d$ ] [.X[$e$] $e$ ] ]
     ]
   ]
  \end{tikzpicture}
  \subcaption{}\label{subfig:binarize:constituent}
 \end{minipage}
 \caption{Implicit binarization process of the oracle described in the body.}\label{fig:binarize}
\end{figure}

\paragraph{Implicit binarization}
We first note the property of the presented oracle that it follows the strategy of performing composition or insert operations when possible.
As we saw in the given example, sometimes {\sc Insert} and {\sc Shift} can both be valid for recovering the gold arcs, though here we always select {\sc Insert}.
Sometimes the same ambiguity exists between {\sc LeftComp} and {\sc LeftPred} or {\sc RightComp} and {\sc RightPred};
we always prefer composition.

\REVISE{
Then, we can show the following theorem about the binarization mechanism of this oracle.
\begin{mytheorem}
 \label{theorem:trans:binarize}
 The presented oracle above implicitly binarizes a dependency tree in the following manner:
 \begin{itemize}
  \item Given a subtree rooted at $h$, if the parent of it is its {\it right} side, or $h$ is the sentence root, $h$'s left children are constructed first.
  \item If the parent of $h$ is its {\it left} side, $h$'s right children are constructed first.
 \end{itemize} 
\end{mytheorem}
Figure \ref{fig:binarize} shows an example.
For example, since the parent of $d$ is $b$, which is in left, the constituent $d$ $e$ is constructed first.
}

An important observation for showing this is the following lemma about the condition for applying {\sc RightComp}.
\begin{newlemma}
 \label{lemma:trans:rightcomp}
 Let $c=(\sigma| \sigma_2 | \sigma_1, \beta, A)$ and $\sigma_2$ be incomplete (next action is a reduce action).
 Then, in the above oracle, {\sc RightComp} never occurs for a configuration on which $\sigma_2$ is rooted at a dummy, i.e., $\sigma_2 = \langle x(\lambda) \rangle$, or $\sigma_1$ has some left children, i.e., $\exists k < \sigma_{11}, (\sigma_{11},k)$.
\end{newlemma}

\begin{proof}
The first constraint, $\sigma_2 \neq \langle x(\lambda) \rangle$ is shown by simulating how a tree after {\sc RightComp} is created in the presented oracle.
Let us assume $\lambda = \{i\}$ (i.e., $\sigma_2$ looks like $i^\curvearrowleft x$), $j = \sigma_{11}$, and $\sigma_1$ be a subtree spanning from $i+1$ to $k$ (i.e., $i+1 \leq j \leq k$).
After {\sc RightComp}, we get a subtree rooted at $j$, which looks like $i^\curvearrowleft j^\curvearrowright x'$ where $x'$ is a new dummy node.
The oracle instead builds the same structure by the following procedures:
After building $i^\curvearrowleft x$ by {\sc LeftPred} to $i$, it first collect all left children of $x$ by consecutive {\sc LeftComp} actions, followed by {\sc Insert}, to obtain a tree $i^\curvearrowleft j$ (omit $j$'s left dependents between $i$ and $j$).
Then it collects the right children of $j$ (corresponding to $x'$) with {\sc RightPred}s.
This is because we prefer {\sc LeftComp} and {\sc Insert} over {\sc LeftPred} and {\sc Shift}, and suggests that $\sigma_2 \neq \langle x(\lambda) \rangle$ before {\sc RightComp}.
 This simulation also implies the second constraint that $\nexists k < \sigma_{11}, (\sigma_{11},k)$, since it never occurs unless {\sc LeftPred} is preferred over {\sc LeftComp}.
\end{proof}

\begin{proof}[Proof of Theorem \ref{theorem:trans:binarize}]
 Let us examine the first case in which the parent of $h$ is in the right side.
 Let this parent index be $h'$, i.e., $h < h'$.
 Note that this right-to-left arc ($h^\curvearrowleft h'$) are only created with {\sc LeftComp} or {\sc LeftPred} and in both cases $h$ must finish collecting every child before reduced, meaning that $h'$ does not affect the construction of a subtree rooted at $h$.
 This is also the case when $h$ is the sentence root.
 Now, if $h$ collects its right children first, that means $h$ collects left children via {\sc RightComp} with subtree rooted at a dummy node (which is later identified to $h$) but this never occurs by Lemma \ref{lemma:trans:rightcomp}.

 In the second case, $h' < h$.
 The arc $h'^\curvearrowleft h$ is established with {\sc RightPred} or {\sc RightComp} when the head of the top stack symbol is $h'$ (instantiated dummy node is later filled with $h$).
 In both cases, if $h$ collects its left children first, that means a subtree rooted at $h$ (the top stack symbol) with left children is substituted to the dummy node with {\sc RightComp} (see Figure \ref{fig:correspondence}).
 However, this situation is prohibited by Lemma \ref{lemma:trans:rightcomp}.
 The oracle instead collects left children of $h$ with successive {\sc LeftComp}s.
 This occurs dues to the oracle's preference for {\sc LeftComp} over {\sc LeftPred}.
\end{proof}

\paragraph{Other notes}
\begin{itemize}
 \item The property of binarization above also indicates that the designed oracle is optimal in terms of stack depth, i.e., it always minimizes the maximum stack depth for a dependency tree, since it will minimize the number of turning points of the zig-zag path.
 \item If we construct another oracle algorithm, we would have different properties regarding implicit binarization, in which case Lemma \ref{lemma:trans:rightcomp} would not be satisfied.
 \item \REVISE{
       Combining the result in Section \ref{subsec:transitionsystem} about the transparency between the stack depth of the transiton system and the left-corner PDA, it is obvious that at each step of this oracle, the incurred stack depth to recognize a dependency tree equals to the stack depth incurred by the left-corner PDA during recognizing the CFG parse given by the presented binarization.
       }
 \item In Section \ref{sec:analysis}, we use this oracle to evaluate the ability of the left-corner parser to recognize natural language sentences within small stack depth bound.
       Note if our interest is just to examine rarity of center-embedded constructions, that is possible without running the oracle in entire sentences, by just counting the degree of center-embedding of the binarized CFG parses.
       The main reason why we do not employ such method is because our interests are not limited to rarity of center-embedded constructions but also lie in the relative performance of the left-corner parser to capture syntactic regularities among other transition-based parsers, such as the arc-eager parser.
       This comparison seems more meaningful when we run every transition system on the same sentences.
       To make this comparison clearer, we next give more detailed analysis on the stack depth behavior of our left-corner transition system.
\end{itemize}

\subsection{Stack depth of the transition system}
\label{sec:memorycost}
\begin{figure}[t]
 \centering
 \begin{minipage}[b]{.2\linewidth}
  \centering
   \begin{dependency}[theme=simple]
    \begin{deptext}[column sep=0.3cm]
     $a$ \& $b$ \& $c$ \& $d$ \\
    \end{deptext}
   \depedge{1}{4}{}
   \depedge{4}{2}{}
   \depedge{2}{3}{}
   \end{dependency}
   \subcaption{}
  \end{minipage}
  \begin{minipage}[b]{.2\linewidth}
   \centering
   \begin{dependency}[theme=simple]
    \begin{deptext}[column sep=0.3cm]
     $a$ \& $b$ \& $c$ \& $d$ \\
    \end{deptext}
   \depedge{1}{2}{}
   \depedge{2}{3}{}
   \depedge{2}{4}{}
   \end{dependency}
   \subcaption{}
 \end{minipage}
 \begin{minipage}[b]{.28\linewidth}
  \centering
   \begin{dependency}[theme=simple]
    \begin{deptext}[column sep=0.3cm]
     $l$ \& $a$ \& $b$ \& $c$ \& $d$ \& $m$ \\
    \end{deptext}
    \depedge{1}{6}{}
    \depedge{6}{2}{}
    \depedge{2}{5}{}
    \depedge{5}{3}{}
    \depedge{3}{4}{}
   \end{dependency}
   \subcaption{}
  \end{minipage}
  \begin{minipage}[b]{.28\linewidth}
   \centering
   \begin{dependency}[theme=simple]
    \begin{deptext}[column sep=0.3cm]
     $l$ \& $a$ \& $b$ \& $c$ \& $d$ \& $m$ \\
    \end{deptext}
    \depedge{1}{2}{}
    \depedge{2}{3}{}
    \depedge{3}{4}{}
    \depedge{3}{5}{}
    \depedge{2}{6}{}
   \end{dependency}
   \subcaption{}
  \end{minipage}
 
 \begin{minipage}[b]{.4\linewidth}
  \centering
  \begin{tikzpicture}[sibling distance=7pt]
   \tikzset{level distance=20pt}
   \Tree
   [.{\bf X}
     [.X[$a$] $a$ ]
     \edge[line width=1pt];
     [.{\bf X}
       \edge[line width=1pt];
       [.{\bf X}
         [.X[$b$] $b$ ]
         \edge[line width=1pt];
         [.X[${\boldsymbol c}$] $c$ ] ]
       [.X[$d$] $d$ ]
     ]
   ]
  \end{tikzpicture}
  \subcaption{}
 \end{minipage}
 \begin{minipage}[b]{.5\linewidth}
  \centering
  \begin{tikzpicture}[sibling distance=7pt]
   \tikzset{level distance=20pt}
   \Tree
   [.{\bf X}
     [.X[$l$] $l$ ]
     \edge[line width=1pt];
     [.{\bf X}
       \edge[line width=1pt];
       [.{\bf X}
         [.X[$a$] $a$ ]
         \edge[line width=1pt];
         [.{\bf X}
           \edge[line width=1pt];
           [.{\bf X}
             [.X[$b$] $b$ ]
             \edge[line width=1pt];
             [.X[${\boldsymbol c}$] $c$ ]
           ]
           [.X[$d$] $d$ ]
         ]
       ]
       [.X[$m$] $m$ ]
     ]
   ]
  \end{tikzpicture}
  \subcaption{}
 \end{minipage}
 \caption{Center-embedded dependency trees and zig-zag patterns observed in the implicit CFG parses: (a)--(b) depth one, (c)--(d) depth two, (e) CFG parse for (a) and (b), and (f) CFG parse for (c) and (d).}
 \label{fig:level}
\end{figure}
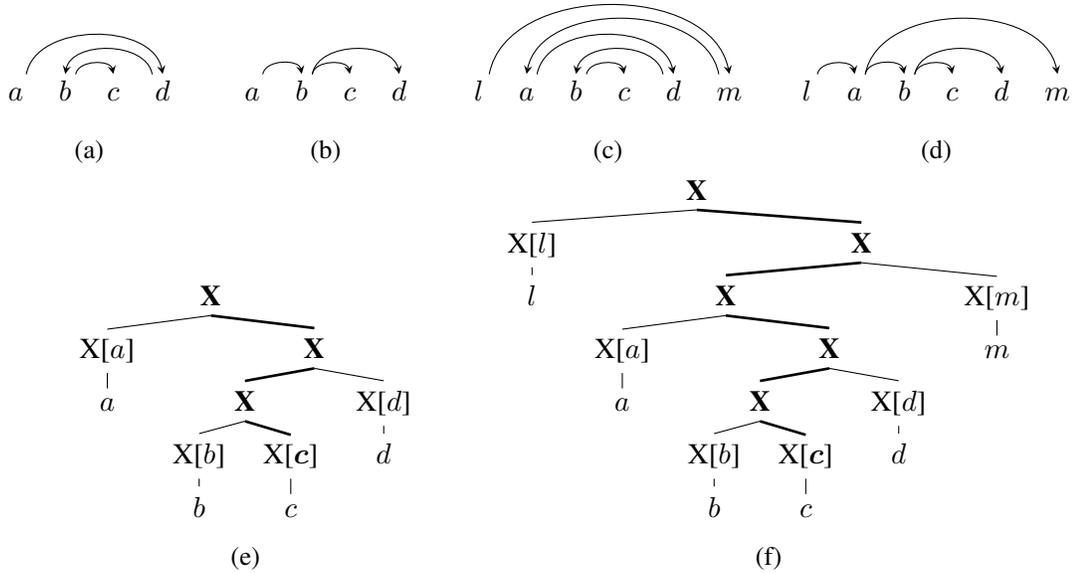

We finally summarize the property of the left-corner transition system in terms of the stack depth.
To do so, let us first introduce two measure, depth$_{re}$ and depth$_{sh}$, with the former representing the stack depth after a reduce step and the latter representing the stack depth after a shift step.
Then, we have:
\begin{itemize}
 \item Depth$_{re} \leq 1$ unless the implicit CFG parse does not contain center-embedding (i.e., is just left-linear or right-linear).
       This linearly increases as the degree of center-embedding increases.
 \item Depth$_{sh} \leq 2$ if the implicit CFG parse does not contain center-embedding.
       The extra element on the stack occurs with a {\sc Shift} action, but it does not imply the existence of center-embedding.
       This linearly increases as the degree of center-embedding increases.
\end{itemize}
The first statement about depth$_{re}$ directly comes from Theorem \ref{thoerem:bg:stack-depth} for the left-corner PDA.
The second statement is about depth$_{sh}$, which we did not touch for the PDA.
Figure \ref{fig:level} shows examples of how the depth of center-embedding increases, with the distinguished zig-zag patterns in center-embedded structures shown in bold.
Note that depth$_{re}$ can capture the degree of center-embedding correctly, by $\max \textrm{depth}_{re} -1$ (Theorem \ref{thoerem:bg:stack-depth}), while depth$_{sh}$ may not;
for example, for parsing a right-branching structure $a^\curvearrowright b^\curvearrowright c$, $b$ must be {\sc Shift}ed (not inserted) before being reduced, resulting in depth$_{sh} = 2$.
We do not precisely discuss the condition with which an extra factor of depth$_{sh}$ occurs.
Of importance here is that both depth$_{re}$ and depth$_{sh}$ increase as the depth of center-embedding in the implicit CFG parse increases, though they may differ only by a constant (just one).

\section{Empirical Stack Depth Analysis}
\label{sec:analysis}

In this section, we evaluate the cross-linguistic coverage of our developed transition system.
We compare our system with other systems by observing the required stack depth as we run oracle transitions for sentences on a set of typologically diverse languages.
We thereby verify the hypothesis that our system consistently demands less stack depth across languages in comparison with other systems.
Note that this claim is not obvious from our theoretical analysis (Table \ref{tab:order}) since the stack depth of the arc-eager system is sometimes smaller than that of the left-corner system (e.g., a subset of center-embedding), which suggests that it may possibly provide a more meaningful measure for capturing the syntactic regularities of a language.

\subsection{Settings}
\label{sec:analysis:settings}
\paragraph{Datasets}
We use two kinds of multilingual corpora introduced in Chapter \ref{chap:corpora}, CoNLL dataset and Universal dependencies (UD), both of which comprises of 19 treebanks.
Below, the first part of analyses in Sections \ref{sec:analysis:general}, \ref{sec:random}, and \ref{sec:token} are performed on CoNLL dataset while the latter analyses in Sections \ref{sec:trans:ud-result} and \ref{sec:trans:relax} are based on UD.

Since all systems presented in this chapter cannot handle nonprojective structures \cite{Nivre:2008}, we projectivize all nonprojective sentences using pseudo-projectivization \cite{nivre-nilsson:2005:ACL} implemented in the MaltParser \cite{NivreMAL07} (see also Section \ref{sec:bg:projective}).
We expect that this modification does not substantially change the overall corpus statistics as nonprojective constructions are relatively rare \cite{nivre-EtAl:2007:EMNLP-CoNLL2007}.
Some treebanks such as the Prague dependency treebanks (including Arabic and Czech) assume that a sentence comprises multiple independent clauses that are connected via a dummy root token.
We place this dummy root node at the end of each sentence, because doing so does not change the behaviors for sentences with a single root token in all systems and improves the parsing accuracy of some systems such as arc-eager across languages as compared with the conventional approach in which the dummy token is placed only at the beginning of each sentence \cite{journals/coling/BallesterosN13}.

\paragraph{Method}
We compare three transition systems: arc-standard, arc-eager, and left-corner.
For each system, we perform oracle transitions for all sentences and languages, measuring stack depth at each configuration.
The arc-eager system sometimes creates a subtree at the beginning of a buffer, in which case we increment stack depth by one.

\paragraph{Oracle}
We run an oracle transition for each sentence with each system.
For the left-corner system, we implemented the algorithm presented in Section \ref{sec:oracle}.
For the arc-standard and arc-eager systems, we implemented oracles preferring reduce actions over shift actions, which minimizes the maximum stack depth.

\subsection{Stack depth for general sentences}
\label{sec:analysis:general}

For each language in CoNLL dataset, we count the number of configurations of a specific stack depth while performing oracles on all sentences.
Figure \ref{fig:load-comparison} shows the cumulative frequencies of configurations as the stack depth increases for the arc-standard, arc-eager, and left-corner systems.
The data answer the question as to which stack depth is required to cover X\% of configurations when recovering all gold trees.
Note that comparing absolute values here is less meaningful since the minimal stack depth to construct an arc is different for each system, e.g., the arc-standard system requires at least two items on the stack, while the arc-eager system can create a right arc if the stack contains one element.
Instead, we focus on the universality of each system's behavior for different languages.

As discussed in Section \ref{sec:standard}, the arc-standard system can only process left-branching structures within a constant stack depth;
such structures are typical in head-final languages such as Japanese or Turkish, and we observe this tendency in the data.
The system performs differently in other languages, so the behavior is not consistent across languages.

The arc-eager and left-corner systems behave similarly for many languages, but we observe that there are some languages for which the left-corner system behaves similarly across numerous languages, while the arc-eager system tends to incur a larger stack depth.
In particular, except Arabic, the left-corner system covers over 90\% (specifically, over 98\%) of configurations with a stack depth $\leq 3$.
The arc-eager system also has 90\% coverage in many languages with a stack depth $\leq 3$, though some exceptions exist, e.g., German, Hungarian, Japanese, Slovene, and Turkish.

We observe that results for Arabic are notably different from other languages.
We suspect that this is because the average length of each sentence is very long (i.e., 39.3 words; see Table \ref{tab:token} for overall corpus statistics).
\newcite{buchholz-marsi:2006:CoNLL-X} noted that the parse unit of the Arabic treebank is not a sentence but a paragraph in which every sentence is combined via a dummy root node.
To remedy this inconsistency of annotation units, we prepared the modified treebank, which we denote as Arabic$^*$ in the figure, by treating each child tree of the root node as a new sentence.\footnote{We removed the resulting sentence if the length was one.}
The results then are closer to other language treebanks, especially Danish, which indicates that the exceptional behavior of Arabic largely originates with the annotation variety.
From this point, we review the results of Arabic$^*$ instead of the original Arabic treebank.

\subsection{Comparing with randomized sentences}
\label{sec:random}
\begin{figure}[p]
\centering
 \resizebox{0.95\textwidth}{!}
 {\includegraphics[]{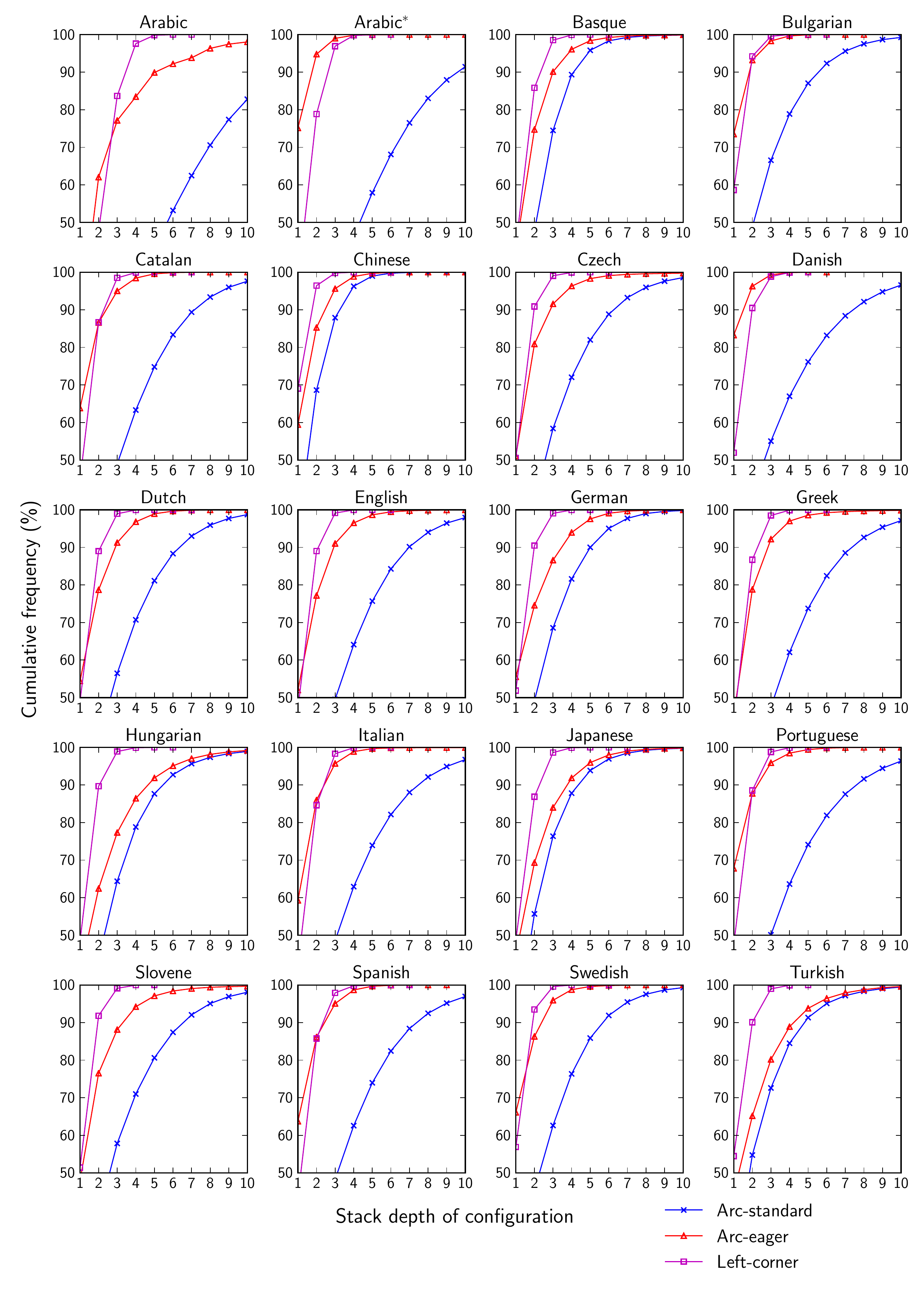}}
\caption{Crosslinguistic comparison of the cumulative frequencies of stack depth during oracle transitions.}
 \label{fig:load-comparison}
\end{figure}

\begin{figure}[p]
\centering
 \resizebox{0.95\textwidth}{!}
 {\includegraphics[]{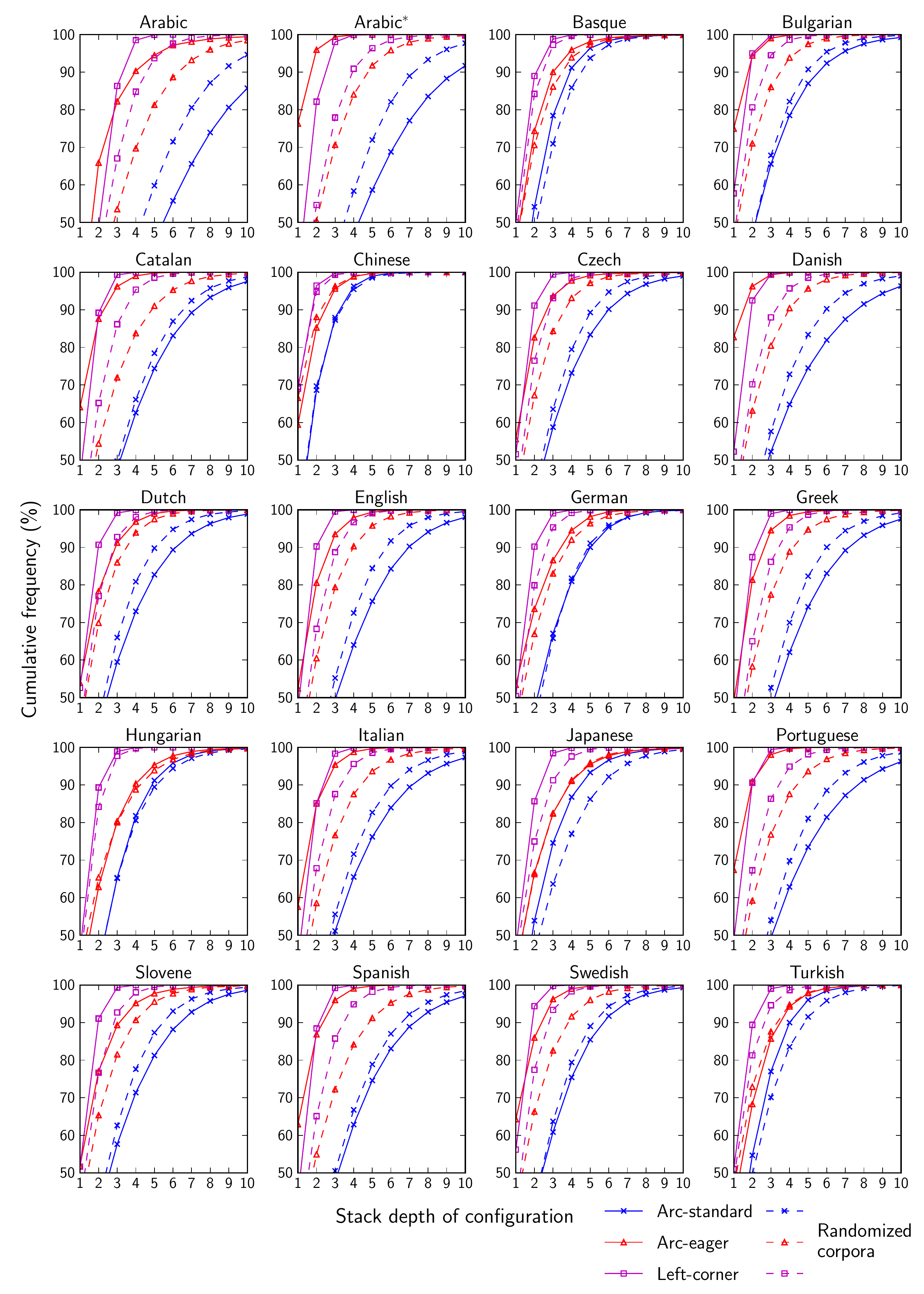}}
\caption{Stack depth results in corpora with punctuation removed; the dashed lines show results on randomly reordered sentences.}
 \label{fig:load-comparison-random}
\end{figure}

The next question we examine is whether the observation from the last experiment, i.e., that the left-corner parser consistently demands less stack depth, holds only for naturally occurring or grammatically correct sentences.
We attempt to answer this question by comparing oracle transitions on original treebank sentences and on (probably) grammatically incorrect sentences.
We create these incorrect sentences using the method proposed by \newcite{gildea-temperley:2007:ACLMain}.
We reorder words in each sentence by first extracting a directed graph from the dependency tree, and then randomly reorder the children of each node while preserving projectivity.
Following \newcite{gildea-temperley:2007:ACLMain}, we remove punctuation from all corpora in this experiment beforehand, since how punctuation is attached to words is not essential.

The dotted lines shown in Figure \ref{fig:load-comparison-random} denote the results of randomized sentences for each system.
There are notable differences in required stack depth between original and random sentences in many languages.
For example, with a stack depth $\leq 3$, the left-corner system cannot reach 90\% of configurations in many randomized treebanks such as Arabic$^*$, Catalan, Danish, English, Greek, Italian, Portuguese, and Spanish.
These results suggest that our system demands less stack depth only for naturally occurring sentences.
For Chinese and Hungarian, the differences are subtle;
however, the differences are also small for the other systems, which implies that these corpora have biases on graphs to reduce the differences.

\subsection{Token-level and sentence-level coverage results}
\label{sec:token}

\begin{table}[p]
\centering
\scalebox{0.90}{
\begin{tabular}[t]{l l r r r r r r r r r r} \hline
& & \multicolumn{1}{r}{{\bf Arabic}}&\multicolumn{1}{r}{{\bf Arabic$^*$}}&\multicolumn{1}{r}{{\bf Basque}}&\multicolumn{1}{r}{{\bf Bulgarian}}&\multicolumn{1}{r}{{\bf Catalan}}&\multicolumn{1}{r}{{\bf Chinese}}&\multicolumn{1}{r}{{\bf Czech}}\\
\multicolumn{2}{l}{\#Sents.} &3,043&4,102&3,523&13,221&15,125&57,647&25,650\\
\multicolumn{2}{l}{Av. len.} &39.3&28.0&16.8&15.8&29.8&6.9&18.0\\ \hline
Token & $\leq 1$& 22.9/21.8&52.8/55.9&57.3/62.7&79.5/80.0&66.2/69.4&83.6/83.6&74.7/74.0\\
      & $\leq 2$& 63.1/65.4&89.6/92.2&92.1/93.3&98.1/98.7&94.8/96.8&98.3/98.3&96.6/97.3\\
      & $\leq 3$& 92.0/94.1&98.9/99.4&99.2/99.3&99.8/99.9&99.5/99.8&99.9/99.9&99.7/99.8\\
      & $\leq 4$& 99.1/99.5&99.9/99.9&99.9/99.9&99.9/99.9&99.9/99.9&99.9/99.9&99.9/99.9\\
\\
Sent. & $\leq 1$& 7.0/7.4&20.8/21.4&15.5/20.8&37.3/39.4&14.7/16.9&58.3/58.3&32.0/34.2\\
      & $\leq 2$& 26.8/27.8&55.4/59.3&69.8/75.8&90.7/93.0&68.3/75.5&95.0/95.0&83.9/86.6\\
      & $\leq 3$& 57.6/61.6&91.7/94.5&95.8/97.0&99.3/99.7&95.6/98.3&99.7/99.7&98.2/99.1\\
      & $\leq 4$& 90.9/94.4&99.5/99.8&99.7/99.8&99.9/99.9&99.7/99.9&99.9/99.9&99.8/99.9\\ \hline
\
& & \multicolumn{1}{r}{{\bf Danish}}&\multicolumn{1}{r}{{\bf Dutch}}&\multicolumn{1}{r}{{\bf English}}&\multicolumn{1}{r}{{\bf German}}&\multicolumn{1}{r}{{\bf Greek}}&\multicolumn{1}{r}{{\bf Hungarian}}&\multicolumn{1}{r}{{\bf Italian}}\\
\multicolumn{2}{l}{\#Sents.} &5,512&13,735&18,791&39,573&2,902&6,424&3,359\\
\multicolumn{2}{l}{Av. len.} &19.1&15.6&25.0&18.8&25.1&22.6&23.7\\ \hline
Token & $\leq 1$& 71.3/75.2&70.2/73.4&69.2/71.3&66.9/66.7&66.7/66.8&65.6/64.1&62.8/64.1\\
      & $\leq 2$& 95.6/97.4&95.9/96.8&96.3/97.5&94.5/94.5&95.2/96.2&95.1/94.9&94.0/94.2\\
      & $\leq 3$& 99.6/99.8&99.7/99.8&99.7/99.9&99.5/99.5&99.6/99.8&99.5/99.5&99.5/99.5\\
      & $\leq 4$& 99.9/99.9&99.9/99.9&99.9/99.9&99.9/99.9&99.9/100&99.9/99.9&99.9/99.9\\
\\
Sent. & $\leq 1$& 26.1/29.7&33.0/37.3&13.5/16.7&22.7/23.7&20.7/22.5&14.0/14.7&25.0/27.3\\
      & $\leq 2$& 77.9/83.4&83.4/87.3&73.4/80.0&71.3/72.8&76.6/80.8&69.3/70.4&76.0/77.2\\
      & $\leq 3$& 96.8/98.9&98.2/98.7&97.8/99.0&96.3/96.6&97.4/98.4&95.8/96.2&97.3/97.5\\
      & $\leq 4$& 99.8/99.9&99.8/99.9&99.8/99.9&99.7/99.7&99.8/100&99.7/99.7&99.8/99.8\\ \hline
\
& & \multicolumn{1}{r}{{\bf Japanese}}&\multicolumn{1}{r}{{\bf Portuguese}}&\multicolumn{1}{r}{{\bf Slovene}}&\multicolumn{1}{r}{{\bf Spanish}}&\multicolumn{1}{r}{{\bf Swedish}}&\multicolumn{1}{r}{{\bf Turkish}}\\
\multicolumn{2}{l}{\#Sents.} &17,753&9,359&1,936&3,512&11,431&5,935\\
\multicolumn{2}{l}{Av. len.} &9.8&23.7&19.1&28.0&18.2&12.7\\ \hline
Token & $\leq 1$& 57.1/55.0&68.7/73.0&76.4/74.9&64.0/67.0&78.5/80.1&65.8/62.7\\
      & $\leq 2$& 90.6/89.5&95.5/97.5&97.1/97.3&93.4/96.1&98.1/98.6&93.9/93.7\\
      & $\leq 3$& 99.1/99.0&99.6/99.9&99.7/99.8&99.1/99.8&99.9/99.9&99.4/99.5\\
      & $\leq 4$& 99.9/99.9&99.9/99.9&99.9/100&99.9/99.9&99.9/99.9&99.9/99.9\\
\\
Sent. & $\leq 1$& 57.3/58.1&27.1/30.8&34.0/40.1&17.8/20.2&32.0/34.4&37.6/38.8\\
      & $\leq 2$& 81.8/81.8&78.7/85.1&85.7/88.5&66.1/73.5&87.8/90.3&80.1/81.0\\
      & $\leq 3$& 97.0/97.1&97.4/99.1&98.3/99.0&94.5/97.9&99.1/99.6&97.1/97.5\\
      & $\leq 4$& 99.8/99.8&99.8/99.9&99.9/100&99.2/99.9&99.9/99.9&99.8/99.8\\ \hline
\end{tabular}
 }
 \caption{Token-level and sentence-level coverage results of left-corner oracles with depth$_{re}$.
 Here, the right-hand numbers in each column are calculated from corpora that exclude all punctuation, e.g., 92\% of tokens in Arabic are covered within a stack depth $\leq 3$, while the number increases to 94.1 when punctuation is removed.
 Further, 57.6\% of sentences (61.6\% without punctuation) can be parsed within a maximum depth$_{re}$ of three, i.e., the maximum degree of center-embedding is at most two in 57.6\% of sentences.
 Av. len. indicates the average number of words in a sentence.}\label{tab:token}
\end{table}

As noted in Section \ref{sec:memorycost}, the stack depth of the left-corner system in our experiments thus far is not the exact measurement of the degree of center-embedding of the construction due to an extra factor introduced by the {\sc Shift} action.
In this section, we focus on depth$_{re}$, which matches the degree of center-embeddeding and may be more applicable to some applications.

Table \ref{tab:token} shows token- and sentence-level statistics with and without punctuations.
The token-level coverage of depth $\leq 2$ substantially improves from the results shown in Figure \ref{fig:load-comparison} in many languages, consistently exceeding 90\% except for Arabic$^*$, which indicates that many configurations of a stack depth of two in previous experiments are due to the extra factor caused by the {\sc Shift} action rather than the deeper center-embedded structures.
Results showing that the token-level coverage reaches 99\% in most languages with depth$_{re} \leq 3$ indicate that the constructions with the degree three of center-embedding occurs rarely in natural language sentences.
Overall, sentence-level coverage results are slightly decreased, but they are still very high, notably 95\% -- 99\% with depth$_{re} \leq 3$ for most languages.

\subsection{Results on UD}
\label{sec:trans:ud-result}

\begin{figure}[p]
\centering
 \resizebox{0.95\textwidth}{!}
 {\includegraphics[]{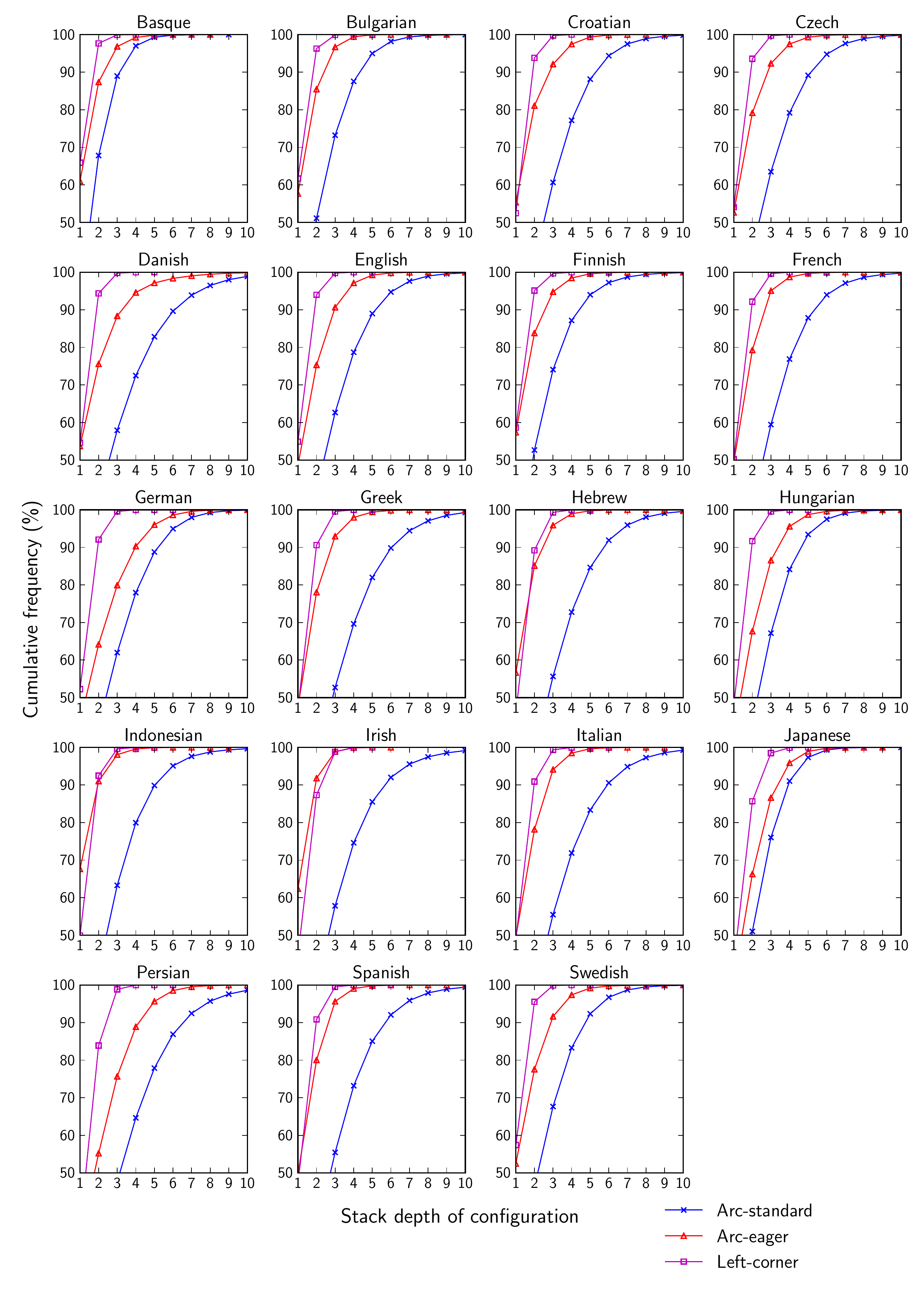}}
\caption{Stack depth results in UD.}
 \label{fig:load-comparison-ud}
\end{figure}

In the following two sections, we move on to UD, in which annotation styles are more consistent across languages.
Figure \ref{fig:load-comparison-ud} shows the result of the same analysis as the comparison in Section \ref{fig:load-comparison} on CoNLL dataset (i.e., Figure \ref{fig:load-comparison}).
We do not observe substantial differences between CoNLL dataset and UD.
Again, the left-corner system is the most consistent across languages.
This result is interesting in that it indicates the stack depth constraint of the left-corner system is less affected by the choice of annotation styles, since the annotation of UD is consistently content head-based while that of CoNLL dataset is (although consistently is lower) mainly function head-based.\footnote{
\REVISE{A theoretical analysis of the effect of the annotation style is interesting, but is beyond the scope of the current study.
We only claim that substantial differences are not observed in the present empirical analysis.
Generally speaking, two dependency representations based on content-head and function-head, do not lead to the identical CFG representation with binarization, but as the meanings that they encode are basically the same (with different notions of head) we expect that the resulting differences in CFG forms are not substantial.
}}
We will see this tendency in more detail by analyzing token-based statistics based on depth$_{re}$ below.

\subsection{Relaxing the definition of center-embedding}
\label{sec:trans:relax}
The token-level analysis on CoNLL dataset in Section \ref{sec:token} (Table \ref{tab:token}) reveals that in most languages depth$_{re} \leq 2$ is a sufficient condition to cover most constructions but there are often relatively large gaps between depth$_{re} \leq 1$ (i.e., no center-embedding) and depth$_{re} \leq 2$ (i.e., at most one degree of center-embedding).
We explore in this section constraints that exist in the middle between these two.
We do so by relaxing the definition of center-embedding that we discussed in Section \ref{sec:bg:embedding}.

Recall that in our definition of center-embedding (Definition \ref{def:bg:embed-depth}), we check whether the length of the most embedded constituent is larger than one (i.e., $|x|\geq 2$ in Eq. \ref{eq:bg:embed-depth}).
In other words, the minimal length of the most embedded constituent for center-embedded structures is {\it two} in this case.
Here, we relax this condition;
for example, if assume the minimal length of most embedded clause is {\it three}, we recognize some portion of singly center-embedded structures (by Definition \ref{def:bg:embed-depth}), in which the size of embedded constituent is one or two, to be not center-embedded.

Due to the transparency between the stack depth and the degree of center-embedding, this can be achieved by not increasing depth$_{re}$ when the size (number of tokens including the dummy node) of the top stack element does not exceed the threshold, which is one in default (thus no reduction occurs).


\begin{figure}[t]
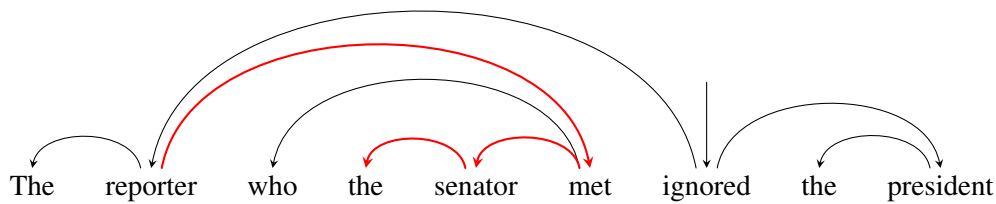

 \centering
 \begin{dependency}[theme=simple]
  \begin{deptext}[column sep=0.5cm]
   The \& reporter \& who \& the \& senator \& met \& ignored \& the \& president \\
  \end{deptext}
  \depedge{2}{1}{}
  \depedge{7}{2}{}
  \depedge{6}{3}{}
  \depedge[thick,red]{5}{4}{}
  \depedge[thick,red]{6}{5}{}
  \depedge[thick,red]{2}{6}{}
  \deproot[edge unit distance=2ex]{7}{}
  \depedge{9}{8}{}
  \depedge{7}{9}{}
 \end{dependency}
 \caption{Following Definition \ref{def:bg:embed-depth}, this tree is recognized as singly center-embedded while is not center-embedded if ``the senator'' is replaced by one word.
 Bold arcs are the cause of center-embedding (zig-zag pattern).}
 \label{fig:trans:relax}
\end{figure}

\paragraph{Motivating example}
In Section \ref{sec:bg:psycho}, we showed that the following sentence is recognized as not center-embedded when we follow Definition \ref{def:bg:embed-depth}:
\enumsentence{
The reporter [who {\bf Mary} met] ignored the president.
}
However, we can see that the following sentence is recognized as singly center-embedded:
\enumsentence{
The reporter [who {\bf the senator} met] ignored the president.
}
Figure \ref{fig:trans:relax} shows the UD-style dependency tree with the emphasis on arcs causing center-embedding.
This observation suggests many constructions that requires depth$_{re} = 2$ might be caught by relaxing the condition of center-embedding discussed above.

\paragraph{Result}
Figure \ref{fig:load-comparison-relax-ud} is the result with such relaxed conditions.
Here we also show the effect of changing maximum sentence length.
We can see in some languages, such as Hungarian, Japanese, and Persian, the effect of this relaxation is substantial while the changes in other languages are rather modest.
We can also see that in most languages depth two is a sufficient condition to conver most constructions, which is again consistent with our observation in CoNLL dataset (Section \ref{sec:token}).

We will explore this relaxation again in the supervised experiments we present below.
Interestingly, there we will observe that the improvements with those relaxations are more substantial in parsing experiments (Section \ref{sec:trans:ud-parse}).

\begin{figure}[p]
\centering
 \resizebox{0.95\textwidth}{!}
 {\includegraphics[]{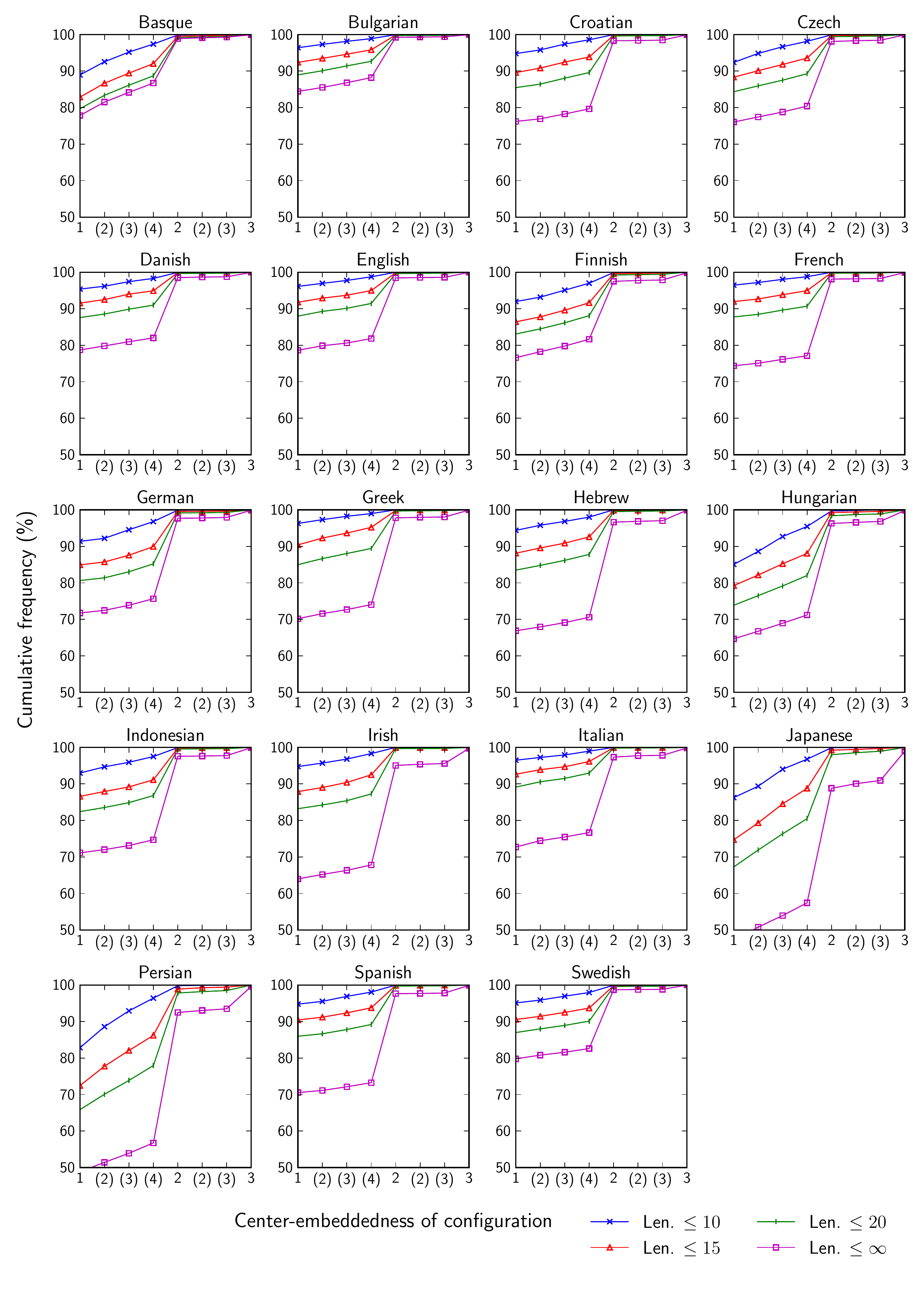}}
 \vspace{-10pt}
 \caption{Stack depth results in UD with a left-corner system (depth$_{re}$) when the definition of center-embedding is relaxed.
 The parenthesized numbers indicate the size of allowed constituents at the bottom of embedding.
 For example (2) next to 2 indicates we allow depth = 3 if the size of subtree on the top of the stack is 1 or 2.
 ${\sf Len.}$ is the maximum sentence length.}
 \label{fig:load-comparison-relax-ud}
\end{figure}

\section{Parsing Experiment}
\label{sec:parse}

Our final experiment is the parsing experiment on unseen sentences.
A transition-based dependency parsing system is typically modeled with a structured discriminative model, such as with the structured perceptron and beam search \cite{zhang-clark:2008:EMNLP,huang-sagae:2010:ACL}.
We implemented and trained the parser model in this framework to investigate the following questions:
\begin{itemize}
 \item How does the stack depth bound at decoding affect parsing performance of each system?
       The underlying concern here is basically the same as in the previous oracle experiment discussed in Section \ref{sec:analysis}, i.e., to determine whether the stack depth of the left-corner system provides a meaningful measure for capturing the syntactic regularities.
       More specifically, we wish to observe whether the observation from the last experiment, i.e., that the behavior of the left-corner system is mostly consistent across languages, also holds with parse errors.
 \item Does our parser perform better than other transition-based parsers?
       One practical disadvantage of our system is that its attachment decisions are made more eagerly, i.e., that it has to commit to a particular structure at an earlier point;
       however, this also means the parser may utilize rich syntactic information as features that are not available in other systems.
       We investigate whether these rich features help disambiguation in practice.
 \item Finally, we examine parser performance of our system under a restriction on features to prohibit lookahead on the buffer.
       This restriction is motivated by the previous model of probabilistic left-corner parsing \cite{journals/coling/SchulerAMS10} in which the central motivation is its cognitive plausibility.
       We report how accuracies drop with the cognitively motivated restriction and discuss a future direction to improve performance.
\end{itemize}

In the following we will investigate the above questions mainly with CoNLL dataset, as in our analysis in Section \ref{sec:analysis}.
In Section \ref{sec:exp:setting}, we explain several experimental setups.
We first compare the performances in the standard English experiments in Section \ref{sec:exp:english-devel}, and then present experiments in CoNLL dataset in Section \ref{sec:exp:conll-parse}.
Finally, we summarize the results in UD in Section \ref{sec:trans:ud-parse}.

\begin{figure}[t]
 \centering
 \begin{tikzpicture}[level distance=20pt, sibling distance=35pt,edge from parent/.style={draw,->}]
  \node at (0, 0) {Incomplete element:};
  \node at (0, -0.5) {$s.{\sf gg}$}
       [sibling distance=45pt]
       child[missing]
       child { node {$s.{\sf gp}$}
         [sibling distance=45pt]
         child[missing]
         child { node {$s.{\sf p}$}
           [sibling distance=45pt]
           child[missing]
           child { node {$x$}
             [sibling distance=15pt]
             child {node {$s.{\sf l}$}}
             child {node {$s.{\sf l_2}$}}
             child[missing]
             child[missing]
           }
         }
       };
  \node at (4.0, 0) {Complete element:};
  \node at (4.0, -0.5) {$q$}
    [sibling distance=30pt]
    child { node {$q.{\sf l}$} }
    child { node {$q.{\sf r}$} };

  \node (R) at (7.0, 0) {Shift mode:};

  \node (s) [below right =0.4cm and 0.5cm of R.west,anchor=west] {Stack:};
  \node (b) [right=1.7cm of s] {Buffer:};

  \draw ($(s)+(0,-0.3)$) -- ($(s)+(2.0,-0.3)$) -- ($(s)+(2.0,-0.8)$) -- ($(s)+(0,-0.8)$);
  \draw ($(s)+(4.5,-0.3)$) -- ($(s)+(2.7,-0.3)$) -- ($(s)+(2.7,-0.8)$) -- ($(s)+(4.5,-0.8)$);

  \node (s1) [below right=0.55cm and 1.4cm of s.west,anchor=west] {$s_1 ~~ s_0$};
  \node[right=0.9cm of s1] {$q_0 ~~ q_1 ~~ q_2$};

  \node (R) [below=1.8cm of R.west,anchor=west] {Reduce mode:};
  \node (s) [below right =0.4cm and 0.5cm of R.west,anchor=west] {Stack:};
  \node (b) [right=1.7cm of s] {Buffer:};

  \draw ($(s)+(0,-0.3)$) -- ($(s)+(2.0,-0.3)$) -- ($(s)+(2.0,-0.8)$) -- ($(s)+(0,-0.8)$);
  \draw ($(s)+(4.5,-0.3)$) -- ($(s)+(2.7,-0.3)$) -- ($(s)+(2.7,-0.8)$) -- ($(s)+(4.5,-0.8)$);

  \node (s1) [below right=0.55cm and 0.9cm of s.west,anchor=west] {$s_1 ~~ s_0 ~~ q_0$};
  \node[right=0.9cm of s1] {$q_1 ~~ q_2$};
 \end{tikzpicture}
 \caption{(Left) Elementary features extracted from an incomplete and complete node, and
 (Right) how feature extraction is changed depending on whether the next step is shift or reduce.
 }\label{fig:feature}
\end{figure}
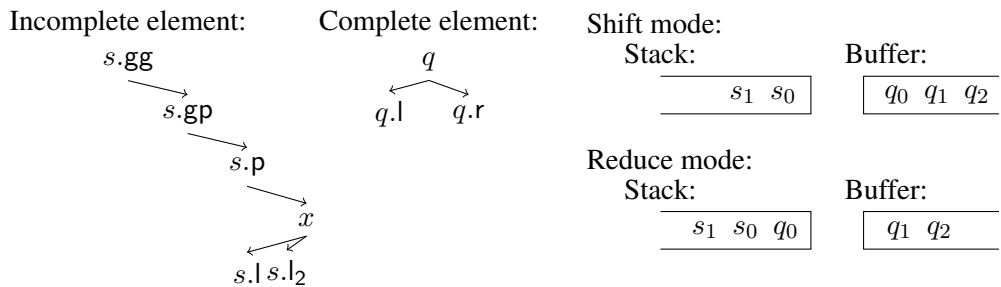
\begin{table}[t]
 \centering
 \begin{tabular}{l l l l} \hline
  $s_0.{\sf p.w}$ & $s_0.{\sf p.t}$ & $s_0.{\sf l.w}$ & $s_0.{\sf l.t}$ \\
  $s_1.{\sf p.w}$ & $s_1.{\sf p.t}$ & $s_1.{\sf l.w}$ & $s_1.{\sf l.t}$ \\ 
  $s_0.{\sf p.w} \circ s_0.{\sf p.t}$ & $s_0.{\sf l.w} \circ s_0.{\sf l.t}$ & $s_1.{\sf p.w} \circ s_1.{\sf p.t}$ & $s_1.{\sf l.w} \circ s_1.{\sf l.t}$ \\
  $q_0.{\sf w}$ & $q_0.{\sf t}$ & $q_0.{\sf w} \circ q_0.{\sf t}$\\ 
  $s_0.{\sf p.w} \circ s_0.{\sf l.w}$ & $s_0.{\sf p.t} \circ s_0.{\sf l.t}$ & \\ \hline
  $s_0.{\sf p.w} \circ s_1.{\sf p.w}$ & $s_0.{\sf l.w} \circ s_1.{\sf l.w}$ & $s_0.{\sf p.t} \circ s_1.{\sf p.t}$ & $s_0.{\sf l.t} \circ s_1.{\sf l.t}$ \\ \hline
  $s_0.{\sf p.w} \circ q_0.{\sf w}$ & $s_0.{\sf l.w} \circ q_0.{\sf w}$ & $s_0.{\sf p.t} \circ q_0.{\sf t}$ & $s_0.{\sf l.t} \circ q_0.{\sf t}$ \\
  $s_0.{\sf p.w} \circ q_0.{\sf w} \circ q_0.{\sf p}$ & $s_0.{\sf p.w} \circ q_0.{\sf w} \circ s_0.{\sf p.t}$ & $s_0.{\sf l.w} \circ q_0.{\sf w} \circ s_0.{\sf l.t}$ & $s_0.{\sf l.w} \circ q_0.{\sf w} \circ s_0.{\sf l.t}$ \\
  $s_0.{\sf p.w} \circ s_0.{\sf p.t} \circ q_0.{\sf t}$ & $s_0.{\sf l.w} \circ s_0.{\sf l.t} \circ q_0.{\sf t}$\\ \hline
  $q_0.{\sf t} \circ q_0.{\sf l.t} \circ q_0.{\sf r.t}$ & $q_0.{\sf w} \circ q_0.{\sf l.t} \circ q_0.{\sf r.t}$ \\ \hline
  $s_0.{\sf p.t} \circ s_0.{\sf gp.t} \circ s_0.{\sf gg.t}$ & $s_0.{\sf p.t} \circ s_0.{\sf gp.t} \circ s_0.{\sf l.t}$ & $s_0.{\sf p.t} \circ s_0.{\sf l.t} \circ s_0.{\sf l_2.t}$ & $s_0.{\sf p.t} \circ s_0.{\sf gp.t} \circ q_0.{\sf t}$\\
  $s_0.{\sf p.t} \circ s_0.{\sf l.t} \circ q_0.{\sf t}$ & $s_0.{\sf p.w} \circ s_0.{\sf l.t} \circ q_0.{\sf t}$ & $s_0.{\sf p.t} \circ s_0.{\sf l.w} \circ q_0.{\sf t}$ & $s_0.{\sf l.t} \circ s_0.{\sf l_2.p} \circ q_0.{\sf t}$ \\
  $s_0.{\sf l.t} \circ s_0.{\sf l_2.t} \circ q_0.{\sf t}$ & $s_0.{\sf p.t} \circ q_0.{\sf t} \circ q_0.{\sf l.t}$ & $s_0.{\sf p.t} \circ q_0.{\sf t} \circ q_0.{\sf r.t}$ \\
  $s_1.{\sf p.t} \circ s_0.{\sf p.t} \circ s_0.{\sf l.t}$ & $s_1.{\sf p.t} \circ s_0.{\sf l.t} \circ q_0.{\sf t}$ & $s_1.{\sf l.t} \circ s_0.{\sf l.t} \circ q_0.{\sf t}$ & $s_1.{\sf l.t} \circ s_0.{\sf l.t} \circ q0.{\sf t}$ \\
  $s_1.{\sf l.t} \circ s_0.{\sf p.t} \circ q_0.{\sf p}$ \\ \hline
 \end{tabular}
 \caption{Feature templates used in both full and restricted feature sets, with ${\sf t}$ representing POS tag and ${\sf w}$ indicating the word form, e.g., $s_0.{\sf l.t}$ refers to the POS tag of the leftmost child of $s_0$. $\circ$ means concatenation.}\label{tab:feature}
\end{table}
\begin{table}[t]
 \centering
 \begin{tabular}{l l l l} \hline
  $q_0.{\sf t} \circ q_1.{\sf t}$ & $q_0.{\sf t} \circ q_1.{\sf t} \circ q_2.{\sf t}$ & $s_0.{\sf p.t} \circ q_0.{\sf p} \circ q_1.{\sf p} \circ q_2.{\sf p}$ & $s_0.{\sf l.t} \circ q_0.{\sf t} \circ q_1.{\sf t} \circ q_2.{\sf t}$ \\
  $s_0.{\sf p.w} \circ q_0.{\sf t} \circ q_1.{\sf t}$ & $s_0.{\sf p.t} \circ q_0.{\sf t} \circ q_1.{\sf t}$ & $s_0.{\sf l.w} \circ q_0.{\sf t} \circ q_1.{\sf t}$ & $s_0.{\sf l.t} \circ q_0.{\sf t} \circ q_1.{\sf t}$ \\ \hline
 \end{tabular}
 \caption{Additional feature templates only used in the full feature model.}\label{tab:additional} 
\end{table}

\subsection{Feature}

The feature set we use is explained in Figure \ref{fig:feature} and Tables \ref{tab:feature} and \ref{tab:additional}.
Our transition system is different from other systems in that it has two modes, i.e., a shift mode in which the next action is either {\sc Shift} or {\sc Insert} and a reduce mode in which we select the next reduce action, thus we use different features depending on the current mode.
Figure \ref{fig:feature} shows how features are extracted from each node for each mode.
In reduce mode, we treat the top node of the stack as if it were the top of buffer ($q_0$), which allows us to use the same feature templates in both modes by modifying only the definitions of elementary features $s_i$ and $q_i$.
A similar technique has been employed in the transition system proposed by \newcite{sartorio-satta-nivre:2013:ACL2013}.

To explore the last question, we develop two feature sets.
Our full feature set consists of features shown in Tables \ref{tab:feature} and \ref{tab:additional}.
For the limited feature set, we remove all features that depend on $q_1$ and $q_2$ in Figure \ref{fig:feature}, which we list in Table \ref{tab:additional}.
Here, we only look at the top node on the buffer in shift mode.
This is the minimal amount of lookahead in our parser and is the same as the previous left-corner PCFG parsers \cite{journals/coling/SchulerAMS10}, which are cognitively motivated.

Our parser cannot capture a head and dependent relationship directly at each reduce step, because all interactions between nodes are via a dummy node, which may be a severe limitation from a practical viewpoint;
however, we can exploit richer context from each subtree on the stack, as illustrated in Figure \ref{fig:feature}.
We construct our feature set with many nodes around the dummy node, including the parent (${\sf p}$), grandparent (${\sf gp}$), and great grandparent (${\sf gg}$).

\subsection{Settings}
\label{sec:exp:setting}

We compare parsers with three transition systems: arc-standard, arc-eager, and left-corner.
The feature set of the arc-standard system is borrowed from \newcite{huang-sagae:2010:ACL}.
For the arc-eager system, we use the feature set of \newcite{zhang-nivre:2011:ACL-HLT2011} from which we exclude features that rely on arc label information.

We train all models with different beam sizes in the violation fixing perceptron framework \cite{huang-fayong-guo:2012:NAACL-HLT}.
Since our goal is not to produce a state-of-the-art parsing system, we use gold POS tags as input both in training and testing.

As noted in Section \ref{subsec:transitionsystem}, the left-corner parser sometimes fails to generate a single tree, in which case the stack contains a complete subtree at the top position (since the last action is always {\sc Insert}) and one or more incomplete subtrees.
If this occurs, we perform the following post-processing steps:
\begin{itemize}
 \item We collapse each dummy node in an incomplete tree. More specifically, if the dummy node is the head of the subtree, we attach all children to the sentence (dummy) root node; otherwise, the children are reattached to the parent of the dummy node.
 \item The resulting complete subtrees are all attached to the sentence (dummy) root node.
\end{itemize}

\subsection{Results on the English Development Set}
\label{sec:exp:english-devel}
\begin{figure}[p]
\centering
 \begin{minipage}[b]{.49\linewidth}
  \centering
  \resizebox{0.9\textwidth}{!}
  {\includegraphics[]{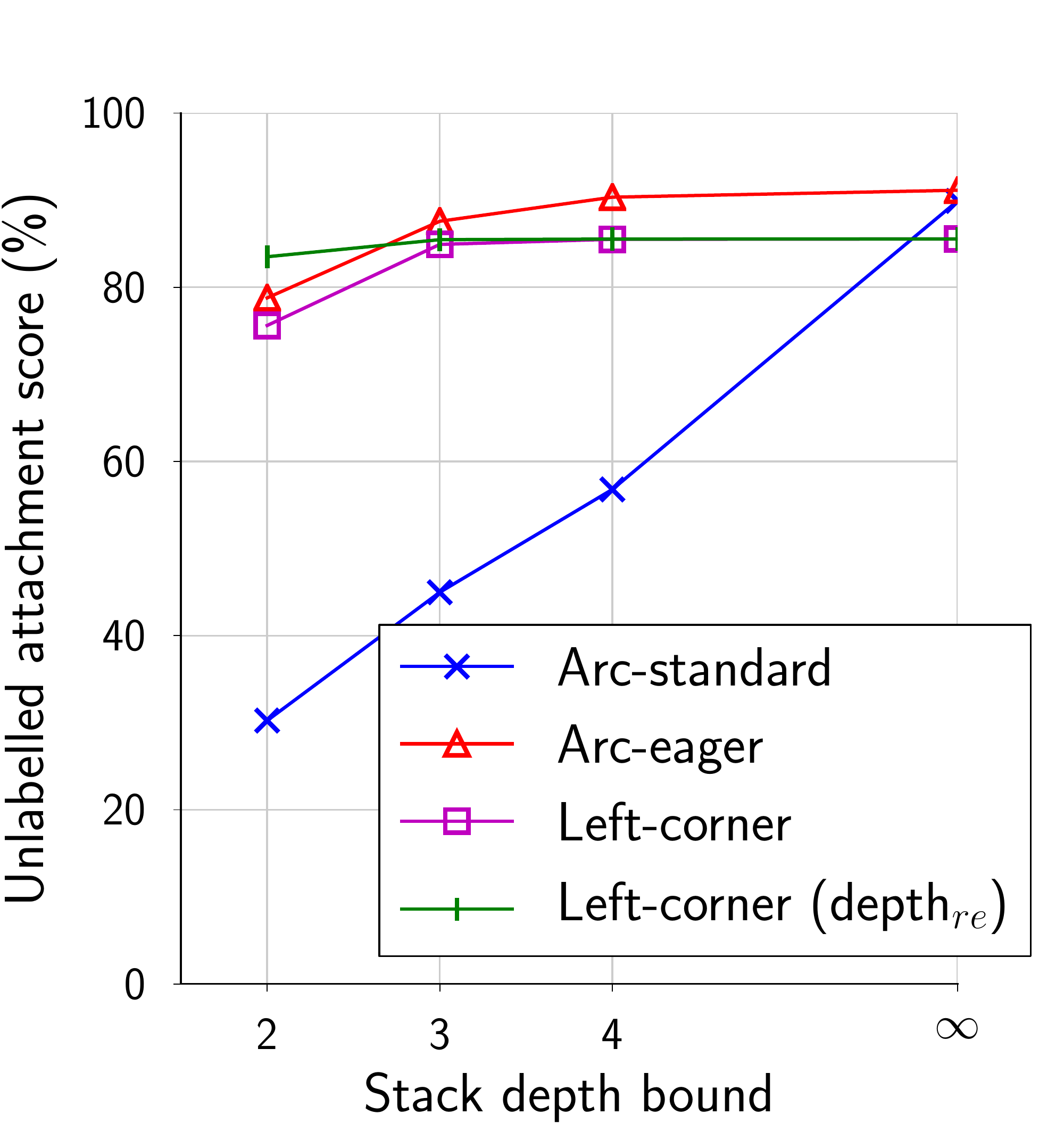}}
  \subcaption{$b=1$}
 \end{minipage}
 \begin{minipage}[b]{.49\linewidth}
  \centering
  \resizebox{0.9\textwidth}{!}
  {\includegraphics[]{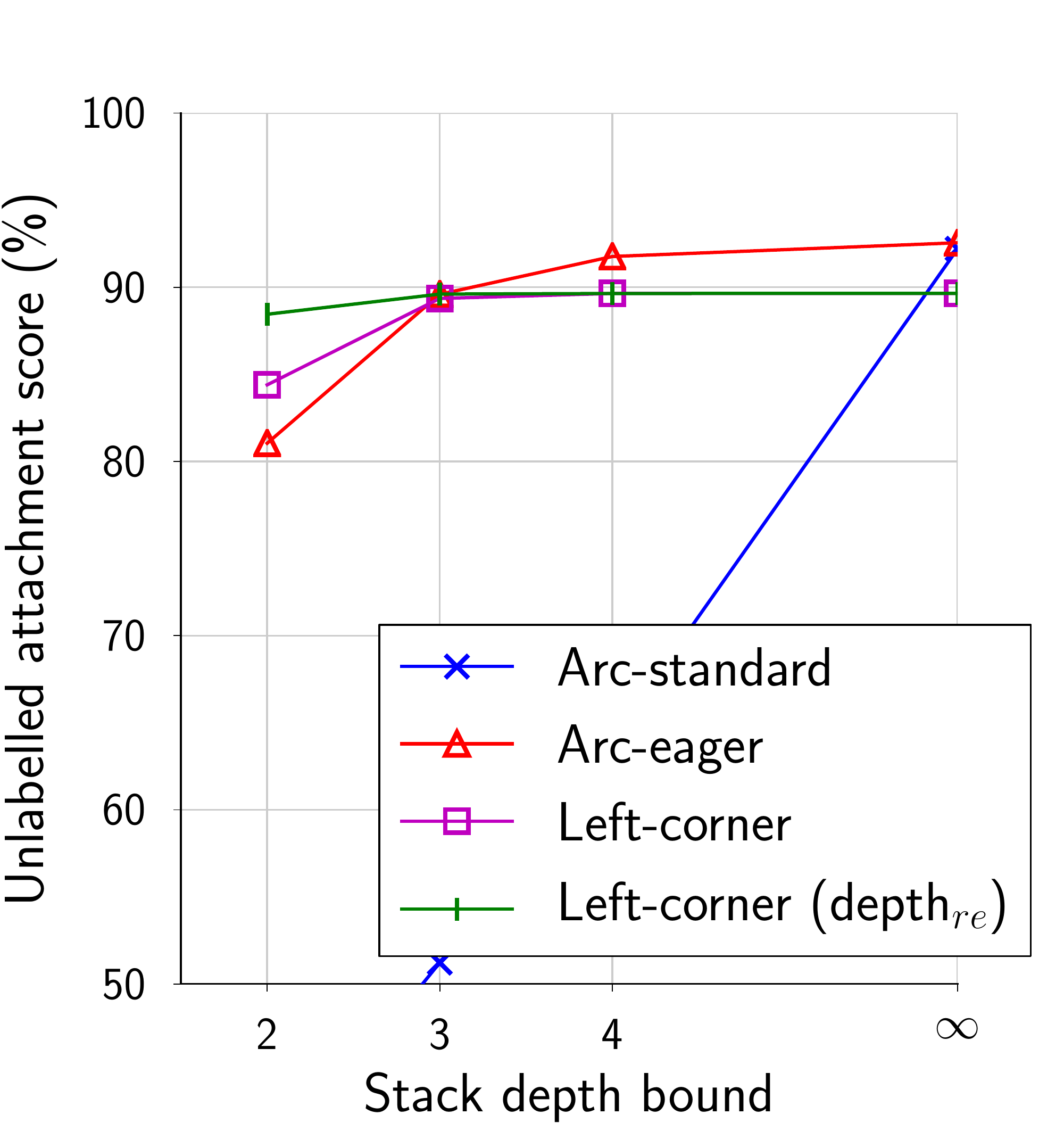}}
  \subcaption{$b=2$}
 \end{minipage}
 \begin{minipage}[b]{.49\linewidth}
  \centering
  \resizebox{0.9\textwidth}{!}
  {\includegraphics[]{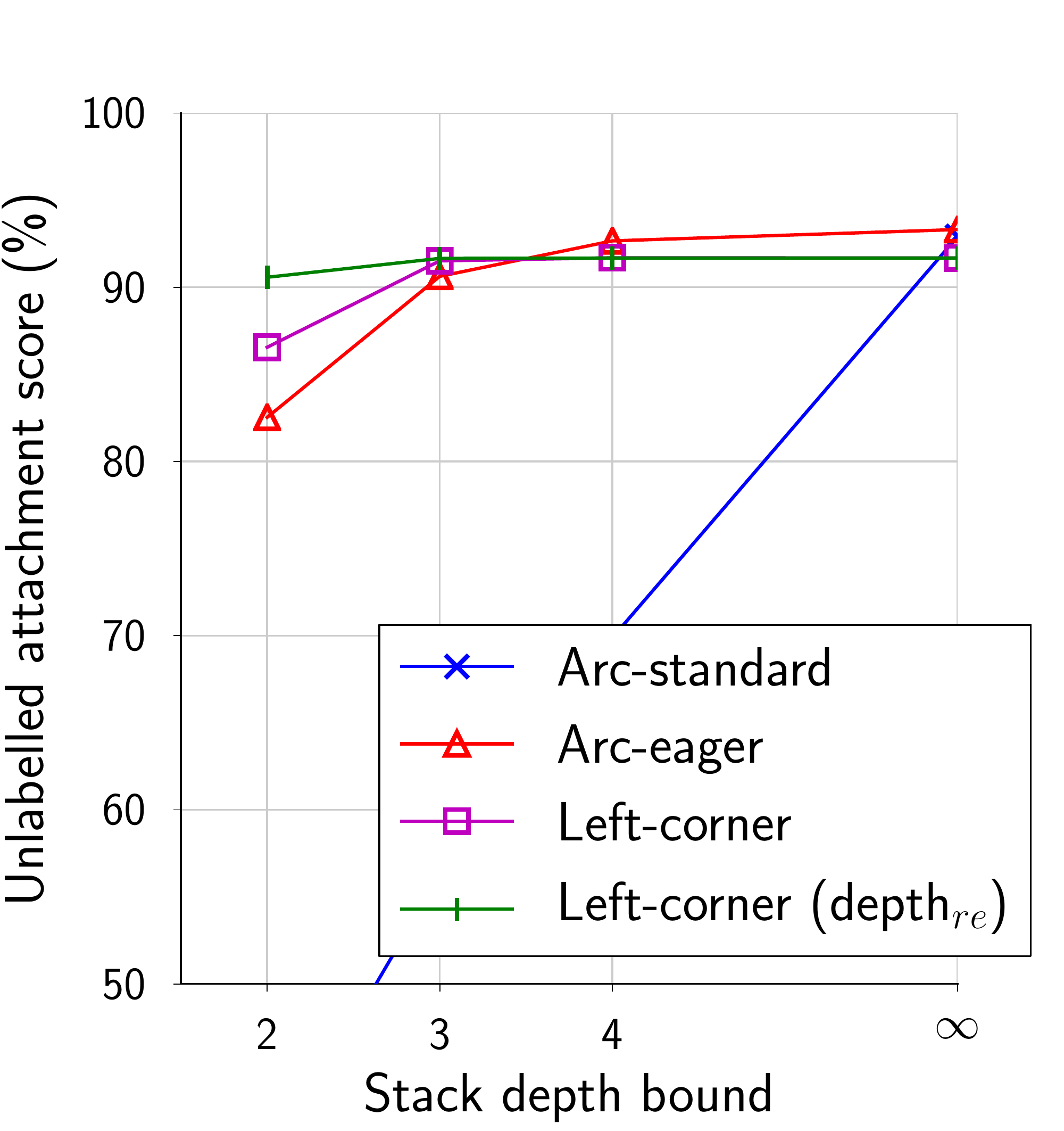}}
  \subcaption{$b=8$}
 \end{minipage}
 \begin{minipage}[b]{.49\linewidth}
  \centering
  \resizebox{0.9\textwidth}{!}
  {\includegraphics[]{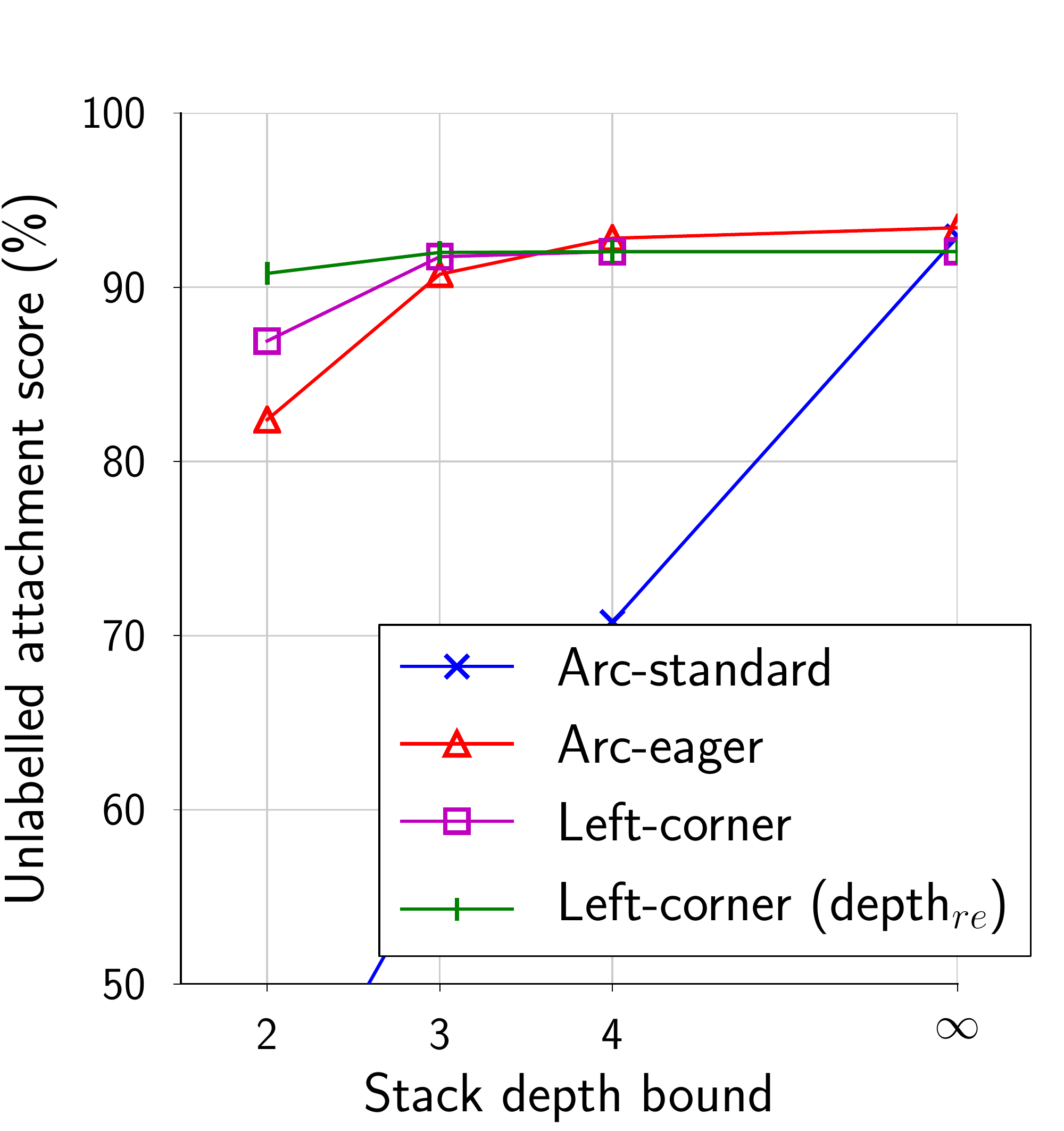}}
  \subcaption{$b=16$}
 \end{minipage}
 \caption{Accuracy vs. stack depth bound at decoding for several beam sizes ($b$).}
 \label{fig:accuracy-depth}
\end{figure}
 
We first evaluate our system on the common English development experiment.
We train the model in section 2-21 of the WSJ Penn Treebank \cite{Marcus93buildinga}, which are converted into dependency trees using the LTH conversion tool\footnote{http://nlp.cs.lth.se/software/treebank\_converter/}.

\paragraph{Impact of Stack Depth Bound}
To explore the first question posed at the beginning of this section, we compare parse accuracies under each stack depth bound with several beam sizes, with results shown in Figure \ref{fig:accuracy-depth}.
In this experiment, we calculate the stack depth of a configuration in the same way as our oracle experiment (see Section \ref{sec:analysis:settings}), and when expanding a beam, we discard candidates for which stack depth exceeds the maximum value.
As discussed in Section \ref{sec:token}, for the left-corner system, depth$_{re}$ might be a more linguistically meaningful measure, so we report scores with both definitions.\footnote{
This can be achieved by allowing any configurations after a shift step.
}
The general tendency across different beam sizes is that our left-corner parser (in particular with depth$_{re}$) is much less sensitive to the value of the stack depth bound.
For example, when the beam size is eight, the accuracies of the left-corner (depth$_{re}$) are 90.6, 91.7, 91.7, and 91.7 with increased stack depth bounds, while the corresponding scores are 82.5, 90.6, 92.6, and 93.3 in the arc-eager system.
This result is consistent with the observation in our oracle coverage experiment discussed in Section \ref{sec:analysis}, and suggests that a depth bound of two or three might be a good constraint for restricting tree candidates for natural language with no (or little) loss of recall.
Next, we examine whether this observation is consistent across languages.

\begin{figure}[t]
\centering
 \resizebox{0.75\textwidth}{!}
 {\includegraphics[]{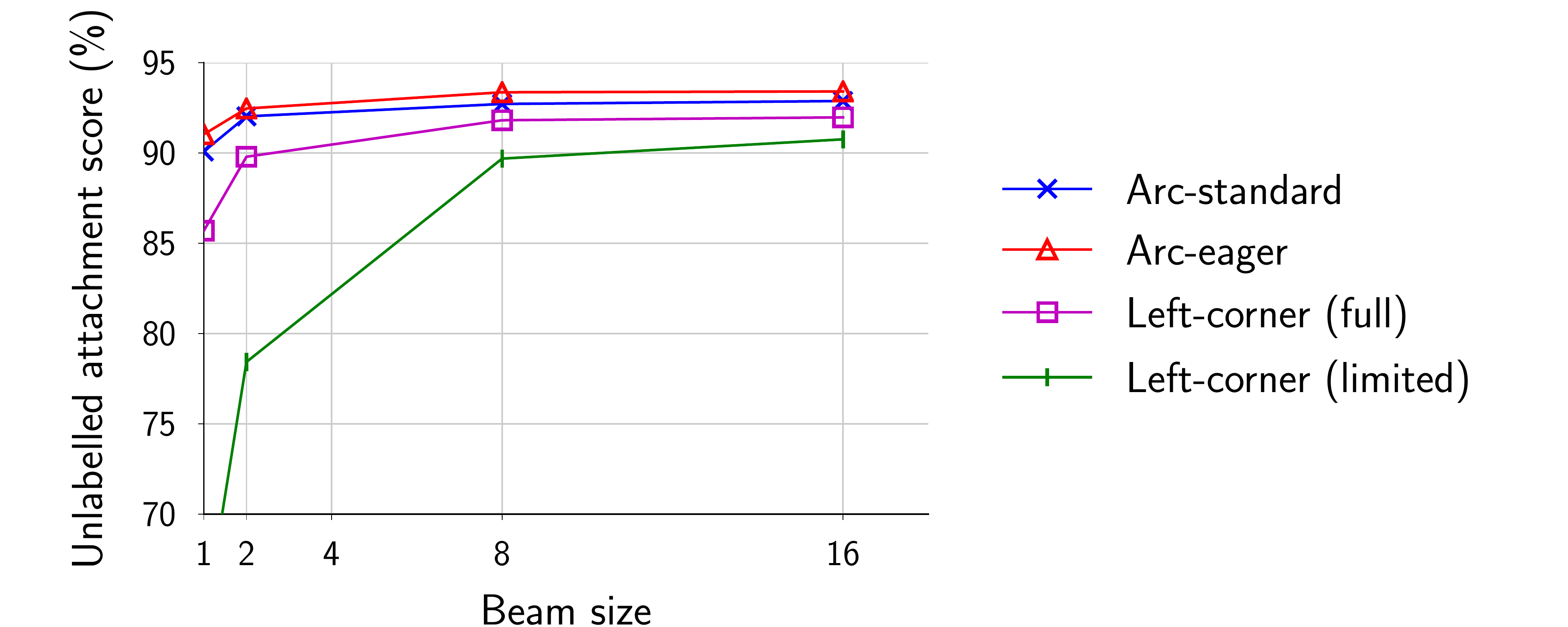}}
 \caption{Accuracy vs. beam size for each system on the English Penn Treebank development set.
 Left-corner (full) is the model with the full feature set, while Left-corner (limited) is the model with the limited feature set.}
 \label{fig:plot}
\end{figure}

\paragraph{Performance without Stack Depth Bound}
Figure \ref{fig:plot} shows accuracies with no stack depth bound when changing beam sizes.
We can see that the accuracy of the left-corner system (full feature) is close to that of the other two systems, but some gap remains.
With a beam size of 16, the scores are
left-corner: 92.0;
arc-standard: 92.9;
arc-eager: 93.4.
Also, the score gaps are relatively large at smaller beam sizes;
e.g., with beam size 1, the score of the left-corner is 85.5, while that of the arc-eager is 91.1.
This result indicates that the prediction with our parser might be structurally harder than other parsers even though ours can utilize richer context from subtrees on the stack.

\paragraph{Performance of Limited Feature Model}
Next we move on to the results with cognitively motivated limited features (Figure \ref{fig:plot}).
When the beam size is small, it performs extremely poorly (63.6\% with beam size 1).
This is not surprising since our parser has to deal with each attachment decision much earlier, which seems hard without lookahead features or larger beam.
However, it is interesting that it achieves a reasonably higher score of 90.6\% accuracy with beam size 16.
In the previous constituency left-corner parsing experiments that concerned their cognitive plausibility \cite{journals/coling/SchulerAMS10,TOPS:TOPS12034}, typically the beam size is quite huge, e.g., 2,000.
The largest difference between our parser and their systems is the model: our model is discriminative while their models are generative.
Though discriminative models are not popular in the studies of human language processing \cite{keller:2010:Short}, the fact that our parser is able to output high quality parses with such smaller beam size would be appealing from the cognitive viewpoint (see Section \ref{sec:relatedwork} for further discussion).

\subsection{Result on CoNLL dataset}
\label{sec:exp:conll-parse}
\begin{figure}[p]
\centering
 \resizebox{0.95\textwidth}{!}
 {\includegraphics[]{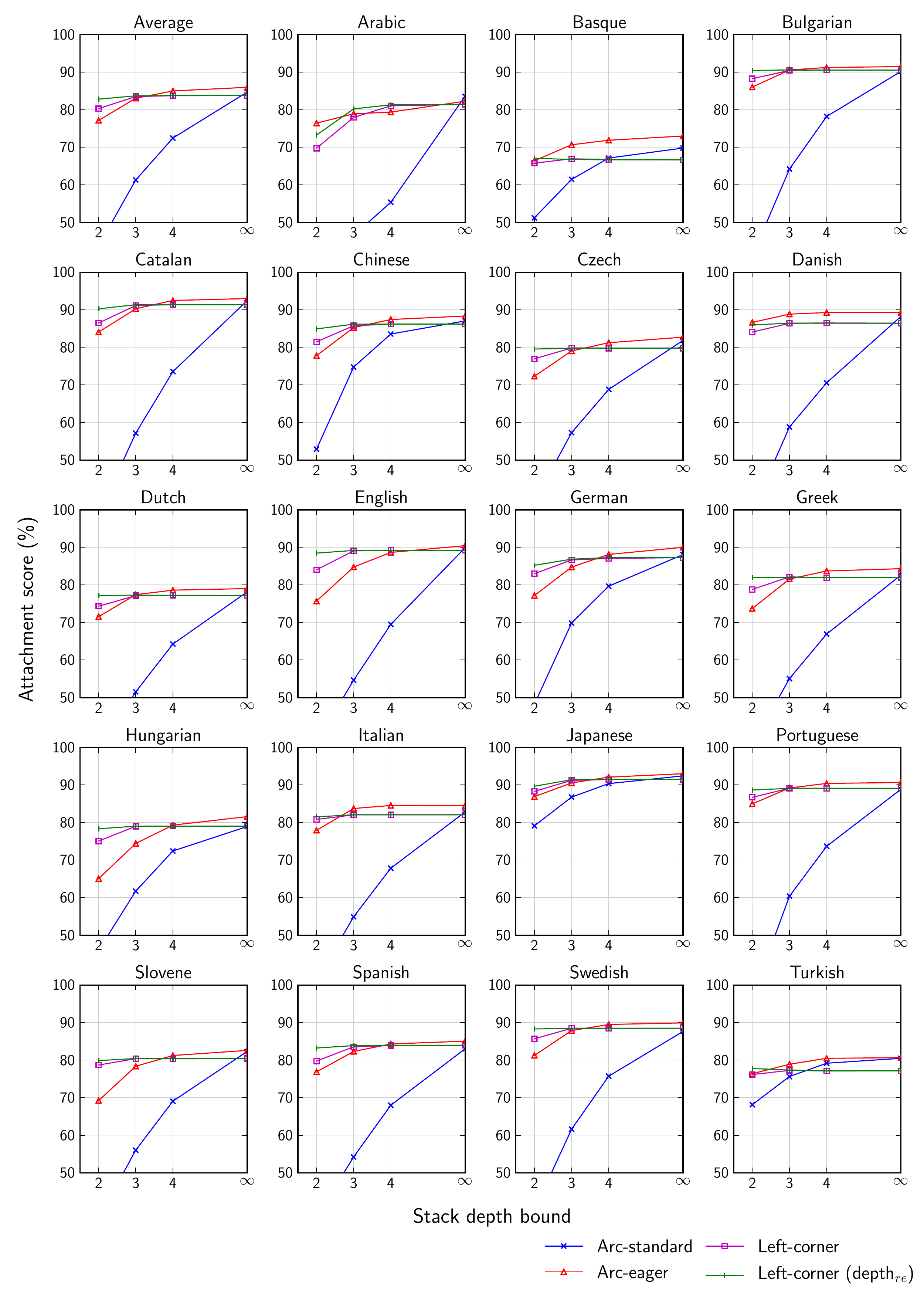}}
\caption{Accuracy vs. stack depth bound in CoNLL dataset.}
 \label{fig:conll_depth_to_accuracy}
\end{figure}

We next examine whether the observations above with English dataset are general across languages using CoNLL dataset.
Note that although we train on the projectivized corpus, evaluation is against the original nonprojective trees.
As our systems are unlabeled, we do not try any post-deprojectivization \cite{nivre-nilsson:2005:ACL}.
In this experiment, we set the beam size to 8.

\paragraph{Effect of Stack Depth Bound}
The cross-linguistic results with stack depth bounds are summarized in Figure \ref{fig:conll_depth_to_accuracy} from which we can see that the overall tendency of each system is almost the same as the English experiment.
Little accuracy drops are observed between models with bounded depth 2 or 3 and the model without depth bound in the left-corner (depth$_{re}$), although the score gaps are larger in the arc-eager.
The arc-standard parser performs extremely poorly with small depth bounds except Japanese and Turkish, and this is consistent with our observation that the arc-standard system demands less stack depth only for head-final languages (Section \ref{sec:analysis:general}).

Notably, in some cases the scores of the left-corner parser (depth$_{re}$) drop when loosing the depth bound (see Basque, Danish, and Turkish), meaning that the stack depth bound of the left-corner sometimes help disambiguation by ignoring linguistically implausible structures (deep center-embedding) during search.
The result indicates the parser performance could be improved by utilizing stack depth information of the left-corner parser, though we leave further investigation as a future work.

\paragraph{Performance without Stack Depth Bound}
Table \ref{tab:parse} summarizes the results without stack depth bounds.
Again, the overall tendency is the same as the English experiment.
The arc-eager performs the best except Arabic.
In some languages (e.g., Bulgarian, English, Spanish, and Swedish), the left-corner (full) performs better than the arc-standard, while the average score is 1.1 point below.
This difference and the average difference between the arc-eager and the arc-standard (85.8 vs. 84.6) are both statistically significant (p $<$ 0.01, the McNemar test).
We can thus conclude that our left-corner parser is not stronger than the other state-of-the-art parsers even with rich features.

\begin{table}[t]
 \centering
 \begin{tabular}[t]{lcccc} \hline
  & {\bf Arc-standard} & {\bf Arc-eager} & {\bf Left-corner} & {\bf Left-corner} \\
  &                    &                 & {\bf full}        & {\bf limited}     \\ \hline
  Arabic & 83.9 & 82.2 & 81.2 & 77.5 \\
  Basque & 70.5 & 72.8 & 66.8 & 64.6 \\
  Bulgarian & 90.2 & 91.4 & 89.9 & 88.1 \\
  Catalan & 92.5 & 93.3 & 91.4 & 89.3 \\
  Chinese & 87.3 & 88.4 & 86.8 & 83.6 \\
  Czech & 81.5 & 82.3 & 80.1 & 77.2 \\
  Danish & 88.0 & 89.1 & 86.8 & 85.5 \\
  Dutch & 77.7 & 79.0 & 77.4 & 74.9 \\
  English & 89.6 & 90.3 & 89.0 & 85.8 \\
  German & 88.1 & 90.0 & 87.2 & 85.7 \\
  Greek & 82.2 & 84.0 & 82.0 & 80.7 \\
  Hungarian & 79.1 & 80.9 & 79.0 & 75.8 \\
  Italian & 82.3 & 84.8 & 81.7 & 79.4 \\
  Japanese & 92.5 & 92.9 & 91.3 & 90.7 \\
  Portuguese & 89.2 & 90.6 & 88.9 & 87.1 \\
  Slovene & 82.3 & 82.3 & 80.8 & 77.1 \\
  Spanish & 83.0 & 85.0 & 83.8 & 80.6 \\
  Swedish & 87.2 & 90.0 & 88.5 & 87.0 \\
  Turkish & 80.8 & 80.8 & 77.5 & 75.4 \\
  \\
  Average & 84.6 & 85.8 & 83.7 & 81.4 \\ \hline
 \end{tabular}
 \caption{Parsing results on CoNLL X and 2007 test sets with no stack depth bound (unlabeled attachment scores).}\label{tab:parse} 
\end{table}

\paragraph{Performance of Limited Feature Model}
With limited features, the left-corner parser performs worse in all languages. The average score is about 2 point below the full feature models (83.7\% vs. 81.4\%) and shows the same tendency as in the English development experiment.
This difference is also statistically significant (p $<$ 0.01, the McNemar test).
The scores of English are relatively low compared with the results in Table \ref{fig:plot}, probably because the training data used in the CoNLL 2007 shared task is small, about half of our development experiment, to reduce the cost of training with large corpora for the shared task participants \cite{nivre-EtAl:2007:EMNLP-CoNLL2007}.

\begin{figure}[p]
\centering
 \begin{minipage}[b]{0.45\linewidth}
  \centering
  \begin{dependency}[theme=simple]
   \begin{deptext}[column sep=0.12cm]
    ... \& to \& appear \& three \& times \& on \& CNN \& ... \\
   \end{deptext}
   \depedge{2}{3}{}
   \depedge{3}{5}{}
   \depedge{5}{4}{}
   \depedge[line width=0.2ex]{5}{6}{}
   \depedge{6}{7}{}
   \depedge{3}{8}{}
  \end{dependency}
  \subcaption{System output (left-corner; $b=8$).}
 \end{minipage}
 \begin{minipage}[b]{0.45\linewidth}
  \centering
  \begin{dependency}[theme=simple]
   \begin{deptext}[column sep=0.12cm]
    ... \& to \& appear \& three \& times \& on \& CNN \& ... \\
   \end{deptext}
   \depedge{2}{3}{}
   \depedge{3}{5}{}
   \depedge{5}{4}{}
   \depedge[line width=0.2ex]{3}{6}{}
   \depedge{6}{7}{}
   \depedge{3}{8}{}
  \end{dependency}
  \subcaption{Gold tree.}
 \end{minipage}
 \begin{minipage}[b]{\linewidth}

  \centering
  \vspace{10pt}
  \begin{tikzpicture}[edge from parent/.style={draw,->},sibling distance=15pt,level distance=16pt]

   \node (begin) at (-3, 0.0) {};
   
   \draw ($(begin)$) rectangle ($(begin)+(6, -2.5)$);

   \node (s) [below right=0.6 and 0.7 of begin.west, anchor=west] {};
   \draw ($(s)+(0,0)$) -- ($(s)+(3.0,0.0)$) -- ($(s)+(3.0,-0.5)$) -- ($(s)+(0,-0.5)$);
   \draw ($(s)+(4.5,0)$) -- ($(s)+(3.3,0)$) -- ($(s)+(3.3,-0.5)$) -- ($(s)+(4.5,-0.5)$);
   
   \node[anchor=south west] (t) at ($(s.west)+(0.25, -0.49)$) {to{\color{white}{p}}}
     child[missing]
     child { node {$x$} };
   
   \node[anchor=south west] (t) at ($(s.west)+(0.9, -0.49)$) {appear}
     child[missing]
     child { node {$x$}
       child { node {three}}
       child[missing] };
   \node[anchor=south west] (t) at ($(s.west)+(3.5, -0.49)$) {times{\color{white}{p}}};

   \draw[->] (-1.0, -2.8) -> (-3.5, -3.5);
   \node at (-3.0, -3.0) {{\sc Shift}};

   \draw[->] (1.0, -2.8) -> (3.5, -3.5);
   \node at (3.0, -3.0) {{\sc Insert}};

   \draw[->] (-3.5, -6.6) -> (-3.5, -7.5) node[left,pos=0.5] {{\sc RightComp}};
   \draw[->] (3.5, -6.6) -> (3.5, -7.5) node[right,pos=0.5] {{\sc RightPred}};

   \node (begin) at (-6.5, -3.8) {};
   
   \draw ($(begin)$) rectangle ($(begin)+(6, -2.5)$);

   \node (s) [below right=0.6 and 0.7 of begin.west, anchor=west] {};
   \draw ($(s)+(0,0)$) -- ($(s)+(3.0,0.0)$) -- ($(s)+(3.0,-0.5)$) -- ($(s)+(0,-0.5)$);
   \draw ($(s)+(4.5,0)$) -- ($(s)+(3.3,0)$) -- ($(s)+(3.3,-0.5)$) -- ($(s)+(4.5,-0.5)$);
   
   \node[anchor=south west] (t) at ($(s.west)+(0.25, -0.49)$) {to{\color{white}{p}}}
     child[missing]
     child { node {$x$} };
   
   \node[anchor=south west] (t) at ($(s.west)+(0.9, -0.49)$) {appear}
     child[missing]
     child { node {$x$}
       child { node {three}}
       child[missing] };
   \node[anchor=south west] (t) at ($(s.west)+(1.9, -0.49)$) {{\color{white}{p}}times};
   \node[anchor=south west] (t) at ($(s.west)+(3.5, -0.49)$) {on{\color{white}{p}}};
   
   \node (begin) at (0.5, -3.8) {};
   
   \draw ($(begin)$) rectangle ($(begin)+(6, -2.5)$);

   \node (s) [below right=0.6 and 0.7 of begin.west, anchor=west] {};
   \draw ($(s)+(0,0)$) -- ($(s)+(3.0,0.0)$) -- ($(s)+(3.0,-0.5)$) -- ($(s)+(0,-0.5)$);
   \draw ($(s)+(4.5,0)$) -- ($(s)+(3.3,0)$) -- ($(s)+(3.3,-0.5)$) -- ($(s)+(4.5,-0.5)$);
   
   \node[anchor=south west] (t) at ($(s.west)+(0.25, -0.49)$) {to{\color{white}{p}}}
     child[missing]
     child { node {$x$} };
   
   \node[anchor=south west] (t) at ($(s.west)+(0.9, -0.49)$) {appear}
     child[missing]
     child { node {times}
       child { node {three}}
       child[missing] };
   \node[anchor=south west] (t) at ($(s.west)+(3.5, -0.49)$) {on{\color{white}{p}}};

   \node (begin) at (-6.5, -7.8) {};
   
   \draw ($(begin)$) rectangle ($(begin)+(6, -2.5)$);

   \node (s) [below right=0.6 and 0.7 of begin.west, anchor=west] {};
   \draw ($(s)+(0,0)$) -- ($(s)+(3.0,0.0)$) -- ($(s)+(3.0,-0.5)$) -- ($(s)+(0,-0.5)$);
   \draw ($(s)+(4.5,0)$) -- ($(s)+(3.3,0)$) -- ($(s)+(3.3,-0.5)$) -- ($(s)+(4.5,-0.5)$);
   
   \node[anchor=south west] (t) at ($(s.west)+(0.25, -0.49)$) {to{\color{white}{p}}}
     child[missing]
     child { node {$x$} };
   
   \node[anchor=south west] (t) at ($(s.west)+(0.9, -0.49)$) {appear}
     child[missing]
     child { node {times}
       child { node {three}}
       child { node[right=0.1] {$x$}} };
   \node[anchor=south west] (t) at ($(s.west)+(3.5, -0.49)$) {on{\color{white}{p}}};

   \node (begin) at (0.5, -7.8) {};
   
   \draw ($(begin)$) rectangle ($(begin)+(6, -2.5)$);

   \node (s) [below right=0.6 and 0.7 of begin.west, anchor=west] {};
   \draw ($(s)+(0,0)$) -- ($(s)+(3.0,0.0)$) -- ($(s)+(3.0,-0.5)$) -- ($(s)+(0,-0.5)$);
   \draw ($(s)+(4.5,0)$) -- ($(s)+(3.3,0)$) -- ($(s)+(3.3,-0.5)$) -- ($(s)+(4.5,-0.5)$);
   
   \node[anchor=south west] (t) at ($(s.west)+(0.25, -0.49)$) {to{\color{white}{p}}}
     child[missing]
     child { node {$x$} };

   \node[anchor=south west] (t) at ($(s.west)+(0.9, -0.49)$) {appear}
     child[missing]
     child[missing]
     child {
       node {times}
         child { node {three}}
         child[missing] }
     child { node[right=0.2] {$x$} };
   \node[anchor=south west] (t) at ($(s.west)+(3.5, -0.49)$) {on{\color{white}{p}}};
   
  \end{tikzpicture}
 \end{minipage}
 \caption{(a)-(b) Example of a parse error by the left-corner parser that may be saved with external syntactic knowledge (limited features and beam size 8).
 (c) Two corresponding configuration paths: the left path leads to (a) and the right path leads to (b).}
 \label{fig:error:example}
\end{figure}
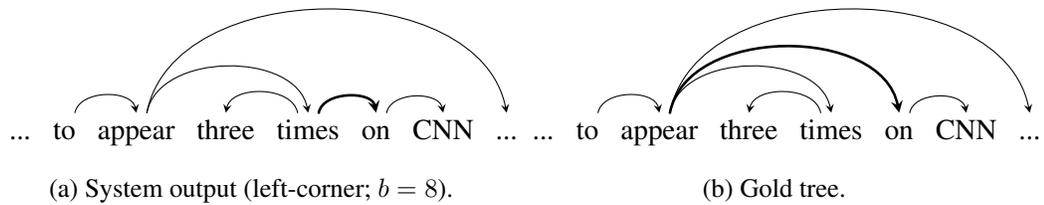
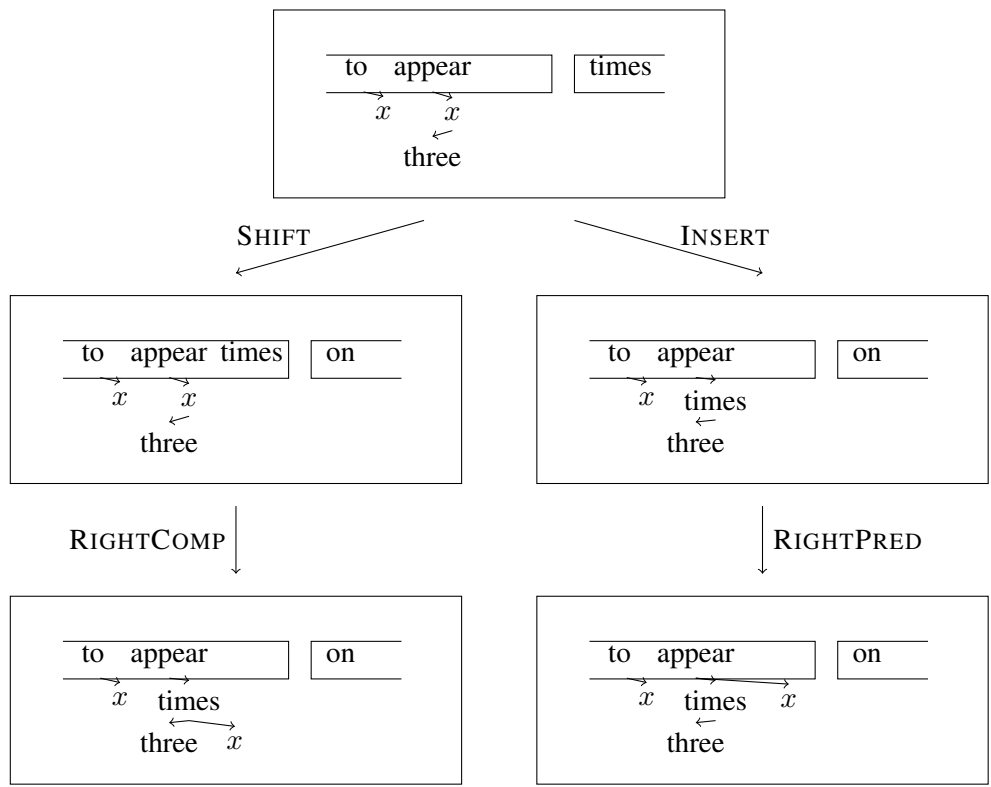

Finally, though the overall score of the left-corner parser is lower,
we suspect that it could be improved by inventing new features, in particular those with external syntactic knowledge.
The analysis below is based on the result with limited features, but we expect a similar technique would also be helpful to the full feature model.

As we have discussed (see the beginning of Section \ref{sec:parse}), an attachment decision of the left-corner parser is more eager, which is the main reason for the lower scores.
One particular difficulty with the left-corner parser is that the parser has to decide whether each token has further (right) arguments with no (or a little) access to the actual right context.
Figure \ref{fig:error:example} shows an example of a parse error in English caused by the left-corner parser with limited features (without stack depth bound).
This is a kind of PP attachment error on {\it on CNN}, though the parser has to deal with this attachment decision implicitly before observing the attached phrase ({\it on CNN}).
When the next token in the buffer is {\it times} (Figure \ref{fig:error:example}(c)), performing {\sc Shift} means {\it times} would take more than one argument in future, while performing {\sc Insert} means the opposite:
{\it times} does not take any arguments.
Resolving this problem would require knowledge on {\it times} that it often takes no right arguments (while {\it appear} generally takes several arguments), but it also suggests that the parser performance could be improved by augmenting such syntactic knowledge on each token as new features, such as with distributional clustering \cite{koo-carreras-collins:2008:ACLMain,BohnetJMA13}, supertagging \cite{ouchi-duh-matsumoto:2014:EACL2014-SP}, or refined POS tags \cite{mueller-EtAl:2014:EMNLP2014}.
All those features are shown to be effective in transition-based dependency parsing;
we expect those are particularly useful for our parser though the further analysis is beyond the scope of this chapter.
In PCFG left-corner parsing, \newcite{TOPS:TOPS12034} reported accuracy improvement with symbol refinements obtained by the Berkeley parser \cite{petrov-EtAl:2006:COLACL} in English.

\subsection{Result on UD}
\label{sec:trans:ud-parse}

Figure \ref{fig:ud_depth_to_accuracy} shows the results in UD.
Again the performance tendency is not changed from the CoNLL dataset;
on average, the left-corner with depth$_{re}$ can parse sentences without dropping accuracies but other systems are largely affected by the constraints.

\begin{figure}[p]
\centering
 \resizebox{0.95\textwidth}{!}
 {\includegraphics[]{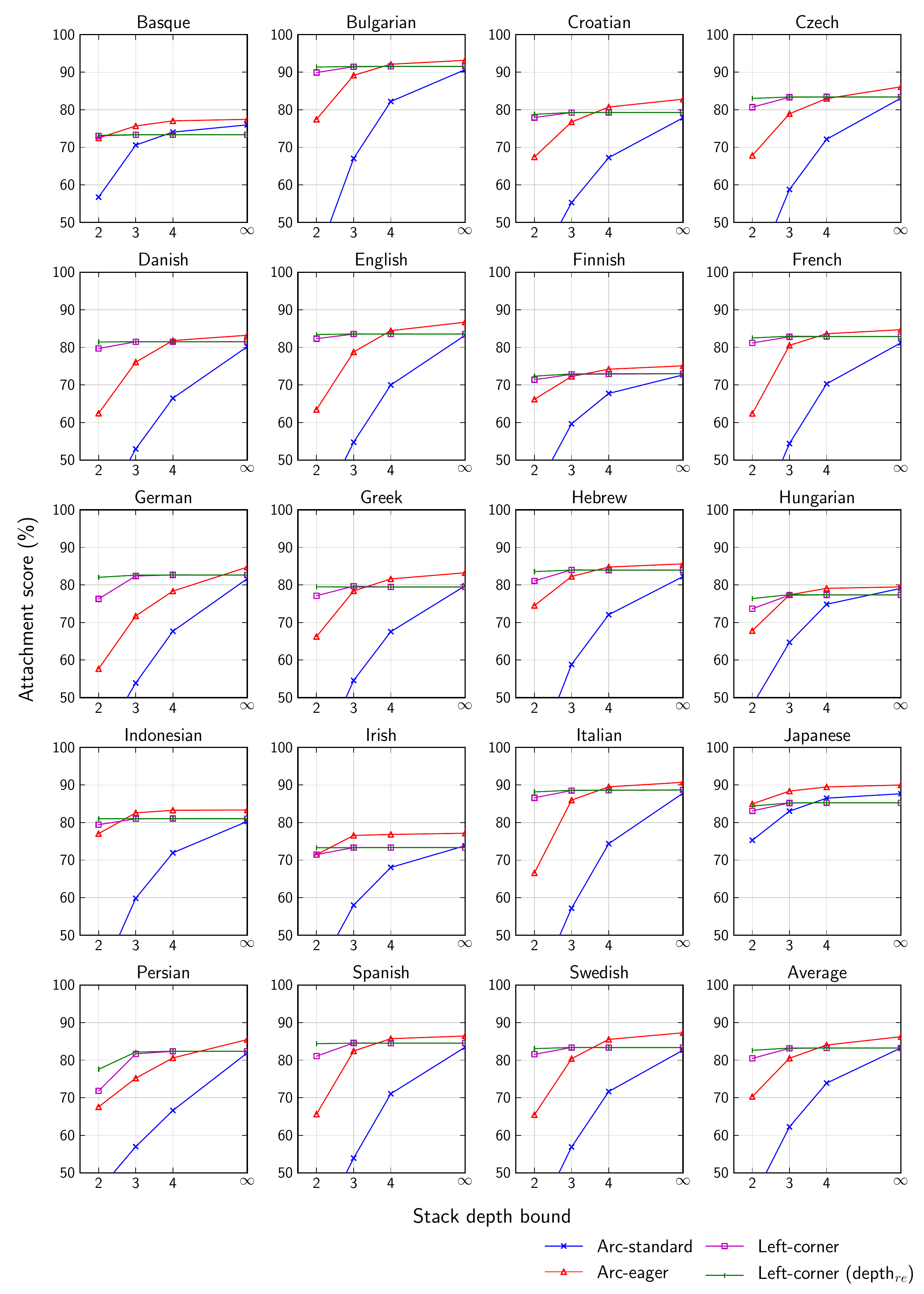}}
\caption{Accuracy vs. stack depth bound in UD.}
 \label{fig:ud_depth_to_accuracy}
\end{figure}

We further examine the behavior of the left-corner parser by relaxing the definition of center-embedding which we discussed in Section \ref{sec:trans:relax}.
Figure \ref{fig:ud_token_depth_to_accuracy} shows the result when we change the definition of depth$_{re}$.
It is intersting to see that compared to Figure \ref{fig:load-comparison-relax-ud},
the number of languages in which this relaxation had a greater impact increases;
 e.g., in Croatian, Czech, Danish, Finnish, Hungarian, Indonesian, Persian, and Swedish, there is about 10\% improvements from the original depth$_{re} \leq 1$ to the relaxed depth$_{re} 1 (3)$ (i.e., when three word constituents are allowed to be embedded).
The reason of this might be in the characteristics of the supervised parsers, which more freely explore the search space (compared to the statistic analysis in Figure \ref{fig:load-comparison-relax-ud}). 


\begin{figure}[p]
\centering
 \resizebox{0.95\textwidth}{!}
 {\includegraphics[]{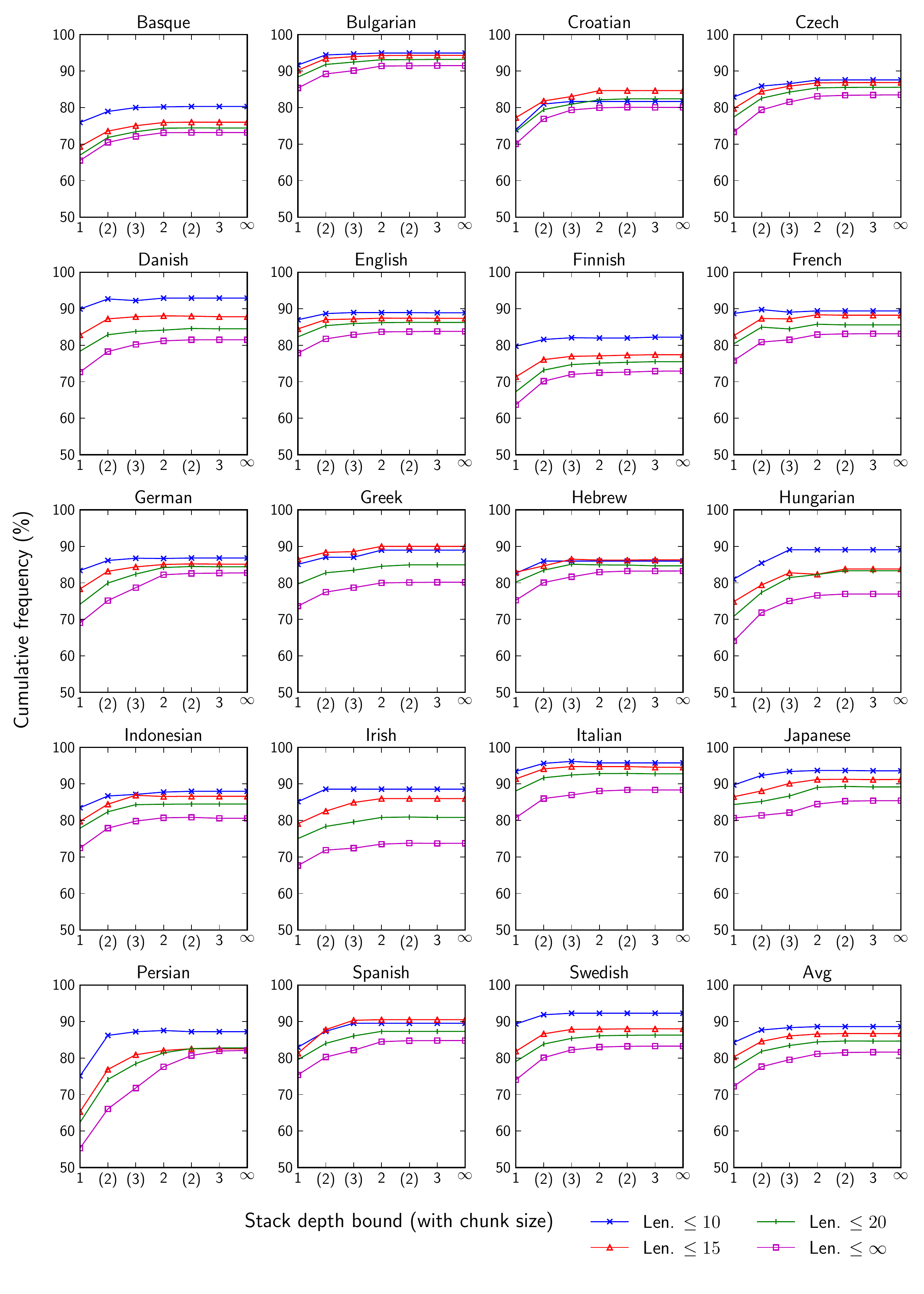}}
\caption{Accuracy vs. stack depth bound with left-corner parsers on UD with different maximum length of test sentences.}
 \label{fig:ud_token_depth_to_accuracy}
\end{figure}

\section{Discussion and Related Work}
\label{sec:relatedwork}

We have presented the left-corner parsing algorithm for dependency structures and showed that our parser demands less stack depth for recognizing most of natural language sentences.
The result also indicates the existence of universal syntactic biases that center-embedded constructions are rare phenomena across languages.
We finally discuss the relevance of the current study to the previous works.

We have reviewed previous works about left-corner parsing (for CFGs) in Section \ref{sec:bg:left-corner} though have little mentioned previous works that study the empirical property of the left-corner parsers.
\newcite{Roark:2001:RPP:933637} is the first attempt of the empirical study.
His idea is instead of modeling left-corner transitions directly as in our parser, incorporating the left-corner strategy into a CFG parser via a left-corner grammar transform \cite{conf/acl/Johnson98}.
This design makes the overall parsing system top-down and makes it possible to compare the pure top-down and the left-corner parsing systems in a unified way.
Note also that as his method is based on \newcite{conf/acl/Johnson98}, the parsing mechanism is basically the same as the left-corner PDA that we introduced as another variant in Section \ref{sec:bg:anothervariant}.
\newcite{journals/coling/SchulerAMS10} examine the empirical coverage result of the left-corner PDA that we formalized in Section \ref{sec:bg:left-corner-pda}, though the experiment is limited on English.

Most of previous left-corner parsing models have been motivated by the study of cognitively plausible parsing models, an interdisciplinary research on psycholinguistics and computational linguistics \cite{keller:2010:Short}.
Though we also evaluated our parser with cognitively motivated limited feature models and got an encouraging result, this is preliminary and we do not claim from this experiment that our parser is cross-linguistically cognitively plausible.
Our parser is able to parse most sentences within a limited stack depth bound.
However, it is skeptical whether there is any connection between the stack of our parser and memory units preserved in human memory.
\newcite{vanschijndel-schuler:2013:NAACL-HLT} calculated several kinds of {\it memory cost} obtained from a configuration of their left-corner parser and discussed which cost is more significant indicator to predict human reading time data, such as the current stack depth and the integration cost in the dependency locality theory \cite{Gibson2000The-dependency-}, which is obtained by calculating the distance between two subtrees at composition.
Discussing cognitive plausibility of a parser requires such kind of careful experimental setup, which is beyond the scope of the current work.

Our main focus in this chapter is rather a syntactic bias exist in language universally.
In this view, our work is more relevant to previous dependency parsing model with a constraint on possible tree structures \cite{eisner2010}.
They studied parsing with a hard constraint on dependency length based on the observation that grammar may favor a construction with shorter dependency lengths \cite{gildea-temperley:2007:ACLMain,DBLP:journals/cogsci/GildeaT10}.
Instead of prohibiting longer dependency lengths, our method prohibits deeper center-embedded structures, and we have shown that this bias is effective to restrict natural language grammar.
The two constraints, length and center-embedding, are often correlated since center-embedding constructions typically lead to longer dependency length.
It is therefore an interesting future topic to explore which bias is more essential for restricting grammar.
This question can be perhaps explored through unsupervised dependency parsing tasks \cite{klein-manning:2004:ACL}, where such kind of light supervision has significant impact on the performance \cite{smith-eisner-2006-acl-sa,marevcek-vzabokrtsky:2012:EMNLP-CoNLL,DBLP:journals/tacl/BiskH13}.

We introduced a dummy node for representing a subtree with an unknown head or dependent.
Recently, Menzel and colleagues \cite{beuck2013,kohn-menzel:2014:ACL} have also studied dependency parsing with a dummy node.
While conceptually similar, the aim of introducing a dummy node is different between our approach and theirs:
We need a dummy node to represent a subtree corresponding to that in Resnik's algorithm, while they introduced it to confirm that every dependency tree on a sentence prefix is fully connected.
This difference leads to a technical difference;
a subtree of their parser can contain more than one dummy node, while we restrict each subtree to containing only one dummy node on a right spine.

\chapter{Grammar Induction with Structural Constraints}
\label{chap:induction}

In the previous chapter, we formulated a left-corner dependency parsing algorithm as a transition system in which its stack size grows only for center-embedded constructions.
Also, we investigated how much the developed parser can capture the syntactic biases found in the manually developed treebanks, and found that very restricted stack depth such as two or one (by allowing small consituents to be embedded) suffices to describe most syntactic constructions across languages.

In this chapter, we will investigate whether the found syntactic bias in the previous chapter would be helpful for the task of {\it unsupervised grammar induction}, where the goal is to learn the model of finding hidden syntactic structures given the surface strings (or part-of-speeches) alone;
see Section \ref{sec:2:unsupervised} for overviews.

There are a number of motivations to consider unsupervised grammar induction, in particular with the {\it universal} syntactic biases as we discussed in Chapter \ref{chap:1}.
Among them our primary motivation is to investigate a good {\it prior} that would be useful for restricting possible tree structures for general natural language sentences (regardless of language).
The structure that we aim to induce is dependency structure;
though this choice mainly stems from computational reasons rather than philosophical ones, i.e., the dependency structure is currently the most feasible structure to be learned, we argue the lesson from the current study would be useful for the problem of inducing other structures including constituent-based representations, e.g., HPSG or CCG.

Another interesting reason to tackle this problem is to understand the mechanism of child language acquisition.
In particular, since the structural constraint that we impose originally is motivated by psycholinguistic observations (Section \ref{sec:bg:psycho}), we can regard the current task as controlled experiments to see whether the (memory) limitation that children may suffer from may in turn facilitate the acquisition of language.
This is however not our primary motivation since there are large gaps between the actual environment in which children acquire language and the current task;
see Section \ref{sec:intro:unsupervised} for the detailed discussion.
We therefore think the current study to be a starting point for the modeling of a more realistic acquisition scenario, such as the joint inference of word categories and syntax.

As in the previous chapter, this chapter starts with the conceptual part, in which the main focus is the learning algorithm with structural constraints, then followed by the empirical part that focuses on experiments.
Our model is basically the dependency model with valence \cite{klein-manning:2004:ACL} that we formalized as a special instance of split bilexical grammar (SBG) in Section \ref{sec:2:dmv}.
We describe how this model can be encoded in a chart parser that simulates left-corner dependency parsing as presented in the previous chapter, which captures the {\it center-embeddedness} of a subderivation at each chart entry.
Intuitively, with this technique we can bias the model to prefer some syntactic patterns, e.g., that do not contain many center-embedding.
We discuss the high level idea and mathematical formalization of this approach in Section \ref{sec:ind:overview} and then present a new chart parsing algorithm that simulates split bilexical grammars in a left-corner parsing strategy \ref{sec:ind:simulate}.
We then empirically evaluate whether such structural constraints would help to learn good parameters for the model (Sections \ref{sec:ind:setup} and \ref{sec:ind:exp}).
As in the previous chapter, we study this effect across diverse languages;
the total number of treebanks that we use is 30 across 24 languages.


Our main empirical finding is that the constraint on center-embeddedness brings at least the same or superior effects as the closely related structural bias on dependency {\it length} \cite{smith-eisner-2006-acl-sa}, i.e., the preference for {\it shorter} dependencies.
In particular, we find that our bias often outperforms length-based ones when additional syntactic cues are given to the model, such as the one that the sentence root should be a verb or a noun.
For example, when such a constraint on the root POS tag is given, our method that penalizes center-embeddedness achieves an attachment score of 62.0 on Google universal treebanks (averaged across 10 languages, evaluated on length $\leq 10$ sentences), which is superior to the model with the bias on dependency length (58.6) and the model utilizing a larger number of hand crafted rules between POS tags (56.0) \cite{naseem-EtAl:2010:EMNLP}.


\section{Approach Overview}
\label{sec:ind:overview}


\subsection{Structure-Constrained Models}

Every model presented in this section can be formalized as the following joint model over a sentence $x$ and a parse tree $z$:
\begin{equation}
 p(x,z|\theta) = \frac{p_\textsc{orig}(x,z|\theta) \cdot f(z,\theta)}{Z(\theta)} \label{eqn:ind:joint}
\end{equation}
where $p_\textsc{orig}(x,z|\theta)$ is a (baseline) model, such as DMV.
$f(z,\theta)$ assigns a value between $[0,1]$ for each $z$, i.e., it works as a penalty or a cost, {\it reducing} the original probability depending on $z$.
One such penalty that we consider is prohibiting any trees that contain any center-embedding, which is represented as follows:
\begin{equation}
 f(z,\theta) = \left\{ \begin{array}{cl}
  1& \textrm{if } z \textrm{ contains no center-embedding;} \\
  0& \textrm{otherwise.} \\
                      \end{array}
               \right. \label{eqn:ind:cost}
\end{equation}
In Section \ref{sec:ind:simulate}, we present a way to encode such a penalty term during the CKY-style algorithm.

Though $f(z,\theta)$ works as adding a penalty to each original probability, the distribution $p(x,z|\theta)$ is still normalized;
here $Z(\theta) = \sum_{x,z} p_\textsc{orig}(x,z|\theta) \cdot f(z,\theta)$.

Intuitively, $f(z,\theta)$ models the preferences that the original model $p_\textsc{orig}(x,z|\theta)$ does not explicitly encode.
Note that we do not try to learn $f(z,\theta)$;
every constraint is given as an {\it external} constraint.

Note that this simple approach to combine two models is not entirely new and has been explored several times.
For example, \newcite{pereira-schabes:1992:ACL} present an EM algorithm that relies on partially bracketed information and \newcite{smith-eisner-2006-acl-sa} model $f(z,\theta)$ as the dependency length-based penalty term.
We explore several kinds of $f(z,\theta)$ in our experiments including the existing one, e.g., dependency length, and our new idea, center-embeddedness, examining which kind of structural constraint is most helpful for learning grammars in a cross-linguistic setting.
Below, we discuss the issues on learning of this model.
The main result was previously shown in \newcite{smith-2006} though we summarize it here in our own terms defined in Chapter \ref{chap:bg} for completeness.

\subsection{Learning Structure-Constrained Models}

At first glance, the normalization constant $Z(\theta)$ in Eq. \ref{eqn:ind:joint} seems to prevent the use of the EM algorithm for parameter estimation for this model.
We show here that in practice we need not care about this constant and the resulting EM algorithm will increase the likelihood of the model of Eq. \ref{eqn:ind:joint}.

Recall that the EM algorithm collects expected counts for each rule $r$, $e(r|\theta)$ at each E-step and then normalizes the counts to update the parameters.
We decomposed $e(r|\theta)$ into the counts for each span of each sentence as follows:
\begin{equation}
 e(r|\theta) = \sum_{x \in \mathbf x} e_x(r|\theta) = \sum_{x \in \mathbf x} \sum_{0 \leq i \leq k \leq j \leq n_x} e_x(z_{i,k,j,r}|\theta).
\end{equation}

We now show that correct $e_x(z_{i,k,j,r}|\theta)$ under the model (Eq. \ref{eqn:ind:joint}) is obtained without a need to compute $Z(\theta)$.
Let $q(x,z|\theta) = p_{\textsc{orig}}(x,z|\theta) \cdot f(z,\theta)$ be an {\it unnormalized} (i.e., deficient) distribution over $x$ and $z$.
Then $p(x,z|\theta) = q(x,z|\theta)/Z(\theta)$.
Note that we can use the inside-outside algorithm to collect counts under the deficient distribution $q(x,z|\theta)$.
For example, we can obtain the (deficient) sentence marginal probability $q(x|\theta) = \sum_{z\in \mathcal{Z}(x)} q(x,z|\theta)$ by modifying rule probabilities appropriately.
More specifically, in the case of eliminating center-embedding, our chart may record the current stack depth at each chart entry that corresponds to some subderivation, and then assign zero probability to a chart entry if the stack depth exceeds some threshold.

We can represent $e_x(z_{i,k,j,r}|\theta)$ using $q(x,z|\theta)$ instead of $p(x,z|\theta)$, which is more complex.
The key observation is that each $e_x(z_{i,k,j,r}|\theta)$ is represented as the ratio of two quantities:
\begin{equation}
 e_x(z_{i,k,j,r}|\theta) = \frac{p(z_{i,k,j,r} = 1, x | \theta)}{p(x|\theta)}. \label{eqn:ind:expected}
\end{equation}
Calculating these quantities is hard due to the normalization constant.
However, as we show below, the normalization constant is canceled in the course of computing the ratio, meaning that the expected counts (under the correct distribution) are obtained with the inside-outside algorithm under the unnormalized distribution $q(x,z|\theta)$.
Let us first consider the denominator in Eq. \ref{eqn:ind:expected}, which can be rewritten as follows:
\begin{equation}
 p(x|\theta) = \sum_{z\in \mathcal Z(x)} p(x,z|\theta) = \sum_{z\in \mathcal Z(x)} \frac{q(x,z|\theta)}{Z(\theta)} = \frac{q(x|\theta)}{Z(\theta)}.
\end{equation}
For the numerator, we first observe that
\begin{align}
 p(z_{i,k,j,r} = 1, x |\theta) &= \sum_{z\in\mathcal{Z}(x)} p(z_{i,k,j,r} = 1, z, x |\theta) \\
 &= \sum_{z\in\mathcal{Z}(x)} p(z, x |\theta) p(z_{i,k,j,r} = 1|z) \\
 &= \sum_{z\in\mathcal{Z}(x)} p(z, x |\theta) \mathbb{I} (z_{i,j,k,r} \in z) \\
 &=  \frac{\sum_{z\in\mathcal{Z}(x)} q(x,z|\theta) \mathbb{I} (z_{i,j,k,r} \in z)}{Z(\theta)}, \label{eq:ind:suff}
\end{align}
where $\mathbb{I}(c)$ is an identity function that returns $1$ if $c$ is satisfied and $0$ otherwise.
The numerator in Eq. \ref{eq:ind:suff} is the value that the inside-outside algorithm calculates for each $z_{i,j,k,r}$ (with Eq. \ref{eqn:2:io}), which we write as $q(z_{i,k,j,r} = 1, x | \theta)$.
Thus, we can skip computing the normalization constant in Eq. \ref{eqn:ind:expected} by replacing the quantities with the ones under $q(x,z|\theta)$ as follows:
\begin{equation}
 e_x(z_{i,k,j,r}|\theta) = \frac{p(z_{i,k,j,r} = 1, x | \theta)}{p(x|\theta)} = \frac{q(z_{i,k,j,r} = 1, x | \theta)}{q(x|\theta)}.
\end{equation}
The result indicates that by running the inside-outside algorithm {\it as if} our model is deficient, using $q(x,z|\theta)$ in place of $p(z,x|\theta)$, we can obtain the model with higher likelihood of $p(x,z|\theta)$ (Eq. \ref{eqn:ind:joint}).
Note that the viterbi parse can also be obtained using $q(x,z|\theta)$ since $\arg\max_{z}q(x,z|\theta) = \arg\max_{z} p(x,z|\theta)$ holds.


\section{Simulating split-bilexical grammars with a left-corner strategy}
\label{sec:ind:simulate}
Here we present the main theoretical result in this chapter.
In Section \ref{sec:2:learning} we showed that the parameters of the very general model for dependency trees called split-bilexical grammars (SBGs) can be learned using the EM algorithm with CKY-style inside-outside calculation.
Also, we formalized the left-corner dependency parsing algorithm as a transition system in Chapter \ref{chap:transition}, which enables capturing the {\it center-embeddedness} of the current derivation via {\it stack depth}.
We combine these two parsing techniques in a non-trivial way, and obtain a new chart parsing method for split-bilexical grammars that enables us to calculate center-embeddedness of each subderivation at each chart entry.


\REVISE{
We describe the algorithm based on the inference rules with items (as the triangles in Section \ref{sec:2:sbg}).
The basic idea is that we {\it memoize} the subderivations of left-corner parsing, which share the same information and look the same under the model.
Basically, each chart item is a stack element;
for example, an item \tikz[baseline=-20pt]{\headtriangleinlinebottomfull{$i$}{$h$}{$j$}} abstracts (complete) subtrees on the stack headed by $h$ spanning $i$ to $j$.
Each inference rule then roughly corresponds to an action of the transition system.\footnote{We also introduce extra rules, which are needed for encoding parameterization of SBGs, or achieving head-splitting as we describe later.}
Thus, if we extract one derivation from the chart, which is a set of inference rules, it can be mapped to a particular sequence of transition actions.
On this chart, each item is further decorated with the current stack depth, which is the key to capture the center-embeddedness efficiently during dynamic programming.


A particular challenge for efficient tabulation is similar to the one that we discussed in Section \ref{sec:2:sbg};
that is, we need to eliminate the spurious ambiguity for correct parameter estimation and for reducing time complexity.
In Section \ref{sec:ind:algorithm}, we describe how this can be achieved by applying the idea of head-splitting into the tabulation of left-corner parsing.
In the following two sections we discuss some preliminaries for developing the algorithm, i.e., how to handle dummy nodes on the chart (Section \ref{sec:ind:handling}) and a note on parameterization of SBGs with the left-corner strategy (Section \ref{sec:bg:head-outward}).
}

\subsection{Handling of dummy nodes}
\label{sec:ind:handling}
An obstacle when designing chart items abstracting many derivations is the existence of predicted nodes, which were previously abstracted with {\it dummy} nodes in the transition system.
Unfortunately, we cannot use the same device in our dynamic programming algorithm because it leads to very inefficient asymptotic runtime.
Figure \ref{fig:ind:dummy-is-bad} explains the reason for this inefficiency.
In the transition system, we postponed scoring of attachment preferences between a dummy token and its left dependents (e.g., Figure \ref{subfig:ind:dummy1}) until {\it filling} the dummy node with an actual token by an {\sc Insert} action;
this mechanism makes the algorithm fully incremental, though it requires remembering every left dependent token (see {\sc Insert} in Figure \ref{fig:actions}) at each step.
This tracking of child information is too expensive for our dynamic programming algorithm.
To solve this problem, we instead fill a dummy node with an actual token when the dummy is first introduced (not when {\sc Insert} is performed).
This is impossible in the setting of a transition system since we do not observe the unread tokens in the portion of the sentence remaining in the buffer.
Figure \ref{subfig:ind:pred-rect} shows an example of an item used in our dynamic programming, which does not abstracts the predicted token as a dummy node, but abstracts the construction of child subtrees spanning $i$ to $j$ below the predicted node $p$.
An arc from $p$ indicates that at least one token between $i$ and $j$ is a dependent of $p$, although the number of dependents as well as the positions are unspecified.

\begin{figure}[t]
 \centering
 \begin{minipage}[b]{.3\linewidth}
  \centering
  \begin{dependency}[theme=simple]
   \begin{deptext}[column sep=0.5cm]
    $1$ \& $2$ \& $3$ \& $x$ \\
   \end{deptext}
   \depedge{4}{1}{}
   \depedge{4}{2}{}
   \depedge{4}{3}{}
  \end{dependency}
  \subcaption{}\label{subfig:ind:dummy1}
 \end{minipage}
 \begin{minipage}[b]{.3\linewidth}
  \centering
  \begin{dependency}[theme=simple]
   \begin{deptext}[column sep=0.5cm]
    $1$ \& $2$ \& $3$ \& $x$ \\
   \end{deptext}
   \depedge{4}{1}{}
   \depedge{4}{2}{}
   \depedge{2}{3}{}
  \end{dependency}
  \subcaption{}
 \end{minipage}
 \begin{minipage}[b]{.3\linewidth}
  \centering
  \begin{tikzpicture}[thick, level distance=0.8cm, >=stealth']
   \draw (0,0) rectangle (2.0, -1.0);
   \draw[->] (3, -1.0) arc (0:180:0.9cm);
   \draw (0, -1.25) node {$i$} ++(2.0, 0) node {$j$} +(1.0, 0) node {$p$};
  \end{tikzpicture}
  \subcaption{}\label{subfig:ind:pred-rect}
 \end{minipage}
 \caption{Dummy nodes ($x$ in (a) and (b)) in the transition system cannot be used in our transition system because with this method, we have to remember every child token of the dummy node to calculate attachment scores at the point when the dummy is filled with an actual token, which leads to an exponential complexity.
 We instead abstract trees in a different way as depicted in (c) by not abstracting the predicted node $p$ but filling with the actual word ($p$ points to some index in a sentence such that $j < p \leq n$).
 If $i=1, j=3$, this representation abstracts both tree forms of (a) and (b) with some fixed $x$ (corresponding to $p$).
 }
 \label{fig:ind:dummy-is-bad}
\end{figure}
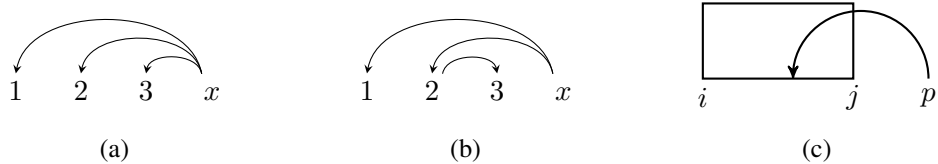

\subsection{Head-outward and head-inward}
\label{sec:bg:head-outward}
The generative process of SBGs described in Section \ref{sec:2:sbg} is {\it head-outward}, in that its state transition $q_1 \xmapsto{~a~} q_2$ is defined as a process of {\it expanding} the tree by generating a new symbol $a$ which is the most distant from its head when the current state is $q_1$.
The process is called {\it head-inward} (e.g., \newcite{alshawi:1996:ACL}) if it is reversed, i.e., when the closest dependent of a head on each side is generated last.
Note that the generation process of left-corner parsing cannot be described fully head-outward.
In particular, its generation of left dependents of a head is inherently head-inward since a parser builds a tree from left to right.
For example, the tree of Figure \ref{subfig:ind:dummy1} is constructed by first doing {\sc LeftPred} when $1$ is recognized, and then attaching $2$ and $3$ in order.
Fortunately, these two processes, head-inward and head-outward, can generally be interchanged by reversing transitions \cite{Eisner:1999:EPB:1034678.1034748}.
In the algorithm described below, we model its left automaton $L_a$ as a head-inward process while right automaton $R_a$ as a head-outward process.
Specifically, that means if we write $q_1 \xmapsto{~a'~} q_2 \in L_a$, the associated weight for this transition is $p(a'|q_2)$ instead of $p(a'|q_1)$.
We also do not modify the meaning of sets $\textit{final}(L_a)$ and $\textit{init}(L_a)$; i.e., the left state is initialized with $q \in \textit{final}(L_a)$ and finishes with $q \in \textit{init}(L_a)$.


\subsection{Algorithm}
\label{sec:ind:algorithm}

\begin{figure}[p]
 \begin{tikzpicture}[thick, level distance=0.8cm, >=stealth']
  \begin{scope}[xshift=0cm,yshift=0cm]
   \node at (0, 0) [anchor=west] {\sc Shift-Left:};
   \node at (1.25, -0.5) {$q \in \textit{final}(L_h)$};
   \draw (0.1, -0.75) -- +(2.5, 0) node[right] {$1 \leq h \leq n$};
   \begin{scope}[xshift=1.5cm, yshift=-1.3cm]
    \lefttriangle{$q:d$}{$h$}{$h$};
   \end{scope}
  \end{scope}
  
  \begin{scope}[xshift=0cm,yshift=-4cm]
   \node at (0, 0) [anchor=west] {\sc Finish-Left:};
   \begin{scope}[xshift=1.0cm, yshift=-0.75cm]
    \lefttriangle{$q:d$}{$i$}{$h$};
   \end{scope}
   \node at (1.5, -1.25) [anchor=west] {$q \in \textit{init}(L_h)$};
   \draw (0.1, -1.75) -- +(4.1, 0);
   \begin{scope}[xshift=2.0cm, yshift=-2.3cm]
    \lefttriangle{$I:d$}{$i$}{$h$};
   \end{scope}
  \end{scope}
  
  \begin{scope}[xshift=5cm,yshift=0cm]
   \node at (0, 0) [anchor=west] {\sc Shift-Right:};
   \node at (1.25, -0.5) {$q \in \textit{init}(R_h)$};
   \draw (0.1, -0.75) -- +(2.5, 0) node[right] {$1 \leq h \leq n$};
   \begin{scope}[xshift=0.75cm, yshift=-1.3cm]
    \righttriangle{$q:d$}{$h$}{$h$};
   \end{scope}
  \end{scope}

  \begin{scope}[xshift=5cm,yshift=-4.cm]
   \node at (0, 0) [anchor=west] {\sc Finish-Right:};
   \begin{scope}[xshift=0.5cm, yshift=-0.75cm]
    \righttriangle{$q:d$}{$h$}{$i$}
   \end{scope}
   \node at (1.5, -1.25) [anchor=west] {$q \in \textit{final}(R_h)$};
   \draw (0.1, -1.75) -- +(4.1, 0);
   \begin{scope}[xshift=1.25cm, yshift=-2.3cm]
    \righttriangle{$F:d$}{$h$}{$i$}
   \end{scope}
  \end{scope}

  \begin{scope}[xshift=10cm,yshift=0cm]
   \node at (0, 0) [anchor=west] {\sc Insert-Left:};
   \begin{scope}[xshift=0.25cm, yshift=-0.5cm]
    \predrectangle{$i$}{$p-1$}{$p$}{$q:d$}
   \end{scope}
   \node at (2.75, -1.0) [anchor=west] {$q \in \textit{init}(L_p)$};
   \draw (0.1, -1.5) -- +(5.1, 0);
   \begin{scope}[xshift=2.5cm, yshift=-2.05cm]
    \lefttriangle{$I:d$}{$i$}{$p$}
   \end{scope}
  \end{scope}

  \begin{scope}[xshift=10cm,yshift=-4cm]
   \node at (0, 0) [anchor=west] {\sc Insert-Right:};
   \begin{scope}[xshift=0.5cm, yshift=-0.75cm]
    \predright{$r:d$}{$h$}{$p-1$}{$p$}{$q$}
   \end{scope}
   \node at (2.75, -1.25) [anchor=west] {$q \in \textit{init}(L_p)$};
   \draw (0.1, -1.75) -- +(5.1, 0);
   \begin{scope}[xshift=2.0cm, yshift=-2.3cm]
    \righttriangle{$r:d$}{$h$}{$p$}
   \end{scope}
  \end{scope}

  \begin{scope}[xshift=0cm,yshift=-8cm]
   \node at (0, 0) [anchor=west] {\sc LeftPred:};
   \begin{scope}[xshift=1.0cm, yshift=-0.75cm]
    \headtriangle{$:d$}{$i$}{$h$}{$j$}
   \end{scope}
   \node at (2.5, -0.75) [anchor=west] {$r \in \textit{final}(L_p)$};
   \node at (2.5, -1.25) [anchor=west] {$r \xmapsto{~h~} q \in L_p$};
   \draw (0.1, -1.75) -- +(5.1, 0);
   \begin{scope}[xshift=1.5cm, yshift=-2cm]
    \predrectangle{$i$}{$j$}{$p$}{$q:d$}
   \end{scope}
  \end{scope}

  \begin{scope}[xshift=6cm,yshift=-8cm]
   \node at (0, 0) [anchor=west] {\sc RightPred:};
   \begin{scope}[xshift=1.0cm, yshift=-0.75cm]
    \righttriangle{$r:d$}{$h$}{$i$}
   \end{scope}
   \node at (2.5, -0.75) [anchor=west] {$q \in \textit{final}(L_p)$};
   \node at (2.5, -1.25) [anchor=west] {$r \xmapsto{~p~} r' \in R_h$};
   \draw (0.1, -1.75) -- +(5.1, 0);
   \begin{scope}[xshift=1.5cm, yshift=-2.25cm]
    \predright{$r':d$}{$h$}{$i$}{$p$}{$q$}
   \end{scope}
  \end{scope}

  \begin{scope}[xshift=0cm,yshift=-12cm]
   \node at (0, 0) [anchor=west] {\sc Combine:};
   \begin{scope}[xshift=1.0cm, yshift=-0.75cm]
    \lefttriangle{$I:d$}{$i$}{$h$}
   \end{scope}
   \begin{scope}[xshift=2.5cm, yshift=-0.75cm]
    \righttriangle{$F:d$}{$h$}{$j$}
   \end{scope}
   \draw (0.2, -1.75) -- +(3.1, 0);
   \begin{scope}[xshift=1.5cm, yshift=-2.25cm]
    \headtriangle{$:d$}{$i$}{$h$}{$j$}
   \end{scope}
  \end{scope}

  \begin{scope}[xshift=5cm,yshift=-12cm]
   \node at (0, 0) [anchor=west] {\sc Accept:};
   \begin{scope}[xshift=1.25cm, yshift=-0.75cm]
    \lefttriangle{$I:1$}{}{};
    \draw[anchor=mid] (-0.65, -0.75) node {$1$} +(1.15, 0) node {$n+1$};
    \end{scope}
   \draw (0.2, -1.75) -- +(2.1, 0);
   \node at (1.25, -2.1) {\textit{accept}};
  \end{scope}
  
 \end{tikzpicture}
 \caption{An algorithm for parsing SBGs with a left-corner strategy in $O(n^4)$ given a sentence of length $n$, except the composition rules which are summarized in Figure \ref{fig:ind:ded-comp}.
 The $n+1$-th token is a dummy root token $\$$, which only has one left dependent (sentence root).
 $i,j,h,p$ are indices of tokens.
 The index of a head which is still collecting its dependents is decorated with a state (e.g., $q$).
 $L_h$ and $R_h$ are left and right FSAs of SBGs given a head index $h$, respectively;
 we reverse the proces of $L_h$ to start with $q\in\textit{final}(L_h)$ and finish with $q\in\textit{init}(L_h)$ (see the body).
 Each item is also decorated with depth $d$ that corresponds to the stack depth incurred when building the corresponding tree with left-corner parsing.
 Since an item with larger depth is only required for composition rules, the depth is unchanged with the rules above, except {\sc Shift-*}, which corresponds to {\sc Shift} transition and can be instantiated with arbitrary depth.
 Note that {\sc Accept} is only applicable for an item with depth 1, which guarantees that the successful parsing process remains a single tree on the stack.
 Each item as well as a statement about a state (e.g., $r\in \textit{final}(L_p)$) has a weight and the weight of a consequence item is obtained by the product of the weights of its antecedent items.
 }
 \label{fig:ind:ded1}
\end{figure}
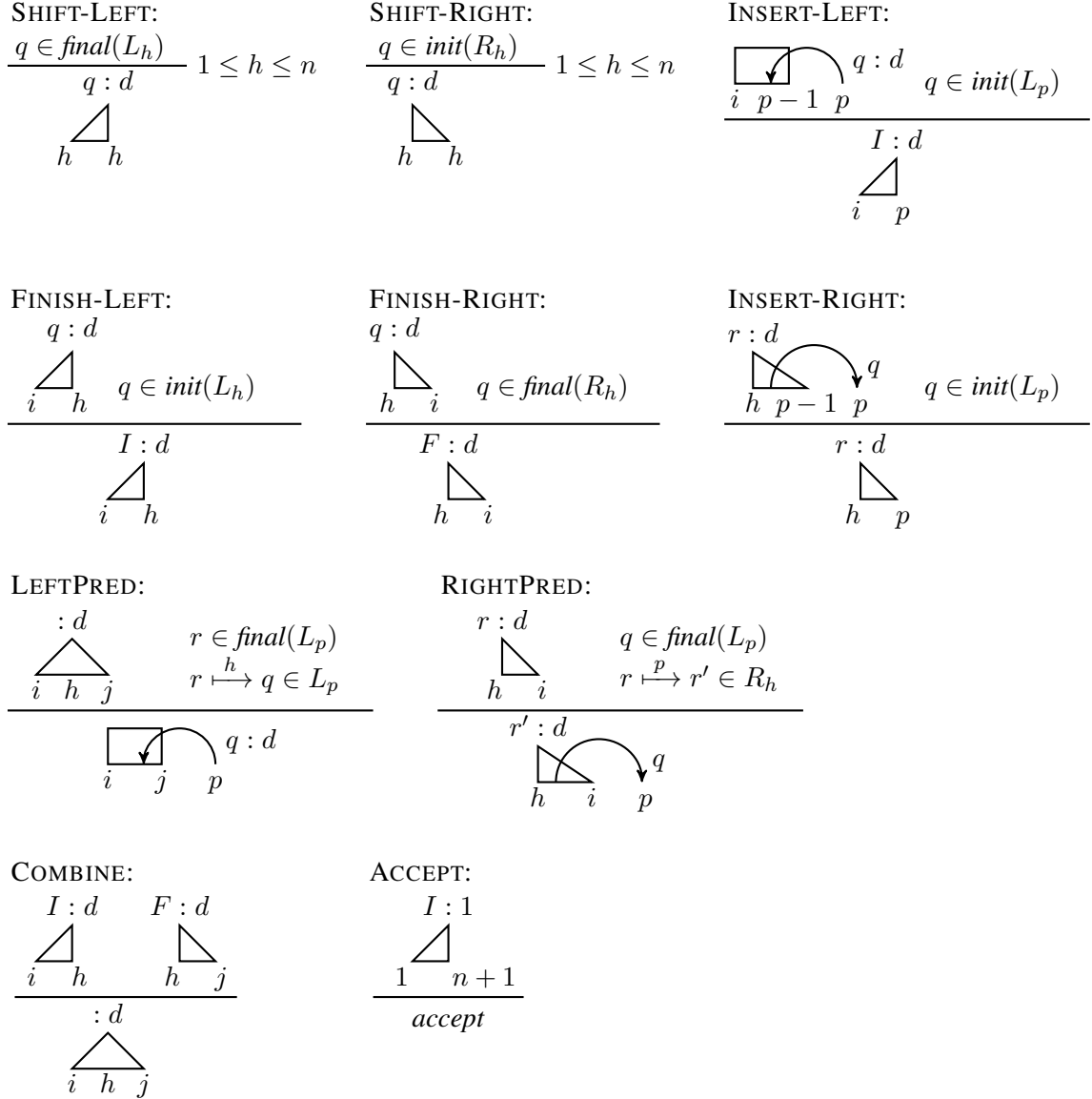

\begin{figure}[p]
 \begin{tikzpicture}[thick, level distance=0.8cm, >=stealth']
  \begin{scope}[xshift=0cm,yshift=0cm]
   \node at (0, 0) [anchor=west] {\sc LeftComp-L-1:};
   \begin{scope}[xshift=0.25cm, yshift=-0.5cm]
    \predrectangle{$i$}{$j$}{$p$}{$q:d$}
   \end{scope}
   \begin{scope}[xshift=4cm, yshift=-0.5cm]
    \lefttriangle{$I:d$}{$j+1$}{$h$}
   \end{scope}
   \draw (0.1, -1.5) -- +(4.4, 0);
   \node at (4.8, -1.25)[anchor=west,text width=3cm] {
   \begin{equation*}
    b= \left\{ \begin{array}{cl}
       1& \textrm{if } h-(j+1) \geq 1 \\
       0& \textrm{otherwise.} \\
              \end{array}
       \right.
   \end{equation*}
   };
   \begin{scope}[xshift=1.75cm, yshift=-2.05cm]
    \halfrectangle{$i$}{$h$}{$p$}{$q:d$}
   \end{scope}
  \end{scope}
  
  \begin{scope}[xshift=0cm,yshift=-4cm]
   \node at (0, 0) [anchor=west] {\sc LeftComp-L-2:};
   \begin{scope}[xshift=0.25cm, yshift=-0.75cm]
    \halfrectangle{$i$}{$h$}{$p$}{$q:d$}
   \end{scope}
   \begin{scope}[xshift=3.5cm, yshift=-0.75cm]
    \righttriangle{$F:d'$}{$h$}{$j$}
   \end{scope}
   \node [anchor=west] at (4.5, -1.25) {$q \xmapsto{~h~} q' \in L_p$};
   \draw (0.1, -1.75) -- +(6.6, 0);
   \node at (7.0, -1.5)[anchor=west,text width=3cm] {
   \begin{equation*}
    d'= \left\{ \begin{array}{ll}
       d+1& \textrm{if } b=1 \textrm{ or } (j-h) \geq 1 \\
       d& \textrm{otherwise.} \\
              \end{array}
       \right.
   \end{equation*}
   };
   \begin{scope}[xshift=2.5cm, yshift=-2.0cm]
    \predrectangle{$i$}{$j$}{$p$}{$q':d$}
   \end{scope}
  \end{scope}
  
  \begin{scope}[xshift=0cm,yshift=-8cm]
   \node at (0, 0.) [anchor=west] {\sc LeftComp-R-1:};
   \begin{scope}[xshift=0.25cm, yshift=-0.75cm]
    \predright{$r:d$}{$i$}{$j$}{$p$}{$q$}
   \end{scope}
   \begin{scope}[xshift=4cm, yshift=-0.75cm]
    \lefttriangle{$I:d$}{$j+1$}{$h$}
   \end{scope}
   \draw (0.1, -1.75) -- +(4.4, 0);
   \node at (4.8, -1.5)[anchor=west,text width=3cm] {
   \begin{equation*}
    b= \left\{ \begin{array}{ll}
       1& \textrm{if } h-(j+1) \geq 1 \\
       0& \textrm{otherwise.} \\
              \end{array}
       \right.
   \end{equation*}
   };
   \begin{scope}[xshift=1.75cm, yshift=-2.3cm]
    \halfright{$r:d$}{$i$}{$h$}{$p$}{$q$}
   \end{scope}
  \end{scope}

  \begin{scope}[xshift=0cm,yshift=-12cm]
   \node at (0, 0) [anchor=west] {\sc LeftComp-R-2:};
   \begin{scope}[xshift=0.25cm, yshift=-0.75cm]
    \halfright{$r:d$}{$i$}{$h$}{$p$}{$q$}
   \end{scope}
   \begin{scope}[xshift=3.5cm, yshift=-0.75cm]
    \righttriangle{$F:d'$}{$h$}{$j$}
   \end{scope}
   \node [anchor=west] at (4.5, -1.25) {$q \xmapsto{~h~} q' \in L_p$};
   \draw (0.1, -1.75) -- +(6.6, 0);
   \node at (7.0, -1.5)[anchor=west,text width=3cm] {
   \begin{equation*}
    d'= \left\{ \begin{array}{ll}
       d+1& \textrm{if } b=1 \textrm{ or } (j-h) \geq 1 \\
       d& \textrm{otherwise.} \\
              \end{array}
       \right.
   \end{equation*}
   };
   \begin{scope}[xshift=2.5cm, yshift=-2.3cm]
    \predright{$r:d$}{$i$}{$j$}{$p$}{$q'$}
   \end{scope}
  \end{scope}

  \begin{scope}[xshift=0cm,yshift=-16cm]
   \node at (0, 0) [anchor=west] {\sc RightComp:};
   \begin{scope}[xshift=0.25cm, yshift=-0.75cm]
    \predright{$r:d$}{$i$}{$h-1$}{$h$}{$q$}
   \end{scope}
   \begin{scope}[xshift=3.5cm, yshift=-0.75cm]
    \righttriangle{$q':d'$}{$h$}{$j$}
   \end{scope}
   \node [anchor=west] at (4.5, -0.15) {$q \in \textit{init}(L_h)$};
   \node [anchor=west] at (4.5, -0.75) {$q' \xmapsto{~p~} q'' \in \textit{final}(R_h)$};
   \node [anchor=west] at (4.5, -1.35) {$s \in \textit{final}(L_p)$};
   \draw (0.1, -1.75) -- +(6.6, 0);
   \node at (7.0, -1.5)[anchor=west,text width=3cm] {
   \begin{equation*}
    d'= \left\{ \begin{array}{ll}
       d+1& \textrm{if } (j-h) \geq 1 \\
       d& \textrm{otherwise.} \\
              \end{array}
       \right.
   \end{equation*}
   };
   \begin{scope}[xshift=2.5cm, yshift=-2.3cm]
    \predright{$r:d$}{$i$}{$j$}{$p$}{$s$}
   \end{scope}
  \end{scope}
  
 \end{tikzpicture}
 \caption{The composition rules that are not listed in Figure \ref{fig:ind:ded1}.
 {\sc LeftComp} is devided into two cases, {\sc LeftComp-L-*} and {\sc LeftComp-R-*} depending on the position of the dummy (predicted) node on the left antecedent item (corresponding to the second top element on the stack).
 They are further divided into two processes, {\sc 1} and {\sc 2} for achieving head-splitting.
 $b$ is an additional annotation on an intermediate item for correct depth computation in {\sc LeftComp}.
 }
 \label{fig:ind:ded-comp}
\end{figure}

Figures \ref{fig:ind:ded1} and \ref{fig:ind:ded-comp} describe the algorithm that parses SBGs with the left-corner parsing strategy.
Each inference rule can be basically mapped to a particular action of the transition, though the mapping is sometimes not one-to-one.
For example {\sc Shift} action is divided into two cases, {\sc Left} and {\sc Right}, for achieving head-splitting.
Some actions, e.g., {\sc Insert} ({\sc Insert-Left} and {\sc Insert-Right}) and {\sc LeftComp} ({\sc LeftComp-L-*} and {\sc LeftComp-R-*}) are further divided into two cases depending on tree structure of a stack element, i.e., whether a predicted node (a dummy) is a head of the stack element or some right dependent.







Each chart item preserves the current stack depth $d$. 
The algorithm only accepts an item spanning the whole sentence (including the dummy root symbol $\$$ at the end of the sentence) with the stack depth one and this condition certifies that the derivation can be converted to a valid sequence of transitions.
When our interest is to eliminate derivations that contain the specific depth of center-embedding, it can be achieved by assigning zero weight to every chart cell in which the depth exceeds the threshold.

See {\sc LeftComp-L-1} and {\sc LeftComp-L-2} ({\sc LeftComp-R-*} is described later);
these are the points where deeper stack depth might be detected.
These two rules are the result of decomposing a single {\sc LeftComp} action in the transition system into the {\it left phase}, which collects only left half constituent given a head $h$, and {\it right phase}, which collects the remaining.
Figure \ref{fig:ind:split-leftcomp} describes how this decomposition looks like in the transition system.
As shown in Figure \ref{subfig:ind:left-comp-after-decomp}, we imagine that the top subtree on the stack is divided into a left and right constituents.\footnote{
This assumption does not change the incurred stack depth at each step since in the left-corner transition system, right half dependents of a head are collected only after its left half construction is finished.
Splitting a head as in Figure \ref{subfig:ind:left-comp-after-decomp} means that we treat these left and right parsing processes independently.
}
In Figure \ref{fig:ind:split-leftcomp} we number each subtree on the stack from left to right.
Note that then the number of the rightmost (top) element on the stack corresponds to the stack depth of the configuration.
This value corresponds to the depth annotated on each item in the algorithm, such as $d'$ in \hspace{-8pt}\tikz[anchor=mid]{\righttriangleinline{$F:d'$}} appeared as the right antecedent item of {\sc LeftComp-L-2} in Figure \ref{fig:ind:ded-comp}.
Then, since the left antecedent item of the same rule, i.e., \tikz[anchor=mid]{\halfrectangleinline}, corresponds to the second top element on the stack, its depth $d$ is generally smaller by one, i.e., $d=d'-1$.

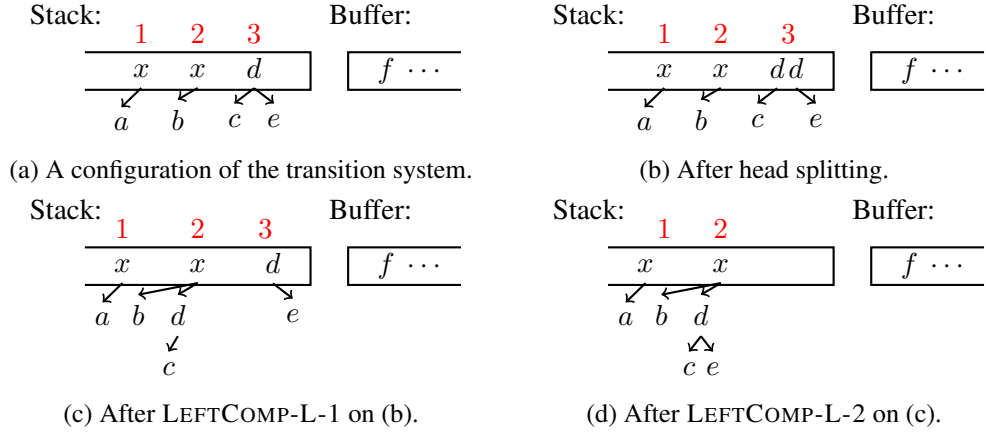
\begin{figure}[t]
 \centering
 \begin{minipage}{0.45\textwidth}
  \centering
  \begin{tikzpicture}[thick,edge from parent/.style={draw,->},sibling distance=15pt,level distance=20pt]
   \node at (-0.25, 0) {Stack:};
   \draw (0, -0.5) -- ++(3, 0) -- ++(0, -0.5) -- ++(-3, 0);
   \node at (0.75, -0.75) {$x$}
   child { node[anchor=mid] {$a$} }
   child[missing];
   \node at (1.5, -0.75) {$x$}
   child { node[anchor=mid] {$b$} }
   child[missing];
   \node[anchor=mid] at (2.25, -0.75) {$d$}
   child { node {$c$} }
   child { node {$e$} };

   \draw (0.75, -0.25) node {\color{red}{$1$}} ++(0.75, 0) node {\color{red}{$2$}} ++(0.75, 0) node {\color{red}{$3$}};

   \node at (3.8, 0) {Buffer:};
   \draw (5.0, -0.5) -- ++(-1.5, 0) -- ++(0, -0.5) -- ++(1.5, 0);
   \node[anchor=west] at (3.75, -0.75) {$f ~ \cdots$};
  \end{tikzpicture}  
  \subcaption{A configuration of the transition system.}
  \label{subfig:ind:left-comp-transition}
 \end{minipage}
 \begin{minipage}{0.45\textwidth}
  \centering
  \begin{tikzpicture}[thick,edge from parent/.style={draw,->},sibling distance=15pt,level distance=20pt]
   \node at (-0.25, 0) {Stack:};
   \draw (0, -0.5) -- ++(3, 0) -- ++(0, -0.5) -- ++(-3, 0);
   \node at (0.75, -0.75) {$x$}
   child { node[anchor=mid] {$a$} }
   child[missing];
   \node at (1.5, -0.75) {$x$}
   child { node[anchor=mid] {$b$} }
   child[missing];
   \node[anchor=mid] at (2.25, -0.75) {$d$}
   child { node {$c$} }
   child[missing];
   \node[anchor=mid] at (2.5, -0.75) {$d$}
   child[missing]
   child { node {$e$} };

   \draw (0.75, -0.25) node {\color{red}{$1$}} ++(0.75, 0) node {\color{red}{$2$}} ++(0.9, 0) node {\color{red}{$3$}};

   \node at (3.8, 0) {Buffer:};
   \draw (5.0, -0.5) -- ++(-1.5, 0) -- ++(0, -0.5) -- ++(1.5, 0);
   \node[anchor=west] at (3.75, -0.75) {$f ~ \cdots$};
  \end{tikzpicture}
  \subcaption{After head splitting.}
  \label{subfig:ind:left-comp-after-decomp}
 \end{minipage}
 \begin{minipage}{0.45\textwidth}
  \centering
  \begin{tikzpicture}[thick,edge from parent/.style={draw,->},sibling distance=15pt,level distance=20pt]
   \node at (-0.25, 0) {Stack:};
   \draw (0, -0.5) -- ++(3, 0) -- ++(0, -0.5) -- ++(-3, 0);
   \node at (0.5, -0.75) {$x$}
   child { node[anchor=mid] {$a$} }
   child[missing];
   \node at (1.5, -0.75) {$x$}
   child { node[anchor=mid] {$b$} }
   child { node[anchor=mid] {$d$}
    [sibling distance=7pt]
    child { node[anchor=mid] {$c$} }
    child[missing]
   }
   child[missing]
   child[missing];
   \node[anchor=mid] at (2.5, -0.75) {$d$}
   child[missing]
   child { node {$e$} };

   \draw (0.5, -0.25) node {\color{red}{$1$}} ++(1.0, 0) node {\color{red}{$2$}} ++(0.9, 0) node {\color{red}{$3$}};

   \node at (3.8, 0) {Buffer:};
   \draw (5.0, -0.5) -- ++(-1.5, 0) -- ++(0, -0.5) -- ++(1.5, 0);
   \node[anchor=west] at (3.75, -0.75) {$f ~ \cdots$};
  \end{tikzpicture}
  \subcaption{After {\sc LeftComp-L-1} on (b).}
 \end{minipage}
 \begin{minipage}{0.45\textwidth}
  \centering
  \begin{tikzpicture}[thick,edge from parent/.style={draw,->},sibling distance=15pt,level distance=20pt]
   \node at (-0.25, 0) {Stack:};
   \draw (0, -0.5) -- ++(3, 0) -- ++(0, -0.5) -- ++(-3, 0);
   \node at (0.5, -0.75) {$x$}
   child { node[anchor=mid] {$a$} }
   child[missing];
   \node at (1.5, -0.75) {$x$}
   child { node[anchor=mid] {$b$} }
   child { node[anchor=mid] {$d$}
     [sibling distance=9pt]
     child { node[anchor=mid] {$c$} }
     child { node[anchor=mid] {$e$} }
   }
   child[missing]
   child[missing];

   \draw (0.75, -0.25) node {\color{red}{$1$}} ++(0.75, 0) node {\color{red}{$2$}};

   \node at (3.8, 0) {Buffer:};
   \draw (5.0, -0.5) -- ++(-1.5, 0) -- ++(0, -0.5) -- ++(1.5, 0);
   \node[anchor=west] at (3.75, -0.75) {$f ~ \cdots$};
  \end{tikzpicture}
  \subcaption{After {\sc LeftComp-L-2} on (c).}
 \end{minipage}
 \caption{We decompose the {\sc LeftComp} action defined for the transition system into two phases, {\sc LeftComp-L-1} and {\sc LeftComp-L-2}, each of which collects only left or right half constituent of a subtree on the top of the stack.
 A number above each stack element is the stack depth decorated on the corresponding chart item.}
 \label{fig:ind:split-leftcomp}
\end{figure}

One complicated point with this depth calculation is that larger stack depth should not always be detected during this computation.
Recall the discussion in Section \ref{sec:memorycost} that there are two kinds of stack depth that we called depth$_{re}$ and depth$_{sh}$, in which only depth$_{re}$ correctly captures the center-embeddedness of a construction.
Depth$_{sh}$ is the depth of a configuration after a shift action, on which the top element is {\it complete}, i.e., contains no dummy (predicted) node.
Note that the depth annotated on each subtree in Figure \ref{fig:ind:ded-comp} is in fact depth$_{sh}$, as the right antecedent item of each rule (corresponds to the top element on the stack) does not contain a predicted node.
Since our goal is to capture the center-embeddedness during parsing, we fix this discrepancy with a small trick, which is described as the side condition of {\sc LeftComp-L-1} and {\sc LeftComp-L-2} in Figure \ref{fig:ind:ded-comp}.
The point at which only depth$_{sh}$ increases by 1 is when a shifted token is immediately reduced with a following composition rule.
We treat this process as a special case and do not increase the stack depth when the span length of the subtree that is reduced with a composition rule (right antecedent) is 1.

The remained problem is that because we split each constituent into left and right constituents, we cannot calculate the size of the reduced constituent immediately.
The additional variable $b$ in Figure \ref{fig:ind:ded-comp} is introduced for checking this condition.
$b$ is set to 1 if the left part, i.e., {\sc LeftComp-L-1} collects a constituent with span length greater than one ($h=j+1$ indicates one word constituent).
If $b=1$, regardless of the size of remaining right constituent, the second phase, or the right part {\sc LeftComp-L-2} always increases the stack depth (i.e., $d'=d+1$).
$b=0$ means the size of left constituent is zero;
in this case, the stack depth is increased when the size of the right constituent is greater than one.
In summary, the stack depth of the right antecedent item is increased {\it unless} three indices, $j+1$ and $h$ in {\sc LeftComp-L-1} and $j$ in {\sc LeftComp-L-2} are identical, which occurs when the reduced constituent is a single shifted token.

Finally, we note that we do not modify the depth of the right antecedent item of {\sc LeftComp-L-1}.
This is correct since the consequent item of this rule \tikz[anchor=mid]{\halfrectangleinline}, which waits for a right half constituent of $h$, can only be used as an antecedent of the following {\sc LeftComp-L-2} rule.
The role of {\sc LeftComp-L-1} is just to annotate the head of the reduced constituent and the span length.
Then {\sc LeftComp-L-2} checks the depth condition and calculates the weight associated with {\sc LeftComp} action (i.e., $q\xmapsto{~h~}q' \in L_p$).

The other part of the algorithm can be understood as follows:
\begin{itemize}
 \item {\sc LeftComp-R-*} is almost the same as {\sc LeftComp-L-*}.
       The difference is in the tree shape of the left antecedent item;
       in the {\sc R} case, it is a right half constituent with a predicted node \tikz[anchor=mid]{\predrightinline}, which corresponds to a subtree in which the predicted node is not the head but the tail of the right spine.
       We distinguish these two trees in the algorithm since they often behave very differently as shown in Figure \ref{fig:ind:ded1}.
       Note that \tikz[anchor=mid]{\predrightinline} has two FSA states since it contains two head tokens collecting its dependents (i.e., the head of the tree and the predicted token).
 \item Differently from {\sc LeftComp}, {\sc RightComp} is summarized as a single rule while it seems a bit complicated.
       In the algorithm, {\sc RightComp} can only be performed when the predicted token of \tikz[anchor=mid]{\predrightinline} finishes collecting its left dependents (indicated as the consecutive indices of $h-1$ and $h$).
       See Section \ref{sec:ind:spurious} for the reason of this restriction.
       Another condition for applying this rule is that the right state $q'$ of head $h$ must be a final state after applying transition $q' \xmapsto{~p~} q''$, which collects new rightmost dependent of $h$, i.e., $p$.
       Under these conditions, the rule performs the following parser actions: 1) attach $p$ as a right dependent of $h$; 2) finishes the right FSA of $h$; and 3) start collecting left dependents of $p$ by setting the final state to the left state of it.
 \item Some rules such as {\sc Finish-*} and {\sc Combine} do not exist in the original transition system.
       We introduce these to represent the generative process of SBGs in the left-corner algorithm.
 \item We do not annotate a state on a triangle \tikz[baseline=-10pt]{\headtriangleinline{$:d$}}.
       This is because it can only be deduced by {\sc Combine}, which combines two finished constituents with the same head.
 \item A parse always finishes with a consecutive application of {\sc LeftPred}, {\sc Insert-Left}, and {\sc Accept} after a parse spanning the sentence \tikz[baseline=-3ex,thick]{
       \draw (0, 0) -- ++(-0.4, -0.4) -- ++(0.8, 0) -- cycle;
       \draw[anchor=mid] (-0.55, -0.5) node {$1$} +(1.1, 0) node {$n$};
       } is recognized.
       {\sc LeftPred} predicts the arc between the dummy root token $\$$ at $n+1$-th position and this parse tree, and then {\sc Insert-Left} removes this predicted arc.
       Note that to get a correct parse the left FSA for $\$$ should be modified appropriately for not collecting more than one dependent (the common parameterization of DMV automatically achieve this).
\end{itemize}




\subsection{Spurious ambiguity and stack depth}
\label{sec:ind:spurious}

We have splitted each head token into left and right, which means each derivation with this algorithm has one-to-one correspondence to a dependency tree.
That is, there is no spurious ambiguity and the EM algorithm works correctly (Section \ref{sec:2:when}).
In Section \ref{sec:oracle}, we developed an oracle for a transition system that returns a sequence of gold actions given a sentence and a dependency tree, and found that the presented oracle is {\it optimal} in terms of incurred stack depth.
This optimality essentially comes from the implicit binarization mechanism of the oracle given a dependency tree (Theorem \ref{theorem:trans:binarize}).

The chart algorithm presented above has the same mechanism of binarization, and thus is optimal in terms of incurred stack depth.
Essentially this is due to our design of {\sc LeftComp} and {\sc RightComp} in Figure \ref{fig:ind:ded-comp}.
We do not allow {\sc RightComp} for a rectangle \tikz[anchor=mid]{\predrectangleinline} as in the case of {\sc LeftComp-L-*}.
Also, there is no second phase in {\sc RightComp}, meaning that the reduced constituent (i.e., the top stack element) does not have left children.
We notice that these two conditions are exactly the same as the statement in Lemma \ref{lemma:trans:rightcomp}, which was the key to prove the binarization mechanism in Theorem \ref{theorem:trans:binarize}.
When we are interested in another {\it nonoptimal} parsing algorithm, what we should do is to modify the allowed condition for {\sc LeftComp} and {\sc RightComp}, e.g., perhaps {\sc RightComp} is divided into several parts and instead the condition that {\sc LeftComp} can be applied is highly restricted.

\subsection{Relaxing the definition of center-embedding}
\label{sec:ind:modif}

In the previous chapter, we have examined a simple relaxation for the definition of center-embedding by allowing constituents up to some length to be at the top on the stack.
Here we demonstrate how this relaxation can be implemented with a simple modification to the presented algorithm.

Let us assume the situation in which we allow constructions with one degree of center-embedding {\it if} the length of embedded constituent is at most three;
that is, we allow a small part of center-embedding.
We write this as $(D,C)=(1,3)$, meaning that the maximum stack depth is generally one ($D=1$) though we partially allow depth two when the length of the embedded constituent is less than or equal to three ($C=3$).
This can be achieved by modifying the role of variable $b$ and equations in Figure \ref{fig:ind:ded-comp}.
As we have seen, the current algorithm does not increase the stack depth of right antecedent item with {\sc LeftComp} or {\sc RightComp} if the length of those reduced constituents is just one (has no dependent), which corresponds to the case of $C=1$.
Our goal is to generalize this calculation and to judge whether the length of the reduced constituent is greater than $C$ or not.
We modify the side condition of {\sc LeftComp-*-1} as follows:
\begin{equation}
 b = \max(C, h - (j + 1)).
\end{equation}
Now $b$ is a variable in the range $[0, C]$.
$b=0$ means the left constituent is one word.
Then, the side condition of {\sc LeftComp-*-2} is changed as follows:
\begin{equation}
 d' = \left\{ \begin{array}{ll}
       d+1& \textrm{if } b + (j - h) \geq C \\
       d& \textrm{otherwise.} \\
              \end{array}
       \right. \label{eqn:ind:mod-depth}
\end{equation}
For example, when $C=3$, and the left constituent is \tikz[baseline=-3ex,thick]{
       \draw (0, 0) -- ++(-0.4, -0.4) -- ++(0.4, 0) -- cycle;
       \draw[anchor=mid] (-0.55, -0.6) node {$3$} +(0.6, 0) node {$4$};
       }
while the right constituent is \tikz[baseline=-3ex,thick]{
       \draw (0, 0) -- ++(0.4, -0.4) -- ++(-0.4, 0) -- cycle;
       \draw[anchor=mid] (-0.15, -0.6) node {$4$} +(0.6, 0) node {$5$};
       },
the depth is unchanged since $b+(j-h) = 2 < 3$ in Eq. \ref{eqn:ind:mod-depth}.
Note that the algorithm may be inefficient when $C$ is larger, although it is not a practical problem as we only explore very small values such as 2 and 3.





\section{Experimental Setup}
\label{sec:ind:setup}
Now we move on to the empirical part of this chapter.
This section summarizes the experimental settings such as the datasets that we use, evaluation method, and possible constraints that we impose to the models.
In particular, we point out the crucial issue in the current evaluation metric in Section \ref{sec:ind:eval} and then propose our solution to alleviate this problem in Section \ref{sec:ind:param-const}.

\subsection{Datasets}
\label{sec:ind:dataset}
We use two different multilingual corpora for our experiments: Universal Dependencies (UD) and Google universal dependency treebanks;
the characteristics of these two corpora are summarized in Chapter \ref{chap:corpora}.
We mainly use UD in this chapter, which comprises of 20 different treebanks.
One problem of UD is that because this is the first study (in our knowledge) to use it in unsupervised dependency grammar induction, we cannot compare our models to previous state-of-the-art approaches.
The Google treebanks are used for this purpose.
It comprises of 10 languages and we discuss the relative performance of our approaches compared to the previously reported results on this dataset.

\paragraph{Preprocess}
Some treebanks of UD are annotated with multiword expressions although we strip them off for simplicity.
Also we remove every punctuation mark in every treebank (both training and testing).
This preprocessing has been performed in many previous studies.
This is easy when the punctuation is at a leaf position;
otherwise, we reattach every child token of that punctuation to its closest ancestor that is not a punctuation.

\paragraph{Input token}
Every model in this chapter only receives annotated part-of-speech (POS) tags given in the treebank.
This is a crucial limitation of the current study both from the practical and cognitive points of view as we discussed in Section \ref{sec:2:unsupervised}.
We learn the model on the {\it unified}, universal tags given in the respective corpora.
In Google treebanks, the total number of tags is 11 (excluding punctuation) while that is 16 in UD.
See Chapter \ref{chap:corpora} for more details.

\paragraph{Sentence Length}
Often unsupervised parsing systems are trained and tested on a subset of the original training/testing set by setting a maximum sentence length and ignoring every sentence longer than the threshold.
The main reason for this filtering during training is efficiency:
running dynamic programming ($O(n^4)$ in our case) for longer sentences in many number of iterations is expensive.
We therefore set the maximum sentence length during training to 15 on UD experiments, which is not so expensive to explore many parameter settings and languages.
We also evaluate our models against test sentences up to length 15.
We choose this value because we are interested more in whether our structural constraint helps to learn basic word order of each language, which may be obscured if we use full sentences as in the supervised parsing experiments since longer sentences are typically conjoined with several clauses.
This setting has been previously used in, e.g., \newcite{DBLP:journals/tacl/BiskH13}.
We call this filterd dataset UD15 in the following.

We use different filtering for Google treebanks and set the maximum length for training and testing to 10.
This is the setting of \newcite{grave-elhadad:2015:ACL-IJCNLP}, which compares several models including the previous state-of-the-art method of \newcite{naseem-EtAl:2010:EMNLP}.
See Tables \ref{tab:ind:stat-ud15} and \ref{tab:ind:stat-google10} for the statistics of the datasets.

\begin{table}[t]
 \centering
 \scalebox{1.0}{
\begin{tabular}[t]{l r r r r} \hline
Language & \#Sents. & \#Tokens & Av. len. & Test ratio\\\hline
Basque & 3,743 & 31,061 & 8.2 & 24.9\\
Bulgarian & 6,442 & 53,737 & 8.3 & 11.0\\
Croatian & 1,439 & 15,285 & 10.6 & 6.6\\
Czech & 46,384 & 388,309 & 8.3 & 12.9\\
Danish & 2,952 & 25,455 & 8.6 & 6.5\\
English & 9,279 & 67,249 & 7.2 & 14.9\\
Finnish & 10,146 & 85,057 & 8.3 & 5.2\\
French & 5,174 & 55,413 & 10.7 & 2.1\\
German & 8,073 & 82,789 & 10.2 & 6.7\\
Greek & 746 & 6,987 & 9.3 & 11.2\\
Hebrew & 1,883 & 19,057 & 10.1 & 9.7\\
Hungarian & 580 & 5,785 & 9.9 & 11.0\\
Indonesian & 2,492 & 25,731 & 10.3 & 11.5\\
Irish & 408 & 3,430 & 8.4 & 18.6\\
Italian & 5,701 & 51,272 & 8.9 & 4.3\\
Japanese & 2,705 & 28,877 & 10.6 & 25.9\\
Persian & 1,972 & 18,443 & 9.3 & 9.9\\
Spanish & 4,249 & 45,608 & 10.7 & 2.2\\
Swedish & 3,545 & 31,682 & 8.9 & 20.7\\\hline
\end{tabular}
 }
 \caption{Statistics on UD15 (after stripping off punctuations). Av. len. is the average length. Test ratio is the token ratio of the test set.}
 \label{tab:ind:stat-ud15}
\end{table}
\begin{table}[t]
\centering
\scalebox{1.0}{
\begin{tabular}[t]{l r r r r} \hline
Language & \#Sents. & \#Tokens & Av. len. & Test ratio\\\hline
German & 3,036 & 23,833 & 7.8 & 8.9\\
English & 4,341 & 31,287 & 7.2 & 5.7\\
Spanish & 1,258 & 9,731 & 7.7 & 3.2\\
French & 1,629 & 13,221 & 8.1 & 2.7\\
Indonesian & 799 & 6,178 & 7.7 & 10.7\\
Italian & 1,215 & 9,842 & 8.1 & 6.1\\
Japanese & 5,434 & 32,643 & 6.0 & 3.6\\
Korean & 3,416 & 21,020 & 6.1 & 7.7\\
Br-Portuguese & 1,186 & 9,199 & 7.7 & 11.7\\
Swedish & 1,912 & 12,531 & 6.5 & 18.5\\\hline
\end{tabular}
 }
 \caption{Statistics on Google trebanks (maximum length = 10).}
 \label{tab:ind:stat-google10}
\end{table}

\subsection{Baseline model}
\label{sec:ind:baseline}
Our baseline model is the featurized DMV model \cite{bergkirkpatrick-EtAl:2010:NAACLHLT}, which we briefly described in Section \ref{sec:bg:loglinear}.
We choose this model as our baseline since it is very simple yet performs competitively to the more complicated state-of-the-art systems.
Other more sophisticated methods exist, but they typically require much complex inference techniques \cite{spitkovsky-alshawi-jurafsky:2013:EMNLP} or external information \cite{marevcek-straka:2013:ACL2013}, which obscure the contribution of our imposing constraints.
Studying the effect of the structural constraints for these more strong models is remained for the future work.

This model contains two tunable parameters, the regularization parameter and the feature templates.
We fix the regularization parameter to 10, which is the same as the value in \newcite{bergkirkpatrick-EtAl:2010:NAACLHLT} since we did not find significant performance changes with this value in our preliminary study.
We also basically use the same feature templates as Berg-Kirkpatrick et al.;
the only difference is that we add additional backoff features for {\sc stop} probabilities that ignore both direction and adjacency, which we found slightly improves the performance.

\subsection{Evaluation}
\label{sec:ind:eval}

Evaluation is one of the {\it unsolved} problems in the unsupervised grammar induction task.
The main source of difficulty is the inherent ambiguity of the notion of {\it heads} in dependency grammar that we mentioned several times in this thesis (see Section \ref{sec:corpora:heads} for details).
Typically the quality of the model is evaluated in the same way as the supervised parsing experiments:
At test time, the model predicts dependency trees on test sentences;
then the {\it accuracy} of the prediction is measured by an {\it unlabelled attachment score} (UAS):
\begin{equation}
 \textrm{UAS} = \frac{\textrm{\# tokens whose head matches the gold head}}{\textrm{\# tokens}}
\end{equation}
The problem of this measure is that it completely ignores the ambiguity of head definitions since its score calculation is against the single gold dependency treebank.
Some attempts to alleviate the problem of UAS exist, such as direction-free (undirected) measure \cite{klein-manning:2004:ACL} and a more sophisticated measure called neutral edge detection (NED) \cite{schwartz-EtAl:2011:ACL-HLT20112}.
NED expands the set of {\it correct} dependency constructions given the predicted tree and the gold tree to save the errors that seem to be caused by annotation variations.
However NED is a too lenient metric and causes different problems.
For example, under NED (also under the undirected measure) the two trees ${\sf dogs^\curvearrowleft ran}$ and ${\sf dogs^\curvearrowright ran}$ are treated equal, although it is apparent that the correct analysis is the former.
We suspect this is the reason why many researchers have not used NED and instead select UAS while recognizing the inherent problems \cite{cohen-2011,DBLP:journals/tacl/BiskH13}.

However, the current situation is really unhealthy for our community.
For example, if we find some method that can boost UAS from 40 to 60, we cannot identify whether this improvement is due to the acquisition of essential word orders such as dependencies between nouns and adjectives, or just overfitting to the current gold treebank.
The latter case occurs, e.g., when the current gold treebank assumes that heads of prepositional phrases are the content words and the improved model changes the analysis for them from functional heads to content heads.
Since our goal is not to obtain a model that can overfit to the gold treebank in a surface level, but to understand the mechanism that the model can acquire better word orders, we want to remove the possibility to make such (fake) improvements in our experiments.

We try to minimize this problem not by revising the evaluation metric but by customizing models.
We basically use UAS since the other metrics have more serious drawbacks.
However, to avoid unexpected improvements/degradations, we constraint the model to explore only trees that may follow the conventions in the current gold data.
This is possible in our corpora as they are annotated under some consistent annotation standard (see Chapter \ref{chap:corpora}).
This approach is conceptually similar to \newcite{naseem-EtAl:2010:EMNLP}, although we do not incorporate many constraints on word orders, such as the dependencies between a verb and a noun.
The detail of the constraints we impose to the models is described next.

\subsection{Parameter-based Constraints}
\label{sec:ind:param-const}

The goal of the current experiments is to see the effect of {\it structural constraint}, which we hope to guide the model to find better parameters.
To do so, on the baseline model (Section \ref{sec:ind:baseline}) we impose several additional constraints in the framework of structural constraint model described in Section \ref{sec:ind:overview}, and examine how performance changes (we list these constraints in Section \ref{sec:ind:structural-const}).

In addition to the structural constraints, we also consider another kind of constraint that we call {\it parameter-based constraint} in the same framework, that is, as a cost function $f(z,\theta)$ in Eq. \ref{eqn:ind:cost}.
The parameter-based constraints are constraints on POS tags in a given sentence and are specified decralatively, e.g., {\it X cannot have a dependent in the sentence}.
Note that our main focus in this experiment is the effect of structural constraints.
As we describe below, the parameter-based ones are introduced to make the empirical comparison of structural constrains more meaningful.

Note that all constraints below are imposed during training only, as we found in our preliminary experiments that the constraints during decoding (at test time) make little performance changes.
This is natural in particular for parameter-based constraints since the model learns the parameters that follow the given constraints during training.


We consider the following constraints in this category:
\begin{description}
 \item[Function words in UD] This constraint is introduced to alleviate the problem of evaluation that we discussed in Section \ref{sec:ind:eval}.
            One characteristic of UD is that its annotation style is consistently content head based, that is, every function word is analyzed as a dependent of the most closely related content word.\footnote{We find small exceptions in each treebank probably due to the remaining annotation errors though they are negligibly small.}
            By forcing the model to explore only structures that follow this convention, we expect we can minimize the problem of {\it arbitrariness} of head choices.
            This constraint can easily be implemented by setting every {\sc stop} probability of DMV for function words to 1.
            We regard the following six POS tags as function words: {\sc adp}, {\sc aux}, {\sc conj}, {\sc det}, {\sc part}, and {\sc sconj}.
            Since most arbitrary constructions are around function words, we hope this makes the performance change due to other factors such as the structural constraints clearer.
            Note that this technique is still not the perfect and cannot neutralize some annotation variations such as internal structures of noun phrases;
            we do not consider further constraints to save such more complex cases.
 \item[Function words in Google treebanks]
            We consider the similar constraints on Google treebanks.
            The Google treebanks uses the following four POS tags for function words: {\sc det}, {\sc conj}, {\sc prt}, and {\sc adp}.
            {\sc prt} is a particle corresponding to {\sc part} in UD.
            As in UD it also follows the annotation standard of Stanford typed dependencies \cite{mcdonald-EtAl:2013:Short} and analyzes most function words as dependents, although it is not the case for {\sc adp}, which is in most cases analyzed as a head of the governing phrase.
            We therefore introduce another kind of constraint for {\sc adp}, which prohibits to become a dependent, i.e., {\sc adp} must have at least one dependent.
            Implementing this constraint in our dynamic programming algorithm is a bit involved compared to the previous {\it unheadable} constraints, mainly due to our {\it split-head} representation.
            We can achieve this constraint in a similar way to the constituent length memoization technique that we introduced in Figure \ref{fig:ind:ded-comp} with variable $b$.
            Specifically, at {\sc LeftComp-L-1}, we remember whether the reduced head $h$ has at least one dependent if $h$ is {\sc adp};
            then at {\sc LeftComp-L-2}, we disallow the rule application if that {\sc adp} is recognized as having no dependent.
            We also disallow {\sc Combine} if the head is {\sc adp} and the sizes of two constituents are both 1 (i.e., $i=h=j$).
            The resulting full constituent would be reduced by {\sc LeftPred} to be some dependent, which is although prohibited.
            Other function words are in most cases analyzed as dependents so we use the same restriction as the function words in UD.
 \item[Candidates for root words] This constraint is also parameter based though should be distinguished from the above two.
             Note that the constraints discussed so far are only for alleviating the problem of annotation variations in that they give no hint for acquiring basic word orders for the model such as the preference of an arc from a verb to a noun.
             This constraint is intended to give such a hint to the model by restricting possible root positions in the sentence.
             We consider two types of such constraints.
             The first one is called the {\it verb-or-noun constraint}, which restricts the possible root word of the sentence to a verb, or a noun.
             The second one is called the {\it verb-otherwise-noun constraint}, which more eagerly restricts the search space by only allowing a verb to become a root, if at least one verb exists in the sentence; otherwise, nouns become a candidate.
            If both do not exist, then every word becomes a candidate.
            In both corpora, {\sc verb} is the only POS tag for representing verbs.
            We regard three POS tags, {\sc noun}, {\sc pron}, and {\sc propn} in UD as nouns.
            In Google treebanks, {\sc noun} and {\sc pron} are considered as nouns.
            This type of knowledge is often employed in the previous unsupervised parsing models in different ways \cite{gimpel-smith:2012:NAACL-HLT2,gormley-eisner:2013:ACL2013,DBLP:conf/aaai/BiskH12,DBLP:journals/tacl/BiskH13} as {\it seed knowledge} given to the model.
            We will see how this simple hint to the target grammar affects the performance.
\end{description}

\subsection{Structural Constraints}
\label{sec:ind:structural-const}

These are the constraints that we focus on in the experiments.
\begin{description}
 \item[Maximum stack depth]
            This constraint removes parses involving center-embedding up to a specific degree and can be implemented by setting the maximum stack depth in the algorithm in Figures \ref{fig:ind:ded1} and \ref{fig:ind:ded-comp}.
            We also investigate the relaxation of the constraint with a small constituent that we described in Section \ref{sec:ind:modif}.
            Studying the effect of this constraint is the main topic in the our experiments.
 \item[Dependency length bias] We also explore another structural constraint that biases the model to prefer shorter dependency length, which has previously been examined in \newcite{smith-eisner-2006-acl-sa}.
            With this constraint, each attachment probability is changed as follows:
            \begin{equation}
             \theta'_{\textsc{a}}(a|h,\textit{dir}) = \theta_{\textsc{a}}(a|h,\textit{dir}) \cdot \exp(-\beta_{len} \cdot (|h - d| - 1)),
            \end{equation}
            where $\theta_{\textsc{a},h,d}$ is the original DMV parameter attaching $d$ as a dependent of $h$.
            Differently from \newcite{smith-eisner-2006-acl-sa}, we subtract 1 in each length cost calculation to add no penalty for an arc between adjacent words.
            As \newcite{smith-eisner-2006-acl-sa} noted, this constraint leads to the following form of $f(z,\theta)$ in Eq. \ref{eqn:ind:cost}:
            \begin{equation}
             f_{len}(z,\theta) = \exp \left( -\beta_{len} \cdot \left( \sum_{1\leq d \leq n} (|h_{z,d} - d|) - n \right) \right), \label{eqn:ind:betalen}
            \end{equation}
            where $h_{z,d}$ is the analyzed head position for a dependent at $d$.
            Notice that $\sum_{1\leq d \leq n} (|h_{z,d} - d|)$ is the sum of dependency lengths in the sentence, which means that this model tries to {\it minimize} the sum of dependency length in the tree and is closely related to the theory of dependency length minimization, a typological hypothesis that grammars may universally favor shorter dependency length \cite{gildea-temperley:2007:ACLMain,DBLP:journals/cogsci/GildeaT10,Futrell18082015,gulordava-merlo-crabbe:2015:ACL-IJCNLP}.
\end{description}


Notice that typically center-embedded constructions are accompanied by longer dependencies.
However the opposite is generally not the case;
there are many constructions that do accompany longer dependencies though do not contain center-embedding.
By comparing these two constraints, we discuss whether center-embedding is the phenomena worth considering than the simpler method of shorter length bias.


\subsection{Other Settings}
\label{sec:ind:setting}

\paragraph{Initialization}

Much previous works of unsupervised dependency induction, in particular DMV and related models, relied on heuristic initialization called {\it harmonic initialization} \cite{klein-manning:2004:ACL,bergkirkpatrick-EtAl:2010:NAACLHLT,cohen-smith:2009:NAACLHLT09,blunsom-cohn:2010:EMNLP}, which is obtained by running the first E-step of the training by setting every attachment probability between $i$ and $j$ to $(|i-j|)^{-1}$.\footnote{There is another variant of harmonic initialization \cite{smith-eisner-2006-acl-sa} though we do not explore this since the method described here is the one that is employed in \newcite{bergkirkpatrick-EtAl:2010:NAACLHLT} (p.c.), which is our baseline model (Section \ref{sec:ind:baseline}).}
Note that this method also biases the model to favor shorter dependencies.

We do not use this initialization with our structural constraints since one of our motivation is to invent a method that does not rely on such heursitics highly connected to a specific model (like DMV).
We therefore initialize the model to be a uniform model.
However, we also compare such uniform + structural constrained models to the harmonic initialized model {\it without} structural constraints to see the relative strength of our approach.

\paragraph{Decoding}
As noted above, every constraint introduced so far is only imposed during training.
At decoding (test time), we do not consider the bias term of Eq. \ref{eqn:ind:joint} and just run the Viterbi algorithm to get the best parse under the original DMV model.

\section{Empirical Analysis}
\label{sec:ind:exp}

We first check the performance differences of several settings on UD and then move on to Google treebanks to compare our approach to the state-of-the-art methods.

\subsection{Universal Dependencies}

\begin{figure}[p]
\centering
 \resizebox{0.99\textwidth}{!}
 {\includegraphics[]{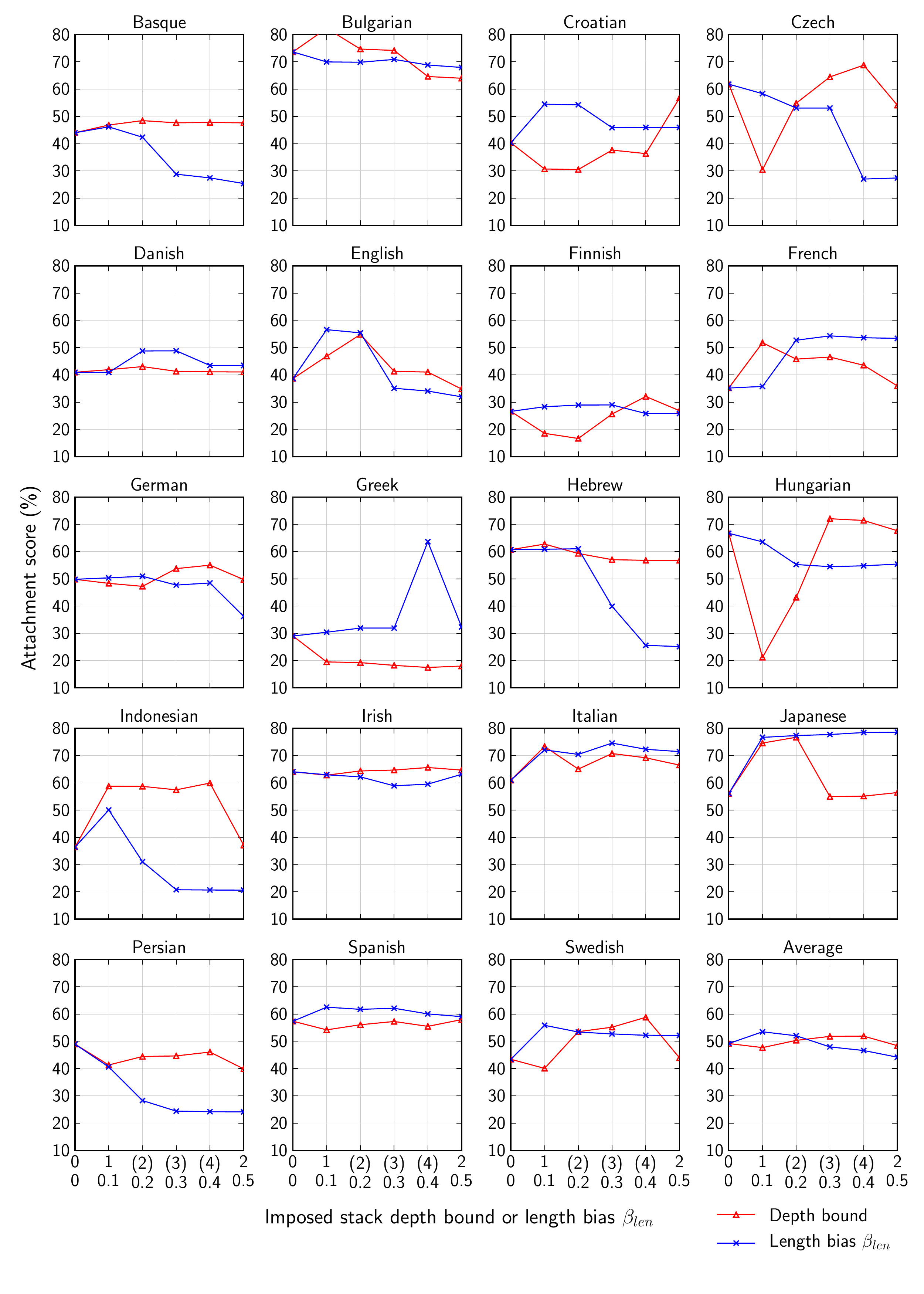}}
 \vspace{-20pt}
 \caption{Attachment accuracies on UD15 with the function word constraint and structural constrains.
 The numbers in parentheses are the maximum length of a constituent allowed to be embedded.
 For example (3) means a part of center-embedding of depth two, in which the length of embedded constituent $\leq 3$, is allowed.}
 \label{fig:ind:ud15-none}
\end{figure}

\paragraph{When no help is given to the root word}
We first see how the performance changes when using different length biases or stack depth biases (Section \ref{sec:ind:structural-const}).
The parameter-based constraint is only the function word constraint of UD, that is, any function word cannot be a head of others.
Figure \ref{fig:ind:ud15-none} summarizes the results.
Although the variance is large, we can make the following observations:
\begin{itemize}
 \item Often small (weak) length biases (e.g., $\beta_{len}=0.1$) work better than more strong biases.
       In many languages, e.g., English, Indonesian, and Croatian, the performance improves with a small bias and degrades as the bias is sharpened.
       Note that the left most point is the score with no constraints, e.g., just the uniformly initialized model.
 \item We try five different stack depth bound between depth one and depth two.
       The result shows in many cases the middle, e.g., 1(2), 1(3), and 1(4) works better.
       The performance of depth two is almost the same as the no-constraint baseline, meaning that stack depth two is too lenient to restrict the search space during learning.
       This observation is consistent with the empirical stack depth analysis in the previous chapter (Section \ref{sec:trans:ud-result}).
 \item On average, we find the best setting is the small length bias $\beta_{len}=0.1$.
       In Table \ref{tab:ind:ud15-none} we summarizes the accuracies of selected configurations in this figure, which work better, as well as the harmonic initialized models.
 \item The performance for some languages, in particular Greek and Finnish, is quite low compared to other languages.
       We inspect the output trees for these languages, and found that the model fails to identify very basic word orders, such as the tendency of a verb to be a root word.
       Essentially, the models so far do not receive such explicit knowledge about grammar, which is known to be particularly hard.
       We thus see next how performances change if a small amount of {\it seed knowledge} about the grammar is given the model.
\end{itemize}

\begin{table}[t]
 \centering
 \begin{tabular}{lccccc}\hline
  &Unif.&$C=2$&$C=3$&$\beta_{len} = 0.1$&Harmonic\\\hline
Basque&44.0&{\bf48.5}&47.6&46.1&44.5\\
Bulgarian&73.6&{\bf74.7}&74.2&70.0&72.5\\
Croatian&40.3&30.5&37.7&{\bf54.5}&47.3\\
Czech&61.8&54.8&{\bf64.5}&58.3&54.2\\
Danish&40.9&{\bf43.0}&41.3&40.9&40.9\\
English&38.7&54.8&41.3&{\bf56.6}&39.1\\
Finnish&26.6&16.7&25.6&{\bf28.3}&26.4\\
French&35.2&45.9&{\bf46.6}&35.7&34.6\\
German&49.8&47.3&{\bf53.6}&50.3&49.5\\
Greek&29.2&19.3&18.3&30.4&{\bf 49.3}\\
Hebrew&60.7&59.3&57.1&{\bf60.9}&57.3\\
Hungarian&66.7&43.2&{\bf72.2}&63.6&65.8\\
Indonesian&36.4&{\bf58.7}&57.4&50.0&40.8\\
Irish&64.1&64.5&{\bf64.8}&63.0&64.4\\
Italian&61.0&65.0&70.7&{\bf72.2}&65.8\\
Japanese&56.1&{\bf76.7}&55.0&{\bf 76.7}&48.2\\
Persian&{\bf49.0}&44.4&44.7&40.7&41.4\\
Spanish&57.3&56.0&57.3&{\bf62.5}&55.4\\
Swedish&43.4&53.6&55.2&{\bf55.9}&49.5\\\\
Avg&49.2&50.4&51.8&{\bf53.5}&49.8\\\hline
 \end{tabular}
 \caption{Accuracy comparison on UD15 for selected configurations including harmonic initialization (Harmonic).
 Unif. is a baseline model without structural constraints.
 $C$ is the allowed constituent length when the maximum stack depth is one.
 $\beta_{len}$ is strength of the length bias.
 }
 \label{tab:ind:ud15-none}
\end{table}


\begin{table}[t]
 \centering
 \begin{tabular}{lcccccccc}\hline
  & \multicolumn{4}{c}{Verb-or-noun constraint} & \multicolumn{4}{c}{Verb-otherwise-noun constraint} \\
  &Unif.&$C=2$&$C=3$&$\beta_{len} = 0.1$&Unif.&$C=2$&$C=3$&$\beta_{len} = 0.1$ \\\hline
Basque    &44.7&{\bf55.2}&54.3&46.4&{\bf55.8}&55.6&54.8&51.0\\
Bulgarian &73.4&{\bf75.8}&75.1&64.1&72.7&{\bf75.8}&75.2&70.6\\
Croatian  &40.1&{\bf52.5}&41.4&47.3&{\bf57.0}&52.5&52.5&55.8\\
Czech     &50.7&54.8&{\bf64.7}&59.2&63.2&54.9&{\bf66.3}&58.1\\
Danish    &40.9&{\bf43.1}&41.3&40.9&48.7&46.9&{\bf50.1}&47.3\\
English   &39.8&{\bf55.8}&41.3&40.2&57.2&55.2&{\bf58.5}&53.9\\
Finnish   &26.2&27.7&27.7&{\bf28.3}&40.3&32.5&34.3&{\bf40.4}\\
French    &35.7&{\bf50.9}&49.5&47.0&44.2&{\bf55.8}&54.6&42.1\\
German    &49.7&47.1&{\bf56.0}&51.2&49.5&55.7&{\bf57.4}&49.9\\
Greek     &61.7&{\bf70.0}&62.1&60.2&60.5&{\bf68.8}&62.0&60.2\\
Hebrew    &52.9&58.7&{\bf60.9}&57.5&54.8&{\bf62.6}&54.2&57.2\\
Hungarian &68.8&41.6&{\bf71.3}&63.6&69.2&65.5&{\bf72.4}&64.8\\
Indonesian&32.0&{\bf58.3}&58.1&43.6&50.2&58.6&58.5&{\bf59.4}\\
Irish     &63.1&64.5&{\bf65.2}&63.0&63.4&64.4&{\bf64.7}&63.9\\
Italian   &62.7&{\bf77.1}&73.6&72.5&69.2&65.2&69.8&{\bf72.4}\\
Japanese  &56.4&70.5&56.9&{\bf73.9}&56.9&69.0&57.0&{\bf73.5}\\
Persian   &46.9&45.1&{\bf51.2}&39.7&48.0&45.1&{\bf51.1}&41.7\\
Spanish   &46.8&56.1&57.3&{\bf63.1}&57.7&56.2&58.5&{\bf62.3}\\
Swedish   &43.5&{\bf44.8}&43.2&43.5&{\bf57.9}&53.3&53.3&56.9\\\\
Avg.&49.3&55.2&{\bf55.3}&52.9&56.7&57.6&{\bf58.2}&56.9\\\hline
 \end{tabular}
 \caption{Accuracy comparison on UD15 for selected configurations with the hard constraints on possible root POS tags.}
 \label{tab:ind:ud15-vn}
\end{table}

\paragraph{Constraining POS tags for root words}

Table \ref{tab:ind:ud15-vn} shows the results when we add two kinds of seed knowledge as parameter-based constraints (Section \ref{sec:ind:param-const}).

We first see the result with the verb-or-noun constraint.
This constraint comes from the main assumption of UD that the root token of a sentence is its main predicate, which is basically a verb, or a noun or a adjective if the main verb is copula.
We remove adjective from this set as we found it is relative rare across languages.
Interestingly in this case, the stack depth constraints ($C=2$ and $C=3$) work the best.
In particular, the average score of the length bias ($\beta_{len}=0.1$) drops.
We inspect the reason of this below.

We next see the effect of another constraint, the verb-otherwise-noun constraint, which excludes nouns from the candidate for the root if both a verb and noun exist.
This probably decreases the recall though we expect that it increases the performance as the majority of predicates is verbs.
As we expected, with this constraint the average performance of baseline uniform model increases sharply from 49.2 to 56.7 (+7.5), which is larger than any increases with structural constraint to the original baseline model.
In this case, though the change is small, again our stack depth constraints perform the best (58.2 with $C=3$);
the average score with the length bias does not increase.

\subsection{Qualitative analysis}

When we inspect the scores of the models without root POS constraints in Table \ref{tab:ind:ud15-none} and the models with the verb-or-noun constraint in Table \ref{tab:ind:ud15-vn}, we notice that the behaviors of our models with stack depth constraints and other models are often quite different.
Specifically,
\begin{itemize}
 \item It is only Greek on which the baseline uniform model improves score from the setting with no root POS constraint.
 \item In other languages, the scores of the uniform model are unchanged or dropped when adding the root POS constraint.
       For example the score for Czech drops from 61.8 to 50.7.
 \item The same tendency is observed in the models of $\beta_{len}=0.1$.
       Its score for Greek improves from 30.4 to 60.2 while other scores are often unchanged or dropped;
       an exception is French, on which the score improves from 35.7 to 47.0.
       For other languages, such as Croatian (-7.2), English (-16.4), Indonesian (-6.4), and Swedish (-12.4), the scores sharply drop.
 \item On the other hand, we observe no significant performance drops in our models with stack depth constraints (i.e., $C=2$ and $C=3$) by adding the root POS constraint.
\end{itemize}

These differences between the length constraints and stack depth constraints are interesting and may shed some light on the characteristics of two approaches.
Here we look into the output parses in English, with which the performance changes are typical, i.e., the model of $\beta_{len}=0.1$ drops while the scores of other models are almost unchanged.

\begin{figure}[p]
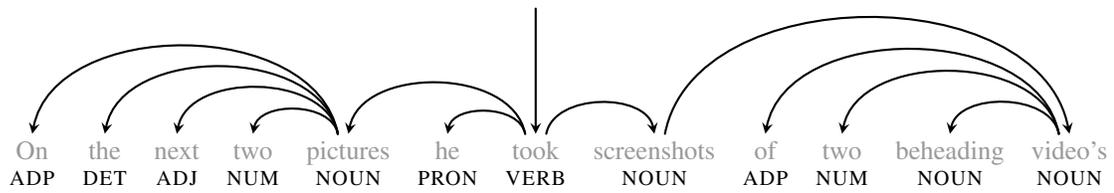
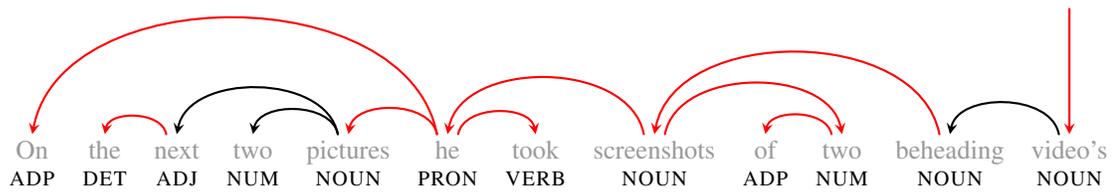
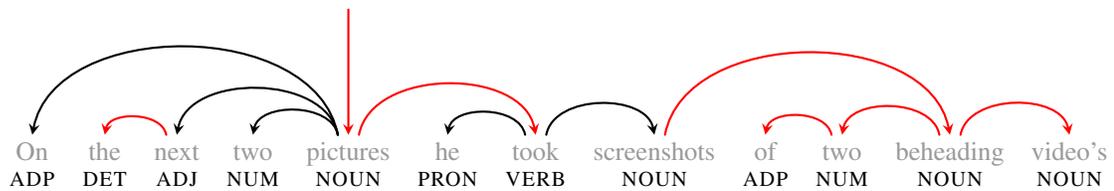
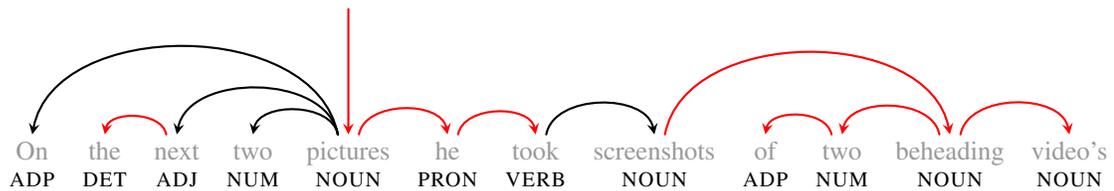
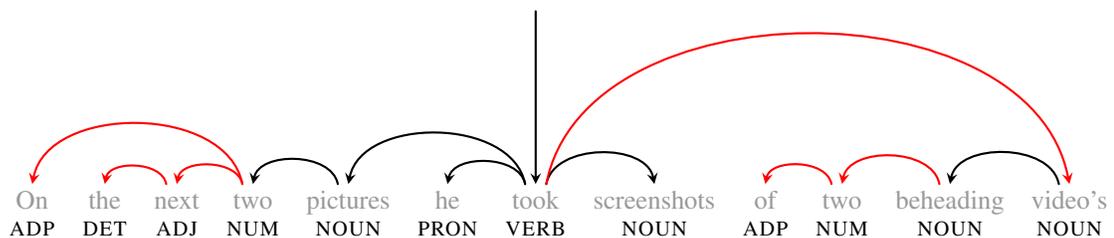

 \centering
 \definecolor{g70}{gray}{0.6}
 \newcommand{\gw}[1]{\color{g70}{#1}}
 \newcommand{\onthenext}{\gw{On} \& \gw{the} \& \gw{next} \& \gw{two} \& \gw{pictures} \& \gw{he} \& \gw{took} \& \gw{screenshots} \& \gw{of} \& \gw{two} \& \gw{beheading} \& \gw{video's}\\
 {\sc adp} \& {\sc det} \& {\sc adj} \& {\sc num} \& {\sc noun} \& {\sc pron} \& {\sc verb} \& {\sc noun} \& {\sc adp} \& {\sc num} \& {\sc noun} \& {\sc noun}
 }
 \begin{minipage}[b]{.99\linewidth}
  \centering
  {\small
  \begin{dependency}[theme=simple,thick]
   \begin{deptext}[column sep=0.2cm]
    \onthenext \\
    \end{deptext}
   \depedge{5}{1}{}
   \depedge{5}{2}{}
   \depedge{5}{3}{}
   \depedge{5}{4}{}
   \depedge{7}{5}{}
   \depedge{7}{6}{}
   \deproot[edge unit distance=3ex]{7}{}
   \depedge{7}{8}{}
   \depedge{12}{9}{}
   \depedge{12}{10}{}
   \depedge{12}{11}{}
   \depedge{8}{12}{}
  \end{dependency}
  }
 \subcaption{Gold parse.}
 \end{minipage}
 \begin{minipage}[b]{.99\linewidth}
  \centering
  {\small
  \begin{dependency}[theme=simple,thick]
   \begin{deptext}[column sep=0.2cm]
    \onthenext \\
    \end{deptext}
   \depedge[red]{6}{1}{}
   \depedge[red]{3}{2}{}
   \depedge{5}{3}{}
   \depedge{5}{4}{}
   \depedge[red]{6}{5}{}
   \depedge[red]{8}{6}{}
   \depedge[red]{6}{7}{}
   \depedge[red]{11}{8}{}
   \depedge[red]{10}{9}{}
   \depedge[red]{8}{10}{}
   \depedge{12}{11}{}
   \deproot[edge unit distance=3ex,red]{12}{}
  \end{dependency}
  }
 \subcaption{Output by uniform, uniform + verb-or-noun, and $\beta=0.1$ + verb-or-noun.}
 \end{minipage}
 \begin{minipage}[b]{.99\linewidth}
  \centering
  {\small
  \begin{dependency}[theme=simple,thick]
   \begin{deptext}[column sep=0.2cm]
    \onthenext \\
    \end{deptext}
   \depedge{5}{1}{}
   \depedge[red]{3}{2}{}
   \depedge{5}{3}{}
   \depedge{5}{4}{}
   \deproot[edge unit distance=3ex,red]{5}{}
   \depedge{7}{6}{}
   \depedge[red]{5}{7}{}
   \depedge{7}{8}{}
   \depedge[red]{10}{9}{}
   \depedge[red]{11}{10}{}
   \depedge[red]{8}{11}{}
   \depedge[red]{11}{12}{}
  \end{dependency}
  }
  \subcaption{Output by $C=2$, $C=2$ + verb-or-noun.}
  \label{fig:ind:c2-eng-ud}
 \end{minipage}
 \begin{minipage}[b]{.99\linewidth}
  \centering
  {\small
  \begin{dependency}[theme=simple,thick]
   \begin{deptext}[column sep=0.2cm]
    \onthenext \\
    \end{deptext}
   \depedge{5}{1}{}
   \depedge[red]{3}{2}{}
   \depedge{5}{3}{}
   \depedge{5}{4}{}
   \deproot[edge unit distance=3ex,red]{5}{}
   \depedge[red]{5}{6}{}
   \depedge[red]{6}{7}{}
   \depedge{7}{8}{}
   \depedge[red]{10}{9}{}
   \depedge[red]{11}{10}{}
   \depedge[red]{8}{11}{}
   \depedge[red]{11}{12}{}
  \end{dependency}
  }
  \subcaption{Output by $C=3$, $C=3$ + verb-or-noun.}
  \label{fig:ind:c3-eng-ud}
 \end{minipage}
 \begin{minipage}[b]{.99\linewidth}
  \centering
  {\small
  \begin{dependency}[theme=simple,thick]
   \begin{deptext}[column sep=0.2cm]
    \onthenext \\
    \end{deptext}
   \depedge[red]{4}{1}{}
   \depedge[red]{3}{2}{}
   \depedge[red]{4}{3}{}
   \depedge{5}{4}{}
   \depedge{7}{5}{}
   \depedge{7}{6}{}
   \deproot[edge unit distance=4ex]{7}{}
   \depedge{7}{8}{}
   \depedge[red]{10}{9}{}
   \depedge[red]{11}{10}{}
   \depedge{12}{11}{}
   \depedge[red]{7}{12}{}
  \end{dependency}
  }
  \subcaption{Output by $\beta_{len}=0.1$.}
  \label{fig:ind:beta-eng-ud}
 \end{minipage}
 \caption{Comparison of output parses of several models on a sentence in English UD.
 The outputs of $C=2$ and $C=3$ do not change with the root POS constraint, while the output of $\beta_{len}=0.1$ changes to the same one of the uniform model with the root POS constraint.
 Colored arcs indicate the wrong predictions.
 Note surface forms are not observed by the models (only POS tags are).
 }
 \label{fig:ind:error-eng-ud}
\end{figure}

When we compare output parses of different models, we notice that often the same tree is predicted by several different models.
Figure \ref{fig:ind:error-eng-ud} shows examples of output parses of different models.
The following observations are made with those errors.
Note that these are typical, in that the same observation can be often made on other sentences as well.
\begin{enumerate}
 \item One strong observation from Figure \ref{fig:ind:error-eng-ud} is that the output of $\beta_{len}=0.1$ reduces to that of the uniform model when the root POS constraint is added to the model.
       As can be seen in other parses, every model in fact predicts that the root token is a noun or a verb, which suggests this explicit root POS constraint is completely redundant in the case of English.
 \item Contrary to $\beta_{len}=0.1$, the stack depth constraints, $C=2$ and $C=3$, are not affected by the root POS constraint.
       This is consistent with the scores in Tables \ref{tab:ind:ud15-none} and \ref{tab:ind:ud15-vn};
       the score of $C=3$ is unchanged and that of $C=2$ increases by just 1.0 point with the root POS constraint.
 \item Whie the scores of the uniform model and $C=3$ are similar in Table \ref{tab:ind:ud15-none} (38.7 and 41.3, respectively), the properties of output parses seem very different.
       The typical errors made by $C=3$ are the root tokens, which are in most cases predicted as nouns as in Figure \ref{fig:ind:c3-eng-ud}, and arcs between nouns and verbs, which also are typically predicted as {\sc noun} $\rightarrow$ {\sc verb}.
       Contrary to these {\it local} mistakes, the uniform model often fail to capture the basic structure of a sentence.
       For example, while $C=3$ correctly identifies that ``of two beheading video's'' comprises a constituent, which modifies ``screenshots'', which in turn becomes an argument of ``took'', the parse of the uniform model is more corrupt in that we cannot identify any semantically coherent units from it.
       See also Figure \ref{fig:ind:error-eng-ud-2} where we compare outputs of these models on another sentence.
\end{enumerate}

\begin{figure}[p]
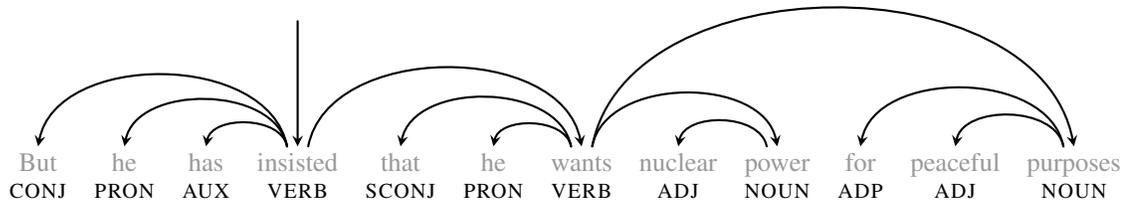
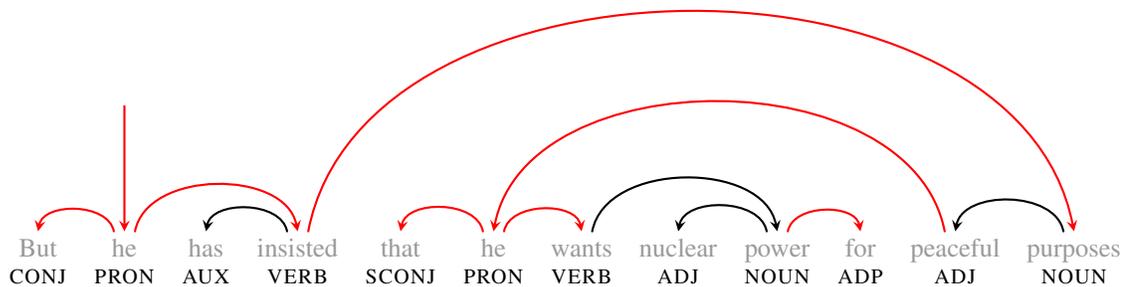
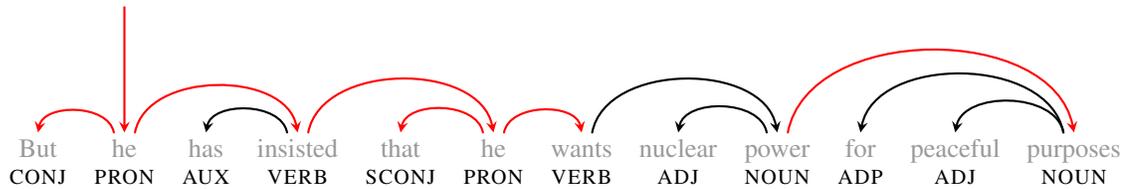
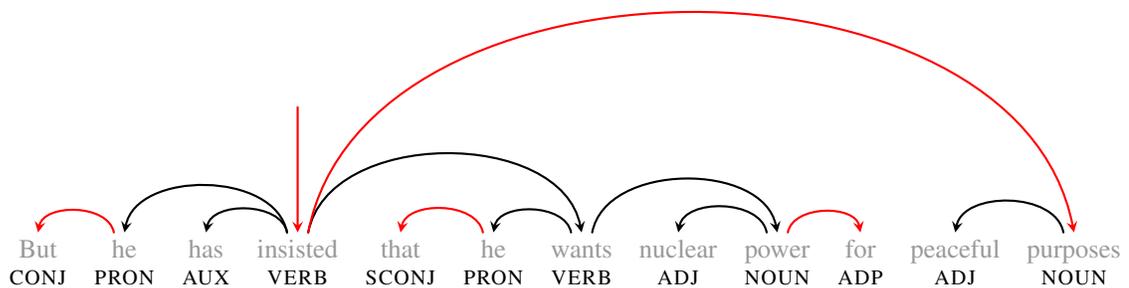

 \centering
 \definecolor{g70}{gray}{0.6}
 \newcommand{\gw}[1]{\color{g70}{#1}}
 \newcommand{\onthenext}{\gw{But} \& \gw{he} \& \gw{has} \& \gw{insisted} \& \gw{that} \& \gw{he} \& \gw{wants} \& \gw{nuclear} \& \gw{power} \& \gw{for} \& \gw{peaceful} \& \gw{purposes}\\
 {\sc conj} \& {\sc pron} \& {\sc aux} \& {\sc verb} \& {\sc sconj} \& {\sc pron} \& {\sc verb} \& {\sc adj} \& {\sc noun} \& {\sc adp} \& {\sc adj} \& {\sc noun}
 }
 \begin{minipage}[b]{.99\linewidth}
  \centering
  {\small
  \begin{dependency}[theme=simple,thick]
   \begin{deptext}[column sep=0.2cm]
    \onthenext \\
    \end{deptext}
   \depedge{4}{1}{}
   \depedge{4}{2}{}
   \depedge{4}{3}{}
   \deproot[edge unit distance=3ex]{4}{}
   \depedge{7}{5}{}
   \depedge{7}{6}{}
   \depedge{4}{7}{}
   \depedge{9}{8}{}
   \depedge{7}{9}{}
   \depedge{12}{10}{}
   \depedge{12}{11}{}
   \depedge{7}{12}{}
  \end{dependency}
  }
  \vspace{-10pt}
 \subcaption{Gold parse.}
 \end{minipage}
 \begin{minipage}[b]{.99\linewidth}
  \centering
  {\small
  \begin{dependency}[theme=simple,thick]
   \begin{deptext}[column sep=0.2cm]
    \onthenext \\
    \end{deptext}
   \depedge[red]{2}{1}{}
   \deproot[edge unit distance=3ex,red]{2}{}
   \depedge{4}{3}{}
   \depedge[red]{2}{4}{}
   \depedge[red]{6}{5}{}
   \depedge[red]{11}{6}{}
   \depedge[red]{6}{7}{}
   \depedge{9}{8}{}
   \depedge{7}{9}{}
   \depedge[red]{9}{10}{}
   \depedge{12}{11}{}
   \depedge[red]{4}{12}{}
  \end{dependency}
  }
  \vspace{-10pt}
 \subcaption{Output by the uniform model.}
 \end{minipage}
 \begin{minipage}[b]{.99\linewidth}
  \centering
  {\small
  \begin{dependency}[theme=simple,thick]
   \begin{deptext}[column sep=0.2cm]
    \onthenext \\
    \end{deptext}
   \depedge[red]{2}{1}{}
   \deproot[edge unit distance=3ex,red]{2}{}
   \depedge{4}{3}{}
   \depedge[red]{2}{4}{}
   \depedge[red]{6}{5}{}
   \depedge[red]{4}{6}{}
   \depedge[red]{6}{7}{}
   \depedge{9}{8}{}
   \depedge{7}{9}{}
   \depedge{12}{10}{}
   \depedge{12}{11}{}
   \depedge[red]{9}{12}{}
  \end{dependency}
  }
  \vspace{-10pt}
  \subcaption{Output by $C=3$.}
  \label{fig:ind:c3-eng-ud-2}
 \end{minipage}
 \begin{minipage}[b]{.99\linewidth}
  \centering
  {\small
  \begin{dependency}[theme=simple,thick]
   \begin{deptext}[column sep=0.2cm]
    \onthenext \\
    \end{deptext}
   \depedge[red]{2}{1}{}
   \depedge{4}{2}{}
   \depedge{4}{3}{}
   \deproot[edge unit distance=3ex,red]{4}{}
   \depedge[red]{6}{5}{}
   \depedge{7}{6}{}
   \depedge{4}{7}{}
   \depedge{9}{8}{}
   \depedge{7}{9}{}
   \depedge[red]{9}{10}{}
   \depedge{12}{11}{}
   \depedge[red]{4}{12}{}
  \end{dependency}
  }
  \vspace{-10pt}
  \subcaption{Output by $\beta_{len}=0.1$.}
  \label{fig:ind:beta-eng-ud-2}
 \end{minipage}
 \caption{Another comparison between outputs of the uniform model and $C=3$ in English UD.
 We also show $\beta_{len}=0.1$ for comparison.
 Although the score difference is small (see Table \ref{tab:ind:ud15-none}), the types of errors are different.
 In particular the most of parse errors by $C=3$ are at local attachments (first-order).
 For example it consistently recognizes a noun is a head of a verb, and a noun is a sentence root.
 Note an error on ``power $\rightarrow$ purposes'' is an example of PP attachment errors, which may not be solved under the current problem setting receiving only a POS tag sequence.
 }
 \label{fig:ind:error-eng-ud-2}
\end{figure}

\paragraph{Discussion}
The first observation, i.e., the output of $\beta_{len}=0.1$ + verb-or-noun reduces to that of the uniform model, is found in most other sentences as well.
Along with the results in other languages, we suspect the effect of the length bias gets weak when the root POS constraint is given.
We do not analyze the cause of this degradation more, but the discussion below on the difference between two constraints, i.e., the stack depth constraint and the length bias, might be relevant to that.

The essential difference between these two approaches is in the assumed structural form to be constrained:
The length bias (i.e., $\beta_{len}$) is a bias for each dependency arcs on the tree, while the stack depth constraint, which corresponds to the center-embeddedness, is inherently the constraint on constituent structures.
Interestingly, we can see the effect of this difference in the output parses in Figures \ref{fig:ind:error-eng-ud} and \ref{fig:ind:error-eng-ud-2}.
Note that we do not use the constraints at decoding and all differences are due to the learned parameters with the constraints during training.

Nevertheless, we can detect some typical errors in two approaches.
One difference between trees in Figure \ref{fig:ind:error-eng-ud} is in the constructions of a phrase ``On ... pictures''.
$\beta_{len}=0.1$ predicts that ``On the next two'' comprises a constituent, which modifies ``pictures'' while $C=2$ and $C=3$ predict that ``the next two pictures'' comprises a constituent, which is correct, although the head of a determiner is incorrectly predicted.
On the other hand, $\beta_{len}=0.1$ works well to find more primitive dependency arcs between POS tags, such as arcs from verbs to nouns, which are often incorrectly recognized by stack depth constraints.
Similar observations can be made in trees in Figure \ref{fig:ind:error-eng-ud-2}.
See the constructions on ``for peaceful purposes''.
In is only $C=3$ (and $C=2$ though we omit) that predicts it becomes a constituent.
In other positions, again, $\beta_{len}=0.1$ works better to find local dependency relationships.
The head of ``purposes'' is predicted differently, but this choice is equally difficult in the current problem setting (see the caption of Figure \ref{fig:ind:error-eng-ud-2}).

These observations may explain the reason why the root POS constraints work better with the stack depth constraints than the dependency length bias.
With the stack depth constraints, the main source of improvements is detections of constituents, but this constraint itself does not help to resolve some dependency relationships, e.g., the dependency direction between verbs and nouns.
The root POS constraints are thus orthogonal to this approach.
They may help to solve the remaining ambiguities, e.g., the head choice between a noun and a verb.
On the other hand, the dependency length bias is the most effective to find basic dependency relationships between POS tags while the resulting tree may contain implausible constituent units.
Thus the effect of the length bias seems somewhat overlapped with the root POS constraints, which may be the reason why they do not well collaborate with each other.

\paragraph{Other languages}

\REVISE{
We further inspect the results of some languages with exceptional behaviors seprately below.  

 \begin{description}
  \item[Japanese] In Figure \ref{fig:ind:ud15-none}, we can see that the performance of Japanese is the best with a strong stack depth constraint, such as depth 1 and $C=2$, and the performance drops when relaxing the constraint.
            This may be counterintuitive from our oracle results in Chapter \ref{chap:transition} (e.g., Figure \ref{fig:load-comparison-relax-ud}) that Japanese is the language in which the ratio of center-embedding is relatively higher.
            
            Inspecting the output parses, we found that these results are essentially due to the word order of Japanese, which is mainly head final.
            With a strong constraint (e.g., the stack depth one), the model tries to build a parse that is purely left- or right-branching.
            An easy way to create such parse is placing a root word at the beginning or the end of the sentence, and then connecting adjacent tokens from left to right, or right to left.
            This is what happened when a severe constraint, e.g., the maximum stack depth of 1 is imposed.
            Since the position of root token is in most cases correctly identified, the score gets relatively higher.
            On the other hand, when relaxing the constraint, the model also try to explore parses in which the root token is not the beginning/end of the sentence, but internal positions, and the model fail to find the head final pattern of Japanese.

            This Japanese result suggests that sometimes our stack depth constraint helps learning even when the imposed stack depth bound does not fit well to the syntax of the target language, though the learning behavior differs from our expectation.
            In this case, the model does not capture the syntax correctly in the sense that Japanese sentences cannot be parsed with a severe stack depth bound, but the model succeeded to find syntactic patterns that are a very rough approximation of the true syntax, resulting in a higher score.
 \item[Finnish] Finnish is an inflectional language with rich morphologies and with little function words.
            This is essentially the reason for consistent lower accuracies of Finnish even when the constraint on root POS tags is given.
            Recall that all our models are imposed the function word constraint (Section \ref{sec:ind:param-const}).
            Though our primary motivation to introduce this constraint is to alleviate problems in evaluation, it also greatly reduces the search space if the ratio of function words is high.
            Also at test time, a higher ratio of function words indicates a higher chance of correct attachments since the head candidates for a function word is limited to other content words.\footnote{Recall that although we remove constraints at test time the model rarely find a parse with function words at internal positions since the model is trained to avoid such parses.}
            Below is an example of a dependency tree in Finnish treebank:

            \begin{center}
            {\small
            \definecolor{g70}{gray}{0.6}
            \newcommand{\gw}[1]{\color{g70}{#1}}
               \begin{dependency}[theme=simple,thick]
               \begin{deptext}[column sep=0.2cm]
                \gw{Liikett\"a} \& \gw{ei} \& \gw{ole} \& \gw{ei} \& \gw{toimintaa} \\
                {\sc noun} \& {\sc verb} \& {\sc verb} \& {\sc verb} \& {\sc noun} \\
               \end{deptext}
               \depedge{3}{1}{}
               \depedge{3}{2}{}
               \deproot[edge unit distance=1.5ex]{3}{}
               \depedge{3}{4}{}
               \depedge{4}{5}{}
               \end{dependency}
             }
            \end{center}

            This sentence comprises of {\sc noun} and {\sc verb} only, and there are a lot of similar sentences.
            This example also explains the reason why the performance of Finnish is still low with the root POS constraints.
            Table \ref{ind:tab:func} lists the statistics about the ratio of function words in the training corpora.
            We can see that it is only Finnish that the ratio of function words is less than 10\%.
            Also, the ratio in Japanese is very high.
            This probably explains the reason for relatively high overall scores of Japanese.
            Thus, the variation of the scores across languages in the current experiment is largely explained by the ratio of function words in each language.
 \item[Greek] In Figure \ref{tab:ind:ud15-none}, the scores on Greek with the stack depth constraints are consistently worse than the uniform baseline.
            Though overall scores are low, the situation largely changes with the root POS constraints, and with them the scores get stable.

            A possible explanation for these exceptional behaviors might be the relatively small ratio of function words (Table \ref{ind:tab:func}) in the data along with the small size of the training data (Table \ref{tab:ind:stat-ud15}), both of which could be partially alleviated with the root POS constraints.

            More linguistically intuitive explanation might be that Greek is a relatively free word order language and our structural constraints do not work well for guiding the model for finding such grammars.
            However, to make such conclusion, we have to set up experiments more carefully, e.g., by eliminating the bias caused by the smaller size of the data.
            We thus leave it our future work to discuss the limitation of the current approach with a typological difference in each language.
 \end{description}

 \begin{table}[t]
  \centering
   \begin{minipage}{.35\textwidth}
    \centering
  \begin{tabular}[t]{l|l}
   & Ratio (\%) \\\hline
basque & 26.57 \\
bulgarian & 25.88 \\
croatian & 24.55 \\
czech & 20.09 \\
danish & 30.66 \\
english & 27.98 \\
finnish & {\bf 9.63} \\
french & 37.84 \\
german & 32.09 \\
greek & 16.94 \\
  \end{tabular}
  \end{minipage}
   \begin{minipage}{.35\textwidth}
    \centering
  \begin{tabular}[t]{l|l}
   & Ratio (\%) \\\hline
hebrew & 32.29 \\   
hungarian & 23.76 \\
indonesian & 19.68 \\
irish & 36.09 \\
italian & 37.73 \\
japanese & {\bf 45.14} \\
persian & 23.25 \\
spanish & 36.99 \\
swedish & 29.64 \\
  \end{tabular}
  \end{minipage}
 \caption{Ratio of function words in the training corpora of UD (sentences of length 15 or less).}
 \label{ind:tab:func}
\end{table}
 }

\subsection{Google Universal Dependency Treebanks}

So far our comparison is limited in the models of our baseline DMV model with some constraints.
Next we see the relative performance of this approach compared to the current state-of-the-art unsupervised systems on another dataset, Google treebanks.

Table \ref{tab:ind:google} shows the result.
The scores of the other systems are borrowed from \newcite{grave-elhadad:2015:ACL-IJCNLP}.
In this experiment, we only focus on the settings where the root word is restricted with the verb-otherwise-noun constraint.
Among our structural constraints, again our stack depth constraints perform the best.
In particular the scores with $C=2$ are stable across languages.

\begin{table}[t]
 \centering
 \begin{tabular}[t]{lcccccc}\hline
&Unif.&$C=2$&$C=3$&$\beta_{len} = 0.1$&Naseem10&Grave15 \\\hline
German&64.5&64.3&64.6&62.5&53.4&60.2 \\
English&57.9&59.5&57.9&56.9&66.2&62.3 \\
Spanish&68.2&71.1&70.5&69.6&71.5&68.8 \\
French&69.2&69.6&70.1&66.4&54.1&72.3 \\
Indonesian&66.8&67.4&66.0&66.7&50.3&69.7 \\
Italian&43.9&67.3&65.9&44.0&46.5&64.3 \\
Japanese&47.5&54.5&47.4&47.6&58.2&57.5 \\
Korean&28.6&30.7&28.3&43.2&48.8&59.0 \\
Br-Portuguese&63.0&67.1&62.7&62.6&46.4&68.3 \\
Swedish&67.4&67.9&67.3&66.4&64.3&66.2 \\\\
Avg&57.7&62.0&60.1&58.6&56.0&64.8 \\\hline
 \end{tabular}
 \caption{Attachment scores on Google universal treebanks (up to length 10).
 All proposed models are trained with the verb-otherwise-noun constraint.
 Naseem10 = the model with manually crafted syntactic rules between POS tags (Naseem et al., 2010);
 Grave15 = also relies on the syntactic rules but is trained discriminatively (Grave and Elhadad, 2015).}
 \label{tab:ind:google}
\end{table}

All our method outperforms the strong baseline model of \newcite{naseem-EtAl:2010:EMNLP}, which encodes manually crafted rules (12 rules) such as {\sc verb $\rightarrow$ noun} and {\sc adp $\rightarrow$ noun} via the posterior regularization method \cite{journals/jmlr/GanchevGGT10}.
Compared to this, our baseline method uses fewer syntactic rules via parameter-based constraints, in total 5 (3 for function words and the verb-otherwise-noun constraint) and is much simpler than their posterior regularization method.

\newcite{grave-elhadad:2015:ACL-IJCNLP} is a more sophisticated model, which utilizes the same syntactic rules as the Naseem et al.'s method.
Our models do not outperform this model, though it is only Korean that ours do not perform competitively to their model.

\section{Discussion}
We found that our stack depth constraints improve the performance of unsupervised grammar induction across languages and datasets in particular when some seed knowledge about grammar is given to the model.
However, we also find that in many languages the improvements from the no structural constraint baseline becomes small when such knowledge is given.
Also, the performance reaches to the current state-of-the-art method, which utilizes much more complex machine learning techniques as well as manually specific syntactic rules.
We thus speculate that our models already reach some limitation under the current problem setting, that is, learning of dependency grammar from the POS input only.

Recall that the annotation of UD is content-head based and every function word is a dependent of the most closely related content word.
This means under the current first-order model on POS tag inputs, many important information that currently a supervised parser would exploit is abandoned.
For example, the model would collapses both some noun phrase and prepositional phrase into its head (probably {\sc noun}) while this information is crucial;
e.g., an adjective cannot be attached to a prepositional phrase, etc.
One way to exploit such clue for disambiguation is to utilize the boundary information \cite{spitkovsky-alshawi-jurafsky:2012:EMNLP-CoNLL,spitkovsky-alshawi-jurafsky:2013:EMNLP}.

Typically, unsupervised learning of structures gets more challenging when employing more structurally complex models.
However, one of our strong observations from the current experiment is that our stack depth constraint rarely hinders, i.e., does not decrease the performance.
Here we have focused on very simple generative model of DMV though it may be more interesting to see what happens when imposing this constraint on more structurally complex models on which learning is much harder.
There remains many rooms for further improvements and such type of study would be important toward one of the goals of unsupervised parsing of identifying which structure can be learned without explicit supervisions \cite{bisk-hockenmaier:2015:ACL-IJCNLP}.

\REVISE{
In this work, our main focus was the imposed structural constraints (or linguistic prior) themselves, and we did not care much about the method to encode these prior knowledge.
That is, our method to inject constraints was a crude way, i.e., via hard constraints in the E step (Eq. \ref{eqn:ind:cost}), and there exist more sophisticated methods with newer machine learning techniques.
Posterior regularization (PR) that we compared the performance with \cite{naseem-EtAl:2010:EMNLP} is one of such techniques.
The crucial difference between our imposing hard constraints and PR is that PR imposes constraints in expectation, i.e., every constraint becomes a {\it soft} constraint.

If we reformulate our model with PR, then that may probably impose a soft constraint, e.g., {\it the expected number of occurrence of center-embedding up to some degree in a parse is less than X}.
{\it X} becomes 1 if we try to resemble the behavior of our hard constraints, but could be any values, and such flexibility is one advantage, which our current method cannot appreciate.
Thus, from a machine learning perspective, one interesting direction is to compare the performances of two approaches with the (conceptually) same constraints.
}

\section{Conclusion}
In this study, we have shown that our imposed stack depth constraint improves the performance of unsupervised grammar induction in many settings.
Specifically, it often does not harm the performance when it already performs well while it reinforces the relatively poorly performed models (Table \ref{tab:ind:google}).
One limitation of the current approach is that the information that the parser can utilize is very superficial (i.e., the first order model on content-head based POS tags).
However, our positive results in the current experiment are an important first step for the current line of research and encourage further study on more structurally complex model beyond the simple DMV model.

\chapter{Conclusions}
\label{chap:conclusion}

Identifying universal syntactic constraints of language is an attractive goal both from the theoretical and empirical viewpoints.
To shed light on this fundamental problem, in this thesis, we pursued the {\it universalness} of the language phenomena of center-embedding avoidance, and its practical utility in natural language processing tasks, in particular unsupervised grammar induction.
Along with these investigations, we develop several computational tools capturing the syntactic regularities stemming from center-embedding.

The tools we presented in this thesis are left-corner parsing methods for dependency grammars.
We formalized two related parsing algorithms.
The transition-based algorithm presented in Chapter \ref{chap:transition} is an incremental algorithm, which operates on the stack, and its stack depth only grows when processing center-embedded constructions.
We then considered tabulation of this incremental algorithm in Chapter \ref{chap:induction}, and obtained an efficient polynomial time algorithm with the left-corner strategy.
In doing so, we applied {\it head-splitting} techniques \cite{Eisner:1999:EPB:1034678.1034748,eisner-2000-iwptbook}, with which we removed the spurious ambiguity and reduced the time complexity from $O(n^6)$ to $O(n^4)$, both of which were essential for our application of inside-outside calculation with filtering.

Dependency grammars were the most suitable choice for our cross-linguistic analysis on language universality, and we obtained the following fruitful empirical findings using the developed tools for them.
\begin{itemize}
 \item Using multilingual dependency treebanks, we quantitatively demonstrate the universalness of center-embedding avoidance.
       We found that most syntactic constructions across languages can be covered within highly restricted bounds on the degree of center-embedding, such as one, or zero, when relaxing the condition of the size of embedded constituent.
 \item From the perspective of parsing {\it algorithms}, the above findings mean that a left-corner parser can be utilized as a tool for exploiting universal constraints during parsing.
       We verified this ability of the parser empirically by comparing the growth of stack depth when analyzing sentences on treebanks with those of existing algorithms, and showed that only the behavior of the left-corner parser is consistent across languages.
 \item Based on these observations, we examined whether the found syntactic constraints help in finding the syntactic patterns (grammars) in the given sentences through experiments on unsupervised grammar induction, and found that our method often boosts the performance from the baseline, and competes with the current state-of-the-art method in a number of languages.
\end{itemize}

We believe the presented study will be the starting point of many future inquiries.
As we have mentioned several times, our choice of dependency grammars for the representation was motivated by its cross-linguistic suitability as well as its computational tractability.
Now we have evidences on the language universality of center-embedding avoidance.
We consider thus one exciting direction would be to explore unsupervised learning of constituent structures exploiting our found constraint, which has not been solved yet with traditional PCFG-based methods.
Note that unlike the dependency length bias, which is only applicable for dependency-based models, our constraint is conceptually free from grammar formalisms.


As we have mentioned in Chapter \ref{chap:1}, recently there has been a growing interest on the tasks of grounding, or semantic parsing.
Also, another direction of grammar induction in particular with more sophisticated grammar formalisms, such as CCG, has been initiated with some success.
There remains many open questions in these settings, e.g., on the necessary amount of seed knowledge to make learning tractable \cite{bisk-hockenmaier:2015:ACL-IJCNLP,AAAI159835}.
Arguably, the system with less initial seed knowledge or less assumption on the specific task is preferable (by keeping accuracies).
We hope our introduced constraint helps in reducing the required assumption, or improving the performance in those more general grammar induction tasks.
Finally, we suspect that the study of child language acquisition would also have to be discussed within the setting of grounding, i.e., with some kind of distant supervision or perception.
Although we have not explored the cognitive plausibility of the presented learning and parsing methods, our empirical finding that when learning from relatively short sentences a severe stack depth constraint (relaxed depth one) often improves the performance may become an appealing starting point for exploring computational models of child language acquisition with human-like memory constraints.

\appendix

\chapter{Analysis of Left-corner PDA}

\label{chap:app:analyze-pda}

\REVISE{
This appendix contains the proof of Theorem \ref{thoerem:bg:stack-depth}, which establishes the connection between the stack depth of the left-corner PDA and the degree of center-embedding.
For proving this, we first need to extend the notion of center-embedding for a {\it token} as follows:
\begin{newdef}
 \label{def:bg:embedding-token}
 Given a sentence and token $e$ (not the initial token) in the sentence, we write the derivation from $S$ to $e$ as follows with the minimal number of $\Rightarrow$:
 \begin{align}
  S \Rightarrow_{lm}^* v \underline{A} \alpha &\Rightarrow_{lm}^+ v w_1 \underline{B_1} \alpha \Rightarrow_{lm}^+ v w_1 \underline{C_1} \beta_1 \alpha \nonumber\\ 
  &\Rightarrow_{lm}^+ v w_1 w_2 \underline{B_2} \beta_1 \alpha
  \Rightarrow_{lm}^+ v w_1 w_2 \underline{C_2} \beta_2 \beta_1 \alpha \nonumber\\
  &\Rightarrow_{lm}^+ \cdots \label{eq:bg:embed-token-degree} \\
  &\Rightarrow_{lm}^+ v w_1 \cdots w_{m_e} \underline{B_{m_e}} \beta_{{m_e}-1} \cdots \beta_1 \alpha
  \Rightarrow_{lm}^{\color{red}{*}} v w_1 \cdots w_{m_e} \underline{C_{m_e}} \beta_{m_e} \beta_{{m_e}-1} \cdots \beta_1 \alpha \nonumber \\
  &\Rightarrow_{lm}^{\color{red}{*}} v w_1 \cdots w_{m_e} {\color{red}{x' \underline{E}}} \beta_{m_e} \beta_{{m_e}-1} \cdots \beta_1 \alpha
  \Rightarrow_{lm} v w_1 \cdots w_{m_e} {\color{red}{x' e}} \beta_{m_e} \beta_{{m_e}-1} \cdots \beta_1 \alpha \nonumber,
 \end{align}
 where the underlined symbol is the expanded symbol by the following $\Rightarrow$.
 Then, the degree of center-embedding for token $e$ is:
 \begin{itemize}
  \item $m_e - 1$ if $C_{m_e} = E$ (i.e., $x' = \varepsilon$) or $B_{m_e} = C_{m_e} = E$ (i.e., $\beta_{m_e} = x' = \varepsilon$); and
  \item $m_e$ otherwise.
 \end{itemize}
 The degree of the token at the beginning of the sentence is defined as 0.
\end{newdef}

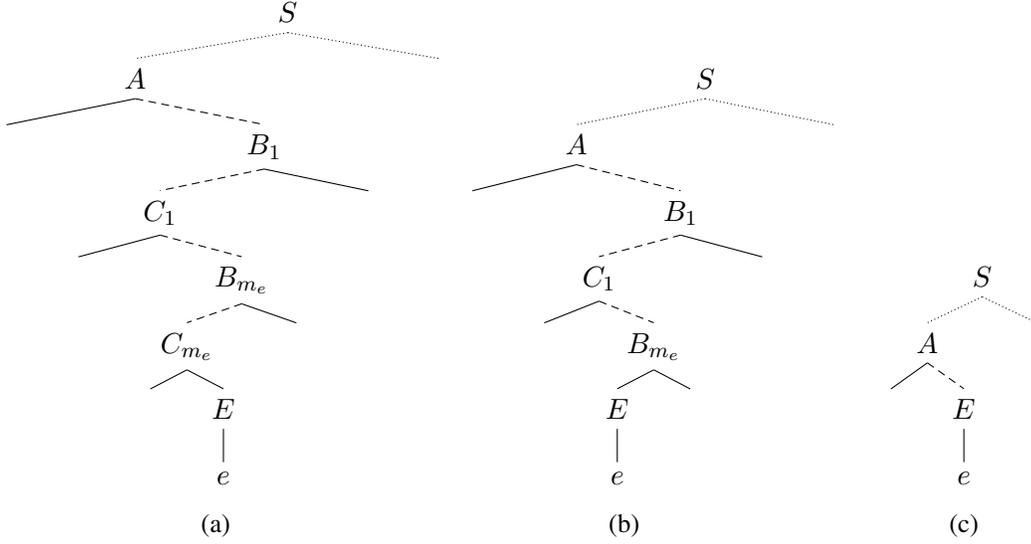
\begin{figure}[t]
 \begin{minipage}[t]{.4\linewidth}
  \centering
  \begin{tikzpicture}[sibling distance=12pt]
   \tikzset{level distance=25pt}
   \Tree[.$S$ \edge[densely dotted];
    [.$A$
     {\color{white}{A}}
     \edge[densely dashed];
     [.$B_1$
      \edge[densely dashed];
      [.$C_1$
       {\color{white}{A}}
       \edge[densely dashed];
       [.$B_{m_e}$
        \edge[densely dashed];
        [.$C_{m_e}$ {\color{white}{A}} [.$E$ $e$ ] ] {\color{white}{A}}
       ]
      ]
      {\color{white}{A}}
     ]
    ]
    \edge[densely dotted]; {\color{white}{A}}
   ]
  \end{tikzpicture}
  \subcaption{}\label{fig:bg:branch-for-token:a}
 \end{minipage}
 \begin{minipage}[t]{.3\linewidth}
  \centering
  \begin{tikzpicture}[sibling distance=12pt]
   \tikzset{level distance=25pt}
   \Tree[.$S$ \edge[densely dotted];
    [.$A$
     {\color{white}{A}}
     \edge[densely dashed];
     [.$B_1$
      \edge[densely dashed];
      [.$C_1$
       {\color{white}{A}}
       \edge[densely dashed];
       [.$B_{m_e}$
        [.$E$ $e$ ] {\color{white}{A}}
       ]
      ]
      {\color{white}{A}}
     ]
    ]
    \edge[densely dotted]; {\color{white}{A}}
   ]
  \end{tikzpicture}
  \subcaption{}\label{fig:bg:branch-for-token:b}
 \end{minipage}
 \begin{minipage}[t]{.28\linewidth}
  \centering
  \begin{tikzpicture}[sibling distance=12pt]
   \tikzset{level distance=25pt}
   \Tree[.$S$ \edge[densely dotted];
    [.$A$
     {\color{white}{A}}
     \edge[densely dashed];
     [.$E$ $e$ ]
    ]
    \edge[densely dotted]; {\color{white}{A}}
   ]
  \end{tikzpicture}
  \subcaption{}\label{fig:bg:branch-for-token:c}
 \end{minipage}
 \caption{Three types of realizations of Eq. \ref{eq:bg:embed-token-degree}.
 Dashed edges may consist of more than one edge (see Figure \ref{fig:bg:branch-hide} for example) while dotted edges may not exist (or consist of more than one edge).
 (a) $E$ is a right child of $C_{m_e}$ and thus the degree of center-embedding is $m_e$.
 (b) $E$ is a left child of $B_{m_e}$ (i.e., $C_{m_e} = E$) and the degree is $m_e-1$; when $m=1$, $C_1 = E$ and thus no center-embedding occurs.
 (c) $B_{m_e} = C_{m_e} = E$; note this happens only when $m_e=1$ (see body).
 }
 \label{fig:bg:branch-for-token}
\end{figure}
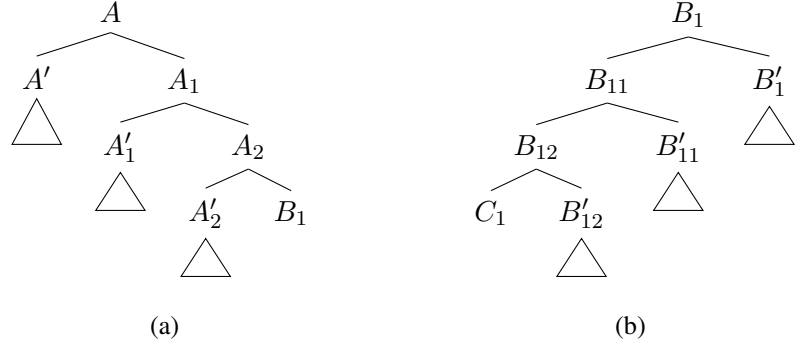
\begin{figure}[t]
 \centering
 \begin{minipage}[t]{.4\linewidth}
  \centering
  \begin{tikzpicture}[sibling distance=12pt]
   \tikzset{level distance=25pt}
   \Tree[.$A$
    [.$A'$ \edge[roof]; {~~~~~} ]
    [.$A_1$
     [.$A_1'$ \edge[roof]; {~~~~~} ]
     [.$A_2$
      [.$A_2'$ \edge[roof]; {~~~~~} ]
      $B_1$
     ]
    ]
   ]
  \end{tikzpicture}
  \subcaption{}\label{fig:bg:branch-hide:a}
 \end{minipage}
 \begin{minipage}[t]{.4\linewidth}
  \centering
  \begin{tikzpicture}[sibling distance=12pt]
   \tikzset{level distance=25pt}
   \Tree[.$B_1$
    [.$B_{11}$
     [.$B_{12}$
      $C_1$
      [.$B_{12}'$ \edge[roof]; {~~~~~} ]
     ]
     [.$B_{11}'$ \edge[roof]; {~~~~~} ]
    ]
    [.$B_{1}'$ \edge[roof]; {~~~~~} ]
   ]
  \end{tikzpicture}
  \subcaption{}\label{fig:bg:branch-hide:b}
 \end{minipage}
 \caption{(a) Example of realization of a path between $A$ and $B_1$ in Figure \ref{fig:bg:branch-for-token}. (b) The one between $B_1$ and $C_1$.}
 \label{fig:bg:branch-hide}
\end{figure}

The main difference of Eq. \ref{eq:bg:embed-token-degree} in this definition from Eq. \ref{eq:bg:embed-depth} is that instead of expanding $C_{m_e}$ to string $x$, we take into account the right edges from $C_{m_e}$ to another nonterminal (preterminal) $E$, which should exist if the requisite in Eq. \ref{eq:bg:embed-depth} that $|x| \geq 2$ is satisfied.
Eq. \ref{eq:bg:embed-token-degree} explains a zig-zag path from the start symbol $S$ to a token (terminal $e$), which can be classified into three cases in Figure \ref{fig:bg:branch-for-token}.
Definition \ref{def:bg:embedding-token} determines the degree of center-embedding of that token depending on the structure of this path, which will be explained further below.
\begin{itemize}
 \item Given terminal $e$, the derivation of the form in Eq. \ref{def:bg:embedding-token} is deterministic, and each $B_{i}$ or $C_{i}$ is determined as a turning point on a zig-zag path;
       see e.g., a path from $c$ to $S$ in Figure \ref{fig:2:depth-2}.
       $A$ is the starting point, which might be identical to $S$.
       This is indicated with dotted edges in Figure \ref{fig:bg:branch-for-token};
       Figure \ref{fig:2:depth-2} is such a case.
 \item We allow an empty transition from $B_{m_e}$ to $C_{m_e}$ and $C_{m_e}$ to $E$ at the last transitions in Eq. \ref{def:bg:embedding-token}, which are important to define the degree in the case where the preterminal $E$ for token $e$ is not the right child of the parent ($C_{m_e} = E$), or no center-embedding is involved (i.e., $B_{m_e} = C_{m_e} = E$ and $m_e=1$).
       Figure \ref{fig:bg:branch-for-token:b} is the complete case without empty transitions, while Figures \ref{fig:bg:branch-for-token:b} and \ref{fig:bg:branch-for-token:c} involve empty transitions.
       Figure \ref{fig:bg:branch-for-token:b} with $m_e=1$, where the degree is $m_e - 1 = 0$, is an example of the parse in Figure \ref{fig:2:not-center-embedding}, where no center-embedding is involved ($b$ corresponds to $e$ in Figure \ref{fig:bg:branch-for-token:b}).
       Figure \ref{fig:bg:branch-for-token:c} is the case where the empty transition from $B_{m_e}$ to $C_{m_e}$ occurs.
       Note that this pattern only occurs for $m_e=1$, which includes the derivation to the last token of the sentence, where the path is always right edges from $S$ (or $A$) to $B_1$ (or $E$) and the degree is 0.
       This is because the derivation with an empty transition from $B_{m_e}$ to $C_{m_e}$ indicates $B_{m_e} = E$, though when $m>1$, it safely reduces to the case of $m_e-1$ in Figure \ref{fig:bg:branch-for-token:a}.
 \item Given a CFG parse, the maximum value of the degree in Definition \ref{def:bg:embedding-token} among tokens in the sentence is identical to the degree of center-embedding defined for that parse (Definition \ref{def:bg:embed-depth}).
\end{itemize}

We next prove the following lemma, which is closely connected to Theorem \ref{thoerem:bg:stack-depth}.
\begin{newlemma}
 \label{lemma:bg:before-shift}
 Given token $e$ (not the initial token) in the sentence, let $m'_e$ be the degree of center-embedding of it, and $\delta_e$ be the stack depth before it is shifted for recognizing that parse on the left-corner PDA.
 Then, $\delta_e = m'_e + 1$. 
\end{newlemma}
\begin{proof}
 The path from $S$ to every token in the sentence except the beginning of the sentence can be classified into three cases in Figure \ref{fig:bg:branch-for-token}.
 We show that in every case between the stack depth $\delta_e$ before $e$ is shifted and the degree of center-embedding $m'_e$, $\delta_e = m'_e + 1$ holds.

 Note first that in all cases, the existence of edges from $S$ to $A$ (i.e., whether $S=A$ or not) does not affect the stack depth $\delta_e$.
 This is due to the basic order of building a parse in the left-corner PDA, which always completes a left subtree first, and then expands it with {\sc Prediction}.
 Thus, in the following, we ignore the existence of $S$, and focus on the stack depth at $e$ during building a subtree rooted at $A$.
 
 \begin{itemize}
  \item[(a)] The path from $C_{m_e}$ to $E$ exists (Figure \ref{fig:bg:branch-for-token:a}):
             In this case, the degree of center-embedding $m'_e = m_e$.
             Before shifting $e$, the following stack configuration occurs:
             \begin{equation}
              A/B_1 ~ \underbrace{C_1/B_2 ~ C_2/B_3 ~ \cdots ~ C_{m_e-1}/B_{m_e} ~ C_{m_e}/E}_{m_e}, \label{eq:bg:config-before}
             \end{equation}
             This can be shown as follows.

             The PDA first makes symbol $A/B_1$.
             Note that the path from $A$ to $B_1$ may contain many nonterminals as shown in Figure \ref{fig:bg:branch-hide:a}.
             During processing these nodes, the PDA first builds a subtree rooted at $A'$, then performs {\sc Prediction}, which results in symbol $A/A_1$.
             After that, $A'_1$ is built with $A/A_1$ being remained on the stack, and then connect them with {\sc Composition}, which results in $A/A_2$.
             Finally $A/B_1$ remains on the stack after repeating this process.

             Then, $C_1/B_2$ is made on the stack in the similar manner, but both symbols remain on the stack, since $A/B_1$ cannot be combined with another subtree unless it is complete (without a predicted node).
             There may exist many nonterminals between $B_1$ and $C_1$ as in Figure \ref{fig:bg:branch-hide:b}, but they does not affect the configuration of the stack;
             for example, $B_{12}$ is first introduced after a subtree rooted at $C_1$ is complete.
             This indicates that the stack accumulates symbols $C_i/B_{i+1}$ as the number of right edges between them increases.
             Finally, after building a subtree rooted at the left child of $C_{m_e}$, it is converted to $C_{m_e}/E$ by {\sc Prediction}, resulting in the stack configuration of Eq. \ref{eq:bg:config-before}.
             This occurs just before $e$ is shifted on the stack by {\sc Scan}.
  \item[(b)] $C_{m_e} = E$ (Figure \ref{fig:bg:branch-for-token:b}):
             In this case $m'_e = m_e - 1$.
             Before shifting $e$, the stack configuration is:
             \begin{equation}
              A/B_1 ~ \underbrace{C_1/B_2 ~ C_2/B_3 ~ \cdots ~ C_{m_e-1}/B_{m_e}}_{m_e-1}. \label{eq:bg:config-before:b}
             \end{equation}
             $e$ is shifted on this stack by {\sc Shift}, and then {\sc Composition} is performed between $C_{m_e-1}/B_{m_e}$ and $E$.
             Thus $\delta_e = m_e = m'_e + 1$.
  \item[(c)] $B_1 = C_1 = E$ (Figure \ref{fig:bg:branch-for-token:c}):
             The stack configuration before shifting $e$ is apparently $A/E$.
             $m' = 0$, so $\delta_e = m' + 1$ holds.
 \end{itemize}
\end{proof}

Now the proof of Theorem \ref{thoerem:bg:stack-depth} is immediate from Lemma \ref{lemma:bg:before-shift}.
\begin{proof}[Proof of Theorem \ref{thoerem:bg:stack-depth}]
 The relationship between Definitions \ref{def:bg:embed-depth} and \ref{def:bg:embedding-token} is that the maximum value of the degree given by Definition \ref{def:bg:embedding-token} for each token is the same as the degree of a parse.
 Given $e, \delta_e, m'_e$ in Lemma \ref{lemma:bg:before-shift}, $\delta_e = m'_e+1$.
 Let $e^* = \arg\max_{e} \delta_e$.
 ``The maximum value of the stack depth after a reduce transition'' in Theorem \ref{thoerem:bg:stack-depth} can be translated to the maximum value {\it before} a reduce transition, which is $\delta_{e^*}$.
 Thus, $\delta_{e^*} = m'_{e^*} + 1$.
 Arranging, $m'_{e^*} = \delta_{e^*} - 1$.
\end{proof}
}

\chapter{Part-of-speech tagset in Universal Dependencies}
\label{chap:app:ud-pos}
Universal Dependencies (UD) uses the following 17 part-of-speech (POS) tags.
\begin{itemize}
 \item {\sc adj}: adjective
 \item {\sc adp}: adposition
 \item {\sc adv}: adverb
 \item {\sc aux}: auxiliary verb
 \item {\sc conj}: coordinating conjunction
 \item {\sc det}: determiner
 \item {\sc intj}: interjection
 \item {\sc noun}: noun
 \item {\sc num}: numeral
 \item {\sc pron}: pronoun
 \item {\sc verb}: verb
 \item {\sc part}: particle
 \item {\sc pron}: pronoun
 \item {\sc sconj}: subordinating conjunction
 \item {\sc punct}: punctuation
 \item {\sc sym}: symbol
 \item {\sc x}: other
\end{itemize}

\bibliographystyle{acl}
\bibliography{dissertation}

\begin{thebibliography}{}

\bibitem[\protect\citename{Abney and Johnson}1991]{abney91memory}
Steven Abney and Mark Johnson.
\newblock 1991.
\newblock Memory requirements and local ambiguities of parsing strategies.
\newblock {\em Journal of Psycholinguistic Research}, 20(3):233--250.

\bibitem[\protect\citename{Abney}1987]{abney1987english}
Stephen~P. Abney.
\newblock 1987.
\newblock {\em The English Noun Phrase In Its Sentential Aspect}.
\newblock {Ph.D.} thesis, M.I.T.

\bibitem[\protect\citename{Aduriz \bgroup et al.\egroup }2003]{eu}
Itzair Aduriz, María~Jesús Aranzabe, Jose~Mari Arriola, Aitziber Atutxa,
  Arantza Díaz~de Ilarraza, Aitzpea Garmendia, and Maite Oronoz.
\newblock 2003.
\newblock Construction of a {Basque} dependency treebank.
\newblock In {\em Proceedings of the 2nd Workshop on Treebanks and Linguistic
  Theories}.

\bibitem[\protect\citename{Afonso \bgroup et al.\egroup }2002]{pt}
Susana Afonso, Eckhard Bick, Renato Haber, and Diana Santos.
\newblock 2002.
\newblock {{``}Floresta sint{\'a}(c)tica{''}:} a treebank for {P}ortuguese.
\newblock In {\em Proceedings of the 3rd International Conference on Language
  Resources and Evaluation (LREC)}, pages 1698--1703, Las Palmas, Spain.

\bibitem[\protect\citename{Agić and Ljubešić}2014]{hr}
Željko Agić and Nikola Ljubešić.
\newblock 2014.
\newblock The {SETimes.HR} linguistically annotated corpus of {Croatian}.
\newblock In {\em Proceedings of LREC 2014}, pages 1724--1727, Reykjavík,
  Iceland.

\bibitem[\protect\citename{Aho and Ullman}1972]{Aho:1972:TPT:578789}
Alfred~V. Aho and Jeffrey~D. Ullman.
\newblock 1972.
\newblock {\em The Theory of Parsing, Translation, and Compiling}.
\newblock Prentice-Hall, Inc., Upper Saddle River, NJ, USA.

\bibitem[\protect\citename{Alshawi}1996]{alshawi:1996:ACL}
Hiyan Alshawi.
\newblock 1996.
\newblock Head automata and bilingual tiling: Translation with minimal
  representations (invited talk).
\newblock In {\em Proceedings of the 34th Annual Meeting of the Association for
  Computational Linguistics}, pages 167--176, Santa Cruz, California, USA,
  June. Association for Computational Linguistics.

\bibitem[\protect\citename{Ammar \bgroup et al.\egroup }2014]{NIPS2014_5344}
Waleed Ammar, Chris Dyer, and Noah~A Smith.
\newblock 2014.
\newblock Conditional random field autoencoders for unsupervised structured
  prediction.
\newblock In Z.~Ghahramani, M.~Welling, C.~Cortes, N.D. Lawrence, and K.Q.
  Weinberger, editors, {\em Advances in Neural Information Processing Systems
  27}, pages 3311--3319. Curran Associates, Inc.

\bibitem[\protect\citename{Atalay \bgroup et al.\egroup }2003]{tr}
Nart~B. Atalay, Kemal Oflazer, and Bilge Say.
\newblock 2003.

\bibitem[\protect\citename{Ballesteros and
  Nivre}2013]{journals/coling/BallesterosN13}
Miguel Ballesteros and Joakim Nivre.
\newblock 2013.
\newblock Going to the roots of dependency parsing.
\newblock {\em Computational Linguistics}, 39(1):5--13.

\bibitem[\protect\citename{Berant \bgroup et al.\egroup
  }2013]{berant-EtAl:2013:EMNLP}
Jonathan Berant, Andrew Chou, Roy Frostig, and Percy Liang.
\newblock 2013.
\newblock Semantic parsing on {Freebase} from question-answer pairs.
\newblock In {\em Proceedings of the 2013 Conference on Empirical Methods in
  Natural Language Processing}, pages 1533--1544, Seattle, Washington, USA,
  October. Association for Computational Linguistics.

\bibitem[\protect\citename{Berg-Kirkpatrick \bgroup et al.\egroup
  }2010]{bergkirkpatrick-EtAl:2010:NAACLHLT}
Taylor Berg-Kirkpatrick, Alexandre Bouchard-C\^{o}t\'{e}, John DeNero, and Dan
  Klein.
\newblock 2010.
\newblock Painless unsupervised learning with features.
\newblock In {\em Human Language Technologies: The 2010 Annual Conference of
  the North American Chapter of the Association for Computational Linguistics},
  pages 582--590, Los Angeles, California, June. Association for Computational
  Linguistics.

\bibitem[\protect\citename{Beuck and Menzel}2013]{beuck2013}
Niels Beuck and Wolfgang Menzel.
\newblock 2013.
\newblock Structural prediction in incremental dependency parsing.
\newblock In Alexander Gelbukh, editor, {\em Computational Linguistics and
  Intelligent Text Processing}, volume 7816 of {\em Lecture Notes in Computer
  Science}, pages 245--257. Springer Berlin Heidelberg.

\bibitem[\protect\citename{Bishop}2006]{Bishop:2006:PRM:1162264}
Christopher~M. Bishop.
\newblock 2006.
\newblock {\em Pattern Recognition and Machine Learning (Information Science
  and Statistics)}.
\newblock Springer-Verlag New York, Inc., Secaucus, NJ, USA.

\bibitem[\protect\citename{Bisk and Hockenmaier}2012]{DBLP:conf/aaai/BiskH12}
Yonatan Bisk and Julia Hockenmaier.
\newblock 2012.
\newblock Simple robust grammar induction with combinatory categorial grammars.
\newblock In {\em Proceedings of the Twenty-Sixth {AAAI} Conference on
  Artificial Intelligence, July 22-26, 2012, Toronto, Ontario, Canada.}

\bibitem[\protect\citename{Bisk and
  Hockenmaier}2013]{DBLP:journals/tacl/BiskH13}
Yonatan Bisk and Julia Hockenmaier.
\newblock 2013.
\newblock An hdp model for inducing combinatory categorial grammars.
\newblock {\em Transactions of the Association for Computational Linguistics},
  1:75--88.

\bibitem[\protect\citename{Bisk and
  Hockenmaier}2015]{bisk-hockenmaier:2015:ACL-IJCNLP}
Yonatan Bisk and Julia Hockenmaier.
\newblock 2015.
\newblock Probing the linguistic strengths and limitations of unsupervised
  grammar induction.
\newblock In {\em Proceedings of the 53rd Annual Meeting of the Association for
  Computational Linguistics and the 7th International Joint Conference on
  Natural Language Processing (Volume 1: Long Papers)}, pages 1395--1404,
  Beijing, China, July. Association for Computational Linguistics.

\bibitem[\protect\citename{Bisk \bgroup et al.\egroup
  }2015]{Bisk:2015:ACLShort}
Yonatan Bisk, Christos Christodoulopoulos, and Julia Hockenmaier.
\newblock 2015.
\newblock Labeled grammar induction with minimal supervision.
\newblock In {\em Proceedings of the 53rd Annual Meeting of the Association for
  Computational Linguistics and the 7th International Joint Conference on
  Natural Language Processing (Volume 2: Short Papers)}, pages 870--876,
  Beijing,China, July.

\bibitem[\protect\citename{Blunsom and Cohn}2010]{blunsom-cohn:2010:EMNLP}
Phil Blunsom and Trevor Cohn.
\newblock 2010.
\newblock Unsupervised induction of tree substitution grammars for dependency
  parsing.
\newblock In {\em Proceedings of the 2010 Conference on Empirical Methods in
  Natural Language Processing}, pages 1204--1213, Cambridge, MA, October.
  Association for Computational Linguistics.

\bibitem[\protect\citename{Bohnet \bgroup et al.\egroup }2013]{BohnetJMA13}
Bernd Bohnet, Joakim Nivre, Igor~M. Boguslavsky, Rich{\'a}rd Farkas, Filip
  Ginter, and Jan Haji{\v{c}}.
\newblock 2013.
\newblock Joint morphological and syntactic analysis for richly inflected
  languages.
\newblock {\em Transactions of the Association for Computational Linguistics},
  1(Oct):429--440.

\bibitem[\protect\citename{Brants \bgroup et al.\egroup }2002]{de}
Sabine Brants, Stefanie Dipper, Silvia Hansen, Wolfgang Lezius, and George
  Smith.
\newblock 2002.
\newblock The {TIGER} treebank.
\newblock In {\em Proceedings of the Workshop on Treebanks and Linguistic
  Theories}, Sozopol.

\bibitem[\protect\citename{Buchholz and
  Marsi}2006]{buchholz-marsi:2006:CoNLL-X}
Sabine Buchholz and Erwin Marsi.
\newblock 2006.
\newblock Conll-x shared task on multilingual dependency parsing.
\newblock In {\em Proceedings of the Tenth Conference on Computational Natural
  Language Learning (CoNLL-X)}, pages 149--164, New York City, June.
  Association for Computational Linguistics.

\bibitem[\protect\citename{Carroll \bgroup et al.\egroup
  }1992]{Carroll92twoexperiments}
Glenn Carroll, Glenn Carroll, Eugene Charniak, and Eugene Charniak.
\newblock 1992.
\newblock Two experiments on learning probabilistic dependency grammars from
  corpora.
\newblock In {\em Working Notes of the Workshop Statistically-Based NLP
  Techniques}, pages 1--13. AAAI.

\bibitem[\protect\citename{Charniak}1993]{books/daglib/0080794}
Eugene Charniak.
\newblock 1993.
\newblock {\em Statistical language learning.}
\newblock MIT Press.

\bibitem[\protect\citename{Chen \bgroup et al.\egroup }2003]{sinica}
Keh-Jiann Chen, Chi-Ching Luo, Ming-Chung Chang, Feng-Yi Chen, Chao-Jan Chen,
  Chu-Ren Huang, and Zhao-Ming Gao.
\newblock 2003.
\newblock Sinica treebank.
\newblock In Anne Abeillé, editor, {\em Treebanks}, volume~20 of {\em Text,
  Speech and Language Technology}, pages 231--248. Springer Netherlands.

\bibitem[\protect\citename{Chen \bgroup et al.\egroup }2005]{Chen2005144}
Evan Chen, Edward Gibson, and Florian Wolf.
\newblock 2005.
\newblock Online syntactic storage costs in sentence comprehension.
\newblock {\em Journal of Memory and Language}, 52(1):144 -- 169.

\bibitem[\protect\citename{Civit and Martí}2004]{spanish}
Montserrat Civit and MaAntònia Martí.
\newblock 2004.
\newblock Building cast3lb: A spanish treebank.
\newblock {\em Research on Language and Computation}, 2(4):549--574.

\bibitem[\protect\citename{Clark}2001]{W01-0713}
Alexander Clark.
\newblock 2001.
\newblock Unsupervised {I}nduction of {S}tochastic {C}ontext-{F}ree {G}rammars
  using {D}istributional {C}lustering.
\newblock In {\em Proceedings of the ACL 2001 Workshop on Computational Natural
  Language Learning (CoNLL)}.

\bibitem[\protect\citename{Cohen and Smith}2009]{cohen-smith:2009:NAACLHLT09}
Shay Cohen and Noah~A. Smith.
\newblock 2009.
\newblock Shared logistic normal distributions for soft parameter tying in
  unsupervised grammar induction.
\newblock In {\em Proceedings of Human Language Technologies: The 2009 Annual
  Conference of the North American Chapter of the Association for Computational
  Linguistics}, pages 74--82, Boulder, Colorado, June. Association for
  Computational Linguistics.

\bibitem[\protect\citename{Cohen}2011]{cohen-2011}
Shay Cohen.
\newblock 2011.
\newblock {\em Computational Learning of Probabilistic Grammars in the
  Unsupervised Setting}.
\newblock {Ph.D.} thesis, Carnegie Mellon University, Pittsburgh, PA.

\bibitem[\protect\citename{Cohn \bgroup et al.\egroup
  }2010]{Cohn:2010:ITG:1756006.1953031}
Trevor Cohn, Phil Blunsom, and Sharon Goldwater.
\newblock 2010.
\newblock Inducing tree-substitution grammars.
\newblock {\em Journal of Machine Learning Research}, 11:3053--3096, December.

\bibitem[\protect\citename{Collins}1997]{P97-1003}
Michael Collins.
\newblock 1997.
\newblock {T}hree {G}enerative, {L}exicalised {M}odels for {S}tatistical
  {P}arsing.
\newblock In {\em 35th Annual Meeting of the Association for Computational
  Linguistics}.

\bibitem[\protect\citename{Collins}1999]{collins-thesis1999}
M.~Collins.
\newblock 1999.
\newblock {\em {Head-Driven Statistical Models for Natural Language Parsing}}.
\newblock {Ph.D.} thesis, University of Pennsylvania.

\bibitem[\protect\citename{Csendes \bgroup et al.\egroup }2005]{hu}
D{\'o}ra Csendes, J{\'a}nos Csirik, Tibor Gyim{\'o}thy, and Andr{\'a}s Kocsor.
\newblock 2005.
\newblock The {S}zeged treebank.
\newblock In {\em TSD}, pages 123--131.

\bibitem[\protect\citename{de Marcken}1999]{carl1999}
C.~de~Marcken.
\newblock 1999.
\newblock On the unsupervised induction of phrase-structure grammars.
\newblock In Susan Armstrong, Kenneth Church, Pierre Isabelle, Sandra Manzi,
  Evelyne Tzoukermann, and David Yarowsky, editors, {\em Natural Language
  Processing Using Very Large Corpora}, volume~11 of {\em Text, Speech and
  Language Technology}, pages 191--208. Springer Netherlands.

\bibitem[\protect\citename{de Marneffe and Manning}2008]{mm2008stdm}
Marie-Catherine de~Marneffe and Christopher~D. Manning.
\newblock 2008.
\newblock Stanford typed dependencies manual, September.

\bibitem[\protect\citename{Dempster \bgroup et al.\egroup }1977]{DEMP1977}
A.~P. Dempster, N.~M. Laird, and D.~B. Rubin.
\newblock 1977.
\newblock Maximum likelihood from incomplete data via the {EM} algorithm.
\newblock {\em Journal of the Royal Statistical Society: Series B}, 39:1--38.

\bibitem[\protect\citename{Doyle and Levy}2013]{doyle-levy:2013:NAACL-HLT}
Gabriel Doyle and Roger Levy.
\newblock 2013.
\newblock Combining multiple information types in bayesian word segmentation.
\newblock In {\em Proceedings of the 2013 Conference of the North American
  Chapter of the Association for Computational Linguistics: Human Language
  Technologies}, pages 117--126, Atlanta, Georgia, June. Association for
  Computational Linguistics.

\bibitem[\protect\citename{Dryer}1992]{Dryer-1992}
Matthew~S. Dryer.
\newblock 1992.
\newblock The greenbergian word order correlations.
\newblock {\em Language}, 68(1):81--138.

\bibitem[\protect\citename{D\v{z}eroski \bgroup et al.\egroup }2006]{sl}
Sa\v{s}o D\v{z}eroski, Toma\v{z} Erjavec, Nina Ledinek, Petr Pajas, Zden\v{e}k
  \v{Z}abokrtsk\'{y}, and Andreja \v{Z}ele.
\newblock 2006.
\newblock Towards a {S}lovene dependency treebank.
\newblock In {\em Proceedings of the Fifth International Language Resources and
  Evaluation Conference, {LREC} 2006}, pages 1388--1391, Genova, Italy.
  European Language Resources Association ({ELRA}).

\bibitem[\protect\citename{Eisner and
  Satta}1999]{Eisner:1999:EPB:1034678.1034748}
Jason Eisner and Giorgio Satta.
\newblock 1999.
\newblock Efficient parsing for bilexical context-free grammars and head
  automaton grammars.
\newblock In {\em Proceedings of the 37th Annual Meeting of the Association for
  Computational Linguistics on Computational Linguistics}, ACL '99, pages
  457--464, Stroudsburg, PA, USA. Association for Computational Linguistics.

\bibitem[\protect\citename{Eisner and Smith}2010]{eisner2010}
Jason Eisner and NoahA. Smith.
\newblock 2010.
\newblock Favor short dependencies: Parsing with soft and hard constraints on
  dependency length.
\newblock In Harry Bunt, Paola Merlo, and Joakim Nivre, editors, {\em Trends in
  Parsing Technology}, volume~43 of {\em Text, Speech and Language Technology},
  pages 121--150. Springer Netherlands.

\bibitem[\protect\citename{Eisner}2000]{eisner-2000-iwptbook}
Jason Eisner.
\newblock 2000.
\newblock {Bilexical Grammars and Their Cubic-Time Parsing Algorithms}.
\newblock In Harry Bunt and Anton Nijholt, editors, {\em Advances in
  Probabilistic and Other Parsing Technologies}, pages 29--62. Kluwer Academic
  Publishers, October.

\bibitem[\protect\citename{Evans and Levinson}2009]{Evans2009}
Nicholas Evans and Stephen~C Levinson.
\newblock 2009.
\newblock The myth of language universals: language diversity and its
  importance for cognitive science.
\newblock {\em The Behavioral and brain sciences}, 32(5):429--48; discussion
  448--494, October.

\bibitem[\protect\citename{Fedzechkina \bgroup et al.\egroup
  }2012]{Fedzechkina30102012}
Maryia Fedzechkina, T.~Florian Jaeger, and Elissa~L. Newport.
\newblock 2012.
\newblock Language learners restructure their input to facilitate efficient
  communication.
\newblock {\em Proceedings of the National Academy of Sciences},
  109(44):17897--17902.

\bibitem[\protect\citename{Finkel \bgroup et al.\egroup
  }2008]{finkel-kleeman-manning:2008:ACLMain}
Jenny~Rose Finkel, Alex Kleeman, and Christopher~D. Manning.
\newblock 2008.
\newblock Efficient, feature-based, conditional random field parsing.
\newblock In {\em Proceedings of ACL-08: HLT}, pages 959--967, Columbus, Ohio,
  June. Association for Computational Linguistics.

\bibitem[\protect\citename{Futrell \bgroup et al.\egroup
  }2015]{Futrell18082015}
Richard Futrell, Kyle Mahowald, and Edward Gibson.
\newblock 2015.
\newblock Large-scale evidence of dependency length minimization in 37
  languages.
\newblock {\em Proceedings of the National Academy of Sciences},
  112(33):10336--10341.

\bibitem[\protect\citename{Ganchev \bgroup et al.\egroup
  }2010]{journals/jmlr/GanchevGGT10}
Kuzman Ganchev, João Graça, Jennifer Gillenwater, and Ben Taskar.
\newblock 2010.
\newblock Posterior regularization for structured latent variable models.
\newblock {\em Journal of Machine Learning Research}, 11:2001--2049.

\bibitem[\protect\citename{Garrette \bgroup et al.\egroup }2015]{AAAI159835}
Dan Garrette, Chris Dyer, Jason Baldridge, and Noah Smith.
\newblock 2015.
\newblock Weakly-supervised grammar-informed bayesian ccg parser learning.

\bibitem[\protect\citename{Gelling \bgroup et al.\egroup
  }2012]{gelling-EtAl:2012:WILS}
Douwe Gelling, Trevor Cohn, Phil Blunsom, and Joao Graca.
\newblock 2012.
\newblock The pascal challenge on grammar induction.
\newblock In {\em Proceedings of the NAACL-HLT Workshop on the Induction of
  Linguistic Structure}, pages 64--80, Montr{\'e}al, Canada, June. Association
  for Computational Linguistics.

\bibitem[\protect\citename{Gibson}1998]{gibson1998dlt}
Edward Gibson.
\newblock 1998.
\newblock Linguistic complexity: Locality of syntactic dependencies.
\newblock {\em Cognition}, 68(1):1--76.

\bibitem[\protect\citename{Gibson}2000]{Gibson2000The-dependency-}
E.~Gibson.
\newblock 2000.
\newblock {The dependency locality theory: A distance-based theory of
  linguistic complexity}.
\newblock In {\em Image, language, brain: Papers from the first mind
  articulation project symposium}, pages 95--126.

\bibitem[\protect\citename{Gildea and
  Temperley}2007]{gildea-temperley:2007:ACLMain}
Daniel Gildea and David Temperley.
\newblock 2007.
\newblock Optimizing grammars for minimum dependency length.
\newblock In {\em Proceedings of the 45th Annual Meeting of the Association of
  Computational Linguistics}, pages 184--191, Prague, Czech Republic, June.
  Association for Computational Linguistics.

\bibitem[\protect\citename{Gildea and
  Temperley}2010]{DBLP:journals/cogsci/GildeaT10}
Daniel Gildea and David Temperley.
\newblock 2010.
\newblock Do grammars minimize dependency length?
\newblock {\em Cognitive Science}, 34(2):286--310.

\bibitem[\protect\citename{Gimpel and Smith}2012]{gimpel-smith:2012:NAACL-HLT2}
Kevin Gimpel and Noah~A. Smith.
\newblock 2012.
\newblock Concavity and initialization for unsupervised dependency parsing.
\newblock In {\em Proceedings of the 2012 Conference of the North American
  Chapter of the Association for Computational Linguistics: Human Language
  Technologies}, pages 577--581, Montr\'{e}al, Canada, June. Association for
  Computational Linguistics.

\bibitem[\protect\citename{Goldberg and Nivre}2013]{GoldbergTDP13}
Yoav Goldberg and Joakim Nivre.
\newblock 2013.
\newblock Training deterministic parsers with non-deterministic oracles.
\newblock {\em Transactions of the Association for Computational Linguistics},
  1(Oct):403--414.

\bibitem[\protect\citename{Goldberg}2011]{he}
Yoav Goldberg.
\newblock 2011.
\newblock {\em Automatic Syntactic Processing of Modern {Hebrew} ({PhD}
  thesis)}.

\bibitem[\protect\citename{Goldwater \bgroup et al.\egroup
  }2009]{Goldwater200921}
Sharon Goldwater, Thomas~L. Griffiths, and Mark Johnson.
\newblock 2009.
\newblock A bayesian framework for word segmentation: Exploring the effects of
  context.
\newblock {\em Cognition}, 112(1):21 -- 54.

\bibitem[\protect\citename{G\'{o}mez-Rodr\'{i}guez and
  Nivre}2013]{RodriguezNivre2013divisible}
Carlos G\'{o}mez-Rodr\'{i}guez and Joakim Nivre.
\newblock 2013.
\newblock {Divisible Transition Systems and Multiplanar Dependency Parsing}.
\newblock {\em Comput. Linguist.}, 39(4):799--845, December.

\bibitem[\protect\citename{G\'{o}mez-Rodr\'{i}guez \bgroup et al.\egroup
  }2011]{journals/coling/Gomez-RodriguezCW11}
Carlos G\'{o}mez-Rodr\'{i}guez, John~A. Carroll, and David~J. Weir.
\newblock 2011.
\newblock Dependency parsing schemata and mildly non-projective dependency
  parsing.
\newblock {\em Computational Linguistics}, 37(3):541--586.

\bibitem[\protect\citename{Gormley and
  Eisner}2013]{gormley-eisner:2013:ACL2013}
Matthew~R. Gormley and Jason Eisner.
\newblock 2013.
\newblock Nonconvex global optimization for latent-variable models.
\newblock In {\em Proceedings of the 51st Annual Meeting of the Association for
  Computational Linguistics (Volume 1: Long Papers)}, pages 444--454, Sofia,
  Bulgaria, August. Association for Computational Linguistics.

\bibitem[\protect\citename{Grave and
  Elhadad}2015]{grave-elhadad:2015:ACL-IJCNLP}
Edouard Grave and No\'{e}mie Elhadad.
\newblock 2015.
\newblock A convex and feature-rich discriminative approach to dependency
  grammar induction.
\newblock In {\em Proceedings of the 53rd Annual Meeting of the Association for
  Computational Linguistics and the 7th International Joint Conference on
  Natural Language Processing (Volume 1: Long Papers)}, pages 1375--1384,
  Beijing, China, July. Association for Computational Linguistics.

\bibitem[\protect\citename{Greenberg}1963]{Greenberg-1963}
Joseph~H. Greenberg.
\newblock 1963.
\newblock Some universals of grammar with particular reference to the order of
  meaningful elements.
\newblock In Joseph~H. Greenberg, editor, {\em Universals of Human Language},
  pages 73--113. MIT Press, Cambridge, Mass.

\bibitem[\protect\citename{Gulordava \bgroup et al.\egroup
  }2015]{gulordava-merlo-crabbe:2015:ACL-IJCNLP}
Kristina Gulordava, Paola Merlo, and Benoit Crabb\'{e}.
\newblock 2015.
\newblock Dependency length minimisation effects in short spans: a large-scale
  analysis of adjective placement in complex noun phrases.
\newblock In {\em Proceedings of the 53rd Annual Meeting of the Association for
  Computational Linguistics and the 7th International Joint Conference on
  Natural Language Processing (Volume 2: Short Papers)}, pages 477--482,
  Beijing, China, July. Association for Computational Linguistics.

\bibitem[\protect\citename{Haghighi and Klein}2006]{haghighi-klein:2006:COLACL}
Aria Haghighi and Dan Klein.
\newblock 2006.
\newblock Prototype-driven grammar induction.
\newblock In {\em Proceedings of the 21st International Conference on
  Computational Linguistics and 44th Annual Meeting of the Association for
  Computational Linguistics}, pages 881--888, Sydney, Australia, July.
  Association for Computational Linguistics.

\bibitem[\protect\citename{Haji{\v{c}} \bgroup et al.\egroup }2006]{cs}
Jan Haji{\v{c}}, Jarmila Panevov{\'{a}}, Eva Haji{\v{c}}ov{\'{a}}, Petr Sgall,
  Petr Pajas, Jan {\v{S}}t{\v{e}}p{\'{a}}nek, Ji{\v{r}}{\'{i}} Havelka, Marie
  Mikulov{\'{a}}, Zden{\v{e}}k {\v{Z}}abokrtsk{\'{y}}, and Magda
  {\v{S}}ev{\v{c}}{\'{i}}kov{\'{a}}-Raz{\'{i}}mov{\'{a}}.
\newblock 2006.
\newblock Prague dependency treebank 2.0.
\newblock CD-ROM, Linguistic Data Consortium, LDC Catalog No.: LDC2006T01,
  Philadelphia.

\bibitem[\protect\citename{Hale}2001]{Hale2001}
John Hale.
\newblock 2001.
\newblock A probabilistic earley parser as a psycholinguistic model.
\newblock In {\em NAACL '01: Second meeting of the North American Chapter of
  the Association for Computational Linguistics on Language technologies 2001},
  pages 1--8, Morristown, NJ, USA. Association for Computational Linguistics.

\bibitem[\protect\citename{Haverinen \bgroup et al.\egroup }2010]{fi}
Katri Haverinen, Timo Viljanen, Veronika Laippala, Samuel Kohonen, Filip
  Ginter, and Tapio Salakoski.
\newblock 2010.
\newblock Treebanking {F}innish.
\newblock In {\em Proceedings of the Ninth International Workshop on Treebanks
  and Linguistic Theories (TLT9)}, pages 79--90.

\bibitem[\protect\citename{Hayashi \bgroup et al.\egroup }2013]{TACL68}
Katsuhiko Hayashi, Shuhei Kondo, and Yuji Matsumoto.
\newblock 2013.
\newblock Efficient stacked dependency parsing by forest reranking.
\newblock {\em Transactions of the Association for Computational Linguistics},
  1:139--150.

\bibitem[\protect\citename{Headden~III \bgroup et al.\egroup
  }2009]{headdeniii-johnson-mcclosky:2009:NAACLHLT09}
William~P. Headden~III, Mark Johnson, and David McClosky.
\newblock 2009.
\newblock Improving unsupervised dependency parsing with richer contexts and
  smoothing.
\newblock In {\em Proceedings of Human Language Technologies: The 2009 Annual
  Conference of the North American Chapter of the Association for Computational
  Linguistics}, pages 101--109, Boulder, Colorado, June. Association for
  Computational Linguistics.

\bibitem[\protect\citename{Henderson}2004]{henderson:2004:ACL}
James Henderson.
\newblock 2004.
\newblock Discriminative training of a neural network statistical parser.
\newblock In {\em Proceedings of the 42nd Meeting of the Association for
  Computational Linguistics (ACL'04), Main Volume}, pages 95--102, Barcelona,
  Spain, July.

\bibitem[\protect\citename{Huang and Sagae}2010]{huang-sagae:2010:ACL}
Liang Huang and Kenji Sagae.
\newblock 2010.
\newblock Dynamic programming for linear-time incremental parsing.
\newblock In {\em Proceedings of the 48th Annual Meeting of the Association for
  Computational Linguistics}, pages 1077--1086, Uppsala, Sweden, July.
  Association for Computational Linguistics.

\bibitem[\protect\citename{Huang \bgroup et al.\egroup
  }2012]{huang-fayong-guo:2012:NAACL-HLT}
Liang Huang, Suphan Fayong, and Yang Guo.
\newblock 2012.
\newblock Structured perceptron with inexact search.
\newblock In {\em Proceedings of the 2012 Conference of the North American
  Chapter of the Association for Computational Linguistics: Human Language
  Technologies}, pages 142--151, Montr\'{e}al, Canada, June. Association for
  Computational Linguistics.

\bibitem[\protect\citename{Hudson}2004]{Hudson:2004-01-01T00:00:00:0929-998X:7}
Richard Hudson.
\newblock 2004.
\newblock Are determiners heads?
\newblock {\em Functions of Language}, 11(1):7--42.

\bibitem[\protect\citename{Jaeger and Tily}2011]{jaeger2011language-gsc}
T.Florian Jaeger and Harry Tily.
\newblock 2011.
\newblock On language utility: Processing complexity and communicative
  efficiency.
\newblock {\em Wiley Interdisciplinary Reviews: Cognitive Science},
  2(3):323--335.

\bibitem[\protect\citename{Johnson \bgroup et al.\egroup
  }2007]{johnson-griffiths-goldwater:2007:main}
Mark Johnson, Thomas Griffiths, and Sharon Goldwater.
\newblock 2007.
\newblock {Bayesian} inference for {PCFG}s via {Markov} chain {Monte} {Carlo}.
\newblock In {\em Human Language Technologies 2007: The Conference of the North
  American Chapter of the Association for Computational Linguistics;
  Proceedings of the Main Conference}, pages 139--146, Rochester, New York,
  April. Association for Computational Linguistics.

\bibitem[\protect\citename{Johnson \bgroup et al.\egroup
  }2012]{johnson-demuth-frank:2012:ACL2012}
Mark Johnson, Katherine Demuth, and Michael Frank.
\newblock 2012.
\newblock Exploiting social information in grounded language learning via
  grammatical reduction.
\newblock In {\em Proceedings of the 50th Annual Meeting of the Association for
  Computational Linguistics (Volume 1: Long Papers)}, pages 883--891, Jeju
  Island, Korea, July. Association for Computational Linguistics.

\bibitem[\protect\citename{Johnson-Laird}1983]{Cognitive:MentalModels}
P.~N. Johnson-Laird.
\newblock 1983.
\newblock {\em Mental models: towards a cognitive science of language,
  inference, and consciousness}.
\newblock Harvard University Press, Cambridge, MA, USA.

\bibitem[\protect\citename{Johnson}1998a]{conf/acl/Johnson98}
Mark Johnson.
\newblock 1998a.
\newblock Finite-state approximation of constraint-based grammars using
  left-corner grammar transforms.
\newblock In Christian Boitet and Pete Whitelock, editors, {\em COLING-ACL},
  pages 619--623. Morgan Kaufmann Publishers / ACL.

\bibitem[\protect\citename{Johnson}1998b]{J98-4004}
Mark Johnson.
\newblock 1998b.
\newblock Pcfg models of linguistic tree representations.
\newblock {\em Computational Linguistics}, 24(4):613--632.

\bibitem[\protect\citename{Johnson}2007]{johnson:2007:ACLMain}
Mark Johnson.
\newblock 2007.
\newblock Transforming projective bilexical dependency grammars into
  efficiently-parsable cfgs with unfold-fold.
\newblock In {\em Proceedings of the 45th Annual Meeting of the Association of
  Computational Linguistics}, pages 168--175, Prague, Czech Republic, June.
  Association for Computational Linguistics.

\bibitem[\protect\citename{Kanayama \bgroup et al.\egroup }2015]{ja-ud}
Hiroshi Kanayama, Yusuke Miyao, Takaaki Tanaka, Shinsuke Mori, Masayuki
  Asahara, and Sumire Uematsu.
\newblock 2015.
\newblock A draft proposal for universal dependencies japanese treebank (in
  japanese).
\newblock In {\em In Proceedings of the 21st annual meeting for Gengo Shori
  Gakkai (The Association for Natural Language Processing)}.

\bibitem[\protect\citename{Kawahara \bgroup et al.\egroup
  }2002]{Kawahara02constructionof}
Daisuke Kawahara, Hongo Sadao, and Koiti Hasida.
\newblock 2002.
\newblock Construction of a japanese relevance-tagged corpus.
\newblock In {\em In Proceedings of the 3rd International Conference on
  Language Resources and Evaluation}.

\bibitem[\protect\citename{Kawata and Bartels}2000]{ja}
Yasuhiro Kawata and Julia Bartels.
\newblock 2000.
\newblock Stylebook for the {Japanese} treebank in {Verbmobil}.
\newblock In {\em Report 240}, Tübingen, Germany, September 29.

\bibitem[\protect\citename{Keller}2010]{keller:2010:Short}
Frank Keller.
\newblock 2010.
\newblock Cognitively plausible models of human language processing.
\newblock In {\em Proceedings of the ACL 2010 Conference Short Papers}, pages
  60--67, Uppsala, Sweden, July. Association for Computational Linguistics.

\bibitem[\protect\citename{Kitagawa and
  Tanaka-Ishii}2010]{kitagawa-tanakaishii:2010:Short}
Kotaro Kitagawa and Kumiko Tanaka-Ishii.
\newblock 2010.
\newblock Tree-based deterministic dependency parsing --- an application to
  nivre's method ---.
\newblock In {\em Proceedings of the ACL 2010 Conference Short Papers}, pages
  189--193, Uppsala, Sweden, July. Association for Computational Linguistics.

\bibitem[\protect\citename{Klein and Manning}2002]{klein-manning:2002:ACL}
Dan Klein and Christopher~D. Manning.
\newblock 2002.
\newblock A generative constituent-context model for improved grammar
  induction.
\newblock In {\em Proceedings of 40th Annual Meeting of the Association for
  Computational Linguistics}, pages 128--135, Philadelphia, Pennsylvania, USA,
  July. Association for Computational Linguistics.

\bibitem[\protect\citename{Klein and Manning}2003]{klein-manning:2003:ACL}
Dan Klein and Christopher~D. Manning.
\newblock 2003.
\newblock Accurate unlexicalized parsing.
\newblock In {\em Proceedings of the 41st Annual Meeting of the Association for
  Computational Linguistics}, pages 423--430, Sapporo, Japan, July. Association
  for Computational Linguistics.

\bibitem[\protect\citename{Klein and Manning}2004]{klein-manning:2004:ACL}
Dan Klein and Christopher Manning.
\newblock 2004.
\newblock Corpus-based induction of syntactic structure: Models of dependency
  and constituency.
\newblock In {\em Proceedings of the 42nd Meeting of the Association for
  Computational Linguistics (ACL'04), Main Volume}, pages 478--485, Barcelona,
  Spain, July.

\bibitem[\protect\citename{Kohn and Menzel}2014]{kohn-menzel:2014:ACL}
Arne Kohn and Wolfgang Menzel.
\newblock 2014.
\newblock Incremental predictive parsing with turboparser.
\newblock In {\em Proceedings of the ACL 2014 Conference Short Papers},
  Baltimore, USA, June. Association for Computational Linguistics.

\bibitem[\protect\citename{Konieczny}2000]{konieczny2000}
Lars Konieczny.
\newblock 2000.
\newblock Locality and parsing complexity.
\newblock {\em Journal of Psycholinguistic Research}, 29(6):627--645.

\bibitem[\protect\citename{Koo \bgroup et al.\egroup
  }2008]{koo-carreras-collins:2008:ACLMain}
Terry Koo, Xavier Carreras, and Michael Collins.
\newblock 2008.
\newblock Simple semi-supervised dependency parsing.
\newblock In {\em Proceedings of ACL-08: HLT}, pages 595--603, Columbus, Ohio,
  June. Association for Computational Linguistics.

\bibitem[\protect\citename{Kromann \bgroup et al.\egroup }2004]{da}
Matthias~T. Kromann, Line Mikkelsen, and Stine~Kern Lynge.
\newblock 2004.

\bibitem[\protect\citename{Kuhlmann \bgroup et al.\egroup
  }2011]{kuhlmann-gomezrodriguez-satta:2011:ACL-HLT2011}
Marco Kuhlmann, Carlos G\'{o}mez-Rodr\'{i}guez, and Giorgio Satta.
\newblock 2011.
\newblock Dynamic programming algorithms for transition-based dependency
  parsers.
\newblock In {\em Proceedings of the 49th Annual Meeting of the Association for
  Computational Linguistics: Human Language Technologies}, pages 673--682,
  Portland, Oregon, USA, June. Association for Computational Linguistics.

\bibitem[\protect\citename{Kuhlmann}2013]{DBLP:journals/coling/Kuhlmann13}
Marco Kuhlmann.
\newblock 2013.
\newblock Mildly non-projective dependency grammar.
\newblock {\em Computational Linguistics}, 39(2):355--387.

\bibitem[\protect\citename{Kushman \bgroup et al.\egroup
  }2014]{kushman-EtAl:2014:P14-1}
Nate Kushman, Yoav Artzi, Luke Zettlemoyer, and Regina Barzilay.
\newblock 2014.
\newblock Learning to automatically solve algebra word problems.
\newblock In {\em Proceedings of the 52nd Annual Meeting of the Association for
  Computational Linguistics (Volume 1: Long Papers)}, pages 271--281,
  Baltimore, Maryland, June. Association for Computational Linguistics.

\bibitem[\protect\citename{Kwiatkowksi \bgroup et al.\egroup
  }2010]{kwiatkowksi-EtAl:2010:EMNLP}
Tom Kwiatkowksi, Luke Zettlemoyer, Sharon Goldwater, and Mark Steedman.
\newblock 2010.
\newblock Inducing probabilistic {CCG} grammars from logical form with
  higher-order unification.
\newblock In {\em Proceedings of the 2010 Conference on Empirical Methods in
  Natural Language Processing}, pages 1223--1233, Cambridge, MA, October.
  Association for Computational Linguistics.

\bibitem[\protect\citename{Kwiatkowski \bgroup et al.\egroup
  }2012]{kwiatkowski-EtAl:2012:EACL2012}
Tom Kwiatkowski, Sharon Goldwater, Luke Zettlemoyer, and Mark Steedman.
\newblock 2012.
\newblock A probabilistic model of syntactic and semantic acquisition from
  child-directed utterances and their meanings.
\newblock In {\em Proceedings of the 13th Conference of the European Chapter of
  the Association for Computational Linguistics}, pages 234--244, Avignon,
  France, April. Association for Computational Linguistics.

\bibitem[\protect\citename{Kwiatkowski \bgroup et al.\egroup
  }2013]{kwiatkowski-EtAl:2013:EMNLP}
Tom Kwiatkowski, Eunsol Choi, Yoav Artzi, and Luke Zettlemoyer.
\newblock 2013.
\newblock Scaling semantic parsers with on-the-fly ontology matching.
\newblock In {\em Proceedings of the 2013 Conference on Empirical Methods in
  Natural Language Processing}, pages 1545--1556, Seattle, Washington, USA,
  October. Association for Computational Linguistics.

\bibitem[\protect\citename{Lafferty \bgroup et al.\egroup
  }2001]{Lafferty:2001:CRF:645530.655813}
John~D. Lafferty, Andrew McCallum, and Fernando C.~N. Pereira.
\newblock 2001.
\newblock Conditional random fields: Probabilistic models for segmenting and
  labeling sequence data.
\newblock In {\em Proceedings of the Eighteenth International Conference on
  Machine Learning}, ICML '01, pages 282--289, San Francisco, CA, USA. Morgan
  Kaufmann Publishers Inc.

\bibitem[\protect\citename{Levy}2008]{Levy20081126}
Roger Levy.
\newblock 2008.
\newblock Expectation-based syntactic comprehension.
\newblock {\em Cognition}, 106(3):1126 -- 1177.

\bibitem[\protect\citename{Lewis and
  Steedman}2014]{lewis-steedman:2014:EMNLP2014}
Mike Lewis and Mark Steedman.
\newblock 2014.
\newblock A* ccg parsing with a supertag-factored model.
\newblock In {\em Proceedings of the 2014 Conference on Empirical Methods in
  Natural Language Processing (EMNLP)}, pages 990--1000, Doha, Qatar, October.
  Association for Computational Linguistics.

\bibitem[\protect\citename{Liang \bgroup et al.\egroup
  }2011]{liang-jordan-klein:2011:ACL-HLT2011}
Percy Liang, Michael Jordan, and Dan Klein.
\newblock 2011.
\newblock Learning dependency-based compositional semantics.
\newblock In {\em Proceedings of the 49th Annual Meeting of the Association for
  Computational Linguistics: Human Language Technologies}, pages 590--599,
  Portland, Oregon, USA, June. Association for Computational Linguistics.

\bibitem[\protect\citename{Liu and Nocedal}1989]{Liu89onthe}
Dong~C. Liu and Jorge Nocedal.
\newblock 1989.
\newblock On the limited memory bfgs method for large scale optimization.
\newblock {\em MATHEMATICAL PROGRAMMING}, 45:503--528.

\bibitem[\protect\citename{Lynn \bgroup et al.\egroup }2014]{ga}
Teresa Lynn, Jennifer Foster, Mark Dras, and Lamia Tounsi.
\newblock 2014.
\newblock Cross-lingual transfer parsing for low-resourced languages: An
  {Irish} case study.
\newblock In {\em {CLTW} 2014}, Dublin, Ireland.

\bibitem[\protect\citename{M.~Ant\'{o}nia~Mart\'{i} and Bertran}2007]{catalan}
Llu\'{i}s~M\'{a}rquez M.~Ant\'{o}nia~Mart\'{i}, Mariona~Taul\'{e} and Manuel
  Bertran.
\newblock 2007.
\newblock {CESS-ECE}: {A} {M}ultilingual and {M}ultilevel {A}nnotated {C}orpus.
\newblock Available at http://www.lsi.upc.edu/~mbertran/cess-ece/publications.

\bibitem[\protect\citename{Manning and Carpenter}2000]{manning2000}
ChristopherD. Manning and Bob Carpenter.
\newblock 2000.
\newblock Probabilistic parsing using left corner language models.
\newblock In Harry Bunt and Anton Nijholt, editors, {\em Advances in
  Probabilistic and Other Parsing Technologies}, volume~16 of {\em Text, Speech
  and Language Technology}, pages 105--124. Springer Netherlands.

\bibitem[\protect\citename{Manning and Sch{\"u}tze}1999]{manning99foundations}
Christopher~D. Manning and Hinrich Sch{\"u}tze.
\newblock 1999.
\newblock {\em Foundations of Statistical Natural Language Processing}.
\newblock The {MIT} Press, Cambridge, Massachusetts.

\bibitem[\protect\citename{Marcus \bgroup et al.\egroup
  }1993]{Marcus93buildinga}
Mitchell~P. Marcus, Beatrice Santorini, and Mary~Ann Marcinkiewicz.
\newblock 1993.
\newblock Building a large annotated corpus of english: The penn treebank.
\newblock {\em COMPUTATIONAL LINGUISTICS}, 19(2):313--330.

\bibitem[\protect\citename{Mare\v{c}ek and
  Straka}2013]{marevcek-straka:2013:ACL2013}
David Mare\v{c}ek and Milan Straka.
\newblock 2013.
\newblock Stop-probability estimates computed on a large corpus improve
  unsupervised dependency parsing.
\newblock In {\em Proceedings of the 51st Annual Meeting of the Association for
  Computational Linguistics (Volume 1: Long Papers)}, pages 281--290, Sofia,
  Bulgaria, August. Association for Computational Linguistics.

\bibitem[\protect\citename{Mare\v{c}ek and
  \v{Z}abokrtsk\'{y}}2012]{marevcek-vzabokrtsky:2012:EMNLP-CoNLL}
David Mare\v{c}ek and Zden\v{e}k \v{Z}abokrtsk\'{y}.
\newblock 2012.
\newblock Exploiting reducibility in unsupervised dependency parsing.
\newblock In {\em Proceedings of the 2012 Joint Conference on Empirical Methods
  in Natural Language Processing and Computational Natural Language Learning},
  pages 297--307, Jeju Island, Korea, July. Association for Computational
  Linguistics.

\bibitem[\protect\citename{Marneffe \bgroup et al.\egroup
  }2014]{DEMARNEFFE14.1062}
Marie-Catherine~De Marneffe, Timothy Dozat, Natalia Silveira, Katri Haverinen,
  Filip Ginter, Joakim Nivre, and Christopher~D. Manning.
\newblock 2014.
\newblock Universal stanford dependencies: a cross-linguistic typology.
\newblock In {\em Proceedings of the Ninth International Conference on Language
  Resources and Evaluation (LREC'14)}, Reykjavik, Iceland, may.

\bibitem[\protect\citename{Matsuzaki \bgroup et al.\egroup
  }2007]{matuzaki:2007}
T.~Matsuzaki, Y.~Miyao, and J.~Tsujii.
\newblock 2007.
\newblock Efficient hpsg parsing with supertagging and cfg-filtering.
\newblock In {\em Proceedings of the Twentieth International Joint Conference
  on Artificial Intelligence (IJCAI-07)}, pages 1671--1676, January.

\bibitem[\protect\citename{McDonald and Pereira}2006]{McDonald2006}
Ryan McDonald and Fernando Pereira.
\newblock 2006.
\newblock Online learning of approximate dependency parsing algorithms.
\newblock In {\em Proceedings of 11th Conference of the European Chapter of the
  Association for Computational Linguistics (EACL-2006))}, volume~6, pages
  81--88.

\bibitem[\protect\citename{McDonald and
  Satta}2007]{mcdonald-satta:2007:IWPT2007}
Ryan McDonald and Giorgio Satta.
\newblock 2007.
\newblock On the complexity of non-projective data-driven dependency parsing.
\newblock In {\em Proceedings of the Tenth International Conference on Parsing
  Technologies}, pages 121--132, Prague, Czech Republic, June. Association for
  Computational Linguistics.

\bibitem[\protect\citename{McDonald \bgroup et al.\egroup
  }2005]{McDonald:2005:NDP:1220575.1220641}
Ryan McDonald, Fernando Pereira, Kiril Ribarov, and Jan Haji\v{c}.
\newblock 2005.
\newblock Non-projective dependency parsing using spanning tree algorithms.
\newblock In {\em Proceedings of the Conference on Human Language Technology
  and Empirical Methods in Natural Language Processing}, HLT '05, pages
  523--530, Stroudsburg, PA, USA. Association for Computational Linguistics.

\bibitem[\protect\citename{McDonald \bgroup et al.\egroup
  }2011]{mcdonald-petrov-hall:2011:EMNLP}
Ryan McDonald, Slav Petrov, and Keith Hall.
\newblock 2011.
\newblock Multi-source transfer of delexicalized dependency parsers.
\newblock In {\em Proceedings of the 2011 Conference on Empirical Methods in
  Natural Language Processing}, pages 62--72, Edinburgh, Scotland, UK., July.
  Association for Computational Linguistics.

\bibitem[\protect\citename{McDonald \bgroup et al.\egroup
  }2013]{mcdonald-EtAl:2013:Short}
Ryan McDonald, Joakim Nivre, Yvonne Quirmbach-Brundage, Yoav Goldberg, Dipanjan
  Das, Kuzman Ganchev, Keith Hall, Slav Petrov, Hao Zhang, Oscar
  T\"{a}ckstr\"{o}m, Claudia Bedini, N\'{u}ria Bertomeu~Castell\'{o}, and
  Jungmee Lee.
\newblock 2013.
\newblock Universal dependency annotation for multilingual parsing.
\newblock In {\em Proceedings of the 51st Annual Meeting of the Association for
  Computational Linguistics (Volume 2: Short Papers)}, pages 92--97, Sofia,
  Bulgaria, August. Association for Computational Linguistics.

\bibitem[\protect\citename{Mel'{\v{c}}uk}1988]{Melcuk:1988}
Igor Mel'{\v{c}}uk.
\newblock 1988.
\newblock {\em Dependency Syntax: Theory and Practice}.
\newblock State University of New York Press.

\bibitem[\protect\citename{Miller and Chomsky}1963]{Miller1963-MILFMO}
George~A. Miller and Noam Chomsky.
\newblock 1963.
\newblock Finitary models of language users.
\newblock In D.~Luce, editor, {\em Handbook of Mathematical Psychology}, pages
  2--419. John Wiley \& Sons.

\bibitem[\protect\citename{Montemagni \bgroup et al.\egroup }2003]{it}
Simonetta Montemagni, Francesco Barsotti, Marco Battista, Nicoletta Calzolari,
  Ornella Corazzari, Alessandro Lenci, Antonio Zampolli, Francesca Fanciulli,
  Maria Massetani, Remo Raffaelli, Roberto Basili, Maria~Teresa Pazienza, Dario
  Saracino, Fabio Zanzotto, Nadia Mana, Fabio Pianesi, and Rodolfo Delmonte.
\newblock 2003.
\newblock Building the {I}talian syntactic-semantic treebank.
\newblock In Anne Abeillé, editor, {\em Building and using Parsed Corpora},
  Language and Speech series, pages 189--210, Dordrecht. Kluwer.

\bibitem[\protect\citename{Mueller \bgroup et al.\egroup
  }2014]{mueller-EtAl:2014:EMNLP2014}
Thomas Mueller, Rich\'{a}rd Farkas, Alex Judea, Helmut Schmid, and hinrich
  schuetze.
\newblock 2014.
\newblock Dependency parsing with latent refinements of part-of-speech tags.
\newblock In {\em Proceedings of the 2014 Conference on Empirical Methods in
  Natural Language Processing (EMNLP)}, pages 963--967, Doha, Qatar, October.
  Association for Computational Linguistics.

\bibitem[\protect\citename{Nakatani and Gibson}2008]{nakatani2008}
K.~Nakatani and E.~Gibson.
\newblock 2008.
\newblock {Distinguishing theories of syntactic expectation cost in sentence
  comprehension: Evidence from Japanese}.
\newblock {\em Linguistics}, 46(1):63--87.

\bibitem[\protect\citename{Nakatani and Gibson}2010]{COGS:COGS1067}
Kentaro Nakatani and Edward Gibson.
\newblock 2010.
\newblock An on-line study of japanese nesting complexity.
\newblock {\em Cognitive Science}, 34(1):94--112.

\bibitem[\protect\citename{Naseem \bgroup et al.\egroup
  }2010]{naseem-EtAl:2010:EMNLP}
Tahira Naseem, Harr Chen, Regina Barzilay, and Mark Johnson.
\newblock 2010.
\newblock Using universal linguistic knowledge to guide grammar induction.
\newblock In {\em Proceedings of the 2010 Conference on Empirical Methods in
  Natural Language Processing}, pages 1234--1244, Cambridge, MA, October.
  Association for Computational Linguistics.

\bibitem[\protect\citename{Naseem \bgroup et al.\egroup
  }2012]{naseem-barzilay-globerson:2012:ACL2012}
Tahira Naseem, Regina Barzilay, and Amir Globerson.
\newblock 2012.
\newblock Selective sharing for multilingual dependency parsing.
\newblock In {\em Proceedings of the 50th Annual Meeting of the Association for
  Computational Linguistics (Volume 1: Long Papers)}, pages 629--637, Jeju
  Island, Korea, July. Association for Computational Linguistics.

\bibitem[\protect\citename{Nederhof and Satta}2004a]{nederhof-satta:2004:ACL1}
Mark-Jan Nederhof and Giorgio Satta.
\newblock 2004a.
\newblock Probabilistic parsing strategies.
\newblock In {\em Proceedings of the 42nd Meeting of the Association for
  Computational Linguistics (ACL'04), Main Volume}, pages 542--549, Barcelona,
  Spain, July.

\bibitem[\protect\citename{Nederhof and
  Satta}2004b]{DBLP:journals/corr/cs-CL-0404009}
Mark-Jan Nederhof and Giorgio Satta.
\newblock 2004b.
\newblock Tabular parsing.
\newblock In Carlos Mart\'{i}n-Vide, Victor Mitrana, and Gheorghe P\u{a}un,
  editors, {\em Formal Languages and Applications}, volume 148 of {\em Studies
  in Fuzziness and Soft Computing}, pages 529--549. Springer Berlin Heidelberg.

\bibitem[\protect\citename{Nederhof}1993]{Nederhof:1993:GLP:976744.976780}
Mark-Jan Nederhof.
\newblock 1993.
\newblock Generalized left-corner parsing.
\newblock In {\em Proceedings of the Sixth Conference on European Chapter of
  the Association for Computational Linguistics}, EACL '93, pages 305--314,
  Stroudsburg, PA, USA. Association for Computational Linguistics.

\bibitem[\protect\citename{Nivre and
  Fern{\'a}ndez-Gonz{\'a}lez}2014]{NivreAEP14}
Joakim Nivre and Daniel Fern{\'a}ndez-Gonz{\'a}lez.
\newblock 2014.
\newblock Arc-eager parsing with the tree constraint.
\newblock {\em Computational Linguistics}, 40(2):259--267.

\bibitem[\protect\citename{Nivre and Nilsson}2005]{nivre-nilsson:2005:ACL}
Joakim Nivre and Jens Nilsson.
\newblock 2005.
\newblock Pseudo-projective dependency parsing.
\newblock In {\em Proceedings of the 43rd Annual Meeting of the Association for
  Computational Linguistics (ACL'05)}, pages 99--106, Ann Arbor, Michigan,
  June. Association for Computational Linguistics.

\bibitem[\protect\citename{Nivre \bgroup et al.\egroup }2006]{sv}
Joakim Nivre, Jens Nilsson, and Johan Hall.
\newblock 2006.
\newblock Talbanken05: A {Swedish} treebank with phrase structure and
  dependency annotation.
\newblock In {\em Proceedings of the Fifth International Conference on Language
  Resources and Evaluation ({LREC} 2006)}, Genova, Italy. European Language
  Resources Association ({ELRA}).

\bibitem[\protect\citename{Nivre \bgroup et al.\egroup
  }2007a]{nivre-EtAl:2007:EMNLP-CoNLL2007}
Joakim Nivre, Johan Hall, Sandra K\"ubler, Ryan McDonald, Jens Nilsson,
  Sebastian Riedel, and Deniz Yuret.
\newblock 2007a.
\newblock The {CoNLL} 2007 shared task on dependency parsing.
\newblock In {\em Proceedings of the CoNLL Shared Task Session of EMNLP-CoNLL
  2007}, pages 915--932.

\bibitem[\protect\citename{Nivre \bgroup et al.\egroup }2007b]{NivreMAL07}
Joakim Nivre, Johan Hall, Jens Nilsson, Atanas Chanev, G{\"u}l{\c{s}}en
  Eryi{\u{g}}it, Sandra K{\"u}bler, Svetoslav Marinov, and Erwin Marsi.
\newblock 2007b.
\newblock {M}alt{P}arser: A language-independent system for data-driven
  dependency parsing.
\newblock {\em Natural Language Engineering}, 13(2):95--135.

\bibitem[\protect\citename{Nivre}2003]{Nivre2003}
Joakim Nivre.
\newblock 2003.
\newblock An efficient algorithm for projective dependency parsing.
\newblock In {\em Proceedings of the 8th International Workshop on Parsing
  Technologies (IWPT)}, pages 149--160.

\bibitem[\protect\citename{Nivre}2004]{nivre:2004:IncrementalParsing}
Joakim Nivre.
\newblock 2004.
\newblock Incrementality in deterministic dependency parsing.
\newblock In Frank Keller, Stephen Clark, Matthew Crocker, and Mark Steedman,
  editors, {\em Proceedings of the ACL Workshop Incremental Parsing: Bringing
  Engineering and Cognition Together}, pages 50--57, Barcelona, Spain, July.
  Association for Computational Linguistics.

\bibitem[\protect\citename{Nivre}2008]{Nivre:2008}
Joakim Nivre.
\newblock 2008.
\newblock Algorithms for deterministic incremental dependency parsing.
\newblock {\em Computational Linguistics}, 34(4):513--553.

\bibitem[\protect\citename{Nivre}2015]{nivre2015ud}
Joakim Nivre.
\newblock 2015.
\newblock Towards a universal grammar for natural language processing.
\newblock In Alexander Gelbukh, editor, {\em Computational Linguistics and
  Intelligent Text Processing}, volume 9041 of {\em Lecture Notes in Computer
  Science}, pages 3--16. Springer International Publishing.

\bibitem[\protect\citename{Noji and Miyao}2014]{noji-miyao:2014:Coling}
Hiroshi Noji and Yusuke Miyao.
\newblock 2014.
\newblock Left-corner transitions on dependency parsing.
\newblock In {\em Proceedings of COLING 2014, the 25th International Conference
  on Computational Linguistics: Technical Papers}, pages 2140--2150, Dublin,
  Ireland, August. Dublin City University and Association for Computational
  Linguistics.

\bibitem[\protect\citename{Noji and Miyao}2015]{noji-miyao:2015:jnlp}
Hiroshi Noji and Yusuke Miyao.
\newblock 2015.
\newblock Left-corner parsing for dependency grammar.
\newblock {\em Journal of Natural Language Processing}, 22(4), December.

\bibitem[\protect\citename{Ouchi \bgroup et al.\egroup
  }2014]{ouchi-duh-matsumoto:2014:EACL2014-SP}
Hiroki Ouchi, Kevin Duh, and Yuji Matsumoto.
\newblock 2014.
\newblock Improving dependency parsers with supertags.
\newblock In {\em Proceedings of the 14th Conference of the European Chapter of
  the Association for Computational Linguistics, volume 2: Short Papers}, pages
  154--158, Gothenburg, Sweden, April. Association for Computational
  Linguistics.

\bibitem[\protect\citename{Pate and Goldwater}2013]{TACL504}
John Pate and Sharon Goldwater.
\newblock 2013.
\newblock Unsupervised dependency parsing with acoustic cues.
\newblock {\em Transactions of the Association for Computational Linguistics},
  1:63--74.

\bibitem[\protect\citename{Pereira and Schabes}1992]{pereira-schabes:1992:ACL}
Fernando Pereira and Yves Schabes.
\newblock 1992.
\newblock Inside-outside reestimation from partially bracketed corpora.
\newblock In {\em Proceedings of the 30th Annual Meeting of the Association for
  Computational Linguistics}, pages 128--135, Newark, Delaware, USA, June.
  Association for Computational Linguistics.

\bibitem[\protect\citename{Petrov \bgroup et al.\egroup
  }2006]{petrov-EtAl:2006:COLACL}
Slav Petrov, Leon Barrett, Romain Thibaux, and Dan Klein.
\newblock 2006.
\newblock Learning accurate, compact, and interpretable tree annotation.
\newblock In {\em Proceedings of the 21st International Conference on
  Computational Linguistics and 44th Annual Meeting of the Association for
  Computational Linguistics}, pages 433--440, Sydney, Australia, July.
  Association for Computational Linguistics.

\bibitem[\protect\citename{Pitler \bgroup et al.\egroup }2013]{TACL23}
Emily Pitler, Sampath Kannan, and Mitchell Marcus.
\newblock 2013.
\newblock Finding optimal 1-endpoint-crossing trees.
\newblock {\em Transactions of the Association for Computational Linguistics},
  1:13--24.

\bibitem[\protect\citename{Ponvert \bgroup et al.\egroup
  }2011]{ponvert-baldridge-erk:2011:ACL-HLT2011}
Elias Ponvert, Jason Baldridge, and Katrin Erk.
\newblock 2011.
\newblock Simple unsupervised grammar induction from raw text with cascaded
  finite state models.
\newblock In {\em Proceedings of the 49th Annual Meeting of the Association for
  Computational Linguistics: Human Language Technologies}, pages 1077--1086,
  Portland, Oregon, USA, June. Association for Computational Linguistics.

\bibitem[\protect\citename{Popel \bgroup et al.\egroup
  }2013]{popel-EtAl:2013:ACL2013}
Martin Popel, David Mare\v{c}ek, Jan \v{S}těp\'{a}nek, Daniel Zeman, and
  Zděněk \v{Z}abokrtsk\'{y}.
\newblock 2013.
\newblock Coordination structures in dependency treebanks.
\newblock In {\em Proceedings of the 51st Annual Meeting of the Association for
  Computational Linguistics (Volume 1: Long Papers)}, pages 517--527, Sofia,
  Bulgaria, August. Association for Computational Linguistics.

\bibitem[\protect\citename{Prokopidis \bgroup et al.\egroup }2005]{el}
Prokopis Prokopidis, Elina Desipri, Maria Koutsombogera, Harris Papageorgiou,
  and Stelios Piperidis.
\newblock 2005.
\newblock Theoretical and practical issues in the construction of a {G}reek
  dependency treebank.
\newblock In {\em In Proc. of the 4th Workshop on Treebanks and Linguistic
  Theories (TLT)}, pages 149--160.

\bibitem[\protect\citename{Resnik}1992]{conf/coling/Resnik92}
Philip Resnik.
\newblock 1992.
\newblock Left-corner parsing and psychological plausibility.
\newblock In {\em COLING}, pages 191--197.

\bibitem[\protect\citename{Roark}2001]{Roark:2001:RPP:933637}
Brian~Edward Roark.
\newblock 2001.
\newblock {\em Robust Probabilistic Predictive Syntactic Processing:
  Motivations, Models, and Applications}.
\newblock {Ph.D.} thesis, Providence, RI, USA.
\newblock AAI3006783.

\bibitem[\protect\citename{Rosenkrantz and Lewis}1970]{4569645}
D.J. Rosenkrantz and P.M. Lewis.
\newblock 1970.
\newblock Deterministic left corner parsing.
\newblock In {\em Switching and Automata Theory, 1970., IEEE Conference Record
  of 11th Annual Symposium on}, pages 139--152, Oct.

\bibitem[\protect\citename{Sartorio \bgroup et al.\egroup
  }2013]{sartorio-satta-nivre:2013:ACL2013}
Francesco Sartorio, Giorgio Satta, and Joakim Nivre.
\newblock 2013.
\newblock A transition-based dependency parser using a dynamic parsing
  strategy.
\newblock In {\em Proceedings of the 51st Annual Meeting of the Association for
  Computational Linguistics (Volume 1: Long Papers)}, pages 135--144, Sofia,
  Bulgaria, August. Association for Computational Linguistics.

\bibitem[\protect\citename{Schuler \bgroup et al.\egroup
  }2010]{journals/coling/SchulerAMS10}
William Schuler, Samir AbdelRahman, Tim Miller, and Lane Schwartz.
\newblock 2010.
\newblock Broad-coverage parsing using human-like memory constraints.
\newblock {\em Computational Linguistics}, 36(1):1--30.

\bibitem[\protect\citename{Schwartz \bgroup et al.\egroup
  }2011]{schwartz-EtAl:2011:ACL-HLT20112}
Roy Schwartz, Omri Abend, Roi Reichart, and Ari Rappoport.
\newblock 2011.
\newblock Neutralizing linguistically problematic annotations in unsupervised
  dependency parsing evaluation.
\newblock In {\em Proceedings of the 49th Annual Meeting of the Association for
  Computational Linguistics: Human Language Technologies}, pages 663--672,
  Portland, Oregon, USA, June. Association for Computational Linguistics.

\bibitem[\protect\citename{Seginer}2007]{seginer:2007:ACLMain}
Yoav Seginer.
\newblock 2007.
\newblock Fast unsupervised incremental parsing.
\newblock In {\em Proceedings of the 45th Annual Meeting of the Association of
  Computational Linguistics}, pages 384--391, Prague, Czech Republic, June.
  Association for Computational Linguistics.

\bibitem[\protect\citename{Seraji}2015]{faupdt}
Mojgan Seraji.
\newblock 2015.
\newblock {\em Morphosyntactic Corpora and Tools for {Persian}. {PhD} Thesis.
  Studia Linguistica Upsaliensia 16}.

\bibitem[\protect\citename{Shravan~Vasishth}2006]{10.2307/4490268}
Richard L.~Lewis Shravan~Vasishth.
\newblock 2006.
\newblock Argument-head distance and processing complexity: Explaining both
  locality and antilocality effects.
\newblock {\em Language}, 82(4):767--794.

\bibitem[\protect\citename{Silveira \bgroup et al.\egroup
  }2014]{silveira14gold}
Natalia Silveira, Timothy Dozat, Marie-Catherine de~Marneffe, Samuel Bowman,
  Miriam Connor, John Bauer, and Christopher~D. Manning.
\newblock 2014.
\newblock A gold standard dependency corpus for {E}nglish.
\newblock In {\em Proceedings of the Ninth International Conference on Language
  Resources and Evaluation (LREC-2014)}.

\bibitem[\protect\citename{Simov and Osenova}2005]{bg}
Kiril Simov and Petya Osenova.
\newblock 2005.
\newblock Extending the annotation of {B}ul{T}ree{B}ank: Phase 2.
\newblock In {\em The Fourth Workshop on Treebanks and Linguistic Theories (TLT
  2005)}, pages 173--184, Barcelona, December.

\bibitem[\protect\citename{Smith and Eisner}2005]{smith-eisner:2005:ACL}
Noah~A. Smith and Jason Eisner.
\newblock 2005.
\newblock Contrastive estimation: Training log-linear models on unlabeled data.
\newblock In {\em Proceedings of the 43rd Annual Meeting of the Association for
  Computational Linguistics (ACL'05)}, pages 354--362, Ann Arbor, Michigan,
  June. Association for Computational Linguistics.

\bibitem[\protect\citename{Smith and Eisner}2006]{smith-eisner-2006-acl-sa}
Noah~A. Smith and Jason Eisner.
\newblock 2006.
\newblock Annealing structural bias in multilingual weighted grammar induction.
\newblock In {\em Proceedings of the International Conference on Computational
  Linguistics and the Association for Computational Linguistics (COLING-ACL)},
  pages 569--576, Sydney, July.

\bibitem[\protect\citename{Smith}2006]{smith-2006}
Noah~A. Smith.
\newblock 2006.
\newblock {\em Novel Estimation Methods for Unsupervised Discovery of Latent
  Structure in Natural Language Text}.
\newblock {Ph.D.} thesis, Johns Hopkins University, Baltimore, MD, October.

\bibitem[\protect\citename{Smr{\v{z}} \bgroup et al.\egroup }2008]{ar}
Otakar Smr{\v{z}}, Viktor Bielick{\'{y}}, Iveta Kou{\v{r}}ilov{\'{a}}, Jakub
  Kr{\'{a}}{\v{c}}mar, Jan Haji{\v{c}}, and Petr Zem{\'{a}}nek.
\newblock 2008.
\newblock {P}rague {A}rabic dependency treebank: A word on the million words.
\newblock In {\em Proceedings of the Workshop on Arabic and Local Languages
  ({LREC} 2008)}, pages 16--23, Marrakech, Morocco. European Language Resources
  Association.

\bibitem[\protect\citename{S{\o}gaard}2012]{sogaard:2012:WILS}
Anders S{\o}gaard.
\newblock 2012.
\newblock Two baselines for unsupervised dependency parsing.
\newblock In {\em Proceedings of the NAACL-HLT Workshop on the Induction of
  Linguistic Structure}, pages 81--83, Montr{\'e}al, Canada, June. Association
  for Computational Linguistics.

\bibitem[\protect\citename{Spitkovsky \bgroup et al.\egroup
  }2010a]{spitkovsky-alshawi-jurafsky:2010:NAACLHLT}
Valentin~I. Spitkovsky, Hiyan Alshawi, and Daniel Jurafsky.
\newblock 2010a.
\newblock From baby steps to leapfrog: How ``less is more'' in unsupervised
  dependency parsing.
\newblock In {\em Human Language Technologies: The 2010 Annual Conference of
  the North American Chapter of the Association for Computational Linguistics},
  pages 751--759, Los Angeles, California, June. Association for Computational
  Linguistics.

\bibitem[\protect\citename{Spitkovsky \bgroup et al.\egroup
  }2010b]{spitkovsky-EtAl:2010:CONLL}
Valentin~I. Spitkovsky, Hiyan Alshawi, Daniel Jurafsky, and Christopher~D.
  Manning.
\newblock 2010b.
\newblock Viterbi training improves unsupervised dependency parsing.
\newblock In {\em Proceedings of the Fourteenth Conference on Computational
  Natural Language Learning}, pages 9--17, Uppsala, Sweden, July. Association
  for Computational Linguistics.

\bibitem[\protect\citename{Spitkovsky \bgroup et al.\egroup
  }2012]{spitkovsky-alshawi-jurafsky:2012:EMNLP-CoNLL}
Valentin~I. Spitkovsky, Hiyan Alshawi, and Daniel Jurafsky.
\newblock 2012.
\newblock Three dependency-and-boundary models for grammar induction.
\newblock In {\em Proceedings of the 2012 Joint Conference on Empirical Methods
  in Natural Language Processing and Computational Natural Language Learning},
  pages 688--698, Jeju Island, Korea, July. Association for Computational
  Linguistics.

\bibitem[\protect\citename{Spitkovsky \bgroup et al.\egroup
  }2013]{spitkovsky-alshawi-jurafsky:2013:EMNLP}
Valentin~I. Spitkovsky, Hiyan Alshawi, and Daniel Jurafsky.
\newblock 2013.
\newblock Breaking out of local optima with count transforms and model
  recombination: A study in grammar induction.
\newblock In {\em Proceedings of the 2013 Conference on Empirical Methods in
  Natural Language Processing}, pages 1983--1995, Seattle, Washington, USA,
  October. Association for Computational Linguistics.

\bibitem[\protect\citename{Steedman}2000]{opac-b1095877}
Mark Steedman.
\newblock 2000.
\newblock {\em The Syntactic process}.
\newblock Language, speech, and communication. MIT Press, Cambridge (Mass.),
  London.
\newblock A Bradford book.

\bibitem[\protect\citename{Stolcke}1995]{Stolcke:1995:EPC:211190.211197}
Andreas Stolcke.
\newblock 1995.
\newblock An efficient probabilistic context-free parsing algorithm that
  computes prefix probabilities.
\newblock {\em Comput. Linguist.}, 21(2):165--201, June.

\bibitem[\protect\citename{T\"{a}ckstr\"{o}m \bgroup et al.\egroup
  }2013]{tackstrom-mcdonald-nivre:2013:NAACL-HLT}
Oscar T\"{a}ckstr\"{o}m, Ryan McDonald, and Joakim Nivre.
\newblock 2013.
\newblock Target language adaptation of discriminative transfer parsers.
\newblock In {\em Proceedings of the 2013 Conference of the North American
  Chapter of the Association for Computational Linguistics: Human Language
  Technologies}, pages 1061--1071, Atlanta, Georgia, June. Association for
  Computational Linguistics.

\bibitem[\protect\citename{Tesni\`{e}re}1959]{Tesniere1959}
L~Tesni\`{e}re.
\newblock 1959.
\newblock {\em Elements de syntaxe structurale}.
\newblock Editions Klincksieck.

\bibitem[\protect\citename{van~der Beek \bgroup et al.\egroup }2002]{nl}
Leonoor van~der Beek, Gosse Bouma, Jan Daciuk, Tanja Gaustad, Robert Malouf,
  Gertjan van Noord, Robbert Prins, and Begoña Villada.
\newblock 2002.
\newblock Chapter 5. the {Alpino} dependency treebank.
\newblock In {\em Algorithms for Linguistic Processing {NWO PIONIER} Progress
  Report}, Groningen, The Netherlands.

\bibitem[\protect\citename{van Schijndel and
  Schuler}2013]{vanschijndel-schuler:2013:NAACL-HLT}
Marten van Schijndel and William Schuler.
\newblock 2013.
\newblock An analysis of frequency- and memory-based processing costs.
\newblock In {\em Proceedings of the 2013 Conference of the North American
  Chapter of the Association for Computational Linguistics: Human Language
  Technologies}, pages 95--105, Atlanta, Georgia, June. Association for
  Computational Linguistics.

\bibitem[\protect\citename{van Schijndel \bgroup et al.\egroup
  }2013]{TOPS:TOPS12034}
Marten van Schijndel, Andy Exley, and William Schuler.
\newblock 2013.
\newblock A model of language processing as hierarchic sequential prediction.
\newblock {\em Topics in Cognitive Science}, 5(3):522--540.

\bibitem[\protect\citename{Xu \bgroup et al.\egroup }2005]{Xu05maximummargin}
Linli Xu, James Neufeld, Bryce Larson, and Dale Schuurmans.
\newblock 2005.
\newblock Maximum margin clustering.
\newblock In {\em Advances in Neural Information Processing Systems 17}, pages
  1537--1544. MIT Press.

\bibitem[\protect\citename{Yamada and Matsumoto}2003]{Yamada03}
H.~Yamada and Y.~Matsumoto.
\newblock 2003.
\newblock {Statistical Dependency Analysis with Support Vector machines}.
\newblock In {\em {The 8th International Workshop of Parsing Technologies
  (IWPT2003)}}.

\bibitem[\protect\citename{Zeman \bgroup et al.\egroup }2014]{halmedt}
Daniel Zeman, Ond\v{r}ej Du\v{s}ek, David Mare\v{c}ek, Martin Popel, Loganathan
  Ramasamy, Jan \v{S}t\v{e}p\'{a}nek, Zden\v{e}k \v{Z}abokrtsk\'{y}, and Jan
  Haji\v{c}.
\newblock 2014.
\newblock Hamledt: Harmonized multi-language dependency treebank.
\newblock {\em Language Resources and Evaluation}, 48(4):601--637.

\bibitem[\protect\citename{Zettlemoyer and
  Collins}2005]{Zettlemoyer05learningto}
Luke~S. Zettlemoyer and Michael Collins.
\newblock 2005.
\newblock Learning to map sentences to logical form: Structured classification
  with probabilistic categorial grammars.
\newblock In {\em In Proceedings of the 21st Conference on Uncertainty in AI},
  pages 658--666.

\bibitem[\protect\citename{Zhang and Clark}2008]{zhang-clark:2008:EMNLP}
Yue Zhang and Stephen Clark.
\newblock 2008.
\newblock A tale of two parsers: {I}nvestigating and combining graph-based and
  transition-based dependency parsing.
\newblock In {\em Proceedings of the 2008 Conference on Empirical Methods in
  Natural Language Processing}, pages 562--571, Honolulu, Hawaii, October.
  Association for Computational Linguistics.

\bibitem[\protect\citename{Zhang and Nivre}2011]{zhang-nivre:2011:ACL-HLT2011}
Yue Zhang and Joakim Nivre.
\newblock 2011.
\newblock Transition-based dependency parsing with rich non-local features.
\newblock In {\em Proceedings of the 49th Annual Meeting of the Association for
  Computational Linguistics: Human Language Technologies}, pages 188--193,
  Portland, Oregon, USA, June. Association for Computational Linguistics.

\bibitem[\protect\citename{Zwicky}1993]{zwicky199313}
Arnold~M Zwicky.
\newblock 1993.
\newblock Heads, bases and functors.
\newblock {\em Heads in Grammatical Theory}, page 292.

\end{thebibliography}

\end{document}